\documentclass[a4paper,11pt,twoside,onecolumn,final,openany]{book}

\usepackage{amsthm}
\usepackage{fancyhdr}
\usepackage{graphicx}
\usepackage{natbib}
\usepackage{fullpage}
\usepackage{url}
\usepackage{xcolor}
\usepackage{hyperref}
\usepackage{lmodern}
\usepackage[T1]{fontenc}

\fancypagestyle{plain}{
\fancyhead{}

\fancyhead[C]{{\color{red} The official publication is available from now publishers via \url{http://dx.doi.org/10.1561/2200000058}}}
}

\fancypagestyle{headings}{
\fancyhead{}

\fancyhead[RE]{\leftmark}
\fancyhead[LO]{\rightmark}
}

\newcommand{\elink}[1]{\footnote{See Exercise~\ref{#1}.}}
\newcommand{\mopt}{\text{\textsf{mOPT}}_{\!f}}

\newcommand{\x}{\vx}
\newcommand{\y}{\vy}
\newcommand{\z}{\vz}
\newcommand{\w}{\vw}
\newcommand{\e}{\ve}
\newcommand{\R}{\bR}
\newcommand{\B}{\cB}
\newcommand{\D}{\cD}
\newcommand{\A}{\cA}

\newcommand{\sign}{\text{sign}}
\newcommand{\xopt}{\vx^\ast}
\newcommand{\sep}{\ |\ }
\newcommand{\nd}{^{\text{nd}}}

\usepackage{amsmath,amsthm,amssymb}
\usepackage{stmaryrd}
\usepackage{algorithm,algorithmic}
\usepackage{olo}
\usepackage{subcaption}

\graphicspath{{figs/}}

\title{Non-convex Optimization for Machine Learning\footnote{\color{red} The official publication is available from now publishers via \newline \url{http://dx.doi.org/10.1561/2200000058}}}

\author{
Prateek Jain \\
Microsoft Research India\\
prajain@microsoft.com
\and
Purushottam Kar\\
IIT Kanpur\\
purushot@cse.iitk.ac.in
}

\begin{document}

\pagestyle{headings}
\sloppy
\flushbottom
\frenchspacing

\frontmatter

\maketitle

\tableofcontents
\listoffigures
\listofalgorithms

\mainmatter

\allowdisplaybreaks

\chapter*{\centering Abstract}
\addcontentsline{toc}{chapter}{Abstract}

A vast majority of machine learning algorithms train their models and perform inference by solving optimization problems. In order to capture the learning and prediction problems accurately, structural constraints such as sparsity or low rank are frequently imposed or else the objective itself is designed to be a non-convex function. This is especially true of algorithms that operate in high-dimensional spaces or that train non-linear models such as tensor models and deep networks.

The freedom to express the learning problem as a non-convex optimization problem gives immense modeling power to the algorithm designer, but often such problems are NP-hard to solve. A popular workaround to this has been to relax non-convex problems to convex ones and use traditional methods to solve the (convex) \emph{relaxed} optimization problems. However this approach may be lossy and nevertheless presents significant challenges for large scale optimization.

On the other hand, direct approaches to non-convex optimization have met with resounding success in several domains and remain the methods of choice for the practitioner, as they frequently outperform relaxation-based techniques -- popular heuristics include projected gradient descent and alternating minimization. However, these are often poorly understood in terms of their convergence and other properties.

This monograph presents a selection of recent advances that bridge a long-standing gap in our understanding of these heuristics. We hope that an insight into the inner workings of these methods will allow the reader to appreciate the unique marriage of task structure and generative models that allow these heuristic techniques to (provably) succeed. The monograph will lead the reader through several widely used non-convex optimization techniques, as well as applications thereof. The goal of this monograph is to both, introduce the rich literature in this area, as well as equip the reader with the tools and techniques needed to analyze these simple procedures for non-convex problems.

\chapter*{Preface}
\label{chap:org}
\addcontentsline{toc}{chapter}{Preface}
\markboth{\sffamily\slshape Preface}
{\sffamily\slshape Preface}

Optimization as a field of study has permeated much of science and technology. The advent of the digital computer and a tremendous subsequent increase in our computational prowess has increased the impact of optimization in our lives. Today, tiny details such as airline schedules all the way to leaps and strides in medicine, physics and artificial intelligence, all rely on modern advances in optimization techniques.

For a large portion of this period of excitement, our energies were focused largely on convex optimization problems, given our deep understanding of the structural properties of convex sets and convex functions. However, modern applications in domains such as signal processing, bio-informatics and machine learning, are often dissatisfied with convex formulations alone since there exist non-convex formulations that better capture the problem structure. For applications in these domains, models trained using non-convex formulations often offer excellent performance and other desirable properties such as compactness and reduced prediction times.

Examples of applications that benefit from non-convex optimization techniques include gene expression analysis, recommendation systems, clustering, and outlier and anomaly detection. In order to get satisfactory solutions to these problems, that are scalable and accurate, we require a deeper understanding of non-convex optimization problems that naturally arise in these problem settings.

Such an understanding was lacking until very recently and non-convex optimization found little attention as an active area of study, being regarded as intractable. Fortunately, a long line of works have recently led areas such as computer science, signal processing, and statistics to realize that the general abhorrence to non-convex optimization problems hitherto practiced, was misled. These works demonstrated in a beautiful way, that although non-convex optimization problems do suffer from intractability in general, those that arise in \emph{natural settings} such as machine learning and signal processing, possess additional structure that allow the intractability results to be circumvented.

The first of these works still religiously stuck to convex optimization as the method of choice, and instead, sought to show that certain classes of non-convex problems which possess suitable additional structure as offered by natural instances of those problems, could be converted to convex problems without any loss. More precisely, it was shown that the original non-convex problem and the modified convex problem possessed a common optimum and thus, the solution to the convex problem would automatically solve the non-convex problem as well! However, these approaches had a price to pay in terms of the time it took to solve these so-called \emph{relaxed} convex problems. In several instances, these relaxed problems, although not intractable to solve, were nevertheless challenging to solve, at large scales.

It took a second wave of still more recent results to usher in provable non-convex optimization techniques which abstained from relaxations, solved the non-convex problems in their native forms, and yet seemed to offer the same quality of results as relaxation methods did. These newer results were accompanied with a newer realization that, for a wide range of applications such as sparse recovery, matrix completion, robust learning among others, these direct techniques are faster, often by an order of magnitude or more, than relaxation-based techniques while offering solutions of similar accuracy.

This monograph wishes to tell the story of this realization and the wisdom we gained from it from the point of view of machine learning and signal processing applications. The monograph will introduce the reader to a lively world of non-convex optimization problems with rich structure that can be exploited to obtain extremely scalable solutions to these problems. Put a bit more dramatically, it will seek to show how problems that were once avoided, having been shown to be NP-hard to solve, now have solvers that operate in near-linear time, by carefully analyzing and exploiting additional task structure! It will seek to inform the reader on how to look for such structure in diverse application areas, as well as equip the reader with a sound background in fundamental tools and concepts required to analyze such problem areas and come up with newer solutions.\\

\noindent\textbf{How to use this monograph} We have made efforts to make this monograph as self-contained as possible while not losing focus of the main topic of non-convex optimization techniques. Consequently, we have devoted entire sections to present a tutorial-like treatment to basic concepts in convex analysis and optimization, as well as their non-convex counterparts. As such, this monograph can be used for a semester-length course on the basics of non-convex optimization with applications to machine learning.

On the other hand, it is also possible to cherry pick portions of the monograph, such the section on sparse recovery, or the EM algorithm, for inclusion in a broader course. Several courses such as those in machine learning, optimization, and signal processing may benefit from the inclusion of such topics. However, we advise that relevant background sections (see Figure~\ref{fig:org}) be covered beforehand.

While striving for breadth, the limits of space have constrained us from looking at some topics in much detail. Examples include the construction of design matrices that satisfy the RIP/RSC properties and pursuit style methods, but there are several others. However, for all such omissions, the bibliographic notes at the end of each section can always be consulted for references to details of the omitted topics. We have also been unable to address several application areas such as dictionary learning, advances in low-rank tensor decompositions, topic modeling and community detection in graphs but have provided pointers to prominent works in these application areas too.

The organization of this monograph is outlined below with Figure~\ref{fig:org} presenting a suggested order of reading the various sections.\\

\begin{figure}[t!]
\includegraphics[width=\columnwidth]{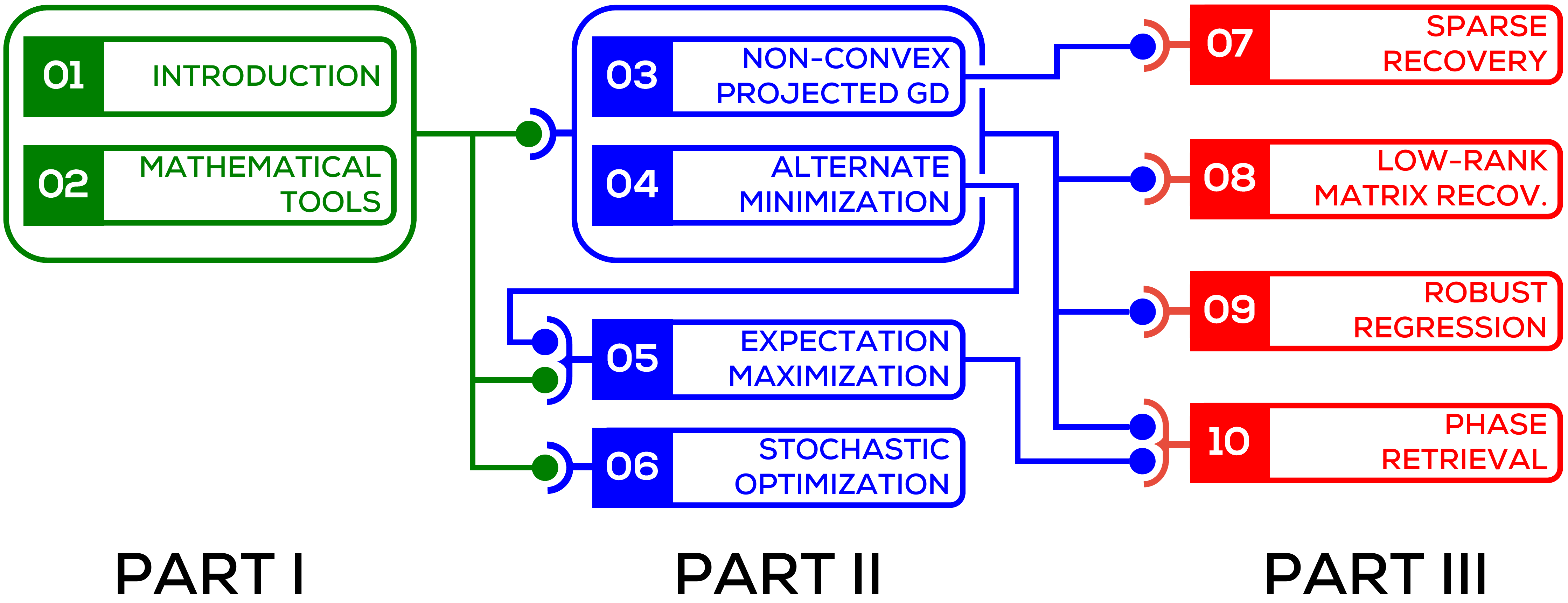}
\caption[Suggested Order of Reading the Sections]{A schematic showing the suggested order of reading the sections. For example, concepts introduced in \S~\ref{chap:pgd} and \ref{chap:altmin} are helpful for \S~\ref{chap:rreg} but a thorough reading of \S~\ref{chap:saddle} is not required for the same. Similarly, we recommend reading \S~\ref{chap:em} after going through \S~\ref{chap:altmin} but a reader may choose to proceed to \S~\ref{chap:spreg} directly after reading \S~\ref{chap:pgd}.}
\label{fig:org}
\end{figure}

\noindent\textbf{Part I: Introduction and Basic Tools}\\
This part will offer an introductory note and a section exploring some basic definitions and algorithmic tools in convex optimization. These sections are recommended to readers not intimately familiar with basics of numerical optimization.\\

\noindent\textbf{Section 1 - Introduction} This section will give a more relaxed introduction to the area of non-convex optimization by discussing applications that motivate the use of non-convex formulations. The discussion will also clarify the scope of this monograph.\\

\noindent\textbf{Section 2 - Mathematical Tools} This section will set up notation and introduce some basic mathematical tools in convex optimization. This section is basically a handy repository of useful concepts and results and can be skipped by a reader familiar with them. Parts of the section may instead be referred back to, as and when needed, using the cross-referencing links in the monograph.

\noindent\textbf{Part II: Non-convex Optimization Primitives}\\
This part will equip the reader with a collection of primitives most widely used in non-convex optimization problems.\\

\noindent\textbf{Section 3 - Non-convex Projected Gradient Descent} This section will introduce the simple and intuitive projected gradient descent method in the context of non-convex optimization. Variants of this method will be used in later sections to solve problems such as sparse recovery and robust learning.\\

\noindent\textbf{Section 4 - Alternating Minimization} This section will introduce the principle of alternating minimization which is widely used in optimization problems over two or more (groups of) variables. The methods introduced in this section will be later used in later sections to solve problems such as low-rank matrix recovery, robust regression, and phase retrieval.\\

\noindent\textbf{Section 5 - The EM Algorithm} This section will introduce the EM algorithm which is a widely used optimization primitive for learning problems with latent variables. Although EM is a form of alternating minimization, given its significance, the section gives it special attention. This section will discuss some recent advances in the analysis and applications of this method and look at two case studies in learning Gaussian mixture models and mixed regression to illustrate the algorithm and its analyses.\\

\noindent\textbf{Section 6 - Stochastic Non-convex Optimization} This section will look at some recent advances in using stochastic optimization techniques for solving optimization problems with non-convex objectives. The section will also introduce the problem of tensor factorization as a case study for the algorithms being studied.\\

\noindent\textbf{Part III - Applications}\\
This part will take a look at four interesting applications in the areas of machine learning and signal processing and explore how the non-convex optimization techniques introduced earlier can be used to solve these problems.

\noindent\textbf{Section 7 - Sparse Recovery} This section will look at a very basic non-convex optimization problem, that of performing linear regression to fit a sparse model to the data. The section will discuss conditions under which it is possible to do so in polynomial time and show how the non-convex projected gradient descent method studied earlier can be used to offer provably optimal solutions. The section will also point to other techniques used to solve this problem, as well as refer to extensions and related results.\\

\noindent\textbf{Section 8 - Low-rank Matrix Recovery} This section will address the more general problem of low rank matrix recovery with specific emphasis on low-rank matrix completion. The section will gently introduce low-rank matrix recovery as a generalization of sparse linear regression that was studied in the previous section and then move on to look at matrix completion in more detail. The section will apply both the non-convex projected gradient descent and alternating minimization methods in the context of low-rank matrix recovery, analyzing simple cases and pointing to relevant literature.\\

\noindent\textbf{Section 9 - Robust Regression} This section will look at a widely studied area of machine learning, namely robust learning, from the point of view of regression. Algorithms that are robust to (adversarial) corruption in data are sought after in several areas of signal processing and learning. The section will explore how to use the projected gradient and alternating minimization techniques to solve the robust regression problem and also look at applications of robust regression to robust face recognition and robust time series analysis.\\

\noindent\textbf{Section 10 - Phase Retrieval} This section will look at some recent advances in the application of non-convex optimization to phase retrieval. This problem lies at the heart of several imaging techniques such as X-ray crystallography and electron microscopy. A lot remains to be understood about this problem and existing algorithms often struggle to cope with the retrieval problems presented in practice.\\

The area of non-convex optimization has considerably widened in both scope and application in recent years and newer methods and analyses are being proposed at a rapid pace. While this makes researchers working in this area extremely happy, it also makes summarizing the vast body of work in a monograph such as this, more challenging. We have striven to strike a balance between presenting results that are the best known, and presenting them in a manner accessible to a newcomer. However, in all cases, the bibliography notes at the end of each section do contain pointers to the state of the art in that area and can be referenced for follow-up readings.\\

\noindent Prateek Jain, Bangalore, India\\
\noindent Purushottam Kar, Kanpur, India\\
\noindent \today

\chapter*{Mathematical Notation}
\label{chap:not}
\addcontentsline{toc}{chapter}{Mathematical Notation}
\markboth{\sffamily\slshape Mathematical Notation}
{\sffamily\slshape Mathematical Notation}
\begin{itemize}
	\item The set of real numbers is denoted by $\R$. The set of natural numbers is denoted by $\bN$.
	\item The cardinality of a set $S$ is denoted by $\abs{S}$.
	\item Vectors are denoted by boldface, lower case alphabets for example, $\x,\y$. The zero vector is denoted by $\vzero$. A vector $\x \in \R^p$ will be in column format. The transpose of a vector is denoted by $\x^\top$. The $i\th$ coordinate of a vector $\x$ is denoted by $\x_i$.
	\item Matrices are denoted by upper case alphabets for example, $A, B$. $A_i$ denotes the $i\th$ column of the matrix $A$ and $A^j$ denotes its $j\th$ row. $A_{ij}$ denotes the element at the $i\th$ row and $j\th$ column.
	\item For a vector $\x \in \R^p$ and a set $S \subset [p]$, the notation $\x_S$ denotes the vector $\z \in \R^p$ such that $\z_i = \x_i$ for $i \in S$, and $\z_i = 0$ otherwise. Similarly for matrices, $A_S$ denotes the matrix $B$ with $B_i = A_i$ for $i \in S$ and $B_i = \vzero$ for $i \neq S$. Also, $A^S$ denotes the matrix $B$ with $B^i = A^i$ for $i \in S$ and $B^i = \vzero^\top$ for $i \neq S$.
	\item The support of a vector $\x$ is denoted by $\supp(x) := \bc{i : \x_i \neq 0}$. A vector $x$ is referred to as $s$-sparse if $\abs{\supp(x)} \leq s$.
	\item The canonical directions in $\R^p$ are denoted by $\e_i,\ i = 1, \ldots, p$.
	\item The identity matrix of order $p$ is denoted by $I_{p\times p}$ or simply $I_p$. The subscript may be omitted when the order is clear from context.
	\item For a vector $\x \in \R^p$, the notation $\norm{x}_q = \sqrt[\leftroot{1}\uproot{1}q]{\sum_{i=1}^p\abs{\x_i}^q}$ denotes its $L_q$ norm. As special cases we define $\norm{\x}_\infty := \max_i\ \abs{\x_i}$, $\norm{\x}_{-\infty} := \min_i\ \abs{\x_i}$, and $\norm{\x}_0 := \abs{\supp(\x)}$.
	\item Balls with respect to various norms are denoted as $\B_q(r) := \bc{\x \in \R^p, \norm{\x}_q \leq r}$. As a special case the notation $\B_0(s)$ is used to denote the set of $s$-sparse vectors.
	\item For a matrix $A \in \R^{m \times n}$, $\sigma_1(A)\geq \sigma_2(A)\geq \ldots\geq\sigma_{\min\bc{m,n}}(A)$ denote its singular values. The Frobenius norm of $A$ is defined as $\norm{A}_F := \sqrt{\sum_{i,j} A_{ij}^2}=\sqrt{\sum_i \sigma_i(A)^2}$. The nuclear norm of $A$ is defined as $\norm{A}_\ast := \sum_i \sigma_i(A)$.
	\item The trace of a square matrix $A \in \bR^{m \times m}$ is defined as $\text{tr}(A) = \sum_{i=1}^m A_{ii}$.
	\item The spectral norm (also referred to as the operator norm) of a matrix $A$ is defined as $\norm{A}_2 := \max_i \sigma_i(A)$.
	\item Random variables are denoted using upper case letters such as $X,Y$.
	\item The expectation of a random variable $X$ is denoted by $\E{X}$. In cases where the distribution of $X$ is to be made explicit, the notation $\bE_{X\sim\D}\bs{X}$, or else simply $\bE_\D\bs{X}$, is used.
	\item $\text{\textsf{Unif}}(\cX)$ denotes the uniform distribution over a compact set $\cX$.
	\item The standard \emph{big-Oh} notation is used to describe the asymptotic behavior of functions. The \emph{soft-Oh} notation is employed to hide poly-logarithmic factors i.e., $f = \softO{g}$ will imply $f = \bigO{g \log^c(g)}$ for some absolute constant $c$.
\end{itemize}

\part{Introduction and Basic Tools}

\chapter{Introduction}
\label{chap:intro}

This section will set the stage for subsequent discussions by motivating some of the non-convex optimization problems we will be studying using real life examples, as well as setting up notation for the same.

\section{Non-convex Optimization}
The generic form of an analytic optimization problem is the following
\begin{align*}
\min_{\vx \in \bR^p}\ & f(\vx)\\
\text{s.t.}\ & \vx \in \cC,
\end{align*}
where $\vx$ is the \emph{variable} of the problem, $f: \bR^p \rightarrow \bR$ is the \emph{objective function} of the problem, and $\cC \subseteq \bR^p$ is the \emph{constraint set} of the problem. When used in a machine learning setting, the objective function allows the algorithm designer to encode proper and expected behavior for the machine learning model, such as fitting well to training data with respect to some loss function, whereas the constraint allows restrictions on the model to be encoded, for instance, restrictions on model size.

An optimization problem is said to be \emph{convex} if the objective is a convex function, as well as the constraint set is a convex set. We refer the reader to \S~\ref{chap:tools} for formal definitions of these terms. An optimization problem that violates either one of these conditions, i.e., one that has a non-convex objective, or a non-convex constraint set, or both, is called a \emph{non-convex} optimization problem. In this monograph, we will discuss non-convex optimization problems with non-convex objectives and convex constraints (\S~\ref{chap:altmin}, \ref{chap:em}, \ref{chap:saddle}, and  \ref{chap:matrec}), as well as problems with non-convex constraints but convex objectives (\S~\ref{chap:pgd}, \ref{chap:spreg}, \ref{chap:rreg}, \ref{chap:phret}, and \ref{chap:matrec}). Such problems arise in a lot of application areas.

\section{Motivation for Non-convex Optimization}
Modern applications frequently require learning algorithms to operate in extremely high dimensional spaces. Examples include web-scale document classification problems where $n$-gram-based representations can have dimensionalities in the millions or more, recommendation systems with millions of items being recommended to millions of users, and signal processing tasks such as face recognition and image processing and bio-informatics tasks such as splice and gene detection, all of which present similarly high dimensional data.

Dealing with such high dimensionalities necessitates the imposition of structural constraints on the learning models being estimated from data. Such constraints are not only helpful in regularizing the learning problem, but often essential to prevent the problem from becoming ill-posed. For example, suppose we know how a user rates some items and wish to infer how this user would rate other items, possibly in order to inform future advertisement campaigns. To do so, it is essential to impose some structure on how a user's ratings for one set of items influences ratings for other kinds of items. Without such structure, it becomes impossible to infer any new user ratings. As we shall soon see, such structural constraints often turn out to be non-convex.

In other applications, the natural objective of the learning task is a non-convex function. Common examples include training deep neural networks and tensor decomposition problems. Although non-convex objectives and constraints allow us to accurately model learning problems, they often present a formidable challenge to algorithm designers. This is because unlike convex optimization, we do not possess a handy set of tools for solving non-convex problems. Several non-convex optimization problems are known to be NP-hard to solve. The situation is made more bleak by a range of non-convex problems that are not only NP-hard to solve optimally, but NP-hard to solve approximately as well \citep{MekaJCD2008}.

\section{Examples of Non-Convex Optimization Problems}
Below we present some areas where non-convex optimization problems arise naturally when devising learning problems.\\

\noindent\textbf{Sparse Regression} The classical problem of linear regression seeks to recover a linear model which can effectively predict a response variable as a linear function of covariates. For example, we may wish to predict the average expenditure of a household (the response) as a function of the education levels of the household members, their annual salaries and other relevant indicators (the covariates). The ability to do allows economic policy decisions to be more informed by revealing, for instance, how does education level affect expenditure.

More formally, we are provided a set of $n$ covariate/response pairs $(\x_1, y_1),\ldots,(\x_n,y_n)$ where $\x_i \in \R^p$ and $y_i \in \R$. The linear regression approach makes the modeling assumption $y_i = \x_i^\top\bto + \eta_i$ where $\bto \in \R^p$ is the underlying linear model and $\eta_i$ is some benign additive noise. Using the data provided $\bc{\x_i,y_i}_{i=1,\ldots,n}$, we wish to recover back the model $\bto$ as faithfully as possible.

A popular way to recover $\bto$ is using the \emph{least squares} formulation
\[
\bth = \underset{\bt\in\R^p}{\arg\min}\ \sum_{i=1}^n\br{y_i - \x_i^\top\bt}^2.
\]
The linear regression problem as well as the least squares estimator, are extremely well studied and their behavior, precisely known. However, this age-old problem acquires new dimensions in situations where, either we expect only a few of the $p$ features/covariates to be actually relevant to the problem but do not know their identity, or else are working in extremely data-starved settings i.e., $n \ll p$.

\begin{figure}
\includegraphics[width=\columnwidth]{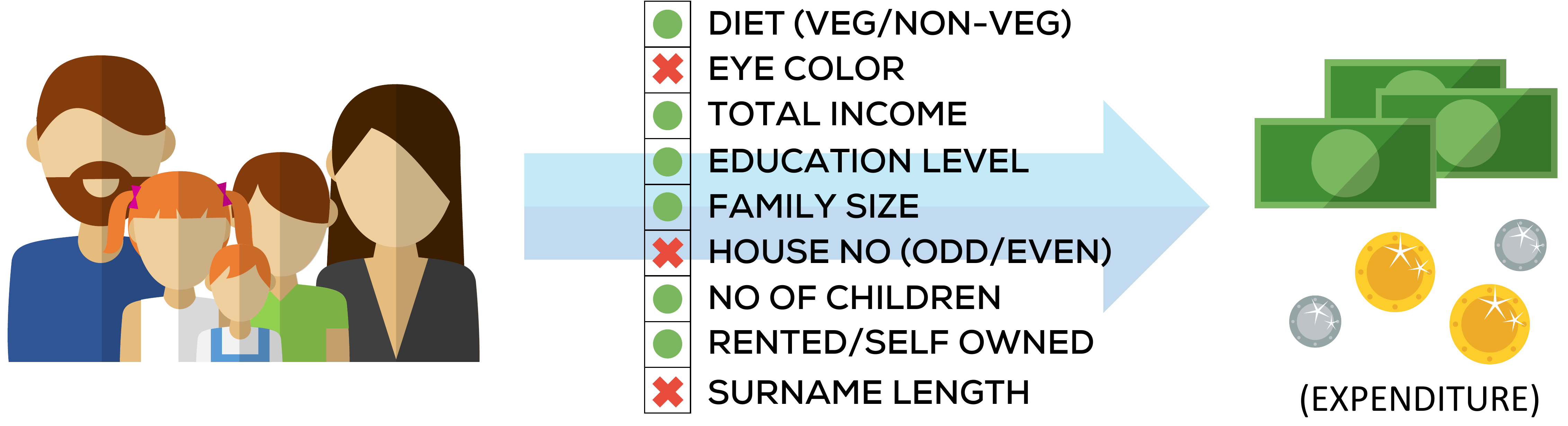}
\caption[Sparse Recovery for Automated Feature Selection]{Not all available parameters and variables may be required for a prediction or learning task. Whereas the family size may significantly influence family expenditure, the eye color of family members does not directly or significantly influence it. Non-convex optimization techniques, such as sparse recovery, help discard irrelevant parameters and promote compact and accurate models.}%
\label{fig:fam}
\end{figure}

The first problem often arises when there is an excess of covariates, several of which may be spurious or have no effect on the response. \S~\ref{chap:spreg} discusses several such practical examples. For now, consider the example depicted in Figure~\ref{fig:fam}, that of expenditure prediction in a situation when the list of indicators include irrelevant ones such as whether the family lives in an odd-numbered house or not, which should arguably have no effect on expenditure. It is useful to eliminate such variables from consideration to promote consistency of the learned model.

The second problem is common in areas such as genomics and signal processing which face moderate to severe \emph{data starvation} and the number of data points $n$ available to estimate the model is small compared to the number of model parameters $p$ to be estimated, i.e., $n \ll p$. Standard statistical approaches require at least $n \geq p$ data points to ensure a consistent estimation of all $p$ model parameters and are unable to offer accurate model estimates in the face of data-starvation.

Both these problems can be handled by the \emph{sparse recovery} approach, which seeks to fit a sparse model vector (i.e., a vector with say, no more than $s$ non-zero entries) to the data. The least squares formulation, modified as a sparse recovery problem, is given below
\begin{align*}
\bth_\text{sp} = \underset{\bt \in \bR^p}{\arg\min}\ & \sum_{i=1}^n\br{y_i - \x_i^\top\bt}^2\\
\text{s.t.}\ & \bt \in \cB_0(s),
\end{align*}

Although the objective function in the above formulation is convex, the constraint $\norm{\bt}_0 \leq s$ (equivalently $\bt \in \cB_0(s)$ -- see list of mathematical notation at the beginning of this monograph) corresponds to a non-convex constraint set\elink{exer:tools-nonconv-sp}. Sparse recovery effortlessly solves the twin problems of discarding irrelevant covariates and countering data-starvation since typically, only $n \geq s\log p$ (as opposed to $n \geq p$) data points are required for sparse recovery to work which drastically reduces the data requirement. Unfortunately however, sparse-recovery is an NP-hard problem \citep{Natarajan1995}.\\

\noindent\textbf{Recommendation Systems} Several internet search engines and e-commerce websites utilize recommendation systems to offer items to users that they would benefit from, or like, the most. The problem of recommendation encompasses benign recommendations for songs etc, all the way to critical recommendations in personalized medicine.

To be able to make accurate recommendations, we need very good estimates of how each user likes each item (song), or would benefit from it (drug). We usually have first-hand information for some user-item pairs, for instance if a user has specifically rated a song or if we have administered a particular drug on a user and seen the outcome. However, users typically rate only a handful of the hundreds of thousands of songs in any commercial catalog and it is not feasible, or even advisable, to administer every drug to a user. Thus, for the vast majority of user-item pairs, we have no direct information.

It is useful to visualize this problem as a \emph{matrix completion} problem: for a set of $m$ users $u_1,\ldots,u_m$ and $n$ items $a_1,\ldots,a_n$, we have an $m \times n$ \emph{preference matrix}  $A = [A_{ij}]$ where $A_{ij}$ encodes the preference of the $i\th$ user for the $j\th$ item. We are able to directly view only a small number of entries of this matrix, for example, whenever a user explicitly rates an item. However, we wish to recover the remaining entries, i.e., complete this matrix. This problem is closely linked to the \emph{collaborative filtering} technique popular in recommendation systems.

Now, it is easy to see that unless there exists some structure in matrix, and by extension, in the way users rate items, there would be no relation between the unobserved entries and the observed ones. This would result in there being no unique way to complete the matrix. Thus, it is essential to impose some structure on the matrix. A structural assumption popularly made is that of low rank: we wish to fill in the missing entries of $A$ assuming that $A$ is a low rank matrix. This can make the problem well-posed and have a unique solution since the additional low rank structure links the entries of the matrix together. The unobserved entries can no longer take values independently of the values observed by us. Figure~\ref{fig:cf} depicts this visually.

\begin{figure}
\includegraphics[width=\columnwidth]{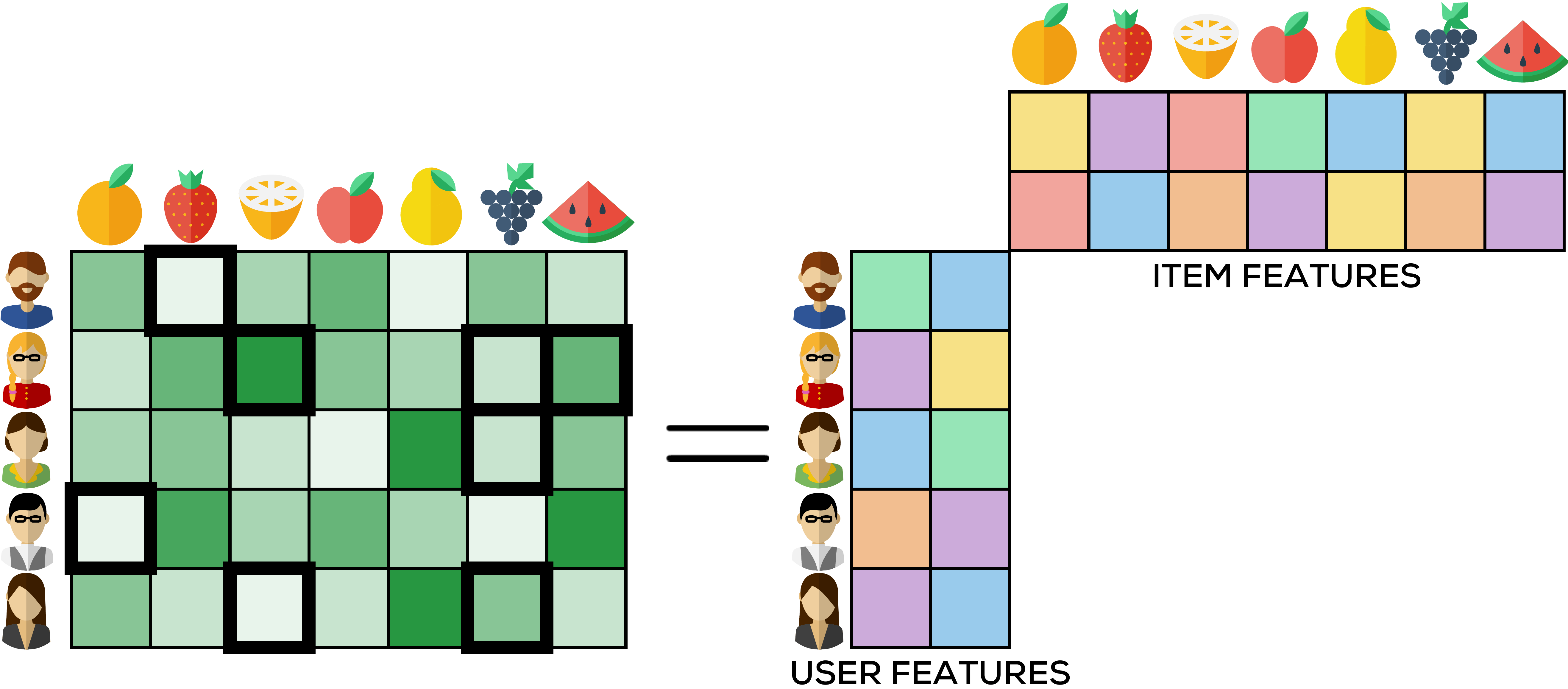}
\caption[Matrix Completion for Recommendation Systems]{Only the entries of the ratings matrix with thick borders are observed. Notice that users rate infrequently and some items are not rated even once. Non-convex optimization techniques such as low-rank matrix completion can help recover the unobserved entries, as well as reveal hidden features that are descriptive of user and item properties, as shown on the right hand side.}%
\label{fig:cf}
\end{figure}

If we denote by $\Omega \subset [m] \times [n]$, the set of observed entries of $A$, then the low rank matrix completion problem can be written as
\begin{align*}
\hat A_\text{lr} = \underset{X \in \R^{m \times n}}{\arg\min}\ & \sum_{(i,j) \in \Omega}\br{X_{ij} - A_{ij}}^2\\
\text{s.t.}\ & \rank(X) \leq r,
\end{align*}

This formulation also has a convex objective but a non-convex rank constraint\elink{exer:pgd-nc-rank}. This problem can be shown to be NP-hard as well. Interestingly, we can arrive at an alternate formulation by imposing the low-rank constraint indirectly. It turns out that\elink{exer:pgd-low-rank} assuming the ratings matrix to have rank at most $r$ is equivalent to assuming that the matrix $A$ can be written as $A = UV^\top$ with the matrices $U \in \bR^{m \times r}$ and $V \in \bR^{n \times r}$ having at most $r$ columns. This leads us to the following alternate formulation
\[
\hat A_\text{lv} = \underset{\substack{U \in \R^{m \times r}\\V \in \R^{n \times r}}}{\arg\min}\ \sum_{(i,j) \in \Omega}\br{U_i^\top V_j - A_{ij}}^2.
\]
There are no constraints in the formulation. However, the formulation requires joint optimization over a pair of variables $(U,V)$ instead of a single variable. More importantly, it can be shown\elink{exer:altmin-marg-conv} that the objective function is non-convex in $(U,V)$.

It is curious to note that the matrices $U$ and $V$ can be seen as encoding $r$-dimensional descriptions of users and items respectively. More precisely, for every user $i \in [m]$, we can think of the vector $U^i \in \bR^r$ (i.e., the $i$-th row of the matrix $U$) as describing user $i$, and for every item $j \in [n]$, use the row vector $V^j \in \bR^r$ to describe the item $j$ in vectoral form. The rating given by user $i$ to item $j$ can now be seen to be $A_{ij} \approx \ip{U^i}{V^j}$. Thus, recovering the rank $r$ matrix $A$ also gives us a bunch of $r$-dimensional latent vectors describing the users and items. These latent vectors can be extremely valuable in themselves as they can help us in understanding user behavior and item popularity, as well as be used in ``content''-based recommendation systems which can effectively utilize item and user features.\\

The above examples, and several others from machine learning, such as low-rank tensor decomposition, training deep networks, and training structured models, demonstrate the utility of non-convex optimization in naturally modeling learning tasks. However, most of these formulations are NP-hard to solve exactly, and sometimes even approximately. In the following discussion, we will briefly introduce a few approaches, classical as well as contemporary, that are used in solving such non-convex optimization problems.

\section{The Convex Relaxation Approach}
Faced with the challenge of non-convexity, and the associated NP-hardness, a traditional workaround in literature has been to modify the problem formulation itself so that existing tools can be readily applied. This is often done by \emph{relaxing} the problem so that it becomes a convex optimization problem. Since this allows familiar algorithmic techniques to be applied, the so-called \emph{convex relaxation} approach has been widely studied. For instance, there exist relaxed, convex problem formulations for both the recommendation system and the sparse regression problems. For sparse linear regression, the relaxation approach gives us the popular LASSO formulation.

Now, in general, such modifications change the problem drastically, and the solutions of the relaxed formulation can be poor solutions to the original problem. However, it is known that if the problem possesses certain nice structure, then under careful relaxation, these distortions, formally referred to as a``relaxation gap'', are absent, i.e., solutions to the relaxed problem would be optimal for the original non-convex problem as well.

Although a popular and successful approach, this still has limitations, the most prominent of them being scalability. Although the relaxed convex optimization problems are solvable in polynomial time, it is often challenging to solve them \emph{efficiently} for large-scale problems.

\section{The Non-Convex Optimization Approach}
Interestingly, in recent years, a new wisdom has permeated the fields of machine learning and signal processing, one that advises not to relax the non-convex problems and instead solve them directly. This approach has often been dubbed the \emph{non-convex optimization} approach owing to its goal of optimizing non-convex formulations directly.

Techniques frequently used in non-convex optimization approaches include simple and efficient primitives such as projected gradient descent, alternating minimization, the expectation-maximization algorithm, stochastic optimization, and variants thereof. These are very fast in practice and remain favorites of practitioners.

\begin{figure}[t]
\begin{subfigure}[t]{.5\columnwidth}
\centering \includegraphics[width=\columnwidth]{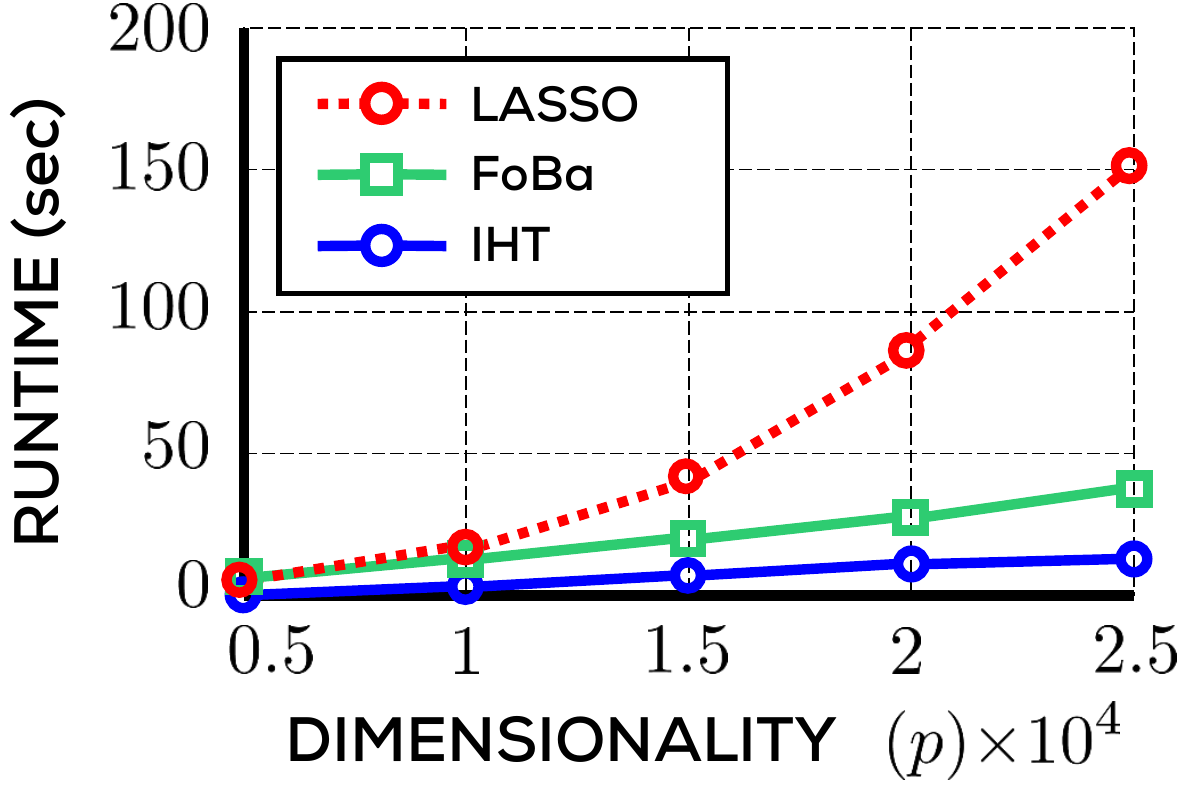}
\caption{Sparse Recovery (\S~\ref{chap:spreg})}
\label{fig:intro-comparison-spreg}
\end{subfigure}
\hfill
\begin{subfigure}[t]{.5\columnwidth}
\centering \includegraphics[width=\columnwidth]{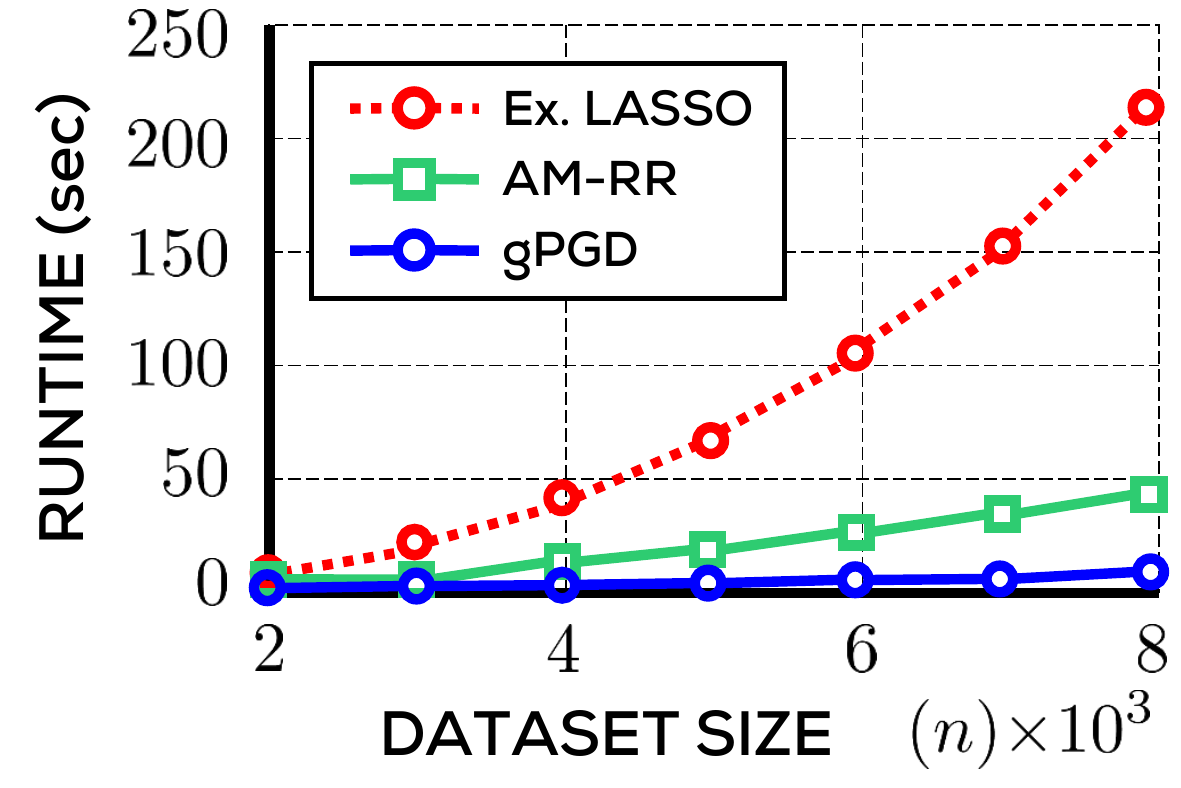}
\caption{Robust Regression (\S~\ref{chap:rreg})}
\label{fig:intro-comparison-rreg}
\end{subfigure}\\
\begin{subfigure}[t]{.5\columnwidth}
\centering \includegraphics[width=\columnwidth]{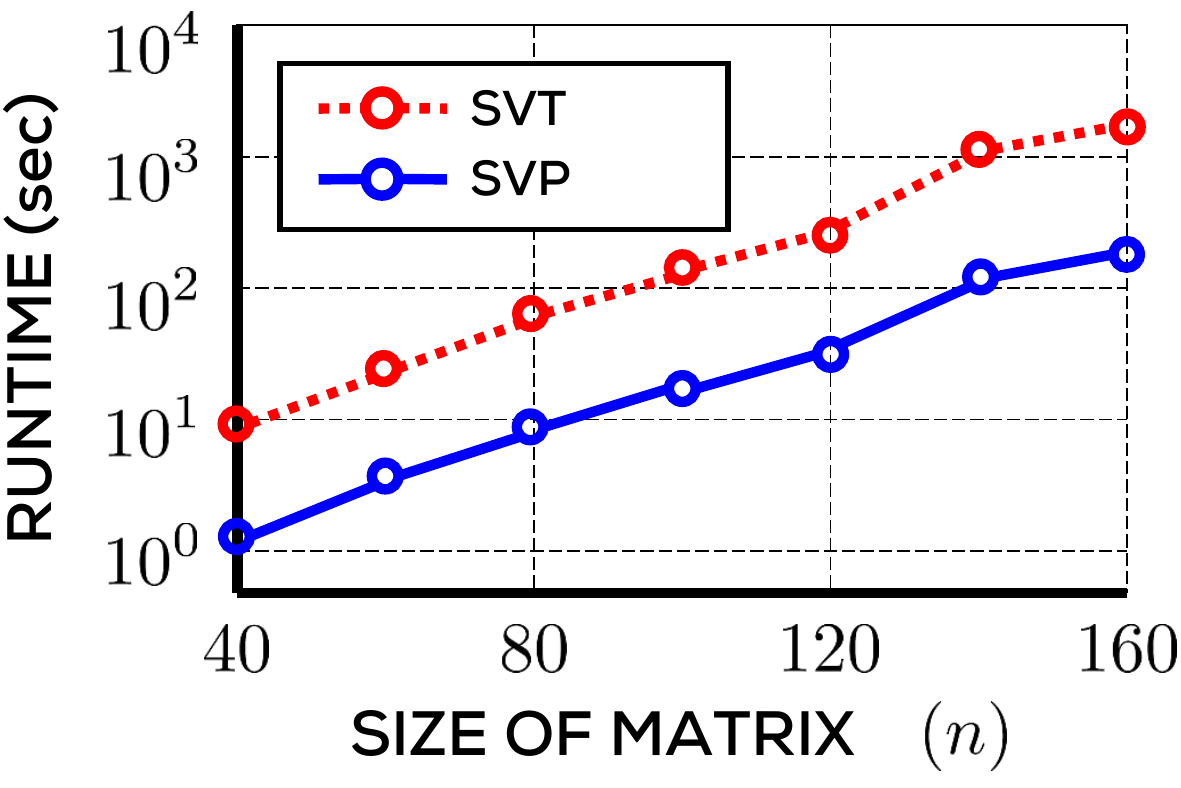}
\caption{Matrix Recovery (\S~\ref{chap:matrec})}
\label{fig:intro-comparison-matrec}
\end{subfigure}
\hfill
\begin{subfigure}[t]{.5\columnwidth}
\centering \includegraphics[width=\columnwidth]{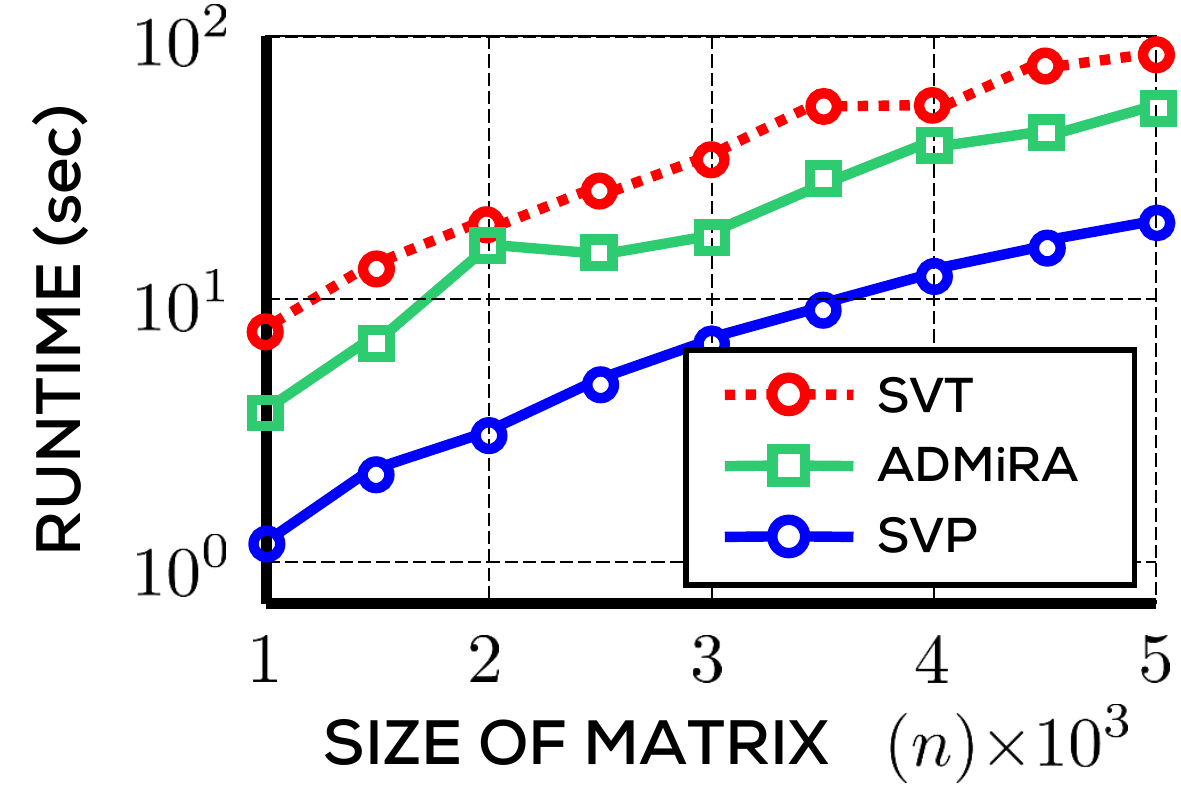}
\caption{Matrix Completion (\S~\ref{chap:matrec})}
\label{fig:intro-comparison-matcomp}
\end{subfigure}%
\caption[Relaxation vs. Non-convex Optimization Methods]{An empirical comparison of run-times offered by various approaches to four different non-convex optimization problems. LASSO, extended LASSO, SVT are relaxation-based methods whereas IHT, gPGD, FoBa, AM-RR, SVP, ADMiRA are non-convex methods. In all cases, non-convex optimization techniques offer routines that are faster, often by an order of magnitude or more, than relaxation-based methods. Note that Figures~\ref{fig:intro-comparison-matrec} and \ref{fig:intro-comparison-matcomp}, employ a $y$-axis at logarithmic scale. The details of the methods are present in the sections linked with the respective figures.}%
\label{fig:intro-comparison}
\end{figure}

At first glance, however, these efforts seem doomed to fail, given to the aforementioned NP-hardness results. However, in a series of deep and illuminating results, it has been repeatedly revealed that if the problem possesses nice structure, then not only do relaxation approaches succeed, but non-convex optimization algorithms do too. In such nice cases, non-convex approaches are able to only avoid NP-hardness, but actually offer provably optimal solutions. In fact, in practice, they often handsomely outperform relaxation-based approaches in terms of speed and scalability. Figure~\ref{fig:intro-comparison} illustrates this for some applications that we will investigate more deeply in later sections.

Very interestingly, it turns out that problem structures that allow non-convex approaches to avoid NP-hardness results, are very similar to those that allow their convex relaxation counterparts to avoid distortions and a large relaxation gap! Thus, it seems that if the problems possess nice structure, convex relaxation-based approaches, as well as non-convex techniques, both succeed. However, non-convex techniques usually offer more scalable solutions.

\section{Organization and Scope}
Our goal of this monograph is to present basic tools, both algorithmic and analytic, that are commonly used in the design and analysis of non-convex optimization algorithms, as well as present results which best represent the non-convex optimization philosophy. The presentation should enthuse, as well as equip, the interested reader and allow further readings, independent investigations, and applications of these techniques in diverse areas.

Given this broad aim, we shall appropriately restrict the number of areas we cover in this monograph, as well as the depth in which we cover each area. For instance, the literature abounds in results that seek to perform optimizations with more and more complex structures being imposed - from sparse recovery to low rank matrix recovery to low rank tensor recovery. However, we shall restrict ourselves from venturing too far into these progressions. Similarly, within the problem of sparse recovery, there exist results for recovery in the simple least squares setting, the more involved setting of sparse M-estimation, as well as the still more involved setting of sparse M-estimation in the presence of outliers. Whereas we will cover sparse least squares estimation in depth, we will refrain from delving too deeply into the more involved sparse M-estimation problems.

That being said, the entire presentation will be self contained and accessible to anyone with a basic background in algebra and probability theory. Moreover, the bibliographic notes given at the end of the sections will give pointers that should enable the reader to explore the state of the art not covered in this monograph.
\chapter{Mathematical Tools}
\label{chap:tools}

This section will introduce concepts, algorithmic tools, and analysis techniques used in the design and analysis of optimization algorithms. It will also explore simple convex optimization problems which will serve as a warm-up exercise.

\section{Convex Analysis}
We recall some basic definitions in convex analysis. Studying these will help us appreciate the structural properties of non-convex optimization problems later in the monograph. For the sake of simplicity, unless stated otherwise, we will assume that functions are continuously differentiable. We begin with the notion of a convex combination.

\begin{definition}[Convex Combination]
A convex combination of a set of $n$ vectors $\x_i \in \bR^p$, $i = 1 \ldots n$ in an arbitrary real space is a vector $\x_{\vtheta} := \sum_{i=1}^n\theta_i\x_i$ where $\vtheta = \br{\theta_1,\theta_2,\ldots,\theta_n}$, $\theta_i \geq 0$ and $\sum_{i=1}^n\theta_i = 1$.
\end{definition}

A set that is closed under arbitrary convex combinations is a convex set. A standard definition is given below. Geometrically speaking, convex sets are those that contain all line segments that join two points inside the set. As a result, they cannot have any inward ``bulges''.

\begin{figure}
\includegraphics[width=\columnwidth]{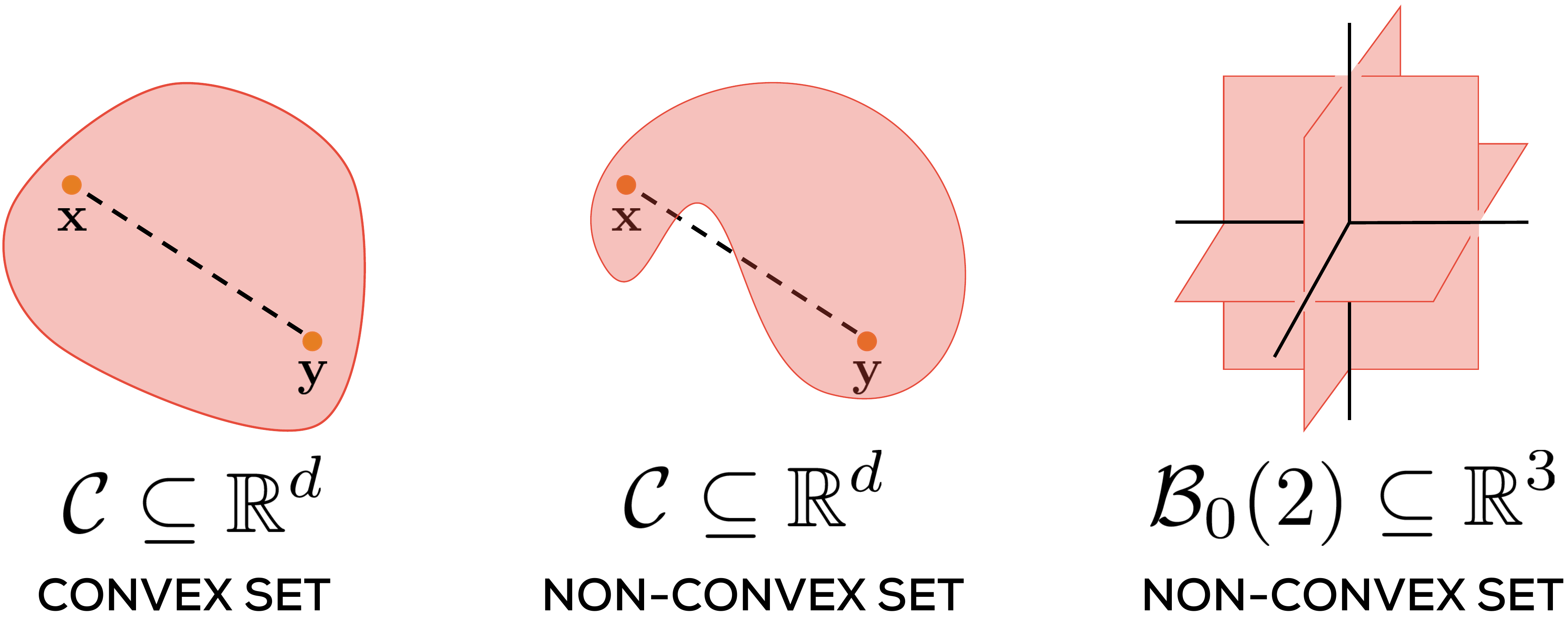}
\caption[Convex and Non-convex Sets]{A convex set is closed under convex combinations. The presence of even a single uncontained convex combination makes a set non-convex. Thus, a convex set cannot have inward ``bulges''. In particular, the set of sparse vectors is non-convex.}%
\label{fig:cvxset}
\end{figure}

\begin{definition}[Convex Set]
A set $\cC \in \R^p$ is considered convex if, for every $\x,\y \in \cC$ and $\lambda \in [0,1]$, we have $(1-\lambda)\cdot\x + \lambda\cdot\y \in \cC$ as well.
\end{definition}

Figure~\ref{fig:cvxset} gives visual representations of prototypical convex and non-convex sets. A related notion is that of convex functions which have a unique behavior under convex combinations. There are several definitions of convex functions, those that are more basic and general, as well as those that are restrictive but easier to use. One of the simplest definitions of convex functions, one that does not involve notions of derivatives, defines convex functions $f : \R^p \rightarrow \R$ as those for which, for every $\x,\y \in \R^p$ and every $\lambda \in [0,1]$, we have $f((1-\lambda)\cdot\x + \lambda\cdot\y) \leq (1-\lambda)\cdot f(\x) + \lambda\cdot f(\y)$. For continuously differentiable functions, a more usable definition follows.

\begin{definition}[Convex Function]
\label{defn:cvx-fn}
A continuously differentiable function $f: \R^p \rightarrow \R$ is considered convex if for every $\x,\y \in \R^p$ we have $f(\y) \geq f(\x) + \ip{\nabla f(\x)}{\y - \x}$, where $\nabla f(\x)$ is the gradient of $f$ at $\x$.
\end{definition}

A more general definition that extends to non-differentiable functions uses the notion of \emph{subgradient} to replace the gradient in the above expression. A special class of convex functions is the class of \emph{strongly convex} and \emph{strongly smooth} functions. These are critical to the study of algorithms for non-convex optimization. Figure~\ref{fig:cvxfn} provides a handy visual representation of these classes of functions.

\begin{figure}
\includegraphics[width=\columnwidth]{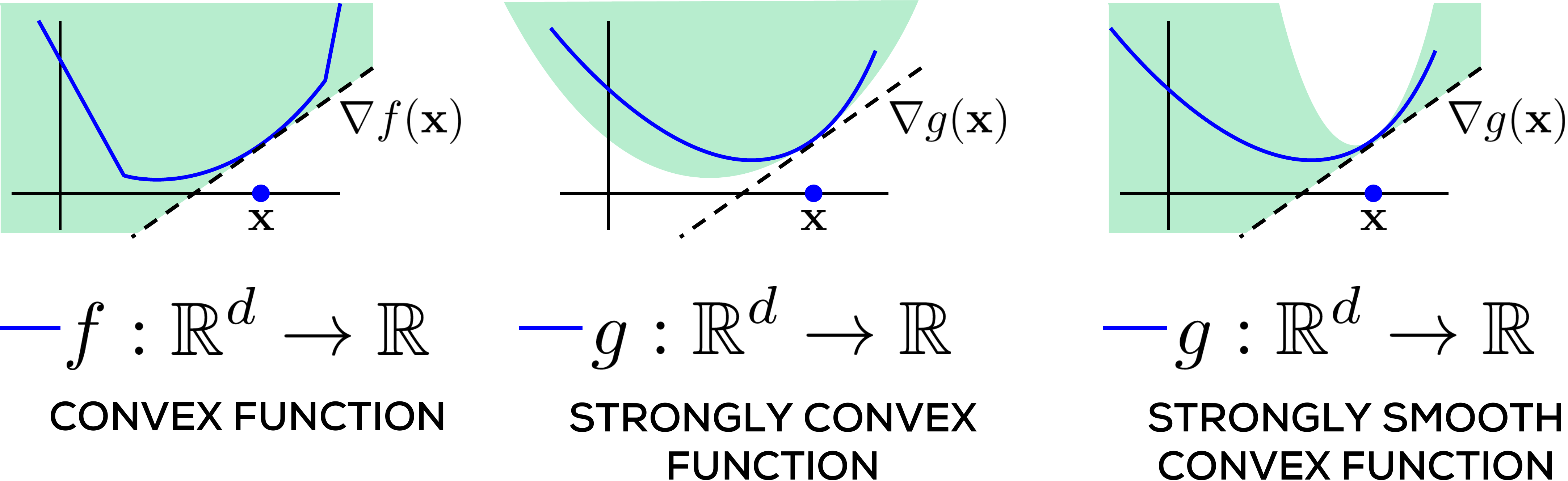}
\caption[Convex, Strongly Convex and Strongly Smooth Functions]{A convex function is lower bounded by its own tangent at all points. Strongly convex and smooth functions are, respectively, lower and upper bounded in the rate at which they may grow, by quadratic functions and cannot, again respectively, grow too slowly or too fast. In each figure, the shaded area describes regions the function curve is permitted to pass through.}%
\label{fig:cvxfn}
\end{figure}

\begin{definition}[Strongly Convex/Smooth Function]
\label{defn:strong-cvx-smooth-fn}
A continuously  differentiable function $f: \R^p \rightarrow \R$ is considered $\alpha$-strongly convex (SC) and $\beta$-strongly smooth (SS) if for every $\x,\y \in \R^p$, we have
\[
\frac{\alpha}{2}\norm{\x-\y}_2^2 \leq f(\y) - f(\x) - \ip{\nabla f(\x)}{\y - \x} \leq \frac{\beta}{2}\norm{\x-\y}_2^2.
\]
\end{definition}

It is useful to note that strong convexity places a quadratic lower bound on the growth of the function at every point -- the function must rise up at least as fast as a quadratic function. How fast it rises is characterized by the SC parameter $\alpha$. Strong smoothness similarly places a quadratic upper bound and does not let the function grow too fast, with the SS parameter $\beta$ dictating the upper limit.

We will soon see that these two properties are extremely useful in forcing optimization algorithms to rapidly converge to optima. Note that whereas strongly convex functions are definitely convex, strong smoothness does not imply convexity\elink{exer:tools-ss-cvx}. Strongly smooth functions may very well be non-convex. A property similar to strong smoothness is that of Lipschitzness which we define below.

\begin{definition}[Lipschitz Function]
\label{defn:lip}
A function $f: \R^p \rightarrow \R$ is $B$-Lipschitz if for every $\x,\y \in \R^p$, we have
\[
\abs{f(\x) - f(\y)} \leq B\cdot\norm{\x - \y}_2.
\]
\end{definition}

Notice that Lipschitzness places a upper bound on the growth of the function that is linear in the perturbation i.e., $\norm{\x-\y}_2$, whereas strong smoothness (SS) places a quadratic upper bound. Also notice that Lipschitz functions need not be differentiable. However, differentiable functions with bounded gradients are always Lipschitz\elink{exer:tools-diff-lip}. Finally, an important property that generalizes the behavior of convex functions on convex combinations is the Jensen's inequality.

\begin{lemma}[Jensen's Inequality]
\label{lem:jensen's}
If $X$ is a random variable taking values in the domain of a convex function $f$, then $\E{f(X)} \geq f(\E{X})$
\end{lemma}
This property will be useful while analyzing iterative algorithms.

\section{Convex Projections}
The projected gradient descent technique is a popular method for constrained optimization problems, both convex as well as non-convex. The \emph{projection} step plays an important role in this technique. Given any closed set $\cC \subset \R^p$, the projection operator $\Pi_\cC(\cdot)$ is defined as
\[
\Pi_\cC(\z) := \underset{\x\in\cC}{\argmin}\ \norm{\x-\z}_2.
\]
In general, one need not use only the $L^2$-norm in defining projections but is the most commonly used one. If $\cC$ is a convex set, then the above problem reduces to a convex optimization problem. In several useful cases, one has access to a closed form solution for the projection.

\begin{figure}
\includegraphics[width=\columnwidth]{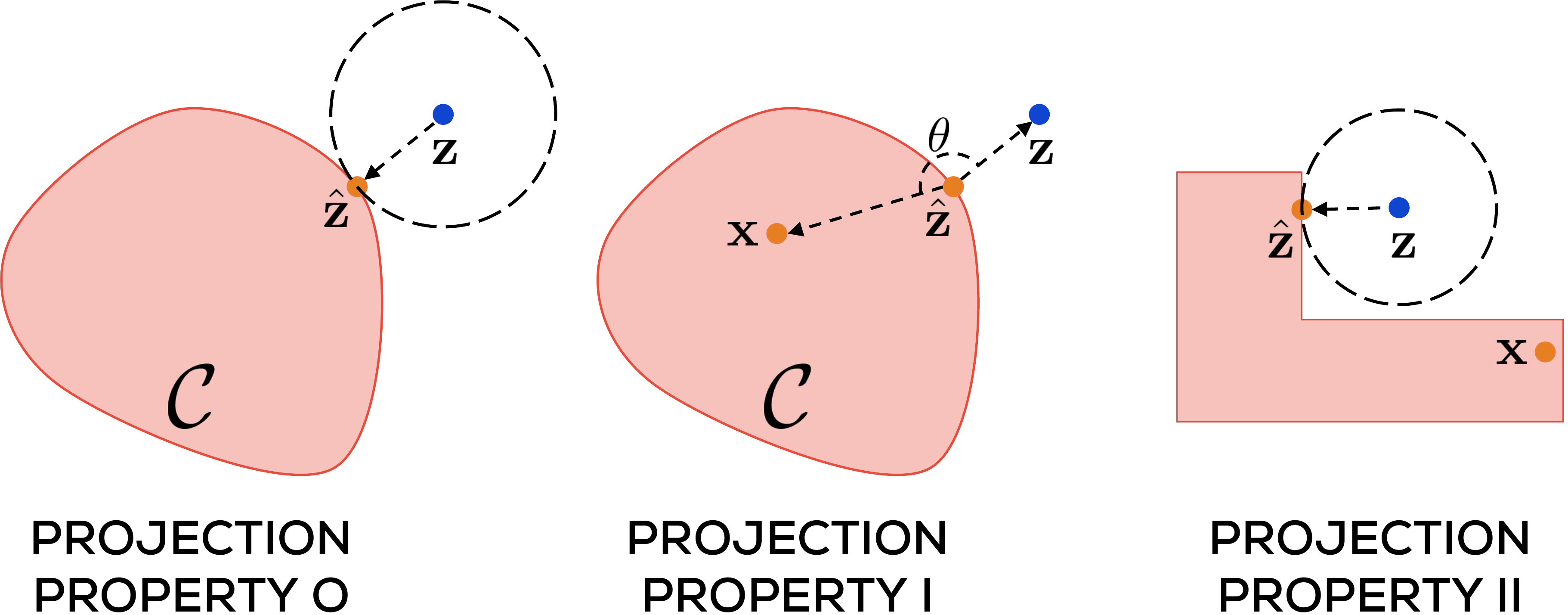}
\caption[Convex Projections and their Properties]{A depiction of projection operators and their properties. Projections reveal a closest point in the set being projected onto. For convex sets, projection property I ensures that the angle $\theta$ is always non-acute. Sets that satisfy projection property I also satisfy projection property II. Projection property II may be violated by non-convex sets. Projecting onto them may take the projected point $\vz$ closer to certain points in the set (for example, $\hat\vz$) but farther from others (for example, $\vx$).}%
\label{fig:proj}
\end{figure}

For instance, if $\cC = \B_2(1)$ i.e., the unit $L_2$ ball, then projection is equivalent\elink{exer:tools-proj-l2} to a normalization step
\[
\Pi_{\B_2(1)}(\z) = \begin{cases}\z/\norm{\z}_2 &\mbox{if } \norm{\z} > 1\\ \z & \mbox{otherwise}\end{cases}.
\]
For the case $\cC = \B_1(1)$, the projection step reduces to the popular \emph{soft thresholding} operation. If $\hat\z := \Pi_{\B_1(1)}(\z)$, then $\hat\z_i = \max\bc{\z_i - \theta, 0}$, where $\theta$ is a threshold that can be decided by a sorting operation on the vector \citep[see][for details]{DuchiS-SSC2008}.

Projections onto convex sets have some very useful properties which come in handy while analyzing optimization algorithms. In the following, we will study three properties of projections. These are depicted visually in Figure~\ref{fig:proj} to help the reader gain an intuitive appeal.

\begin{lemma}[Projection Property-O]
\label{lem:proj-prop-o}
For any set (convex or not) $\cC \subset \R^p$ and $\z \in \R^p$, let $\hat\z := \Pi_\cC(\z)$. Then for all $\x \in \cC$, $\norm{\hat\z - \z}_2 \leq \norm{\x - \z}_2$.
\end{lemma}

This property follows by simply observing that the projection step solves the the optimization problem $\min_{\x\in\cC}\norm{\x-\z}_2$. Note that this property holds for all sets, whether convex or not. However, the following two properties necessarily hold only for convex sets.

\begin{lemma}[Projection Property-I]
\label{lem:proj-prop-1}
For any convex set $\cC \subset \R^p$ and any $\z \in \R^p$, let $\hat\z := \Pi_\cC(\z)$. Then for all $\x \in \cC$, $\ip{\x - \hat\z}{\z - \hat\z} \leq 0$.
\end{lemma}
\begin{proof}

To prove this, assume the contra-positive. Suppose for some $\x \in \cC$, we have $\ip{\x - \hat\z}{\z - \hat\z} > 0$. Now, since $\cC$ is convex and $\hat\z,\x \in \cC$, for any $\lambda \in [0,1]$, we have $\x_\lambda := \lambda\cdot\x + (1 - \lambda)\cdot\hat\z \in \cC$. We will now show that for some value of $\lambda \in [0,1]$, it must be the case that $\norm{\z - \x_\lambda}_2 < \norm{\z - \hat\z}_2$. This will contradict the fact that $\hat\z$ is the closest point in the convex set to $\z$ and prove the lemma. All that remains to be done is to find such a value of $\lambda$. The reader can verify that any value of $0 < \lambda < \min\bc{1, \frac{2\ip{\x-\hat\z}{\z-\hat\z}}{\norm{\x-\hat\z}_2^2}}$ suffices. Since we assumed $\ip{\x - \hat\z}{\z - \hat\z} > 0$, any value of $\lambda$ chosen this way is always in $(0,1]$.
\end{proof}
Projection Property-I can be used to prove a very useful \emph{contraction property} for convex projections. In some sense, a convex projection brings a point closer to \emph{all} points in the convex set simultaneously.
\begin{lemma}[Projection Property-II]
\label{lem:proj-prop-2}
For any convex set $\cC \subset \R^p$ and any $\z \in \R^p$, let $\hat\z := \Pi_\cC(\z)$. Then for all $\x \in \cC$, $\norm{\hat\z - \x}_2 \leq \norm{\z - \x}_2$.
\end{lemma}
\begin{proof}
We have the following elementary inequalities
\begin{align*}
\norm{\z - \x}_2^2 &= \norm{(\hat\z - \x) - (\hat\z - \z)}_2^2\\
									 &= \norm{\hat\z - \x}_2^2 + \norm{\hat\z - \z}_2^2 - 2\ip{\hat\z - \x}{\hat\z - \z}\\
									 &\geq \norm{\hat\z - \x}_2^2 + \norm{\hat\z - \z}_2^2 \tag*{(Projection Property-I)}\\
									 &\geq \norm{\hat\z - \x}_2^2\tag*{\qedhere}
\end{align*}
\end{proof}
Note that Projection Properties-I and II are also called \emph{first order} properties and can be violated if the underlying set is non-convex. However, Projection Property-O, often called a \emph{zeroth order} property, always holds, whether the underlying set is convex or not.

\section{Projected Gradient Descent}
We now move on to study the projected gradient descent algorithm. This is an extremely simple and efficient technique that can effortlessly scale to large problems. Although we will apply this technique to non-convex optimization tasks later, we first look at its behavior on convex optimization problems as a warm up exercise. We warn the reader that the proof techniques used in the convex case do not apply directly to non-convex problems. Consider the following optimization problem:
\begin{equation}
	\label{eq:cons-cvx-opt}\tag*{(CVX-OPT)}
	\begin{split}
		\min_{\vx \in \bR^p}\ & f(\vx) \\
		\text{s.t.}\ & \vx \in \cC.
	\end{split}
\end{equation}
In the above optimization problem, $\cC \subset \R^p$ is a convex constraint set and $f: \R^p \rightarrow \R$ is a convex objective function. We will assume that we have oracle access to the gradient and projection operators, i.e., for any point $\x \in \R^p$ we are able to access $\nabla f(\x)$ and $\Pi_\cC(\x)$.
\begin{algorithm}[t]
	\caption{Projected Gradient Descent (PGD)}
	\label{algo:pgd}
	\begin{algorithmic}[1]
		{
			\REQUIRE Convex objective $f$, convex constraint set $\cC$, step lengths $\eta_t$
			\ENSURE A point $\hat\x \in \cC$ with near-optimal objective value
			\STATE $\x^1 \leftarrow \vzero$
			\FOR{$t = 1, 2, \ldots, T$}
				\STATE $\z^{t+1} \leftarrow \x^t - \eta_t\cdot\nabla f(\x^t)$
				\STATE $\x^{t+1} \leftarrow \Pi_\cC(\z^{t+1})$
			\ENDFOR
			\STATE (OPTION 1) \textbf{return} {$\hat\x_{\text{final}} = \x^T$}
			\STATE (OPTION 2) \textbf{return} {$\hat\x_{\text{avg}} = (\sum_{t=1}^T \x^t)/T$}
			\STATE (OPTION 3) \textbf{return} {$\hat\x_{\text{best}} = \arg\min_{t\in[T]} f(\x^t)$}
		}
	\end{algorithmic}
\end{algorithm}			

The projected gradient descent algorithm is stated in Algorithm~\ref{algo:pgd}. The procedure generates iterates $\x^t$ by taking steps guided by the gradient in an effort to reduce the function value locally. Finally it returns either the final iterate, the average iterate, or the best iterate.

\section{Convergence Guarantees for PGD}
We will analyze PGD for objective functions that are either a) convex with bounded gradients, or b) strongly convex and strongly smooth. Let $f^\ast = \min_{\x \in \cC}\ f(\x)$ be the optimal value of the optimization problem. A point $\hat\x \in \cC$ will be said to be an $\epsilon$-optimal solution if $f(\hat\x) \leq f^\ast + \epsilon$.

\subsection{Convergence with Bounded Gradient Convex Functions}
Consider a convex objective function $f$ with bounded gradients over a convex constraint set $\cC$ i.e., $\norm{f(\x)}_2 \leq G$ for all $\x \in \cC$.

\begin{theorem}
\label{thm:pgd-conv-proof}
Let $f$ be a convex objective with bounded gradients and Algorithm~\ref{algo:pgd} be executed for $T$ time steps with step lengths $\eta_t = \eta = \frac{1}{\sqrt T}$. Then, for any $\epsilon > 0$, if $T = \bigO{\frac{1}{\epsilon^2}}$, then $\frac{1}{T}\sum_{t=1}^Tf(\x^t) \leq f^\ast + \epsilon$.
\end{theorem}

We see that the PGD algorithm in this setting ensures that the function value of the iterates approaches $f^\ast$ \emph{on an average}. We can use this result to prove the convergence of the PGD algorithm. If we use OPTION 3, i.e., $\hat\x_{\text{best}}$, then since by construction, we have $f(\hat\x_{\text{best}}) \leq f(\x^t)$ for all $t$, by applying Theorem~\ref{thm:pgd-conv-proof}, we get
\[
f(\hat\x_{\text{best}}) \leq \frac{1}{T}\sum_{t=1}^Tf(\x^t) \leq f^\ast + \epsilon,
\]
If we use OPTION 2, i.e., $\hat\x_{\text{avg}}$, which is cheaper since we do not have to perform function evaluations to find the best iterate, we can apply Jensen's inequality (Lemma~\ref{lem:jensen's}) to get the following
\[
f(\hat\x_{\text{avg}}) = f\br{\frac{1}{T}\sum_{t=1}^T\x^t} \leq \frac{1}{T}\sum_{t=1}^Tf(\x^t) \leq f^\ast + \epsilon.
\]
Note that the Jensen's inequality may be applied only when the function $f$ is convex. Now, whereas OPTION 1 i.e., $\hat\x_{\text{final}}$, is the cheapest and does not require any additional operations, $\hat\x_{\text{final}}$ does not converge to the optimum for convex functions in general and may oscillate close to the optimum. However, we shall shortly see that $\hat\x_{\text{final}}$ does converge if the objective function is strongly smooth. Recall that strongly smooth functions may not grow at a faster-than-quadratic rate.

The reader would note that we have set the step length to a value that depends on the total number of iterations $T$ for which the PGD algorithm is executed. This is called a \emph{horizon-aware} setting of the step length. In case we are not sure what the value of $T$ would be, a \emph{horizon-oblivious} setting of $\eta_t = \frac{1}{\sqrt t}$ can also be shown to work\elink{exer:tools-hor}.

\begin{proof}[Proof (of Theorem~\ref{thm:pgd-conv-proof}).]
Let $\x^\ast \in \arg\min_{\x \in \cC}\ f(\x)$ denote any point in the constraint set where the optimum function value is achieved. Such a point always exists if the constraint set is closed and the objective function continuous. We will use the following \emph{potential function} $\Phi_t = f(\x^t)-f(\x^\ast)$ to track the progress of the algorithm. Note that $\Phi_t$ measures the sub-optimality of the $t$-th iterate. Indeed, the statement of the theorem is equivalent to claiming that $\frac{1}{T}\sum_{t=1}^T\Phi_t \leq \epsilon$.\\

\noindent\textbf{(Apply Convexity)} We apply convexity to upper bound the potential function at every step. Convexity is a global property and very useful in getting an upper bound on the level of sub-optimality of the current iterate in such analyses.
\[
\Phi_t = f(\x^t)-f(\x^\ast) \leq \ip{\nabla f(\x^t)}{\x^t-\x^\ast}
\]
We now do some elementary manipulations
\begin{align*}
&\ip{\nabla f(\x^t)}{\x^t-\x^\ast} = \frac{1}{\eta}\ip{\eta\cdot\nabla f(\x^t)}{\x^t-\x^\ast}\\
																	&= \frac{1}{2\eta}\br{\norm{\x^t-\x^\ast}_2^2 + \norm{\eta\cdot\nabla f(\x^t)}_2^2 - \norm{\x^t - \eta\cdot\nabla f(\x^t) -\x^\ast}_2^2}\\
																	&= \frac{1}{2\eta}\br{\norm{\x^t-\x^\ast}_2^2 + \norm{\eta\cdot\nabla f(\x^t)}_2^2 - \norm{\z^{t+1} -\x^\ast}_2^2}\\
																	&\leq \frac{1}{2\eta}\br{\norm{\x^t-\x^\ast}_2^2 + \eta^2G^2 - \norm{\z^{t+1} -\x^\ast}_2^2},
\end{align*}
where the first step applies the identity $2ab = a^2 + b^2 - (a+b)^2$, the second step uses the update step of the PGD algorithm that sets $\z^{t+1} \leftarrow \x^t - \eta_t\cdot\nabla f(\x^t)$, and the third step uses the fact that the objective function $f$ has bounded gradients.

\noindent\textbf{(Apply Projection Property)} We apply Lemma~\ref{lem:proj-prop-2} to get
\[
\norm{\z^{t+1} -\x^\ast}_2^2 \geq \norm{\x^{t+1} -\x^\ast}_2^2
\]
Putting all these together gives us
\[
\Phi_t \leq \frac{1}{2\eta}\br{\norm{\x^t-\x^\ast}_2^2 - \norm{\x^{t+1} -\x^\ast}_2^2} + \frac{\eta G^2}{2}
\]
The above expression is interesting since it tells us that, apart from the $\eta G^2/2$ term which is small as $\eta = \frac{1}{\sqrt T}$, the current sub-optimality $\Phi_t$ is small if the consecutive iterates $\x^t$ and $\x^{t+1}$ are close to each other (and hence similar in distance from $\x^\ast$).

This observation is quite useful since it tells us that once PGD stops making a lot of progress, it actually converges to the optimum! In hindsight, this is to be expected. Since we are using a constant step length, only a vanishing gradient can cause PGD to stop progressing. However, for convex functions, this only happens at global optima. Summing the expression up across time steps, performing telescopic cancellations, using $\x^1 = \vzero$, and dividing throughout by $T$ gives us
\begin{align*}
\frac{1}{T}\sum_{t=1}^T\Phi_t &\leq \frac{1}{2\eta T}\br{\norm{\x^\ast}_2^2 - \|{\x^{T+1} - \x^\ast}\|_2^2} + \frac{\eta G^2}{2}\\
				&\leq \frac{1}{2\sqrt T}\br{\norm{\x^\ast}_2^2 + G^2},
\end{align*}
where in the second step, we have used the fact that $\norm{\x^{t+1} - \x^\ast}_2 \geq 0$ and $\eta = 1/\sqrt T$. This gives us the claimed result.
\end{proof}

\subsection{Convergence with Strongly Convex and Smooth Functions}
We will now prove a stronger guarantee for PGD when the objective function is strongly convex and strongly smooth (see Definition~\ref{defn:strong-cvx-smooth-fn}).
\begin{theorem}
\label{thm:pgd-sc-ss-proof}
Let $f$ be an objective that satisfies the $\alpha$-SC and $\beta$-SS properties. Let Algorithm~\ref{algo:pgd} be executed with step lengths $\eta_t = \eta = \frac{1}{\beta}$. Then after at most $T = \bigO{\frac{\beta}{\alpha}\log\frac{\beta}{\epsilon}}$ steps, we have $f(\x^T) \leq f(\x^\ast) + \epsilon$.
\end{theorem}

This result is particularly nice since it ensures that the final iterate $\hat\x_{\text{final}} = \x^T$ converges, allowing us to use OPTION 1 in Algorithm~\ref{algo:pgd} when the objective is SC/SS. A further advantage is the accelerated rate of convergence. Whereas for general convex functions, PGD requires $\bigO{\frac{1}{\epsilon^2}}$ iterations to reach an $\epsilon$-optimal solution, for SC/SS functions, it requires only $\bigO{\log\frac{1}{\epsilon}}$ iterations.

The reader would notice the insistence on the step length being set to $\eta = \frac{1}{\beta}$. In fact the proof we show below crucially uses this setting. In practice, for many problems, $\beta$ may not be known to us or may be expensive to compute which presents a problem. However, as it turns out, it is not necessary to set the step length exactly to $1/\beta$. The result can be shown to hold even for values of $\eta < 1/\beta$ which are nevertheless large enough, but the proof becomes more involved. In practice, the step length is tuned globally by doing a grid search over several $\eta$ values, or per-iteration using line search mechanisms, to obtain a step length value that assures good convergence rates.

\begin{proof}[Proof (of Theorem~\ref{thm:pgd-sc-ss-proof}).]
This proof is a nice opportunity for the reader to see how the SC/SS properties are utilized in a convergence analysis. As with convexity in the proof of Theorem~\ref{thm:pgd-conv-proof}, the strong convexity property is a global property that will be useful in assessing the progress made so far by relating the optimal point $\x^\ast$ with the current iterate $\x^t$. Strong smoothness on the other hand, will be used locally to show that the procedure makes significant progress between iterates.

We will prove the result by showing that after at most $T = \bigO{\frac{\beta}{\alpha}\log\frac{1}{\epsilon}}$ steps, we will have $\norm{\x^T - \x^\ast}_2^2 \leq \frac{2\epsilon}{\beta}$. This already tells us that we have reached very close to the optimum. However, we can use this to show that $\x^T$ is $\epsilon$-optimal in function value as well. Since we are very close to the optimum, it makes sense to apply strong smoothness to upper bound the sub-optimality as follows
\[
f(\x^T) \leq f(\x^\ast) + \ip{\nabla f(\x^\ast)}{\x^T - \x^\ast} + \frac{\beta}{2}\norm{\x^T-\x^\ast}_2^2.
\]
Now, since $\x^\ast$ is an optimal point for the constrained optimization problem with a convex constraint set $\cC$, the first order optimality condition \cite[see][Proposition 1.3]{Bubeck2015} gives us $\ip{\nabla f(\x^\ast)}{\x - \x^\ast} \leq 0$ for any $\x \in \cC$. Applying this condition with $\x = \x^T$ gives us
\[
f(\x^T) - f(\x^\ast) \leq \frac{\beta}{2}\norm{\x^T-\x^\ast}_2^2 \leq \epsilon,
\]
which proves that $\x^T$ is an $\epsilon$-optimal point. We now show $\norm{\x^T - \x^\ast}_2^2 \leq \frac{2\epsilon}{\beta}$. Given that we wish to show convergence in terms of the iterates, and not in terms of the function values, as we did in Theorem~\ref{thm:pgd-conv-proof}, a natural potential function for this analysis is $\Phi_t = \norm{\x^t - \x^\ast}_2^2$.\\

\noindent\textbf{(Apply Strong Smoothness)} As discussed before, we use it to show that PGD always makes significant progress in each iteration.
\begin{align*}
&f(\x^{t+1}) - f(\x^t) \leq \ip{\nabla f(\x^t)}{\x^{t+1}-\x^t} + \frac{\beta}{2}\norm{\x^t-\x^{t+1}}_2^2\\
											&= \ip{\nabla f(\x^t)}{\x^{t+1}-\x^\ast} + \ip{\nabla f(\x^t)}{\x^\ast-\x^t} + \frac{\beta}{2}\norm{\x^t-\x^{t+1}}_2^2\\
											&= \frac{1}{\eta}\ip{\x^t - \z^{t+1}}{\x^{t+1}-\x^t} + \ip{\nabla f(\x^t)}{\x^\ast-\x^t} + \frac{\beta}{2}\norm{\x^t-\x^{t+1}}_2^2
\end{align*}

\noindent\textbf{(Apply Projection Rule)} The above expression contains an unwieldy term $\z^{t+1}$. Since this term only appears during projection steps, we eliminate it by applying Projection Property-I (Lemma~\ref{lem:proj-prop-1}) to get
\begin{align*}
\ip{\x^t - \z^{t+1}}{\x^{t+1}-\x^\ast} &\leq \ip{\x^t - \x^{t+1}}{\x^{t+1}-\x^\ast}\\
																			&= \frac{\norm{\x^t-\x^\ast}_2^2 - \norm{\x^t - \x^{t+1}}_2^2 - \norm{\x^{t+1}-\x^\ast}_2^2}{2}
\end{align*}
Using $\eta = 1/\beta$ and combining the above results gives us
\[
f(\x^{t+1}) - f(\x^t) \leq \ip{\nabla f(\x^t)}{\x^\ast-\x^t} + \frac{\beta}{2}\br{\norm{\x^t-\x^\ast}_2^2 - \norm{\x^{t+1}-\x^\ast}_2^2}
\]

\noindent\textbf{(Apply Strong Convexity)} The above expression is perfect for a telescoping step but for the inner product term. Fortunately, this can be eliminated using strong convexity.
\[
\ip{\nabla f(\x^t)}{\x^\ast-\x^t} \leq f(\x^\ast) - f(\x^t) - \frac{\alpha}{2}\norm{\x^t-\x^\ast}_2^2
\]
Combining with the above this gives us
\[
f(\x^{t+1}) - f(\x^\ast) \leq \frac{\beta-\alpha}{2}\norm{\x^t-\x^\ast}_2^2 - \frac{\beta}{2}\norm{\x^{t+1}-\x^\ast}_2^2.
\]
The above form seems almost ready for a telescoping exercise. However, something much stronger can be said here, especially due to the $\frac{-\alpha}{2}\norm{\x^t-\x^\ast}_2^2$ term. Notice that we have $f(\x^{t+1}) \geq f(\x^\ast)$. This means
\[
\frac{\beta}{2}\norm{\x^{t+1}-\x^\ast}_2^2 \leq \frac{\beta-\alpha}{2}\norm{\x^t-\x^\ast}_2^2,
\]
which can be written as
\[
\Phi_{t+1} \leq \br{1 - \frac{\alpha}{\beta}}\Phi_t \leq \exp\br{-\frac{\alpha}{\beta}}\Phi_t,
\]
where we have used the fact that $1 - x \leq \exp(-x)$ for all $x \in \R$. What we have arrived at is a very powerful result as it assures us that the potential value goes down by a constant fraction at every iteration! Applying this result recursively gives us
\[
\Phi_{t+1} \leq \exp\br{-\frac{\alpha t}{\beta}}\Phi_1 = \exp\br{-\frac{\alpha t}{\beta}}\norm{\x^\ast}_2^2,
\]
since $\x^1 = \vzero$. Thus, we deduce that $\Phi_T = \norm{\x^T - \x^\ast}_2^2 \leq \frac{2\epsilon}{\beta}$ after at most $T = \bigO{\frac{\beta}{\alpha}\log\frac{\beta}{\epsilon}}$ steps which finishes the proof
\end{proof}

We notice that the convergence of the PGD algorithm is of the form $\norm{\x^{t+1}-\x^\ast}_2^2 \leq \exp\br{-\frac{\alpha t}{\beta}}\norm{\x^\ast}_2^2$. The number $\kappa := \frac{\beta}{\alpha}$ is the \emph{condition number} of the optimization problem. The concept of condition number is central to numerical optimization. Below we give an informal and generic definition for the concept. In later sections we will see the condition number appearing repeatedly in the context of the convergence of various optimization algorithms for convex, as well as non-convex problems. The exact numerical form of the condition number (for instance here it is $\beta/\alpha$) will also change depending on the application at hand. However, in general, all these definitions of condition number will satisfy the following property.

\begin{definition}[Condition Number - Informal]
\label{defn:condition-number}
The condition number of a function $f : \cX \rightarrow \R$ is a scalar $\kappa \in \bR$ that bounds how much the function value can change relative to a perturbation of the input.
\end{definition}

Functions with a small condition number are stable and changes to their input do not affect the function output values too much. However, functions with a large condition number can be quite jumpy and experience abrupt changes in output values even if the input is changed slightly. To gain a deeper appreciation of this concept, consider a differentiable function $f$ that is also $\alpha$-SC and $\beta$-SS. Consider a stationary point for $f$ i.e., a point $\x$ such that $\nabla f(\x) = \vzero$. For a general function, such a point can be a local optima or a saddle point. However, since $f$ is strongly convex, $\x$ is the (unique) global minima\elink{exer:tools-sc-unique-minima} of $f$. Then we have, for any other point $\y$
\[
\frac{\alpha}{2}\norm{\x-\y}_2^2 \leq f(\y) - f(\x) \leq \frac{\beta}{2}\norm{\x-\y}_2^2
\]
Dividing throughout by $\frac{\alpha}{2}\norm{\x-\y}_2^2$ gives us
\[
\frac{f(\y) - f(\x)}{\frac{\alpha}{2}\norm{\x-\y}_2^2} \in \bs{1,\frac{\beta}{\alpha}} := [1,\kappa]
\]
Thus, upon perturbing the input from the global minimum $\x$ to a point $\norm{\x-\y}_2 =: \epsilon$ distance away, the function value does change much -- it goes up by an amount at least $\frac{\alpha\epsilon^2}{2}$ but at most $\kappa\cdot\frac{\alpha\epsilon^2}{2}$. Such well behaved response to perturbations is very easy for optimization algorithms to exploit to give fast convergence.

The condition number of the objective function can significantly affect the convergence rate of algorithms. Indeed, if $\kappa = \frac{\beta}{\alpha}$ is small, then $\exp\br{-\frac{\alpha}{\beta}} = \exp\br{-\frac{1}{\kappa}}$ would be small, ensuring fast convergence. However, if $\kappa \gg 1$ then $\exp\br{-\frac{1}{\kappa}} \approx 1$ and the procedure might offer slow convergence.

\section{Exercises}
\begin{exer}
\label{exer:tools-ss-cvx}
Show that strong smoothness does not imply convexity by constructing a non-convex function $f: \R^p \rightarrow \R$ that is $1$-SS.
\end{exer}
\begin{exer}
\label{exer:tools-diff-lip}
Show that if a differentiable function $f$ has bounded gradients i.e., $\norm{\nabla f(\vx)}_2 \leq G$ for all $\x \in \R^d$, then $f$ is Lipschitz. What is its Lipschitz constant?\\
\textit{Hint}: use the mean value theorem.
\end{exer}
\begin{exer}
\label{exer:tools-proj-l2}
Show that for any point $\z \notin \B_2(r)$, the projection onto the ball is given by $\Pi_{\B_2(r)}(\z) = \frac{r}{\norm{\z}_2}\cdot\z$.
\end{exer}
\begin{exer}
\label{exer:tools-hor}
Show that a \emph{horizon-oblivious} setting of $\eta_t = \frac{1}{\sqrt t}$ while executing the PGD algorithm with a convex function with bounded gradients also ensures convergence.\\
\textit{Hint}: the convergence rates may be a bit different for this setting.
\end{exer}
\begin{exer}
\label{exer:tools-sc-unique-minima}
Show that if $f: \bR^p \rightarrow \bR$ is a strongly convex function that is differentiable, then there is a unique point $\x^\ast \in \bR^p$ that minimizes the function value $f$ i.e., $f(\x^\ast) = \min_{\x \in \bR^p}\ f(\x)$.
\end{exer}
\begin{exer}
\label{exer:tools-nonconv-sp}
Show that the set of sparse vectors $\cB_0(s) \subset \R^p$ is non-convex for any $s < p$. What happens when $s = p$?
\end{exer}
\begin{exer}
\label{exer:pgd-nc-rank}
Show that $\cB_\text{rank}(r) \subseteq \R^{n \times n}$, the set of $n \times n$ matrices with rank at most $r$, is non-convex for any $r < n$. What happens when $r = n$?
\end{exer}
\begin{exer}
Consider the Cartesian product set $\cC = \R^{m \times r} \times \R^{n \times r}$. Show that it is convex.
\end{exer}
\begin{exer}
Consider a least squares optimization problem with a strongly convex and smooth objective. Show that the condition number of this problem is equal to the condition number of the Hessian matrix of the objective function.
\end{exer}
\begin{exer}
\label{exer:tools-sc-unique-minima-cons}
Show that if $f: \bR^p \rightarrow \bR$ is a strongly convex function that is differentiable, then optimization problems with $f$ as an objective and a convex constraint set $\cC$ always have a unique solution i.e., there is a unique point $\x^\ast \in \cC$ that is a solution to the optimization problem $\arg\min_{\x\in\cC}\ f(x)$. This generalizes the result in Exercise~\ref{exer:tools-sc-unique-minima}.\\
\textit{Hint}: use the first order optimality condition (see proof of Theorem~\ref{thm:pgd-sc-ss-proof})
\end{exer}

\section{Bibliographic Notes}
\label{sec:tools-bib}
The sole aim of this discussion was to give a self-contained introduction to concepts and tools in convex analysis and descent algorithms in order to seamlessly introduce non-convex optimization techniques and their applications in subsequent sections. However, we clearly realize our inability to cover several useful and interesting results concerning convex functions and optimization techniques given the paucity of scope to present this discussion. We refer the reader to literature in the field of optimization theory for a much more relaxed and deeper introduction to the area of convex optimization. Some excellent examples include \citep{Bertsekas2016,BoydV2004,Bubeck2015,Nesterov2013,SraNW2011}.

\part{Non-convex Optimization Primitives}

\chapter{Non-Convex Projected Gradient Descent}
\label{chap:pgd}

In this section we will introduce and study gradient descent-style methods for non-convex optimization problems. In \S~\ref{chap:tools}, we studied the projected gradient descent method for convex optimization problems. Unfortunately, the algorithmic and analytic techniques used in convex problems fail to extend to non-convex problems. In fact, non-convex problems are NP-hard to solve and thus, no algorithmic technique should be expected to succeed on these problems in general.

However, the situation is not so bleak. As we discussed in \S~\ref{chap:intro}, several breakthroughs in non-convex optimization have shown that non-convex problems that possess nice additional structure can be solved not just in polynomial time, but rather efficiently too. Here, we will study the inner workings of projected gradient methods on such structured non-convex optimization problems.

The discussion will be divided into three parts. The first part will take a look at constraint sets that, despite being non-convex, possess additional structure so that projections onto them can be carried out efficiently. The second part will take a look at structural properties of objective functions that can aid optimization. The third part will present and analyze a simple extension of the PGD algorithm for non-convex problems. We will see that for problems that do possess nicely structured objective functions and constraint sets, the PGD-style algorithm does converge to the global optimum in polynomial time with a linear rate of convergence.

We would like to point out to the reader that our emphasis in this section will be on generality and exposition of basic concepts. We will seek to present easily accessible analyses for problems that have non-convex objectives. However, the price we will pay for this generality is in the fineness of the results we present. The results discussed in this section are not the best possible and more refined and problem-specific results will be discussed in subsequent sections where specific applications will be discussed in detail.

\section{Non-Convex Projections}
\label{sec:non-cvx-proj}
Executing the projected gradient descent algorithm with non-convex problems requires projections onto non-convex sets. Now, a quick look at the projection problem
\[
\Pi_\cC(\z) := \underset{\x\in\cC}{\argmin}\ \norm{\x-\z}_2
\]
reveals that this is an optimization problem in itself. Thus, when the set $\cC$ to be projected onto is non-convex, the projection problem can itself be NP-hard. However, for several well-structured sets, projection can be carried out efficiently despite the sets being non-convex.

\subsection{Projecting into Sparse Vectors}
In the sparse linear regression example discussed in \S~\ref{chap:intro},
\[
\bth = \underset{\norm{\bt}_0 \leq s}{\argmin}\ \sum_{i=1}^n\br{y_i - \x_i^\top\bt}^2,
\]
applying projected gradient descent requires projections onto the set of $s$-sparse vectors i.e., $\B_0(s) := \bc{\x\in\R^p, \norm{\x}_0 \leq s}$. The following result shows that the projection $\Pi_{\B_0(s)}(\z)$ can be carried out by simply sorting the coordinates of the vector $\z$ according to magnitude and setting all except the top-$s$ coordinates to zero.

\begin{lemma}
\label{lem:sparse-proj}
For any vector $\z \in \R^p$, let $\sigma$ be the permutation that sorts the coordinates of $\z$ in decreasing order of magnitude, i.e., $\abs{\z_{\sigma(1)}} \geq \abs{\z_{\sigma(2)}} \geq \ldots \geq \abs{\z_{\sigma(p)}}$. Then the vector $\hat\z := \Pi_{\B_0(s)}(\z)$ is obtained by setting $\hat\z_i = \z_i$ if $\sigma(i) \leq s$ and $\hat\z_i = 0$ otherwise.
\end{lemma}
\begin{proof}
We first notice that since the function $x \mapsto x^2$ is an increasing function on the positive half of the real line, we have $\underset{\x\in\cC}{\argmin}\ \norm{\x-\z}_2 = \underset{\x\in\cC}{\argmin}\ \norm{\x-\z}_2^2$. Next, we observe that the vector $\hat\z := \Pi_{\B_0(s)}$ must satisfy $\hat\z_i = \z_i$ for all $i \in \supp(\hat\z)$ otherwise we can decrease the objective value $\norm{\hat\z -\z}_2^2$ by ensuring this. Having established this gives us $\norm{\hat\z - \z}_2^2 = \sum_{i \notin \supp(\hat\z)} \z_i^2$. This is clearly minimized when $\supp(\hat\z)$ has the coordinates of $\z$ with largest magnitude.
\end{proof}

\subsection{Projecting into Low-rank Matrices}
In the recommendation systems problem, as discussed in \S~\ref{chap:intro} 
\[
\hat A_\text{lr} = \underset{\rank(X) \leq r}{\arg\min}\ \sum_{(i,j) \in \Omega}\br{X_{ij} - A_{ij}}^2,
\]
we need to project onto the set of low-rank matrices. Let us first define this problem formally. Consider matrices of a certain order, say $m \times n$ and let $\cC \subset \R^{m\times n}$ be an arbitrary set of matrices. Then, the projection operator $\Pi_\cC(\cdot)$ is defined as follows: for any matrix $A \in R^{m\times n}$,
\[
\Pi_\cC(A) := \underset{X\in\cC}{\argmin}\ \norm{A - X}_F,
\]
where $\norm{\cdot}_F$ is the Frobenius norm over matrices. For low rank projections we require $\cC$ to be the set of low rank matrices $\B_{\rank}(r) := \bc{A \in \R^{m\times n}, \rank(A) \leq r}$. Yet again, this projection can be done efficiently by performing a \emph{Singular Value Decomposition} on the matrix $A$ and retaining the top $r$ singular values and vectors. The Eckart-Young-Mirsky theorem proves that this indeed gives us the projection.

\begin{theorem}[Eckart-Young-Mirsky theorem]
\label{thm:eym-thm}
For any matrix $A \in \R^{m \times n}$, let $U\Sigma V^\top$ be the singular value decomposition of $A$ such that $\Sigma = \diag(\sigma_1,\sigma_2,\ldots,\sigma_{\min(m,n)})$ where $\sigma_1 \geq \sigma_2 \geq \ldots \geq \sigma_{\min(m,n)}$. Then for any $r \leq {\min(m,n)}$, the matrix $\hat A_{(r)} := \Pi_{\B_{\rank}(r)}(A)$ can be obtained as $U_{(r)}\Sigma_{(r)}V_{(r)}^\top$ where $U_{(r)} := \bs{U_1 U_2 \ldots U_r}$, $V{(r)} := \bs{V_1 V_2 \ldots V_r}$, and $\Sigma_{(r)} := \diag(\sigma_1,\sigma_2,\ldots,\sigma_r)$.
\end{theorem}

Although we have stated the above result for projections with the Frobenius norm defining the projections, the Eckart-Young-Mirsky theorem actually applies to any unitarily invariant norm including the Schatten norms and the operator norm. The proof of this result is beyond the scope of this monograph.

Before moving on, we caution the reader that the ability to efficiently project onto the non-convex sets mentioned above does not imply that non-convex projections are as nicely behaved as their convex counterparts. Indeed, none of the projections mentioned above satisfy projection properties I or II  (Lemmata~\ref{lem:proj-prop-1} and \ref{lem:proj-prop-2}). This will pose a significant challenge while analyzing PGD-style algorithms for non-convex problems since, as we would recall, these properties were crucially used in all convergence proofs discussed in \S~\ref{chap:tools}.

\section{Restricted Strong Convexity and Smoothness}
In \S~\ref{chap:tools}, we saw how optimization problems with convex constraint sets and objective functions that are convex and have bounded gradients, or else are strongly convex and smooth, can be effectively optimized using PGD, with much faster rates of convergence if the objective is strongly convex and smooth. However, when the constraint set fails to be convex, these results fail to apply.

There are several workarounds to this problem, the simplest being to convert the constraint set into a convex one, possibly by taking its \emph{convex hull}\footnote{The convex hull of any set $\cC$ is the ``smallest'' convex set $\bar\cC$ that contains $\cC$. Formally, we define $\bar\cC = \!\!\!\!\!\underset{\substack{S \supseteq \cC\\S \text{ is convex}}}\bigcap \!\!\!\!\!S$. If $\cC$ is convex then it is its own convex hull.}, which is what relaxation methods do. However, a much less drastic alternative exists that is widely popular in non-convex optimization literature.

The intuition is a simple one and generalizes much of the insights we gained from our discussion in \S~\ref{chap:tools}. The first thing we need to notice\elink{exer:pgd-rsc} is that the convergence results for the PGD algorithm in \S~\ref{chap:tools} actually do not require the objective function to be convex (or strongly convex/strongly smooth) over the entire $\R^p$. These properties are only required to be satisfied over the constraint set being considered. A natural generalization that emerges from this insight is the concept of \emph{restricted} properties that are discussed below.

\begin{definition}[Restricted Convexity]
\label{defn:res-cvx-fn}
A continuously differentiable function $f: \R^p \rightarrow \R$ is said to satisfy restricted convexity over a (possibly non-convex) region $\cC \subseteq \R^p$ if for every $\x,\y \in \cC$ we have $f(\y) \geq f(\x) + \ip{\nabla f(\x)}{\y - \x}$, where $\nabla f(\x)$ is the gradient of $f$ at $\x$.
\end{definition}

As before, a more general definition that extends to non-differentiable functions, uses the notion of subgradient to replace the gradient in the above expression.

\begin{definition}[Restricted Strong Convexity/Smoothness]
\label{defn:res-strong-cvx-smooth-fn}
A continuously differentiable function $f: \R^p \rightarrow \R$ is said to satisfy $\alpha$-restricted strong convexity (RSC) and $\beta$-restricted strong smoothness (RSS) over a (possibly non-convex) region $\cC \subseteq \R^p$ if for every $\x,\y \in \cC$, we have
\[
\frac{\alpha}{2}\norm{\x-\y}_2^2 \leq f(\y) - f(\x) - \ip{\nabla f(\x)}{\y - \x} \leq \frac{\beta}{2}\norm{\x-\y}_2^2.
\]
\end{definition}

Note that, as Figure~\ref{fig:rsc} demonstrates, even non-convex functions can demonstrate the RSC/RSS properties over suitable subsets. Conversely, functions that satisfy RSC/RSS need not be convex. It turns out that in several practical situations, such as those explored by later sections, the objective functions in the non-convex optimization problems do satisfy the RSC/RSS properties described above, in some form.

\begin{figure}
\includegraphics[width=\columnwidth]{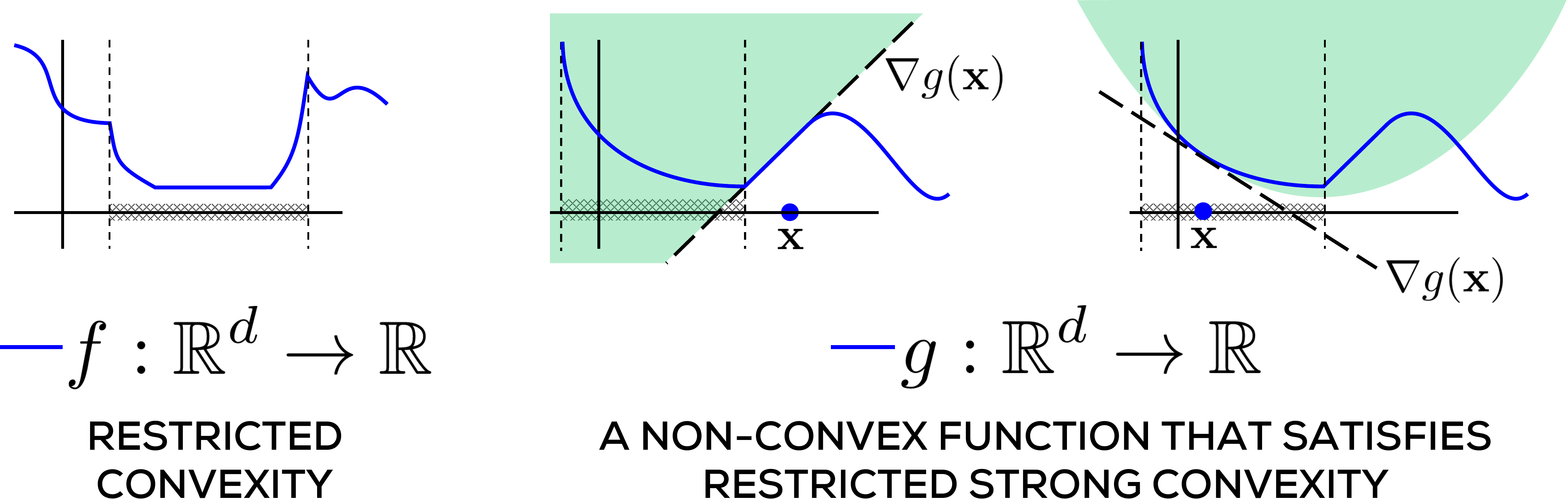}
\caption[Restricted Strong Convexity and Strong Smoothness]{A depiction of restricted convexity properties. $f$ is clearly non-convex over the entire real line but is convex within the cross-hatched region bounded by the dotted vertical lines. $g$ is a non-convex function that satisfies restricted strong convexity. Outside the cross-hatched region (again bounded by the dotted vertical lines), $g$ fails to even be convex as its curve falls below its tangent, but within the region, it actually exhibits strong convexity.}%
\label{fig:rsc}
\end{figure}

We also remind the reader that the RSC/RSS definitions presented here are quite generic and presented to better illustrate basic concepts. Indeed, for specific non-convex problems such as sparse recovery, low-rank matrix recovery, and robust regression, the later sections will develop more refined versions of these properties that are better tailored to those problems. In particular, for sparse recovery problems, the RSC/RSS properties can be shown to be related\elink{exer:spreg-rsc-is-rsc} to the well-known restricted isometry property (RIP).

\section{Generalized Projected Gradient Descent}
We now present the generalized projected gradient descent algorithm (gPGD) for non-convex optimization problems. The procedure is outlined in Algorithm~\ref{algo:gpgd}. The reader would find it remarkably similar to the PGD procedure in Algorithm~\ref{algo:pgd}. However, a crucial difference is in the projections made. Whereas PGD utilized convex projections, the gPGD procedure, if invoked with a non-convex constraint set $\cC$, utilizes non-convex projections instead.

We will perform the convergence analysis for the gPGD algorithm assuming that the projection step in the algorithm is carried out exactly. As we saw in the preceding discussion, this can be accomplished efficiently for non-convex sets arising in several interesting problem settings. However, despite this, the convergence analysis will remain challenging due to the non-convexity of the problem.

Firstly, we will not be able to assume that the objective function we are working with is convex over the entire $\R^p$. Secondly, non-convex projections do not satisfy projection properties I or II. Finally, the first order optimality condition (\cite[Proposition 1.3]{Bubeck2015}) we used to prove Theorem~\ref{thm:pgd-sc-ss-proof} also fails to hold for non-convex constraint sets. Since the analyses for the PGD algorithm crucially used these results, we will have to find workarounds to all of them. We will denote the optimal function value as $f^\ast = \min_{\x \in \cC}f(\x)$ and any optimizer as $\x^\ast \in \cC$ such that $f(\x^\ast) = f^\ast$.

To simplify the presentation we will assume that $\nabla f(\x^\ast) = \vzero$. This assumption is satisfied whenever the objective function is differentiable and the optimal point $\x^\ast$ lies in the interior of the constraint set $\cC$. However, many sets such as $\cB_0(s)$ do not possess an interior (although they may still possess a \emph{relative} interior) and this assumption fails by default on such sets. Nevertheless, this assumption will greatly simplify the presentation as well as help us focus on the key issues. Moreover, convergence results can be shown without making this assumption too.

\begin{algorithm}[t]
	\caption{Generalized Projected Gradient Descent (gPGD)}
	\label{algo:gpgd}
	\begin{algorithmic}[1]
		{
			\REQUIRE Objective function $f$, constraint set $\cC$, step length $\eta$
			\ENSURE A point $\hat\x \in \cC$ with near-optimal objective value
			\STATE $\x^1 \leftarrow \vzero$
			\FOR{$t = 1, 2, \ldots, T$}
				\STATE $\z^{t+1} \leftarrow \x^t - \eta\cdot\nabla f(\x^t)$
				\STATE $\x^{t+1} \leftarrow \Pi_\cC(\z^{t+1})$
			\ENDFOR
			\STATE \textbf{return} {$\hat\x_{\text{final}} = \x^T$}
		}
	\end{algorithmic}
\end{algorithm}

Theorem~\ref{thm:gpgd-rsc-rss-proof} gives the convergence proof for gPGD. The reader will notice that the convergence rate offered by gPGD is similar to the one offered by the PGD algorithm for convex optimization (see Theorem~\ref{thm:pgd-sc-ss-proof}). However, the gPGD algorithm requires a more careful analysis of the structure of the objective function and constraint set since it is working with a non-convex optimization problem.

\begin{theorem}
\label{thm:gpgd-rsc-rss-proof}
Let $f$ be a (possibly non-convex) function satisfying the $\alpha$-RSC and $\beta$-RSS properties over a (possibly non-convex) constraint set $\cC$ with $\beta/\alpha < 2$. Let Algorithm~\ref{algo:gpgd} be executed with a step length $\eta = \frac{1}{\beta}$. Then after at most $T = \bigO{\frac{\alpha}{2\alpha - \beta}\log\frac{1}{\epsilon}}$ steps, $f(\x^T) \leq f(\x^\ast) + \epsilon$.
\end{theorem}

This result holds even when the step length is set to values that are large enough but yet smaller than $1/\beta$. However, setting $\eta = \frac{1}{\beta}$ simplifies the proof and allows us to focus on the key concepts.

\begin{proof}[Proof (of Theorem~\ref{thm:gpgd-rsc-rss-proof}).]
Recall that the proof of Theorem~\ref{thm:pgd-conv-proof} used the SC/SS properties for the analysis. We will replace these by the RSC/RSS properties -- we will use RSC to track the global convergence of the algorithm and RSS to locally assess the progress made by the algorithm in each iteration. We will use $\Phi_t = f(\x^{t+1}) - f(\x^\ast)$ as the potential function.\\

\noindent\textbf{(Apply Restricted Strong Smoothness)} Since both $\x^t,\x^{t+1} \in \cC$ due to the projection steps, we apply the $\beta$-RSS property to them.
\begin{align*}
f(\x^{t+1}) - f(\x^t) &\leq \ip{\nabla f(\x^t)}{\x^{t+1}-\x^t} + \frac{\beta}{2}\norm{\x^t-\x^{t+1}}_2^2\\
											&= \frac{1}{\eta}\ip{\x^t - \z^{t+1}}{\x^{t+1}-\x^t} + \frac{\beta}{2}\norm{\x^t-\x^{t+1}}_2^2\\
											&= \frac{\beta}{2}\br{\norm{\x^{t+1} - \z^{t+1}}_2^2 - \norm{\x^t - \z^{t+1}}_2^2}
\end{align*}
Notice that this step crucially uses the fact that $\eta = 1/\beta$.\\

\noindent\textbf{(Apply Projection Property)} We are again stuck with the unwieldy $\z^{t+1}$ term. However, unlike before, we cannot apply projection properties I or II as non-convex projections do not satisfy  them. Instead, we resort to Projection Property-O (Lemma~\ref{lem:proj-prop-o}), that all projections (even non-convex ones) must satisfy. Applying this property gives us
\begin{align*}
f(\x^{t+1}) - f(\x^t) &\leq \frac{\beta}{2}\br{\norm{\x^\ast - \z^{t+1}}_2^2 - \norm{\x^t - \z^{t+1}}_2^2}\\
											&= \frac{\beta}{2}\br{\norm{\x^\ast - \x^t}_2^2 + 2\ip{\x^\ast - \x^t}{\x^t - \z^{t+1}}}\\
											&= \frac{\beta}{2}\norm{\x^\ast - \x^t}_2^2 + \ip{\x^\ast - \x^t}{\nabla f(\x^t)}
\end{align*}

\noindent\textbf{(Apply Restricted Strong Convexity)} Since both $\x^t,\x^\ast \in \cC$, we apply the $\alpha$-RSC property to them. However, we do so in two ways:
\begin{align*}
f(\x^\ast) - f(\x^t) &\geq \ip{\nabla f(\x^t)}{\x^\ast - \x^t} + \frac{\alpha}{2}\norm{\x^t - \x^\ast}_2^2\\
f(\x^t) - f(\x^\ast) &\geq \ip{\nabla f(\x^\ast)}{\x^t - \x^\ast} + \frac{\alpha}{2}\norm{\x^t - \x^\ast}_2^2 \geq \frac{\alpha}{2}\norm{\x^t - \x^\ast}_2^2,
\end{align*}
where in the second line we used the fact that we assumed $\nabla f(\x^\ast) = \vzero$. We recall that this assumption can be done away with but makes the proof more complicated which we wish to avoid. Simple manipulations with the two equations give us
\[
\ip{\nabla f(\x^t)}{\x^\ast - \x^t} + \frac{\beta}{2}\norm{\x^\ast - \x^t}_2^2 \leq \br{2 - \frac{\beta}{\alpha}}\br{f(\x^\ast) - f(\x^t)}
\]
Putting this in the earlier expression gives us
\[
f(\x^{t+1}) - f(\x^t) \leq \br{2 - \frac{\beta}{\alpha}}\br{f(\x^\ast) - f(\x^t)}
\]
The above inequality is quite interesting. It tells us that the larger the gap between $f(\x^\ast)$ and $f(\x^t)$, the larger will be the drop in objective value in going from $\x^t$ to $\x^{t+1}$. The form of the result is also quite fortunate as it assures us that we will cover a constant fraction $\br{2-\frac{\beta}{\alpha}}$ of the remaining ``distance'' to $\x^\ast$ at each step! Rearranging this gives
\[
\Phi_{t+1} \leq (\kappa-1)\Phi_t,
\]
where $\kappa = \beta/\alpha$. Note that we always have $\kappa \geq 1$\elink{exer:pgd-kappa} and by assumption $\kappa = \beta/\alpha < 2$, so that we always have $\kappa - 1 \in [0,1)$. This proves the result after simple manipulations.
\end{proof}
We see that the condition number has yet again played in crucial role in deciding the convergence rate of the algorithm, this time for a non-convex problem. However, we see that the condition number is defined differently here, using the RSC/RSS constants instead of the SC/SS constants as we did in \S~\ref{chap:tools}.

The reader would notice that while there was no restriction on the condition number $\kappa$ in the analysis of the PGD algorithm (see Theorem~\ref{thm:pgd-sc-ss-proof}), the analysis of the gPGD algorithm does require $\kappa < 2$. It turns out that this restriction can be done away with for specific problems. However, the analysis becomes significantly more complicated. Resolving this issue in general is beyond the scope of this monograph but we will revisit this question in \S~\ref{chap:spreg} when we study sparse recovery in ill-conditioned settings. i.e., with large condition numbers.

In subsequent sections, we will see more refined versions of the gPGD algorithm for different non-convex optimization problems, as well as more refined and problem-specific analyses. In all cases we will see that the RSC/RSS assumptions made by us can be fulfilled in practice and that gPGD-style algorithms offer very good performance on practical machine learning and signal processing problems.

\section{Exercises}
\begin{exer}
\label{exer:pgd-rsc}
Verify that the basic convergence result for the PGD algorithm in Theorem~\ref{thm:pgd-conv-proof}, continues to hold when the constraint set $\cC$ is convex and $f$ only satisfies \emph{restricted convexity} over $\cC$ (i.e., $f$ is not convex over the entire $\R^p$). Verify that the result for strongly convex and smooth functions in Theorem~\ref{thm:pgd-sc-ss-proof}, also continues to hold if $f$ satisfies RSC and RSS over a convex constraint set $\cC$.
\end{exer}
\begin{exer}
\label{exer:pgd-kappa}
Let the function $f$ satisfy the $\alpha$-RSC and $\beta$-RSS properties over a set $\cC$. Show that the condition number $\kappa = \frac{\beta}{\alpha} \geq 1$. Note that the function $f$ and the set $\cC$ may both be non-convex.
\end{exer}
\begin{exer}
\label{exer:pgd-low-rank}
Recall the recommendation systems problem we discussed in \S~\ref{chap:intro}. Show that assuming the ratings matrix to be rank-$r$ is equivalent to assuming that with every user $i \in [m]$ there is associated a vector $\vu_i \in \bR^r$ describing that user, and with every item $j \in [n]$ there is associated a vector $\vv_i \in \bR^r$ describing that item such that the rating given by user $i$ to item $j$ is $A_{ij} = \vu_i^\top\vv_j$.\\
\textit{Hint}: Use the singular value decomposition for $A$.
\end{exer}
\chapter{Alternating Minimization}
\label{chap:altmin}

In this section we will introduce a widely used non-convex optimization primitive, namely the alternating minimization principle. The technique is extremely general and its popular use actually predates the recent advances in non-convex optimization by several decades. Indeed, the popular Lloyd's algorithm \citep{Lloyd1982} for k-means clustering and the EM algorithm \citep{DempsterLR1977} for latent variable models are problem-specific variants of the general alternating minimization principle. The technique continues to inspire new algorithms for several important non-convex optimization problems such as matrix completion, robust learning, phase retrieval and dictionary learning.

Given the popularity and breadth of use of this method, our task to present an introductory treatment will be even more challenging here. To keep the discussion focused on core principles and tools, we will refrain from presenting the alternating minimization principle in all its varieties. Instead, we will focus on showing, in a largely problem-independent manner, what are the challenges that face alternating minimization when applied to real-life problems, and how they can be overcome. Subsequent sections will then show how this principle can be applied to various machine learning and signal processing tasks. In particular, \S~\ref{chap:em} will be devoted to the EM algorithm which embodies the alternating minimization principle and is extremely popular for latent variable estimation problems in statistics and machine learning.

The discussion will be divided into four parts. In the first part, we will look at some useful structural properties of functions that frequently arise in alternating minimization settings. In the second part, we will present a general implementation of the alternating minimization principle and discuss some challenges faced by this algorithm in offering convergent behavior in real-life problems. In the third part, as a warm-up exercise, we will show how these challenges can be overcome when the optimization problem being solved is convex. Finally in the fourth part, we will discuss the more interesting problem of convergence of alternating minimization for non-convex problems.

\section{Marginal Convexity and Other Properties}
Alternating Minimization is most often utilized in settings where the optimization problem concerns two or more (groups of) variables. For example, recall the matrix completion problem in recommendation systems from \S~\ref{chap:intro} which involved two variables $U$ and $V$ denoting respectively, the latent factors for the users and the items. In several such cases, the optimization problem, more specifically the objective function, is not \emph{jointly convex} in all the variables.

\clearpage

\begin{definition}[Joint Convexity]
\label{defn:joint-cvx-fn}
A continuously differentiable function in two variables $f: \R^p \times \R^q \rightarrow \R$ is considered jointly convex if for every $(\x^1,\y^1), (\x^2,\y^2) \in \R^p \times \R^q$ we have
\[
f(\x^2,\y^2) \geq f(\x^1,\y^1) + \ip{\nabla f(\x^1,\y^1)}{(\x^2,\y^2) - (\x^1,\y^1)},
\]
where $\nabla f(\x^1,\y^1)$ is the gradient of $f$ at the point $(\x^1,\y^1)$.
\end{definition}

The definition of joint convexity is not different from the one for convexity Definition~\ref{defn:cvx-fn}. Indeed the two coincide if we assume $f$ to be a function of a single variable $\z = (\x,\y) \in \R^{p+q}$ instead of two variables. However, not all multivariate functions that arise in applications are jointly convex. This motivates the notion of \emph{marginal convexity}.

\begin{definition}[Marginal Convexity]
\label{defn:marg-cvx-fn}
A continuously differentiable function of two variables $f: \R^p \times \R^q\rightarrow \R$ is considered marginally convex in its first variable if for every value of $\y \in \R^q$, the function $f(\cdot,\y): \R^p \rightarrow \R$ is convex, i.e., for every $\x^1,\x^2 \in \R^p$, we have
\[
f(\x^2,\y) \geq f(\x^1,\y) + \ip{\nabla_\x f(\x^1,\y)}{\x^2 - \x^1},
\]
where $\nabla_\x f(\x^1,\y)$ is the partial gradient of $f$ with respect to its first variable at the point $(\x^1,\y)$. A similar condition is imposed for $f$ to be considered marginally convex in its second variable.
\end{definition}
Although the definition above has been given for a function of two variables, it clearly extends to functions with an arbitrary number of variables. It is interesting to note that whereas the objective function in the matrix completion problem mentioned earlier is not jointly convex in its variables, it is indeed marginally convex in both its variables\elink{exer:altmin-marg-conv}.

\begin{figure}
\includegraphics[width=\columnwidth]{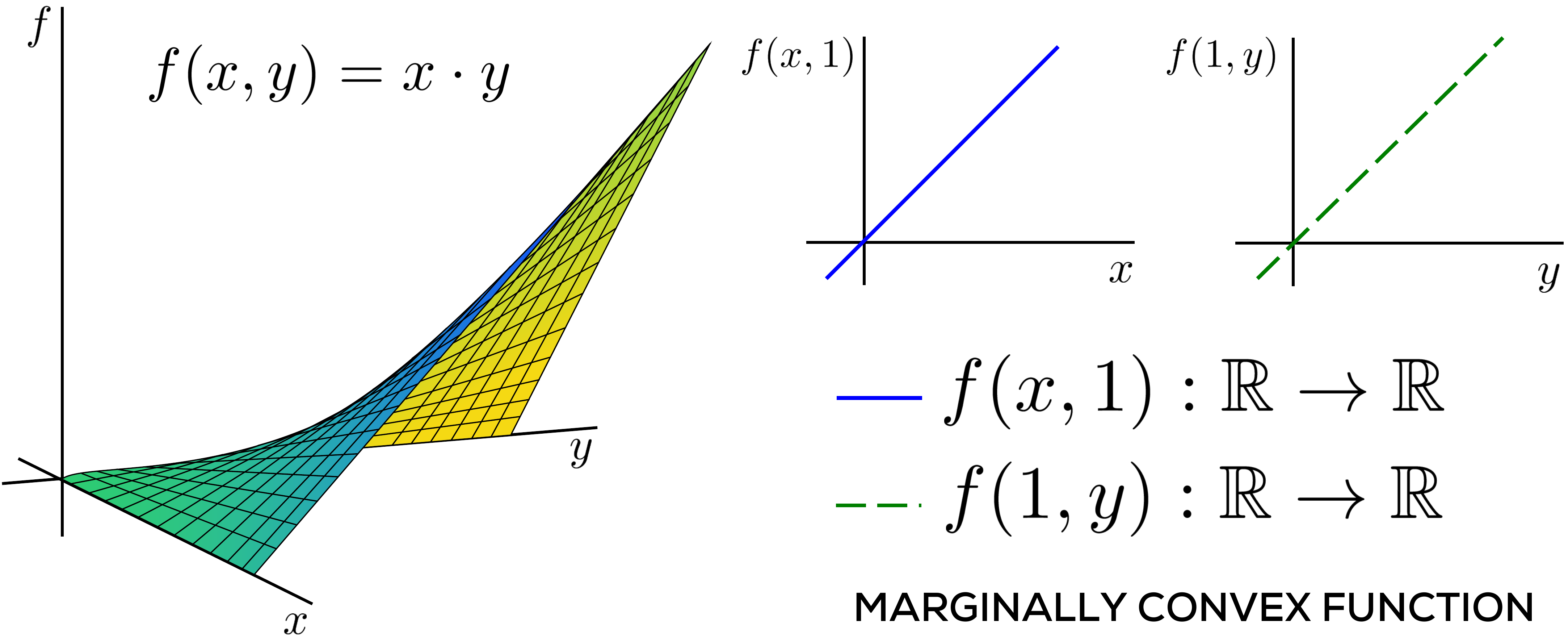}
\caption[Marginal Convexity]{A marginally convex function is not necessarily (jointly) convex. The function $f(x,y) = x\cdot y$ is marginally linear, hence marginally convex, in both its variables, but clearly not a (jointly) convex function.}%
\label{fig:marg}
\end{figure}

It is also useful to note that even though a function that is marginally convex in all its variables need not be a jointly convex function (see Figure~\ref{fig:marg}), the converse is true\elink{exer:altmin-joint-marg}. We will find the following notions of \emph{marginal strong convexity and smoothness} to be especially useful in our subsequent discussions.

\begin{definition}[Marginally Strongly Convex/Smooth Function]
\label{defn:marg-strong-cvx-smooth-fn}
A continuously differentiable function $f: \R^p \times \R^q \rightarrow \R$ is considered (uniformly) $\alpha$-marginally strongly convex (MSC) and (uniformly) $\beta$-marginally strongly smooth (MSS) in its first variable if for every value of $\y \in \R^q$, the function $f(\cdot,\y): \R^p \rightarrow \R$ is $\alpha$-strongly convex and $\beta$-strongly smooth, i.e., for every $\x^1,\x^2 \in \R^p$, we have
\[
\frac{\alpha}{2}\norm{\x^2 - \x^1}_2^2 \leq f(\x^2,\y) - f(\x^1,\y) - \ip{\vg}{\x^2 - \x^1} \leq \frac{\beta}{2}\norm{\x^2 - \x^1}_2^2,
\]
where $\vg = \nabla_\x f(\x^1,\y)$ is the partial gradient of $f$ with respect to its first variable at the point $(\x^1,\y)$. A similar condition is imposed for $f$ to be considered (uniformly) MSC/MSS in its second variable.
\end{definition}

The above notion is a ``uniform'' one since the parameters $\alpha, \beta$ do not depend on the $\y$ coordinate. It is instructive to relate MSC/MSS to the RSC/RSS properties from \S~\ref{chap:tools}. MSC/MSS extend the idea of functions that are not ``globally'' convex (strongly or otherwise) but do exhibit such properties under ``qualifications''. MSC/MSS use a different qualification than RSC/RSS did. Note that a function that is MSC with respect to \emph{all} its variables, need not be a convex function\elink{exer:altmin-msc-sc}.

\begin{algorithm}[t]
	\caption{Generalized Alternating Minimization (gAM)}
	\label{algo:gam}
	\begin{algorithmic}[1]
		{
			\REQUIRE Objective function $f: \cX \times \cY \rightarrow \R$
			\ENSURE A point $(\hat\x,\hat\y) \in \cX \times \cY$ with near-optimal objective value
			\STATE $(\x^1,\y^1) \leftarrow \text{\textsf{INITALIZE}}()$
			\FOR{$t = 1, 2, \ldots, T$}
				\STATE $\x^{t+1} \leftarrow \arg\min_{\x \in \cX} f(\x,\y^t)$
				\STATE $\y^{t+1} \leftarrow \arg\min_{\y \in \cY} f(\x^{t+1},\y)$
			\ENDFOR
			\STATE \textbf{return} {$(\x^T,\y^T)$}
		}
	\end{algorithmic}
\end{algorithm}

\section{Generalized Alternating Minimization}
The alternating minimization algorithm (gAM) is outlined in Algorithm~\ref{algo:gam} for an optimization problem on two variables constrained to the sets $\cX$ and $\cY$ respectively. The procedure can be easily extended to functions with more variables, or have more complicated constraint sets\elink{exer:altmin-gen-am} of the form $\cZ \subset \cX \times \cY$. After an initialization step, gAM alternately fixes one of the variables and optimizes over the other.

This approach of solving several intermediate \emph{marginal} optimization problems instead of a single big problem is the key to the practical success of gAM. Alternating minimization is mostly used when these marginal problems are easy to solve. Later, we will see that there exist simple, often closed form solutions to these marginal problems for applications such as matrix completion, robust learning etc. 

There also exist ``descent'' versions of gAM which do not completely perform marginal optimizations but take gradient steps along the variables instead. These are described below
\begin{align*}
\x^{t+1} &\leftarrow \x^t - \eta_{t,1}\cdot\nabla_\x f(\x^t,\y^t)\\
\y^{t+1} &\leftarrow \y^t - \eta_{t,2}\cdot\nabla_\y f(\x^{t+1},\y^t)
\end{align*}
These descent versions are often easier to execute but may also converge more slowly. If the problem is nicely structured, then progress made on the intermediate problems offers fast convergence to the optimum. However, from the point of view of convergence, gAM faces several challenges. To discuss those, we first introduce some more concepts.

\begin{definition}[Marginally Optimum Coordinate]
\label{defn:marg-opt}
Let $f$ be a function of two variables constrained to be in the sets $\cX,\cY$ respectively. For any point $\y \in \cY$, we say that $\tilde\x$ is a marginally optimal coordinate with respect to $\y$, and use the shorthand $\tilde\x \in \mopt(\y)$, if $f(\tilde\x,\y) \leq f(\x,\y)$ for all $\x \in \cX$. Similarly for any $\x \in \cX$, we say $\tilde\y \in \mopt(\x)$ if $\tilde\y$ is a marginally optimal coordinate with respect to $\x$.
\end{definition}

\begin{definition}[Bistable Point]
\label{defn:bistable}
Given a function $f$ over two variables constrained within the sets $\cX,\cY$ respectively, a point $(\x,\y) \in \cX\times\cY$ is considered a \emph{bistable} point if $\y \in \mopt(\x)$ and $\x \in \mopt(\y)$ i.e., both coordinates are marginally optimal with respect to each other.
\end{definition}

It is easy to see\elink{exer:altmin-opt-bs} that the optimum of the optimization problem must be a bistable point. The reader can also verify that the gAM procedure must stop after it has reached a bistable point. However, two questions arise out of this. First, how fast does gAM approach a bistable point and second, even if it reaches a bistable point, is that point guaranteed to be (globally) optimal?

The first question will be explored in detail later. It is interesting to note that the gAM procedure has no parameters, such as step length. This can be interpreted as a benefit as well as a drawback. While it relieves the end-user from spending time tweaking parameters, it also means that the user has less control over the progress of the algorithm. Consequently, the convergence of the gAM procedure is totally dependent on structural properties of the optimization problem. In practice, it is common to switch between gAM updates as given in Algorithm~\ref{algo:gam} and descent versions thereof discussed earlier. The descent versions do give a step length as a tunable parameter to the user.

\begin{figure}
\includegraphics[width=\columnwidth]{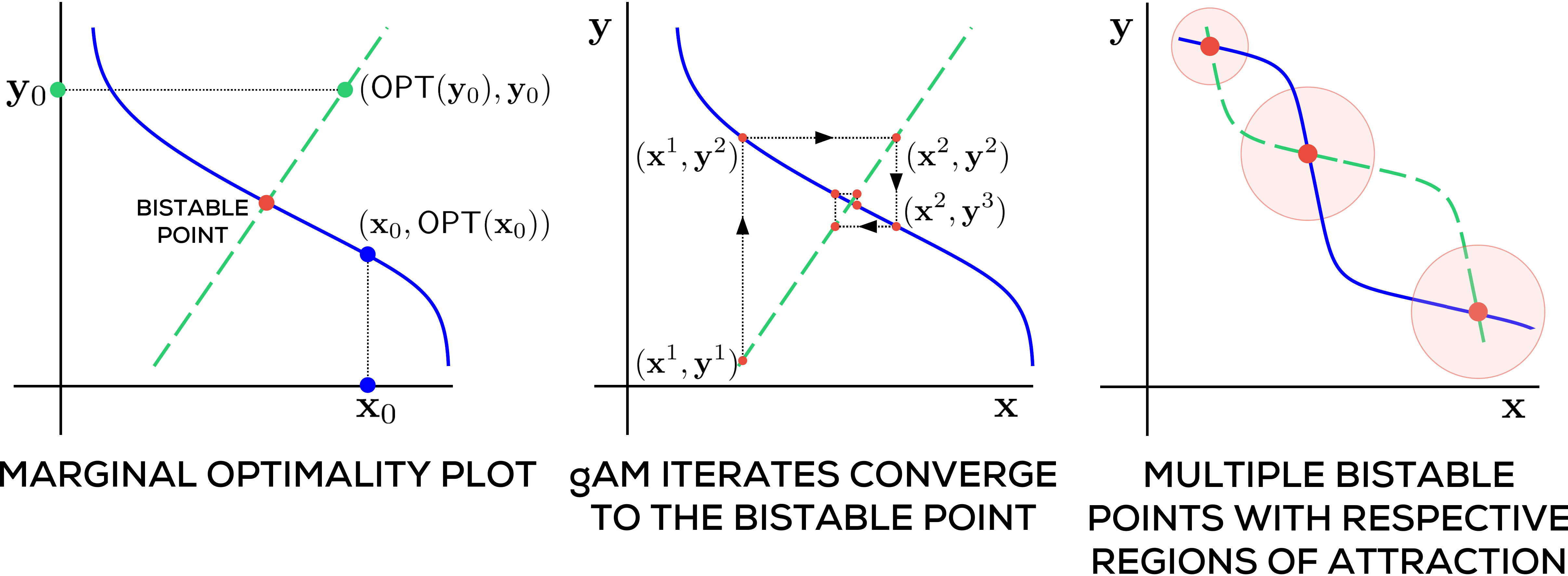}
\caption[Bistable Points and Convergence of gAM]{The first plot depicts the marginally optimal coordinate curves for the two variables whose intersection produces bistable points. gAM is adept at converging to bistable points for well-behaved functions. Note that gAM progresses only along a single variable at a time. Thus, in the 2-D example in the second plot, the progress lines are only vertical or horizontal. A function may have multiple bistable points, each with its own region of attraction, depicted as shaded circles.}
\label{fig:bistable}
\end{figure}

The second question requires a closer look at the interaction between the objective function and the gAM process. Figure~\ref{fig:bistable} illustrates this with toy bi-variate functions over $\R^2$. In the first figure, the bold solid curve plots the function $g: \cX \rightarrow \cX \times \cY$ with $g(\x) = (\x,\mopt(\x))$ (in this toy case, the marginally optimal coordinates are taken to be unique for simplicity, i.e., $\abs{\mopt(\x)} = 1$ for all $\x \in \cX$). The bold dashed curve similarly plots $h: \cY \rightarrow \cX \times \cY$ with $h(\y) = (\mopt(\y),\y)$. These plots are quite handy in demonstrating the convergence properties of the gAM algorithm.

It is easy to see that bistable points lie precisely at the intersection of the bold solid and the bold dashed curves. The second illustration shows how the gAM process may behave when instantiated with this toy function -- clearly gAM exhibits rapid convergence to the bistable point. However, the third illustration shows that functions may have multiple bistable points. The figure shows that this may happen even if the marginally optimal coordinates are unique i.e., for every $\x$ there is a unique $\tilde\y$ such that $\tilde\y = \mopt(\x)$ and vice versa.

In case a function taking bounded values possesses multiple bistable points, the bistable point to which gAM eventually converges depends on where the procedure was initialized. This is exemplified in the third illustration where each bistable region has its own ``region of attraction''. If initialized inside a particular region, gAM converges to the bistable point corresponding to that region. This means that in order to converge to the globally optimal point, gAM must be initialized inside the region of attraction of the global optimum.

The above discussion shows that it may be crucial to properly initialize the gAM procedure to ensure convergence to the global optimum. Indeed, when discussing gAM-style algorithms for learning latent variable models, matrix completion and phase retrieval in later sections, we will pay special attention to initialize the procedure ``close'' to the optimum. The only exception will be that of robust regression in \S~\ref{chap:rreg} where it seems that the problem structure ensures a unique bistable point and so, a careful initialization is not required.

\section{A Convergence Guarantee for gAM for Convex Problems}
As usual, things do become nice when the optimization problems are convex. For instance, for differentiable convex functions, all bistable points are global minima\elink{exer:altmin-conv-gam} and thus, converging to any one of them is sufficient. In fact, approaches similar to gAM, commonly known as \emph{Coordinate Minimization} (CM), are extremely popular for large scale convex optimization problems. The CM procedure simply treats a single $p$-dimensional variable $\x \in \R^p$ as $p$ one-dimensional variables $\bc{\x_1,\ldots,\x_p}$ and executes gAM-like steps with them, resulting in the intermediate problems being uni-dimensional. However, it is worth noting that gAM can struggle with non-differentiable objectives. 

In the following, we will analyze the convergence of the gAM algorithm for the case when the objective function is jointly convex in both its variables. We will then see what happens to Algorithm~\ref{algo:gam} if the function $f$ is non-convex. To keep the discussion focused, we will consider an unconstrained optimization problem.

\begin{theorem}
\label{thm:gam-smooth-cvx-proof}
Let $f: \R^p \times \R^q \rightarrow \R$ be jointly convex, continuously differentiable, satisfy $\beta$-MSS in both its variables, and $f^\ast = \min_{\x,\y}f(\x,\y) > -\infty$. Let the region $S_0 = \bc{\x,\y: f(\x,\y) \leq f(\vzero,\vzero)} \subset \R^{p+q}$ be bounded, i.e., satisfy $S_0 \subseteq \cB_2((\vzero,\vzero),R)$ for some $R > 0$. Let Algorithm~\ref{algo:gam} be executed with the initialization $(\x^1,\y^1) = (\vzero,\vzero)$. Then after at most $T = \bigO{\frac{1}{\epsilon}}$ steps, we have $f(\x^T,\y^T) \leq f^\ast + \epsilon$.
\end{theorem}
\begin{proof}
The first property of the gAM algorithm that we need to appreciate is \emph{monotonicity}. It is easy to see that due to the marginal minimizations carried out, we have at all time steps $t$,
\[
f(\x^{t+1},\y^{t+1}) \leq f(\x^{t+1},\y^t) \leq f(\x^t,\y^t)
\]
The region $S_0$ is the \emph{sublevel set} of $f$ at the initialization point. Due to the monotonicity property, we have $f(\x^t,\y^t) \leq f(\x^1,\y^1)$ for all $t$ i.e., $(\x^t,\y^t) \in S_0$ for all $t$. Thus, gAM remains restricted to the bounded region $S_0$ and does not diverge. We notice that this point underlies the importance of proper initialization: gAM benefits from being initialized at a point at which the sublevel set of $f$ is bounded.

We will use $\Phi_t = \frac{1}{f(\x^t,\y^t)-f^\ast}$ as the potential function. This is a slightly unusual choice of potential function but its utility will be clear from the proof. Note that $\Phi_t > 0$ for all $t$ and that convergence is equivalent to showing $\Phi_t \rightarrow \infty$. We will, as before, use smoothness to analyze the per-iteration progress made by gAM and use convexity for global convergence analysis. For any time step $t \geq 2$, consider the hypothetical update we could have made had we done a gradient step instead of the marginal minimization step gAM does in step 3.
\[
\tilde\x^{t+1} = \x^t - \frac{1}{\beta}\nabla_\x f(\x^t,\y^t)
\]
\textbf{(Apply Marginal Strong Smoothness)} We get
\begin{align*}
f(\tilde\x^{t+1},\y^t) &\leq f(\x^t,\y^t) + \ip{\nabla_\x f(\x^t,\y^t),\tilde\x^{t+1} - \x^t} + \frac{\beta}{2}\norm{\tilde\x^{t+1} - \x^t}_2^2\\
&= f(\x^t,\y^t) - \frac{1}{2\beta}\norm{\nabla_\x f(\x^t,\y^t)}_2^2
\end{align*}
\textbf{(Apply Monotonicity of gAM)} Since $\x^{t+1} \in \mopt(\y^t)$, we must have $f(\x^{t+1},\y^t) \leq f(\tilde\x^{t+1},\y^t)$, which gives us
\[
f(\x^{t+1},\y^{t+1}) \leq f(\x^{t+1},\y^t) \leq f(\x^t,\y^t) - \frac{1}{2\beta}\norm{\nabla_\x f(\x^t,\y^t)}_2^2
\]
Now since $t \geq 2$, we must have had $\y^t \in \argmin_\y f(\x^t,\y)$. Since $f$ is differentiable, we must have (\cite[see][Proposition 1.2]{Bubeck2015}) $\nabla_\y f(\x^t,\y^t) = \vzero$. Applying the Pythagoras' theorem now gives us as a result, $\norm{\nabla f(\x^t,\y^t)}_2^2 = \norm{\nabla_\x f(\x^t,\y^t)}_2^2$.\\

\noindent\textbf{(Apply Convexity)} Since $f$ is jointly convex, we can state
\[
f(\x^t,\y^t) - f^\ast \leq \ip{\nabla f(\x^t,\y^t)}{(\x^t,\y^t)-(\x^\ast,\y^\ast)} \leq 2R\norm{\nabla f(\x^t,\y^t)}_2,
\]
where we have used the Cauchy-Schwartz inequality and the fact that $(\x^t,\y^t),(\x^\ast,\y^\ast) \in S_0$. Putting these together gives us
\[
f(\x^{t+1},\y^{t+1}) \leq f(\x^t,\y^t) - \frac{1}{4\beta R^2}\br{f(\x^t,\y^t) - f^\ast}^2,
\]
or in other words,
\[
\frac{1}{\Phi_{t+1}} \leq \frac{1}{\Phi_t} - \frac{1}{4\beta R^2}\frac{1}{\Phi_t^2} \leq \frac{1}{\Phi_t} - \frac{1}{4\beta R^2}\frac{1}{\Phi_t\Phi_{t+1}},
\]
where the second step follows from monotonicity. Rearranging gives us
\[
\Phi_{t+1} - \Phi_t \geq \frac{1}{4\beta R^2},
\]
which upon telescoping, and using $\Phi_2 \geq 0$ gives us
\[
\Phi_T \geq \frac{T}{4\beta R^2},
\]
which proves the result. Note that the result holds even if $f$ is jointly convex and satisfies the MSS property only \emph{locally} in the region $S_0$.
\end{proof}

\section{A Convergence Guarantee for gAM under MSC/MSS}
We will now investigate what happens when the gAM algorithm is executed with a non-convex function. Note that the previous result crucially uses convexity and will not extend here. Moreover, for non-convex functions, there is no assurance that all bistable points are global minima. Instead we will have to fend for fast convergence to a bistable point, as well as show that the function is globally minimized there.

Doing the above will require additional structure on the objective function. In the following, we will denote $f^\ast = \min_{\x,\y}f(\x,\y)$ to be the optimum value of the objective function. We will fix $(\x^\ast,\y^\ast)$ to be any point such that $f(\x^\ast,\y^\ast) = f^\ast$ (there may be several). We will also let $\cZ^\ast \subset \R^p \times \R^q$ denote the set of all bistable points for $f$.

First of all, notice that if a continuously differentiable function $f: \R^p \times \R^q \rightarrow \R$ is marginally convex (strongly or otherwise) in both its variables, then its bistable points are exactly its stationary points.

\begin{lemma}
\label{lemma:bistable-stat}
A point $(\x,\y)$ is bistable with respect to a continuously differentiable function $f: \R^p \times \R^q$ that is marginally convex in both its variables iff $\nabla f(\x,\y) = \vzero$.
\end{lemma}
\begin{proof}
It is easy to see that partial derivatives must vanish at a bistable point since the function is differentiable (\cite[see][Proposition 1.2]{Bubeck2015}) and thus we get $\nabla f(\x,\y) = [\nabla_\x f(\x,\y),\nabla_\y f(\x,\y)] = \vzero$. Arguing the other way round, if the gradient, and by extension the partial derivatives, vanish at $(\x,\y)$, then by marginal convexity, for any $\x'$
\[
f(\x',\y) - f(\x,\y) \geq \ip{\nabla_\x f(\x,\y)}{\x'-\x} = 0
\]
Similarly, $f(\x,\y') \geq f(\x,\y)$ for any $\y'$. Thus $(\x,\y)$ is bistable.
\end{proof}

The above tells us that $\cZ^\ast$ is also the set of all stationary points of $f$. However, not all points in $\cZ^\ast$ may be global minima. Addressing this problem requires careful initialization and problem-specific analysis, that we will carry out for problems such as matrix completion etc in later sections. For now, we introduce a generic \emph{robust} bistability property that will be very useful in the analysis. Similar properties are frequently used in the analysis of gAM-style algorithms.

\begin{definition}[Robust Bistability Property]
\label{defn:rob-bistable}
A function $f: \R^p \times \R^q \rightarrow \R$ satisfies the $C$-robust bistability property if for some $C > 0$, for every $(\x,\y) \in \R^p \times \R^q$, $\tilde\y \in \mopt(\x)$ and $\tilde\x \in \mopt(\y)$, we have
\[
f(\x,\y^\ast) + f(\x^\ast,\y) - 2f^\ast \leq C\cdot\br{2f(\x,\y) - f(\x,\tilde\y) - f(\tilde\x,\y)}.
\]
\end{definition}

The right hand expression captures how much one can reduce the function value \emph{locally} by performing marginal optimizations. The property suggests\elink{exer:altmin-robust-bistable} that if not much local improvement can be made (i.e., if $f(\x,\tilde\y) \approx f(\x,\y) \approx f(\tilde\x,\y)$) then we are close to the optimum. This has a simple corollary that all bistable points achieve the (globally) optimal function value. We now present a convergence analysis for gAM.

\begin{theorem}
\label{thm:gam-msc-mss-proof}
Let $f: \R^p \times \R^q \rightarrow \R$ be a continuously differentiable (but possibly non-convex) function that, within the region $S_0 = \bc{\x,\y: f(\x,\y) \leq f(\vzero,\vzero)} \subset \R^{p+q}$, satisfies the properties of $\alpha$-MSC, $\beta$-MSS in both its variables, and $C$-robust bistability. Let Algorithm~\ref{algo:gam} be executed with the initialization $(\x^1,\y^1) = (\vzero,\vzero)$. Then after at most $T = \bigO{\log\frac{1}{\epsilon}}$ steps, we have $f(\x^T,\y^T) \leq f^\ast + \epsilon$.
\end{theorem}
Note that the MSC/MSS and robust bistability properties need only hold within the sublevel set $S_0$. This again underlines the importance of proper initialization. Also note that gAM offers rapid convergence despite the non-convexity of the objective. In order to prove the result, the following consequence of $C$-robust bistability will be useful.
\begin{lemma}
\label{lemma:local-conv-gam}
Let $f$ satisfy the properties mentioned in Theorem~\ref{thm:gam-msc-mss-proof}. Then for any $(\x,\y) \in \R^p \times \R^q$, $\tilde\y \in \mopt(\x)$ and $\tilde\x \in \mopt(\y)$,
\[
\norm{\x-\x^\ast}_2^2 + \norm{\y-\y^\ast}_2^2 \leq \frac{C\beta}{\alpha}\br{\norm{\x-\tilde\x}_2^2 + \norm{\y-\tilde\y}_2^2}
\]
\end{lemma}
\begin{proof}
Applying MSC/MSS repeatedly gives us
\begin{align*}
f(\x,\y^\ast) + f(\x^\ast,\y) &\geq 2f^\ast + \frac{\alpha}{2}\br{\norm{\x-\x^\ast}_2^2 + \norm{\y-\y^\ast}_2^2}\\
2f(\x,\y) &\leq f(\x,\tilde\y) + f(\tilde\x,\y) + \frac{\beta}{2}\br{\norm{\x-\tilde\x}_2^2+\norm{\y-\tilde\y}_2^2}
\end{align*}
Applying robust stability then proves the result.
\end{proof}
It is noteworthy that Lemma~\ref{lemma:local-conv-gam} relates local convergence to global convergence and assures us that reaching an almost bistable point is akin to converging to the optimum. Such a result can be crucial, especially for non-convex problems. Indeed, similar properties are used in other proofs concerning coordinate minimization as well, for example, the \emph{local error bound} used in \citep{LuoT1993}.
\begin{proof}[Proof (of Theorem~\ref{thm:gam-msc-mss-proof}).]
We will use $\Phi_t = f(\x^t,\y^t) - f^\ast$ as the potential function. Since the intermediate steps in gAM are marginal optimizations and not gradient steps, we will actually find it useful to apply marginal strong convexity at a local level, and apply marginal strong smoothness at a global level instead.\\

\noindent\textbf{(Apply Marginal Strong Smoothness)} As $\nabla_\x f(\x^\ast,\y^\ast) = \vzero$, applying MSS gives us
\[
f(\x^{t+1},\y^\ast) - f(\x^\ast,\y^\ast) \leq \frac{\beta}{2}\norm{\x^{t+1}-\x^\ast}_2^2.
\]
Further, the gAM updates ensure $\y^{t+1} \in \mopt(\x^{t+1})$, which gives
\[
\Phi_{t+1} = f(\x^{t+1},\y^{t+1}) - f^\ast \leq f(\x^{t+1},\y^\ast) - f^\ast \leq \frac{\beta}{2}\norm{\x^{t+1}-\x^\ast}_2^2,
\]

\noindent\textbf{(Apply Marginal Strong Convexity)} Since $\nabla_\x f(\x^{t+1},\y^t) = \vzero$,
\begin{align*}
f(\x^t,\y^t) &\geq f(\x^{t+1},\y^t) + \frac{\alpha}{2}\norm{\x^{t+1}-\x^t}_2^2\\
&\geq f(\x^{t+1},\y^{t+1}) + \frac{\alpha}{2}\norm{\x^{t+1}-\x^t}_2^2,
\end{align*}
which gives us
\[
\Phi_t - \Phi_{t+1} \geq \frac{\alpha}{2}\norm{\x^{t+1}-\x^t}_2^2.
\]
This shows that appreciable progress is made in a single step. Now, with $(\x^t,\y^t)$ for any $t \geq 2$, due to the nature of the gAM updates, we know that $\y^t \in \mopt(\x^t)$ and $\x^{t+1} \in \mopt(\y^t)$. Applying Lemma~\ref{lemma:local-conv-gam} then gives us the following inequality
\[
\norm{\x^t-\x^\ast}_2^2 \leq \norm{\x^t-\x^\ast}_2^2 + \norm{\y^t-\y^\ast}_2^2 \leq \frac{C\beta}{\alpha}\norm{\x^t-\x^{t+1}}_2^2
\]
Putting these together and using $(a+b)^2 \leq 2(a^2+b^2)$ gives us
\begin{align*}
\Phi_{t+1} &\leq \frac{\beta}{2}\norm{\x^{t+1}-\x^\ast}_2^2 \leq \beta\br{\norm{\x^{t+1}-\x^t}_2^2 + \norm{\x^t-\x^\ast}_2^2}\\
					 &\leq \beta(1 + C\kappa)\norm{\x^{t+1}-\x^t}_2^2 \leq 2\kappa(1 + C\kappa)\br{\Phi_t - \Phi_{t+1}},
\end{align*}
where $\kappa = \frac{\beta}{\alpha}$ is the effective condition number of the problem. Rearranging gives us
\[
\Phi_{t+1} \leq \eta_0\cdot\Phi_t,
\]
where $\eta_0 = \frac{2\kappa(1 + C\kappa)}{1+2\kappa(1 + C\kappa)} < 1$ which proves the result.
\end{proof}

Notice that the condition number $\kappa$ makes an appearance in the convergence rate of the algorithm but this time, with a fresh definition in terms of the MSC/MSS parameters. As before, small values of $\kappa$ and $C$ ensure fast convergence, whereas large values of $\kappa,C$ promote $\eta_0 \rightarrow 1$ which slows the procedure down.

Before we conclude, we remind the reader that in later sections, we will see more precise analyses of gAM-style approaches, and the structural assumptions will be more problem specific. However, we hope the preceding discussion has provided some insight into the inner workings of alternating minimization techniques.

\section{Exercises}
\begin{exer}
\label{exer:altmin-marg-conv}
Recall the low-rank matrix completion problem in recommendation systems from \S~\ref{chap:intro}
\[
\hat A_\text{lv} = \underset{\substack{U \in \R^{m \times r}\\V \in \R^{n \times r}}}{\min}\ \sum_{(i,j) \in \Omega}\br{U_i^\top V_j - A_{ij}}^2.
\]
Show that the objective in this optimization problem is not jointly convex in $U$ and $V$. Then show that the objective is nevertheless, marginally convex in both the variables.
\end{exer}
\begin{exer}
\label{exer:altmin-joint-marg}
Show that a function that is jointly convex is necessarily marginally convex as well. Similarly show that a (jointly) strongly convex and smooth function is marginally so as well.
\end{exer}
\begin{exer}
\label{exer:altmin-msc-sc}
Marginal strong convexity does not imply convexity. Show this by giving an example of a function $f: \R^p \times \R^q \rightarrow \R$ that is marginally strongly convex in \emph{both} its variables, but non-convex.\\
\textit{Hint}: use the fact that the function $f(x) = x^2$ is $2$-strongly convex.
\end{exer}
\begin{exer}
\label{exer:altmin-gen-am}
Design a variant of the gAM procedure that can handle a general constraint set $\cZ \subset \cX \times \cY$. Attempt to analyze the convergence of your algorithm.
\end{exer}
\begin{exer}
\label{exer:altmin-opt-bs}
Show that $(\x^\ast,\y^\ast) \in \arg\min_{\x\in\cX,\y\in\cY}f(\x,\y)$ must be a bistable point for any function even if $f$ is non-convex.
\end{exer}
\begin{exer}
\label{exer:altmin-conv-gam}
Let $f: \R^p \times \R^q \rightarrow \R$ be a differentiable, \emph{jointly} convex function. Show that any bistable point of $f$ is a global minimum for $f$.\\
\textit{Hint}: first show that directional derivatives vanish at bistable points.
\end{exer}
\begin{exer}
\label{exer:altmin-robust-bistable}
For a robustly bistable function $f$, any \emph{almost} bistable point is \emph{almost} optimal as well. Show this by proving, for any $(\x,\y)$, $\tilde\y \in \mopt(\x)$, $\tilde\x \in \mopt(\y)$ such that $\max\bc{f(\x,\tilde\y),f(\tilde\x,\y)} \leq f(\x,\y) + \epsilon$, that $f(\x,\y) \leq f^\ast + \bigO{\epsilon}$. Conclude that if $f$ satisfies robust bistability, then any bistable point $(\x,\y) \in \cZ^\ast$ is optimal.
\end{exer}
\begin{exer}
Show that marginal strong convexity is additive i.e., if $f, g: \R^p \times \R^q \rightarrow \R$ are two functions such that $f$ is respectively $\alpha_1$ and $\alpha_2$-MSC in its two variables and $g$ is $\bar\alpha_1$ and $\bar\alpha_2$-MSC in its variables, then the function $f+g$ is $(\alpha_1+\bar\alpha_1)$ and $(\alpha_2+\bar\alpha_2)$-MSC in its variables.
\end{exer}
\begin{exer}
The alternating minimization procedure may oscillate if the optimization problem is not well-behaved. Suppose for an especially nasty problem, the gAM procedure enters into the following loop
\[
(\x^t,\y^t) \rightarrow (\x^{t+1},\y^t) \rightarrow (\x^{t+1},\y^{t+1}) \rightarrow (\x^t,\y^{t+1}) \rightarrow (\x^t, \y^t)
\]
Show that all four points in the loop are bistable and share the same function value. Can you draw a hypothetical set of marginally optimal coordinate curves which may cause this to happen (see Figure~\ref{fig:bistable})?
\end{exer}

\section{Bibliographic Notes}
\label{sec:altmin-bib}
The descent version for CM is aptly named \emph{Coordinate Descent} (CD) and only takes descent steps along the coordinates \citep{SahaT2013}. There exist versions of CM and CD for constrained optimization problems as well \citep{LuoT1993,Nesterov2012}. A variant of CM/CD splits variables into \emph{blocks} of multi-dimensional variables. The resulting algorithm is appropriately named \emph{Block Coordinate Descent} and minimizes over one block of variables at each time step.

The coordinate/block being processed at each step is chosen carefully to ensure rapid convergence. Several ``rules'' exist for this, for example, the Gauss-Southwell rule (that chooses the coordinate/block along which the objective gradient is the largest), the cyclic rule that simply performs a round-robin selection of coordinates/blocks, and random choice that chooses a random coordinate/block at each time step, independent of previous such choices.

For several problem areas such support vector machines, CM/CD methods are at the heart of some of the fastest solvers available due to their speed and ease of implementation \citep{FanCHWL2008,LuoT1992,Shalev-ShwartzZ2013}.

\chapter{The EM Algorithm}
\label{chap:em}

\newcommand{\vth}{\vtheta}
\newcommand{\popt}{\vth^\ast}

In this section we will take a look at the \emph{Expectation Maximization} (EM) principle. The principle forms the basis for widely used learning algorithms such as those used for learning Gaussian mixture models, the Baum-Welch algorithm for learning hidden Markov models (HMM), and mixed regression. The EM algorithm is also a close cousin to the Lloyd's algorithm for clustering with the k-means objective.

Although the EM algorithm, at a surface level, follows the alternating minimization principle which we studied in \S~\ref{chap:em}, given its wide applicability in learning latent variable models in probabilistic learning settings, we feel it is instructive to invest in a deeper understanding of the EM method. To make the reading experience self-contained, we will first devote some time developing intuitions and notation in probabilistic learning methods.\\

\noindent\textbf{Notation} A parametric distribution over a domain $\cX$ with parameter $\vth$ is denoted by $f(\cdot\cond\vth)$ or $f_\vth$. The notation is abused to let $f(\x\cond\vth)$ denote the probability mass or density function (as the case may be) of the distribution at the point $\x \in \cX$. The notation is also abused to let $X \sim f(\cdot\cond\vth)$ or $X \sim f_\vth$ denote a sample drawn from this distribution.

\section{A Primer in Probabilistic Machine Learning}
\label{sec:em-pml}
Suppose we observe i.i.d. samples $\x_1,\x_2,\ldots,\x_n$ of a random variable $X \in \cX$ drawn from an unknown distribution $f^\ast$. Suppose also, that it is known that the distribution generating these data samples belongs to a \emph{parametric family} of distributions $\cF = \bc{f_\vth: \vth \in \Theta}$ such that $f^\ast = f_{\popt}$ for some unknown parameter $\popt \in \Theta$.

How may we recover (an accurate estimate of) the true parameter $\popt$, using only the samples $\x_1,\ldots,\x_n$? There are several ways to do so, popular among them being the \emph{maximum likelihood} estimate. Since the samples were generated independently, one can, for any parameter $\vth^0 \in \Theta$, write their joint density function as follows
\[
f(\x_1,\x_2,\ldots,\x_n\cond\vth^0) = \prod_{i=1}^nf(\x_i\cond\vth^0)
\]
The above quantity is also known as the \emph{likelihood} of the data parametrized on $\vth^0$ as it captures the probability that the observed data was generated by the parameter $\vth^0$. In fact, we can go ahead and define the likelihood function for any parameter $\vth \in \Theta$ as follows
\[
\cL(\vth;\x_1,\x_2,\ldots,\x_n) := f(\x_1,\x_2,\ldots,\x_n\cond\vth)
\]
The maximum likelihood estimate (MLE) of $\popt$ is simply the parameter that maximizes the above likelihood function i.e, the parameter which seems to be the ``most likely'' to have generated the data.
\[
\hat\vth_{\text{MLE}} := \underset{\vth \in \Theta}{\arg\max}\ \cL(\vth;\x_1,\x_2,\ldots,\x_n)
\]
It is interesting to study the convergence of $\hat\vth_{\text{MLE}} \rightarrow \popt$ as $n \rightarrow \infty$ but we will not do so in this monograph. We note that there do exist other estimation techniques apart from MLE, such as the \emph{Maximum a Posteriori} (MAP) estimate that incorporates a \emph{prior} distribution over $\vtheta$, but we will not discuss those here either.\\

\noindent\textbf{Least Squares Regression} As a warmup exercise, let us take the example of linear regression and reformulate it in a probabilistic setting to better understand the above framework. Let $y \in \bR$ and $\bt,\x \in \bR^p$ and consider the following parametric distribution over the set of reals, parametrized by $\bt$ and $\x$
\[
f(y\cond\x,\bt) = \frac{1}{\sqrt{2\pi}}\exp\br{-\frac{(y - \x^\top\bt)^2}{2}},
\]
Note that this distribution exactly encodes the responses in a linear regression model with unit variance Gaussian noise. More specifically, if $y \sim f(\cdot\cond\x,\bt)$, then
\[
y \sim \cN(\x^\top\bt,1)
\]
The above observation allows us to cast linear regression as a parameter estimation problem. Consider the parametric distribution family
\[
\cF = \bc{f_\bt = f(\cdot\cond\cdot,\bt) : \norm{\bt}_2 \leq 1}.
\]
Suppose now that we have $n$ covariate samples $\x_1,\x_2,\ldots,\x_n$ and there is a true parameter $\bto$ such that the distribution $f_{\bto} \in \cF$ (i.e., $\norm{\bto}_2 \leq 1$) is used to generate the responses i.e., $y_i \sim f(\cdot\cond\x_i,\bto)$. It is easy to see that the likelihood function in this setting is\footnote{The reader would notice that we are modeling only the process that generates the responses \emph{given} the covariates. However, this is just for sake of simplicity. It is possible to model the process that generates the covariates $\x_i$ as well using, for example, a mixture of Gaussians (that we will study in this very section). A model that accounts for the generation of both $\x_i$ and $y_i$ is called a \emph{generative} model.}
\[
\cL(\bt;\bc{(\x_i,y_i)}_{i=1}^n) = \prod_{i=1}^nf(y_i\cond\x_i,\bt) = \frac{1}{\sqrt{2\pi}}\prod_{i=1}^n\exp\br{-\frac{(y_i - \x_i^\top\bt)^2}{2}}
\]
Since the logarithm function is a strictly increasing function, maximizing the \emph{log-likelihood} will also yield the MLE, i.e.,
\[
\bth_{\text{MLE}} = \underset{\norm{\bt}_2 \leq 1}{\arg\max}\ \log\cL(\bt;\bc{(\x_i,y_i)}_{i=1}^n) = \underset{\norm{\bt}_2 \leq 1}{\arg\min}\ \sum_{i=1}^n(y_i - \x_i^\top\bt)^2
\]
Thus, the MLE for linear regression under Gaussian noise is nothing but the common least squares estimate! The theory of maximum likelihood estimates and their consistency properties is well studied and under suitable conditions, we indeed have $\bth_{\text{MLE}} \rightarrow \bto$. ML estimators are members of a more general class of estimators known as \emph{M-estimators} \citep{HuberR2009}.

\section{Problem Formulation}
In practical applications, the probabilistic models one encounters are often more challenging than the one for linear regression due to the presence of \emph{latent variables} which necessitates the use of more careful routines like the EM algorithm to even calculate the MLE.

Consider a statistical model that generates two random variables $Y \in \cY$ and $Z \in \cZ$, instead of one, using a distribution from a parametric family $\cF  = \bc{f_\vth = f(\cdot, \cdot \cond \vth): \vth \in \Theta}$ i.e., $(Y,Z) \sim f_{\popt}$ for some $\popt \in \Theta$. However, we only get to see realizations of the $Y$ components and not the $Z$ components. More specifically, although the (unknown) parameter $\popt$ generates pairs of samples $(\y_1,\z_1), (\y_2,\z_2), \ldots, (\y_n,\z_n)$, only $\y_1,\y_2,\ldots,\y_n$ are revealed to us. The missing $Z$ components are often called \emph{latent variables} since they are hidden from us.

The above situation arises in several practical situations in data modeling, clustering, and analysis. We encourage the reader to momentarily skip to \S~\ref{sec:em-app} to look at a few nice examples before returning to proceed with this discussion. A first attempt at obtaining the MLE in such a scenario would be to maximize the \emph{marginal likelihood} function instead. Assume for the sake of simplicity that the support set of the random variable $Z$, i.e., $\cZ$, is discrete. Then the marginal likelihood function is defined as the following
\[
\cL(\vth;\y_1,\y_2,\ldots,\y_n) = \prod_{i=1}^nf(\y_i\cond\vth) = \prod_{i=1}^n\sum_{\z_i\in\cZ}f(\y_i,\z_i\cond\vth).
\]
In most practical situations, using the marginal likelihood function $\cL(\vth;\y_1,\y_2,\ldots,\y_n)$ to perform ML estimation becomes intractable since the expression on the right hand side, when expanded as a sum, contains $\abs{\cZ}^n$ terms which makes it difficult to even write down the expression fully let alone optimize using it as an objective function!

For comparison, the log-likelihood expression for the linear regression problem with $n$ data points (the least squares expression) was a summation of $n$ terms. Indeed, the problem of maximizing the marginal likelihood function $\cL(\vth;\y_1,\y_2,\ldots,\y_n)$ is often NP-hard and as a consequence, direct optimization techniques for finding the MLE fail for even small scale problems.

\begin{algorithm}[t]
	\caption{AltMax for Latent Variable Models (AM-LVM)}
	\label{algo:em-hard-altopt}
	\begin{algorithmic}[1]
		{
			\REQUIRE Data points $\y_1,\ldots,\y_n$
			\ENSURE An approximate MLE $\hat\vth \in \Theta$
			\STATE $\vth^1 \leftarrow \text{\textsf{INITALIZE}}()$
			\FOR{$t = 1, 2, \ldots$}
				\FOR{$i = 1, 2, \ldots, n$}
					\STATE $\hat\z_i^t \leftarrow \underset{\z \in \cZ}{\arg\max}\ f(\z\cond\y_i,\vth^t)$ \hfill {(Estimate latent variables)}%
				\ENDFOR
				\STATE $\vth^{t+1} \leftarrow \underset{\vth\in\Theta}{\arg\max}\ \log\cL(\vth;\bc{(\y_i,\hat\z_i^t)}_{i=1}^n)$ \hfill {(Update parameter)}
			\ENDFOR
			\STATE \textbf{return} {$\btt$}
		}
	\end{algorithmic}
\end{algorithm}

\section{An Alternating Maximization Approach}
Notice that the key reason for the intractability of the MLE problem in the previous discussion was the missing information about the latent variables. Had the latent variables $\z_1,\ldots,\z_n$ been magically provided to us, it would have been simple to find the MLE solution as follows
\[
\hat\vth_{\text{MLE}} = \underset{\vth\in\Theta}{\arg\max}\ \log\cL(\vth;\bc{(\y_i,\z_i)}_{i=1}^n)
\]
However, notice that it is also true that, had the identity of the true parameter $\vth^\ast$ been provided to us (again magically), it would have been simple to estimate the latent variables using a maximum posterior probability estimate for $\z_i$ as follows
\[
\hat\z_i = \underset{\z \in \cZ}{\arg\max}\ f(\z\cond\y_i,\vth^\ast)
\]

Given the above, it is tempting to apply a gAM-style algorithm to solve the MLE problem in the presence of latent variables. Algorithm~\ref{algo:em-hard-altopt} outlines such an adaptation of gAM to the latent variable learning problem. Note that steps 4 and 6 in the algorithm can be very efficiently carried out for several problem cases. In fact, it can be shown\elink{exer:em-em-k-means} that for the Gaussian Mixture modeling problem, AM-LVM reduces to the popular Llyod's algorithm for k-means clustering.

However, the AM-LVM algorithm has certain drawbacks, especially when the space of latent variables $\cZ$ is large. At every time step $t$, AM-LVM makes a ``hard assignment'', assigning the data point $\y_i$ to just one value of the latent variable $\z_i^t \in \cZ$. This can amount to throwing away a lot of information, especially when there may be other values $\z' \in \cZ$ present such that $f(\z'\cond\y_i,\vth^t)$ is also large but nevertheless $f(\z'\cond\y_i,\vth^t) < f(\z^t_i\cond\y_i,\vth^t)$ so that AM-LVM neglects $\z'$. The EM algorithm tries to remedy this.

\section{The EM Algorithm}
\label{sec:em-em}
Given the drawbacks of the hard assignment approach adopted by the AM-LVM algorithm, the EM algorithm presents an alternative approach that can be seen as making ``soft'' assignments.

At a very high level, the EM algorithm can be seen as doing the following. At time step $t$, instead of assigning the point $\y_i$ to a single value of the latent variable $\hat\z^t_i = \underset{\z \in \cZ}{\arg\max}\ f(\z\cond\y_i,\vth^t)$, the EM algorithm chooses to make a partial assignment of $\y_i$ to \emph{all} possible values of the latent variable in the set $\cZ$. The EM algorithm ``assigns'' $\y_i$ to a value $\z \in \cZ$ with affinity/weight equal to $f(\z\cond\y_i,\vth^t)$.

Thus, $\y_i$ gets partially assigned to all possible values of the latent variable, to some with more weight, to others with less weight. This avoids any loss of information. Note that these weights are always positive and sum up to unity since $\sum_{\z\in\cZ}f(\z\cond\y_i,\vth^t) = 1$. Also note that EM still assigns $\hat\z^t_i = \underset{\z \in \cZ}{\arg\max}\ f(\z\cond\y_i,\vth^t)$ the highest weight. In contrast, AM-LVM can be now seen as putting all the weight on $\hat\z_i^t$ alone and zero weight on any other latent variable value.

We now present a more formal derivation of the EM algorithm. Instead of maximizing the likelihood in a single step, the EM algorithm tries to efficiently \emph{encourage} an increase in the likelihood over several steps. Define the \emph{point-wise likelihood} function as
\[
\cL(\vth;\y) = \sum_{\z\in\cZ}f(\y,\z\cond\vth).
\]
Note that we can write the marginal likelihood function as $\cL(\vth;\y_1,\y_2,\ldots,\y_n) = \prod_{i=1}^n\cL(\vth;\y_i)$. Our goal is to maximize $\cL(\vth;\y_1,\y_2,\ldots,\y_n)$ but doing so directly is too expensive. So the next best thing is to do so indirectly. Suppose we had a proxy function that lower bounded the likelihood function but was also easy to optimize. Then maximizing the proxy function would also lead to an increase in the likelihood if the proxy were really good.

This is the key to the EM algorithm: it introduces a proxy function called the \emph{$Q$-function} that lower bounds the marginal likelihood function and casts the parameter estimation problem as a bi-variate problem, the two variables being the parameter $\vth$ and the $Q$-function.

Given an initialization, $\vth^0 \in \Theta$, EM constructs a $Q$-function out of it, uses that as a proxy to obtain a better parameter $\vth^1$, uses the newly obtained parameter to construct a better $Q$-function, uses the better $Q$-function to obtain a still better parameter $\vth^2$, and so on. Thus, it essentially performs alternating optimization steps, with better estimations of the $\vth$ parameter leading to better constructions of the $Q$-function and vice versa.

To formalize this notion, we will abuse notation to let $f(\z\cond\y,\vth^0)$ denote the conditional probability function for the random variable $Z$ given the variable $Y$ and the parameter $\vth^0$. Then we have
\[
\log\cL(\vth;\y) = \log\sum_{\z\in\cZ}f(\y,\z\cond\vth) = \log\sum_{\z\in\cZ}f(\z\cond\y,\vth^0)\frac{f(\y,\z\cond\vth)}{f(\z\cond\y,\vth^0)}
\]
The summation in the last expression can be seen to be simply an expectation with respect to the random variable $Z$ being sampled from the conditional distribution $f(\cdot\cond\y,\vth^0)$. Using this, we get
\begin{align*}
\log\cL(\vth;\y) &= \log \bE_{\z \sim f(\cdot\cond\y,\vth^0)}\bs{\frac{f(\y,\z\cond\vth)}{f(\z\cond\y,\vth^0)}}\\
									 &\geq \bE_{\z \sim f(\cdot\cond\y,\vth^0)}\bs{\log\frac{f(\y,\z\cond\vth)}{f(\z\cond\y,\vth^0)}}\\
									 &= \underbrace{\bE_{\z \sim f(\cdot\cond\y,\vth^0)}\bs{\log f(\y,\z\cond\vth)}}_{Q_{\y}(\vth\cond\vth^0)} - \underbrace{\bE_{\z \sim f(\cdot\cond\y,\vth^0)}\bs{\log f(\z\cond\y,\vth^0)}}_{R_{\y}(\vth^0)}.
\end{align*}
The inequality follows from Jensen's inequality as the logarithm function is concave. The function $Q_{\y}(\vth\cond\vth^0)$ is called the \emph{point-wise $Q$-function}. Now, the $Q$-function can be interpreted as a \emph{weighted} point-wise likelihood function
\[
Q_{\y}(\vth\cond\vth^0) =  \sum_{\z \in \cZ}w_\z\cdot\log f(\y,\z\cond\vth),
\]
with weights $w_\z = f(\z\cond\y,\theta^0)$. Notice that this exactly corresponds to assigning $\y$ to every $\z \in \cZ$ with weight $w_\z$ instead of assigning it to just one value $\hat\z = \arg\max_{\z\in\cZ}\ f(\z\cond\y,\vth^0)$ (with weight $1$) as AM-LVM does. We will soon see that due to the way the $Q$-function is used, the EM algorithm can be seen as performing AM-LVM with a soft-assignment step instead of a hard assignment step.

Using the point-wise $Q$-function, we define the $Q$-function.
\[
Q(\vth\cond\vth^0) = \frac1n\sum_{i=1}^nQ_{\y_i}(\vth\cond\vth^0)
\]

The $Q$-function has all the properties we desired from our proxy: parameters maximizing the function $Q(\vth\cond\vth^0)$ do indeed improve the likelihood $\cL(\vth;\y_1,\y_2,\ldots,\y_n)$. More importantly, for several applications, it is possible to both construct as well as optimize the $Q$-function, very efficiently.

Algorithm~\ref{algo:em} gives an overall skeleton of the EM algorithm. Implementing EM requires two routines, one to construct the $Q$-function corresponding to the current iterate (the \emph{Expectation step} or E-step), and the other to maximize the $Q$-function (the \emph{Maximization step} or M-step) to get the next iterate. We will next give precise constructions of these steps.

EM works well in many situations where the marginal likelihood function is inaccessible. This is because the $Q$-function only requires access to the conditional probability function $f(\cdot\cond\y,\vth^0)$ and the joint probability function $f(\cdot,\cdot\cond\vth)$, both of which are readily accessible in several applications. We will see examples of such applications shortly.

\section{Implementing the E/M steps}
\label{sec:em-implement}
We will now look at various implementations of the E and the M steps in the EM algorithm. Some would be easier to implement in practice whereas others would be easier to analyze.

\begin{algorithm}[t]
	\caption{Expectation Maximization (EM)}
	\label{algo:em}
	\begin{algorithmic}[1]
		{
			\REQUIRE Implementations of the E-step $E(\cdot)$, and the M-step $M(\cdot)$
			\ENSURE A good parameter $\hat\vth \in \Theta$
			\STATE $\vth^1 \leftarrow \text{\textsf{INITALIZE}}()$
			\FOR{$t = 1, 2, \ldots$}
				\STATE $Q_t(\cdot\cond\vth^t) \leftarrow E(\vth^t)$ \hfill {(E-step)}%
				\STATE $\vth^{t+1} \leftarrow M(\vth^t,Q_t)$ \hfill {(M-step)}
			\ENDFOR
			\STATE \textbf{return} {$\btt$}
		}
	\end{algorithmic}
\end{algorithm}

\noindent\textbf{E-step Constructions} Given the definition of the point-wise $Q$-function, one can use it to construct the $Q$-function in several ways. The first construction is of largely theoretical interest but reveals a lot of the inner workings of the EM algorithm by simplifying the proofs. This \emph{population} construction requires access to the marginal distribution on the Y variable i.e., the probability function $f(\y\cond\popt)$. Recall that $\popt$ is the true parameter generating these samples. Given this, the population $E$-step constructs the $Q$-function as
\begin{align*}
Q^\text{pop}_t(\vth\cond\vth^t) &= \bE_{\y \sim f_{\popt}}Q_{\y}(\vth\cond\vth^t)\\
									 &= \bE_{\y \sim f_{\popt}}\bE_{\z \sim f(\cdot\cond\y,\vth^t)}\bs{\log f(\y,\z\cond\vth)}\\
									 &= \sum_{\y \in \cY}\sum_{\z \in \cZ}f(\y\cond\popt)\cdot f(\z\cond\y,\vth^t)\cdot\log f(\y,\z\cond\vth)
\end{align*}
Clearly this construction is infeasible in practice. A much more realistic \emph{sample} construction works with just the observed samples $\y_1,\y_2,\ldots,\y_n$ (note that these were indeed drawn from the distribution $f(\y\cond\popt)$). The sample $E$-step constructs the $Q$-function as
\begin{align*}
Q^\text{sam}_t(\vth\cond\vth^t) &= \frac{1}{n}\sum_{i=1}^nQ_{\y_i}(\vth\cond\vth^t)\\
									 &= \frac{1}{n}\sum_{i=1}^n\bE_{\z \sim f(\cdot\cond\y_i,\vth^t)}\bs{\log f(\y_i,\z\cond\vth)}\\
									 &= \frac{1}{n}\sum_{i=1}^n\sum_{\z \in \cZ}f(\z\cond\y,\vth^t)\cdot\log f(\y_i,\z\cond\vth)
\end{align*}
Note that this expression has $n\cdot\abs{\cZ}$ terms instead of the $\abs{\cZ}^n$ terms which the marginal likelihood expression had. This drastic reduction is a key factor behind the scalability of the EM algorithm.\\

\noindent\textbf{M-step Constructions}
Recall that in \S~\ref{chap:altmin} we considered alternating approaches that fully optimize with respect to a variable, as well as those that merely perform a descent step, improving the function value along that variable but not quite optimizing it.

Similar variants can be developed for EM as well. Given a $Q$-function, the simplest strategy is to optimize it completely with the M-step simply returning the maximizer of the $Q$-function. This is the \emph{fully corrective} version of EM. It is useful to remind ourselves here that whereas in previous sections we looked at minimization problems, the problem here is that of likelihood \emph{maximization}.
\[
M^\text{fc}(\vth^t,Q_t) = \underset{\vth \in \Theta}{\arg\max}\ Q_t(\vth\cond\vth^t)
\]
Since this can be expensive in large scale optimization settings, a \emph{gradient descent} version exists that makes the M-step faster by performing just a gradient step with respect to the $Q$-function i.e.,
\[
M^\text{grad}(\vth^t,Q_t) = \vth^t + \alpha_t\cdot\nabla Q_t(\vth^t\cond\vth^t).
\]

\noindent\textbf{Stochastic EM Construction}
A highly scalable version of the algorithm is the \emph{stochastic update} version that uses the point-wise $Q$-function of a single, randomly chosen sample $Y_t \sim \text{\textsf{Unif}}[n]$ to execute a gradient update at each time step $t$. It can be shown\elink{exer:em-sgd} that on expectation, this executes the sample E-step and the gradient M-step.
\[
M^\text{sto}(\vth^t) = \vth^t + \alpha_t\cdot\nabla Q_{Y_t}(\vth^t\cond\vth^t)
\]

\section{Motivating Applications}
\label{sec:em-app}
We will now look at a few applications of the EM algorithm and see how the E and M steps are executed efficiently. We will revisit these applications while discussing convergence proofs for the EM algorithm.

\subsection{Gaussian Mixture Models}
Mixture models are ubiquitous in applications such as clustering, topic modeling, and segmentation tasks. Let $\cN(\cdot;\vmu,\Sigma)$ denote the multivariate normal distribution with mean $\vmu \in \bR^p$ and covariance $\Sigma \in \bR^{p \times p}$. A mixture model can be constructed by combining two such normal distributions to obtain a density function of the form
\[
f(\cdot,\cdot\cond\bc{\phi_i,\vmu^{\ast,i},\Sigma_i}_{i \in \bc{0,1}}) = \phi^\ast_0\cdot\cN_0 + \phi^\ast_1\cdot\cN_1,
\]
where $\cN_i = \cN(\cdot;\vmu^{\ast,i},\Sigma^\ast_i), i = 0,1$ are the \emph{mixture components} and $\phi^\ast_i \in (0,1), i = 0,1$ are the \emph{mixture coefficients}. We insist that $\phi^\ast_0+\phi^\ast_1 = 1$ to ensure that $f$ is indeed a probability distribution. Note that we consider a mixture with just two components for simplicity. Mixture models with larger number of components can be similarly constructed.

A sample $(\y,z) \in \bR^p \times \bc{0,1}$ can be drawn from this distribution by first tossing a Bernoulli coin with bias $\phi_1$ to choose a component $z \in \bc{0,1}$ and then drawing a sample $\y \sim \cN_{z}$ from that component.

However, despite drawing the samples $(\y_1,z_1),(\y_2,z_2),\ldots,(\y_n,z_n)$, what is presented to us is $\y_1,\y_2,\ldots,\y_n$ i.e., the identities $z_i$ of the components that actually resulted in these draws is hidden from us. For instance in topic modeling tasks, the underlying topics being discussed in documents is hidden from us and we only get to see surface realizations of words in the documents that have topic-specific distributions. The goal here is to recover the mixture components as well as coefficients in an efficient manner from such partially observed draws.

For the sake of simplicity, we will look at a balanced, isotropic mixture i.e., where we are given that $\phi^\ast_0 = \phi^\ast_1 = 0.5$ and $\Sigma^\ast_0 = \Sigma^\ast_1 = I_p$. This will simplify our updates and analysis as the only unknown parameters in the model are $\vmu^{\ast,0}$ and $\vmu^{\ast,1}$. Let $\vM = (\vmu^0,\vmu^1) \in \R^{p\times p}$ denote an ensemble describing such a parametric mixture. Our job is to recover $\vM^\ast = (\vmu^{\ast,0},\vmu^{\ast,1})$.\\

\noindent\textbf{E-step Construction} For any $\vM = (\vmu^0,\vmu^1)$, we have the mixture
\[
f(\cdot,\cdot\cond\vM) = 0.5\cdot\cN(\cdot;\vmu^0,I_p) + 0.5\cdot\cN(\cdot;\vmu^1,I_p),
\]
i.e., $\cN_i = \cN(\cdot;\vmu^i,I_p)$. In this case, the $Q$-function actually has a closed form expression. For any $\y \in \bR^p$, $z \in \bc{0,1}$, and $\vM$, we have
\begin{align*}
f(\y,z\cond\vM) &= \cN_z(\y) = \exp\br{-\frac{\norm{\y - \vmu^z}_2^2}{2}}\\
f(z\cond\y,\vM) &= \frac{f(\y,z\cond\vM)}{f(\y\cond\vM)} = \frac{f(\y,z\cond\vM)}{f(\y,z\cond\vM) + f(\y,1-z\cond\vM)}.
\end{align*}
Thus, even though the marginal likelihood function was inaccessible, the point-wise $Q$-function has a nice closed form. For the E-step construction, for any two ensembles $\vM^t = (\vmu^{t,0},\vmu^{t,1})$ and $\vM = (\vmu^0,\vmu^1)$,
\begin{multline*}
Q_{\y}(\vM\cond\vM^t) = \bE_{z \sim f(\cdot\cond\y,\vM^t)}\bs{\log f(\y,z\cond\vM)}\\
				= f(0\cond\y,\vM^t)\cdot\log f(\y,0\cond\vM) + f(1\cond\y,\vM^t)\cdot\log f(\y,1\cond\vM)\\
				= -\frac{1}{2}\br{w^0_t(\y)\cdot\norm{\y - \vmu^0}_2^2 + w^1_t(\y)\cdot\norm{\y - \vmu^1}_2^2},
\end{multline*}
where $w^z_t(\y) = e^{-\frac{\norm{\y - \vmu^{t,z}}_2^2}{2}}\br{e^{-\frac{\norm{\y - \vmu^{t,0}}_2^2}{2}} + e^{-\frac{\norm{\y - \vmu^{t,1}}_2^2}{2}}}^{-1}$. Note that $w^z_t(\y) \geq 0$ for $z = 0,1$ and $w^0_t(\y) + w^1_t(\y) = 1$. Also note that $w^z_t(\y)$ is larger if $\y$ is closer to $\vmu^{t,z}$ i.e., it measures the affinity of a point to the center. Given a sample of data points $\y_1,\y_2,\ldots,\y_n$, the $Q$-function is
\begin{align*}
Q(\vM\cond\vM^t) &= \frac{1}{n}\sum_{i=1}^nQ_{\y_i}(\vM\cond\vM^t)\\
									 &= -\frac{1}{2n}\sum_{i=1}^n {w^0_t(\y_i)\cdot\norm{\y_i - \vmu^0}_2^2 + w^1_t(\y_i)\cdot\norm{\y_i - \vmu^1}_2^2}
\end{align*}

\begin{figure}
\includegraphics[width=\columnwidth]{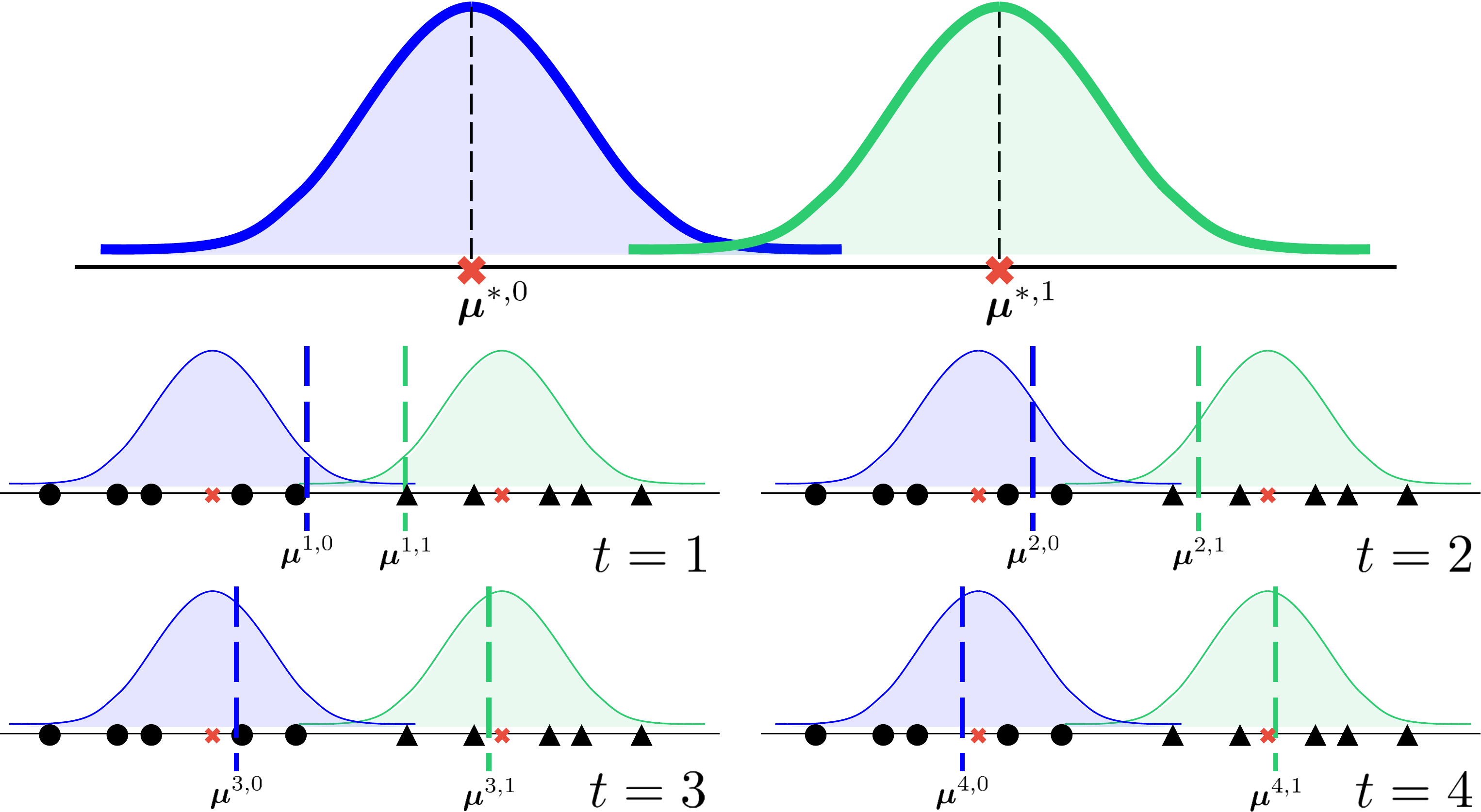}
\caption[Gaussian Mixture Models and the EM Algorithm]{The data is generated from a mixture model: circle points from $\cN(\vmu^{\ast,0},I_p)$ and triangle points from $\cN(\vmu^{\ast,1},I_p)$ but their origin is unknown. The EM algorithm performs soft clustering assignments to realize this and keeps increasing $w^0_t$ values for circle points and increasing $w^1_t$ values for triangle points. As a result, the estimated means $\vmu^{t,z}$ rapidly converge to the true means $\vmu^{\ast,z}$, $z=0,1$.}
\label{fig:gmm}
\end{figure}

\noindent\textbf{M-step Construction} This also has a closed form solution in this case\elink{exer:em-gmm-m}. If $\vM^{t+1} = (\vmu^{t+1,0},\vmu^{t+1,1}) = {\arg\max}_\vM\ Q(\vM\cond\vM^t)$, then
\[
\vmu^{t+1,z} = \sum_{i=1}^n w^z_t(\y_i)\y_i
\]
Note that the M-step can be executed in linear time and does not require an explicit construction of the $Q$-function at all! One just needs to use the M-step repeatedly -- the $Q$-function is implicit in the M-step.

The reader would notice a similarity between the EM algorithm for Gaussian mixture models and the Llyod's algorithm \citep{Lloyd1982} for k-means clustering. In fact\elink{exer:em-em-k-means}, the Llyod's algorithm implements exactly the AM-LVM algorithm (see Algorithm~\ref{algo:em-hard-altopt}) that performs ``hard'' assignments, assigning each data point completely to one of the clusters whereas the EM algorithm makes ``soft'' assignments, allowing each point to have different levels of affinity to different clusters. Figure~\ref{fig:gmm} depicts the working of the EM algorithm on a toy GMM problem.

\subsection{Mixed Regression}
As we have seen in \S~\ref{chap:intro}, regression problems and their variants have several applications in data analysis and machine learning. One such variant is that of mixed regression. Mixed regression is especially useful when we suspect that our data is actually composed of several \emph{sub-populations} which cannot be explained well using a single model.

For example, consider the previous example of predicting family expenditure. Although we may have data from families across a nation, it may be unwise to try and explain it using a single model due to various reasons. The prices of various commodities and services may vary across urban and rural areas and similar consumption in two regions may very well result in different expenditures. Moreover, there may exist parameters such as total income which are not revealed in a survey due to privacy issues, but nevertheless influence expenditure.

Thus, there may actually be several models, each corresponding to a certain income bracket or a certain geographical region, which \emph{together} explain the data very well. This poses a challenge since the income bracket, or geographical location of a family may not have been recorded as a part of the survey due to privacy or other reasons!

To formalize the above scenario, consider two linear models $\bt^{\ast,0}, \bt^{\ast,1} \in \bR^p$. For each data point $\x_i \in \bR^p$, first one of the models is selected by performing a Bernoulli trial with bias $\phi_1$ to get $z_i \in \bc{0,1}$ and then the response is generated as
\[
y_i = \x_i^\top\bt^{\ast,z_i} + \eta_i
\]
where $\eta_i$ is i.i.d. Gaussian noise $\eta_i \sim \cN(0,\sigma_{z_i}^2)$. This can be cast as a parametric model by considering density functions of the form
\[
f(\cdot\cond\cdot,\bc{\phi_z,\bt^{\ast,z},\sigma_z}_{z=0,1}) = \phi_0\cdot g(\cdot\cond\cdot,\bt^{\ast,0},\sigma_0) + \phi_1\cdot g(\cdot\cond\cdot,\bt^{\ast,1},\sigma_1),
\]
where $\sigma_z,\phi_z > 0$, $\phi_0 + \phi_1 = 1$, and for any $(\x,y) \in \bR^p\times\bR$, we have
\[
g(y\cond\x,\bt^{\ast,z},\sigma_z) = \exp\br{-\frac{(y - \x^\top\bt^{\ast,z})^2}{2\sigma_z^2}}.
\]
Note that although such a model generates data in the form of triplets $(\x_1,y_1,z_1),(\x_2,y_2,z_2),\ldots,(\x_n,y_n,z_n)$, we are only allowed to observe $(\x_1,y_1),(\x_2,y_2),\ldots,(\x_n,y_n)$ as the data. For the sake of simplicity, we will yet again look at the special case when the Bernoulli trials are fair i.e., $\phi_0 = \phi_1 = 0.5$ and $\sigma_1 = \sigma_2 = 1$. Thus, the only unknown parameters are the models $\bt^{\ast,0}$ and $\bt^{\ast,1}$. Let $\vW = (\bt^0,\bt^1) \in \R^{p \times p}$ denote the parametric mixed model. Our job is to recover $\vW^\ast = (\bt^{\ast,0},\bt^{\ast,1})$.

A particularly interesting special case arises when we further impose the constraint $\bt^{\ast,0} = -\bt^{\ast,1}$, i.e. the two models in the mixture are tied together to be negative of each other. This model is especially useful in the \emph{phase retrieval} problem. Although we will study this problem in more generality in \S~\ref{chap:phret}, we present a special case here.

Phase retrieval is a problem that arises in several imaging situations such as X-ray crystallography where, after data $\bc{(\vx_i,y_i)}_{i=1,\ldots,n}$ has been generated as $y_i = \ip{\bto}{\vx_i}$, the sign of the response (or more generally, the phase of the response if the response is complex-valued) is omitted and we are presented with just $\bc{(\vx_i,\abs{y_i})}_{i=1,\ldots,n}$. In such a situation, we can use the latent variable $z_i = \sign(y_i)$ to denote the omitted sign information. In this setting, it can be seen that the mixture model with $\bt^{\ast,0} = -\bt^{\ast,1}$ is very appropriate since each data point $(\vx_i,\abs{y_i})$ will be nicely explained by either $\bto$ or $-\bto$ depending on the value of $\sign(y_i)$.

We will revisit this problem in detail in \S~\ref{chap:phret}. For now we move on to discuss the E and M-step constructions for the mixed regression problem. We leave details of the constructions as an exercise\elink{exer:em-mr-em}.\\

\noindent\textbf{E-step Construction} The point-wise $Q$-function has a closed form expression. Given two ensembles $\vW^t = (\bt^{t,0},\bt^{t,1})$ and $\vW = (\bt^{0},\bt^{1})$,
\[
Q_{(\x,y)}(\vW|\vW^t) = -\frac{1}{2}\br{\alpha^0_{t,(\x,y)}\cdot(y - \x^\top\bt^0)^2 - \alpha^1_{t,(\x,y)}\cdot(y - \x^\top\bt^1)^2},
\]
where $\alpha^z_{t,(\x,y)} = e^{-\frac{(y - \x^\top\bt^{t,z})^2}{2}}\bs{e^{-\frac{(y - \x^\top\bt^{t,0})^2}{2}} + e^{-\frac{(y - \x^\top\bt^{t,1})^2}{2}}}^{-1}$. Note that $\alpha^z_{t,(\x,y)} \geq 0$ for $z = 0,1$ and $\alpha^0_{t,(\x,y)} + \alpha^1_{t,(\x,y)} = 1$. Also note that $\alpha^z_{t,(\x,y)}$ is larger if $\bt^{t,z}$ gives less regression error for the point $(\x,y)$ than $\bt^{t,1-z}$ i.e., if $(y - \x^\top\bt^{t,z})^2 < (y - \x^\top\bt^{t,1-z})^2$. Thus, the data point $(\x,y)$ feels greater affinity to the model that fits it better, which is intuitively, an appropriate thing to do.\\

\noindent\textbf{M-step Construction} The maximizer of the $Q$-function has a closed form solution in this case as well. If $\vW^{t+1} = (\bt^{t+1,0},\bt^{t+1,1}) = \arg\max_{\vW}Q(\vW|\vW^t)$, where $Q(\vW|\vW^t)$ is the sample $Q$-function created from a data sample $(\x_1,y_1),(\x_2,y_2),\ldots,(\x_n,y_n)$, then it is easy to see that the M-step update is given by the solution to two weighted least squares problems with weights given by $\alpha^z_{t,(\x_i,y_i)}$ for $z = \bc{0,1}$, which have closed form solutions given by
\[
\bt^{t+1,z} = \br{\sum_{i=1}^n\alpha^z_{t,(\x_i,y_i)}\cdot \x_i\x_i^\top}^{-1}\sum_{i=1}^n\alpha^z_{t,(\x_i,y_i)}\cdot y_i\x_i
\]
Note that to update $\bt^{t,0}$, for example, the M-step essentially performs least squares only over those points $(\x_i,y_i)$ such that $\alpha^0_{t,(\x,y)}$ is large and ignores points where $\alpha^0_{t,(\x,y)} \approx 0$. This is akin to identifying points that belong to sub-population $0$ strongly and performing regression only over them. One need not construct the $Q$-function explicitly here either and may simply keep repeating the M-step. Figure~\ref{fig:mrem} depicts the working of the EM algorithm on a toy mixed regression problem.\\

\begin{figure}
\includegraphics[width=\columnwidth]{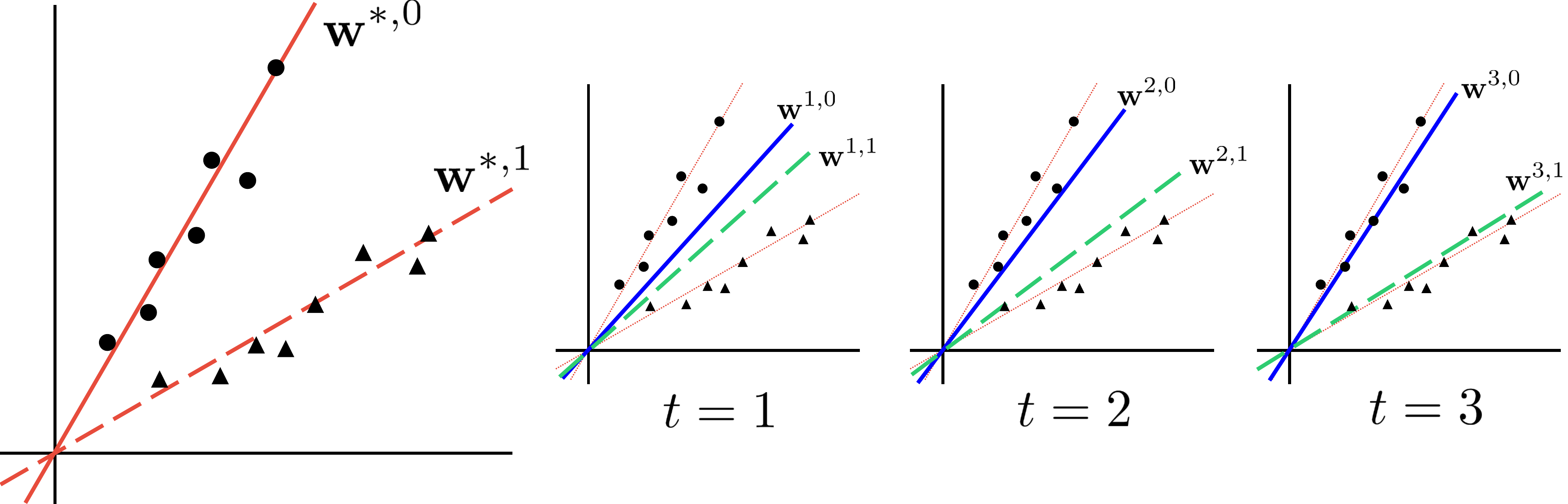}
\caption[Mixed Regression and the EM Algorithm]{The data contains two sub-populations (circle and triangle points) and cannot be properly explained by a single linear model. EM rapidly realizes that the circle points should belong to the solid model and the triangle points to the dashed model. Thus, $\alpha^0_t$ keeps going up for circle points and $\alpha^1_t$ keeps going up for triangle points. As a result, only circle points contribute significantly to learning the solid model and only triangle points contribute significantly to the dashed model.}
\label{fig:mrem}
\end{figure}

\noindent\textbf{Note on Initialization}: when executing the EM algorithm, it is important to initialize the models properly. As we saw in \S~\ref{chap:altmin}, this is true of all alternating strategies. Careless initialization can lead to poor results. For example, if the EM algorithm is initialized for the mixed regression problem with $\vW^1 = (\bt^{1,0},\bt^{1,1})$ such that $\bt^{1,0} = \bt^{1,1}$ then it is easy to see that the EM algorithm will never learn two distinct models and we will have $\bt^{t,0} = \bt^{t,1}$ for all $t$. We will revisit initialization issues shortly when discussing convergence analyses for the EM.

We refer the reader to \citep{BalakrishnanWY2017,YangBW2015} for examples of the EM algorithm applied to other problems such as learning hidden Markov models and regression with missing covariates.

\section{A Monotonicity Guarantee for EM}
\label{sec:em-monotone}
We now present a simple but useful monotonicity property of the EM algorithm that guarantees that the procedure never worsens the likelihood of the data during its successive updates. This assures us that EM does not diverge, although it does not ensure convergence to the true parameter $\popt$ either. For gAM, a similar result was immediate from the nature of the alternating updates.

\begin{theorem}
\label{thm:em-monotone}
The EM algorithm (Algorithm~\ref{algo:em}), when executed with a population E-step and fully corrective M-step, ensures that the population likelihood never decreases across iterations, i.e., for all $t$,
\[
\bE_{\y\sim f_{\popt}}f(\y\cond\vth^{t+1}) \geq \bE_{\y\sim f_{\popt}}f(\y\cond\vth^t),
\]
If executed with the sample E-step on data $\y_1,\ldots,\y_n$, EM ensures that the sample likelihood never decreases across iterations, i.e., for all $t$,
\[
\cL(\vth^{t+1};\y_1,\ldots,\y_n) \geq \cL(\vth^t;\y_1,\ldots,\y_n).
\]
\end{theorem}

A similar result holds for gradient M-steps too but we do not consider that here. To prove this result, we will need the following simple observation\elink{exer:em-monotone}. Recall that for any $\vth^0 \in \Theta$ and $\y \in \cY$, we defined the terms $Q_{\y}(\vth\cond\vth^0) = \Ee{\z \sim f(\cdot\cond\y,\vth^0)}{\log f(\y,\z\cond\vth)}$ and $R_{\y}(\vth^0) = \bE_{\z \sim f(\cdot\cond\y,\vth^0)}\bs{\log f(\z\cond\y,\vth^0)}$.

\begin{lemma}
\label{lem:em-monotone-point}
For any $\vth^0 \in \Theta$, we have
\[
\log f(\y\cond\vth^0) = Q_{\y}(\vth^0\cond\vth^0) - R_{\y}(\vth^0)
\]
\end{lemma}

This is a curious result since we have earlier seen that for any $\vth \in \Theta$, we have the inequality $\log f(\y\cond\vth) \geq Q_{\y}(\vth\cond\vth^0) - R_{\y}(\vth^0)$ by an application of the Jensen's inequality (see \S~\ref{sec:em-em}). The above lemma suggests that the inequality is actually tight. We now prove Theorem~\ref{thm:em-monotone}.

\begin{proof}[Proof (of Theorem~\ref{thm:em-monotone}).]
Consider the sample E-step with $Q(\vth\cond\vth^t) = \frac{1}{n}\sum_{i=1}^nQ_{\y_i}(\vth\cond\vth^t)$. A similar argument works for population E-steps. The fully corrective M-step ensures $Q(\vth^{t+1}\cond\vth^t) \geq Q(\vth^t\cond\vth^t)$ i.e.,
\[
\frac{1}{n}\sum_{i=1}^nQ_{\y_i}(\vth^{t+1}\cond\vth^t) \geq \frac{1}{n}\sum_{i=1}^nQ_{\y_i}(\vth^t\cond\vth^t).
\]
Subtracting the same terms from both sides gives us
\[
\frac{1}{n}\sum_{i=1}^n\br{Q_{\y_i}(\vth^{t+1}\cond\vth^t) - R_{\y_i}(\vth^t)} \geq \frac{1}{n}\sum_{i=1}^n\br{Q_{\y_i}(\vth^t\cond\vth^t) - R_{\y_i}(\vth^t)}.
\]
Using the inequality $\log f(\y\cond\vth^{t+1}) \geq Q_{\y}(\vth^{t+1}\cond\vth^0) - R_{\y}(\vth^0)$ and applying Lemma~\ref{lem:em-monotone-point} gives us
\[
\frac{1}{n}\sum_{i=1}^n\log f(\y_i\cond\vth^{t+1}) \geq \frac{1}{n}\sum_{i=1}^n\log f(\y_i\cond\vth^t),
\]
which proves the result.
\end{proof}

\section{Local Strong Concavity and Local Strong Smoothness}
In order to prove stronger convergence guarantees for the EM algorithm, we need to introduce a few structural properties of parametric distributions. Let us recall the EM algorithm with a population E-step and fully corrective M-step for sake of simplicity.
\begin{enumerate}
	\item (E-step) $Q(\cdot\cond\vth^t) = \bE_{\y \sim f_{\popt}}Q_{\y}(\cdot\cond\vth^t)$
	\item (M-step) $\vth^{t+1} = \underset{\vth \in \Theta}{\arg\max}\ Q_t(\vth\cond\vth^t)$
\end{enumerate}
For any $\vth^0 \in \Theta$, we will use $q_{\vth^0}(\cdot) = Q(\cdot\cond\vth^0)$ as a shorthand for the $Q$-function with respect to $\vth^0$ (constructed at the population or sample level depending on the E-step) and let $M(\vth^0) := \underset{\vth \in \Theta}{\arg\max}\ q_{\vth^0}(\vth)$ denote the output of the M-step if $\vth^0$ is the current parameter. Let $\popt$ denote a parameter that optimizes the population likelihood
\[
\popt \in \underset{\vth \in \Theta}{\arg\max}\ \Ee{\y \sim f(\y|\popt)}{\cL(\vth;\y)} = \underset{\vth \in \Theta}{\arg\max}\ \Ee{\y \sim f(\y|\popt)}{f(\y\cond\vth)}
\]
Recall that our overall goal is indeed to recover a parameter such as $\popt$ that maximizes the population level likelihood (or else sample likelihood if using sample E-steps). Now the $Q$-function satisfies\elink{exer:em-self-consis-q} the following \emph{self-consistency} property.
\[
\popt \in \underset{\vth \in \Theta}{\arg\max}\ q_{\popt}(\vth)
\]
Thus, if we could somehow get hold of the $Q$-function $q_{\popt}(\cdot)$, then a single M-step would solve the problem! However, this is a circular argument since getting hold of $q_{\popt}(\cdot)$ would require finding $\popt$ first.

To proceed along the previous argument, we need to refine this observation. Not only should the M-step refuse to deviate from the optimum $\popt$ if initialized there, it should behave in relatively calm manner in the neighborhood of the optimum as well. The following properties characterize ``nice'' $Q$ functions that ensure this happens.

For sake of simplicity, we will assume that all $Q$-functions are continuously differentiable, as well as that the estimation problem is unconstrained i.e., $\Theta = \R^p$. Note that this ensures $\nabla q_{\popt}(\popt) = \vzero$ due to self-consistency. Also note that since we are looking at maximization problems, we will require the $Q$ function to satisfy ``concavity'' properties instead of ``convexity'' properties.

\begin{definition}[Local Strong Concavity]
A statistical estimation problem with a population likelihood maximizer $\popt$ satisfies the $(r,\alpha)$-Local Strong Concavity (LSC) property if there exist $\alpha, r > 0$, such that the function $q_{\popt}(\cdot)$ is $\alpha$-strongly concave in neighborhood ball of radius $r$ around $\popt$ i.e., for all $\vth^1,\vth^2 \in \cB_2(\popt,r)$,
\[
q_{\popt}(\vth^1) - q_{\popt}(\vth^2) - \ip{\nabla q_{\popt}(\vth^2)}{\vth^1 - \vth^2} \leq -\frac{\alpha}{2}\norm{\vth^1 - \vth^2}_2^2
\]
\end{definition}

The reader would find LSC similar to restricted strong convexity (RSC) in Definition~\ref{defn:res-strong-cvx-smooth-fn}, with the ``restriction'' being the neighborhood $\cB_2(\popt,r)$ of $\popt$. Also note that only $q_{\popt}(\cdot)$ is required to satisfy the LSC property, and not $Q$-functions corresponding to every $\vth$.

We will also require a counterpart to restricted strong smoothness (RSS). For that, we introduce the notion of \emph{Lipschitz} gradients.
\begin{definition}[Lipschitz Gradients]
\label{defn:lip-grad}
A differentiable function $f: \R^p \rightarrow \R$ is said to have $\beta$-Lipschitz gradients if for all $\x,\y\in\R^p$, we have
\[
\norm{\nabla f(\x) - \nabla f(\y)}_2 \leq \beta\cdot\norm{\x-\y}_2.
\]
\end{definition}
We advise the reader to relate this notion to that of Lipschitz functions (Definition~\ref{defn:lip}). It can be shown\elink{exer:em-ss-lip-grad} that all functions with $L$-Lipschitz gradients are also $L$-strongly smooth (Definition~\ref{defn:strong-cvx-smooth-fn}). Using this notion we are now ready to introduce our next properties.

\begin{definition}[Local Strong Smoothness]
\label{defn:lss-em}
A statistical estimation problem with a population likelihood maximizer $\popt$ satisfies the $(r,\beta)$-Local Strong Smoothness (LSS) property if there exist $\beta, r > 0$, such that for all $\vth^1,\vth^2 \in \cB_2(\popt,r)$, the function $q_{\popt}(\cdot)$ satisfies
\[
\norm{\nabla q_{\popt}(M(\vth^1)) - \nabla q_{\popt}(M(\vth^2))}_2 \leq \beta\cdot\norm{\vth^1-\vth^2}_2
\]
\end{definition}

The above property ensures that in the restricted neighborhood around the optimum, the $Q$-function $q_{\popt}(\cdot)$ is strongly smooth. The similarity to RSS is immediate. Note that this property also generalizes the self-consistency property we saw a moment ago.

Self-consistency forces $\nabla q_{\popt}(\popt) = \vzero$ at the optimum. LSS forces such behavior to extend around the optimum as well. To see this, simply set $\vth^2 = \popt$ with LSS and observe the corollary $\norm{\nabla q_{\popt}(M(\vth^1))}_2 \leq \beta\cdot\norm{\vth^1-\popt}_2$. The curious reader may wish to relate this corollary to the Robust Bistability property (Definition~\ref{defn:rob-bistable}) and the \emph{Local Error Bound} property introduced by \cite{LuoT1993}.

The LSS property offers, as another corollary, a useful property of statistical estimation problems called the \emph{First Order Stability} property (introduced by \cite{BalakrishnanWY2017} in a more general setting).

\begin{definition}[First Order Stability \citep{BalakrishnanWY2017}]
A statistical estimation problem with a population likelihood maximizer $\popt$ satisfies the $(r,\gamma)$-First Order Stability (FOS) property if there exist $\gamma >0, r > 0$ such that the the gradients of the functions $q_\vth(\cdot)$ are stable in a neighborhood of $\popt$ i.e., for all $\vth \in \cB_2(\popt,r)$,
\[
\norm{\nabla q_\vth(M(\vth)) - \nabla q_{\popt}(M(\vth))}_2 \leq \gamma\cdot\norm{\vth - \popt}_2
\]
\end{definition}

\begin{lemma}
\label{lemma:lss-lrc-fos}
A statistical estimation problem that satisfies the $(r,\beta)$-LSS property, also satisfies the $(r,\beta)$-FOS property.
\end{lemma}
\begin{proof}
Since $M(\vth)$ maximizes the function $q_\vth(\cdot)$ due to the M-step, and the problem is unconstrained and the $Q$-functions differentiable, we get $\nabla q_{\vth}(M(\vth)) = \vzero$. Thus, we have, using the triangle inequality,
\begin{align*}
&\norm{\nabla q_\vth(M(\vth)) - \nabla q_{\popt}(M(\vth))}_2 = \norm{\nabla q_{\popt}(M(\vth))}_2 \\
&\leq \norm{\nabla q_{\popt}(M(\vth)) - \nabla q_{\popt}(M(\popt))}_2 + \norm{\nabla q_{\popt}(M(\popt))}_2\\
&\leq \beta\cdot\norm{\vth-\popt}_2,
\end{align*}
using self-consistency to get $M(\popt) = \popt$ and $\nabla q_{\popt}(\popt) = \vzero$.
\end{proof}

\section{A Local Convergence Guarantee for EM}
We will now prove a convergence result much stronger than Theorem~\ref{thm:em-monotone}: if EM is initialized ``close'' to the optimal parameter $\popt$ for ``nice'' problems, then it approaches $\popt$ at a linear rate. For the sake of simplicity, we will analyze EM with population E-steps and fully corrective M-steps. We refer the reader to \citep{BalakrishnanWY2017} for analyses of sample E-steps and stochastic M-steps. The following \emph{contraction} lemma will be crucial in the analysis of the EM algorithm.

\begin{lemma}
\label{lem:em-local-contract}
Suppose we have a statistical estimation problem with a population likelihood maximizer $\popt$ that, for some $\alpha,\beta,r > 0$, satisfies the $(r,\alpha)$-LSC and $(r,\beta)$-LSS properties. Then in the region $\cB_2(\popt,r)$, the $M$ operator corresponding to the fully corrective M-step is contractive, i.e., for all $\vth \in \cB_2(\popt,r)$,
\[
\norm{M(\vth) - M(\popt)}_2 \leq \frac{\beta}{\alpha}\cdot\norm{\vth - \popt}_2
\]
\end{lemma}

Since by the self consistency property, we have $M(\popt) = \popt$ and the EM algorithm sets $\vth^{t+1} = M(\vth^t)$ due to the M-step, Lemma~\ref{lem:em-local-contract} immediately guarantees the  following local convergence property\elink{exer:em-conv-like}.

\begin{theorem}
\label{thm:em-local-conv}
Suppose a statistical estimation problem with population likelihood maximizer $\popt$ satisfies the $(r,\alpha)$-LSC and $(r,\beta)$-LSS properties such that $\beta < \alpha$. Let the EM algorithm (Algorithm~\ref{algo:em}) be initialized with $\vth^1 \in \cB_2(\popt,r)$ and executed with population E-steps and fully corrective M-steps. Then after at most $T = \bigO{\log\frac{1}{\epsilon}}$ steps, we have $\norm{\vth^t - \popt}_2 \leq \epsilon$.
\end{theorem}

Note that the above result holds only if $\beta < \alpha$, in other words, if the condition number $\kappa = \beta/\alpha < 1$. We hasten to warn the reader that whereas in previous sections we always had $\kappa \geq 1$, here the LSC and LSS properties are defined differently (LSS involves the M-step whereas LSC does not) and thus it is possible that we have $\kappa < 1$.

Also, since all functions satisfy $(0,0)$-LSC, it is plausible that for well behaved problems, even for $\alpha > \beta$, there should exist some small radius $r(\alpha)$ so that the $(r(\alpha),\alpha)$-LSC property holds. This may require the EM algorithm to be initialized closer to the optimum for the convergence properties to kick in. We now prove Lemma~\ref{lem:em-local-contract} below.

\begin{proof}[Proof of Lemma~\ref{lem:em-local-contract}]
Since we have differentiable $Q$-functions and an unconstrained estimation problem, we immediately get a lot of useful results. We note that this lemma holds even for constrained estimation problems but the arguments are more involved which we wish to avoid. Let $\vth \in \cB_2(\popt,r)$ be any parameter in the $r$-neighborhood of $\popt$.\\

\noindent\textbf{(Apply Local Strong Concavity)} Upon a two sided application of LSC and using $\nabla q_{\popt}(\popt) = \vzero$, we get
\begin{align*}
q_{\popt}(\popt) - q_{\popt}(M(\vth)) - \ip{\nabla q_{\popt}(M(\vth))}{\popt - M(\vth)} &\leq -\frac{\alpha}{2}\norm{M(\vth) - \popt}_2^2\\
q_{\popt}(M(\vth)) - q_{\popt}(\popt) &\leq -\frac{\alpha}{2}\norm{\popt - M(\vth)}_2^2,
\end{align*}
adding which gives us the inequality
\[
\ip{\nabla q_{\popt}(M(\vth))}{\popt - M(\vth)} \geq \alpha\cdot\norm{M(\vth) - \popt}_2^2
\]
\noindent\textbf{(Apply Local Strong Smoothness)} Since $M(\vth)$ maximizes the function $q_\vth(\cdot)$ due to the M-step, we get $\nabla q_{\vth}(M(\vth)) = \vzero$. Thus,
\begin{align*}
\ip{\nabla q_{\popt}(M(\vth))}{\popt - M(\vth)} = \ip{\nabla q_{\popt}(M(\vth)) - \nabla q_{\vth}(M(\vth))}{\popt - M(\vth)}\\
\leq \norm{\nabla q_{\popt}(M(\vth)) - \nabla q_{\vth}(M(\vth))}_2\norm{\popt - M(\vth)}_2
\end{align*}
Using Lemma~\ref{lemma:lss-lrc-fos} to invoke the $(r,\beta)$-FOS property further gives us
\[
\ip{\nabla q_{\popt}(M(\vth))}{\popt - M(\vth)} \leq \beta\cdot\norm{\vth-\popt}_2\norm{\popt - M(\vth)}_2
\]
Combining the two results proved above gives us
\[
\alpha\cdot\norm{M(\vth) - \popt}_2^2 \leq \beta\cdot\norm{M(\vth) - \popt}_2\cdot\norm{\vth - \popt}_2,
\]
which finish the proof.
\end{proof}

\subsection{A Note on the Application of Convergence Guarantees}
We conclude the discussion with some comments on the feasibility of the structural assumptions we used to prove the convergence guarantees, in practical settings. Rigorous proofs that these properties are indeed satisfied in practical applications is beyond the scope of this monograph. These can be found in \citep{BalakrishnanWY2017}. Note that for the convergence guarantees (specifically Lemma~\ref{lem:em-local-contract}) to hold, a problem need only satisfy the LSC and FOS properties, with LSS being equivalent to FOS due to Lemma~\ref{lemma:lss-lrc-fos}.\\

\noindent\textbf{Gaussian Mixture Models} To analyze the LSC and FOS properties, we need to look at the population version of the $Q$-function. Given the point-wise $Q$-function derivation of the in \S~\ref{sec:em-app}, we get for $\vM = (\vmu^0,\vmu^1)$
\[
Q(\vM\cond\vM^\ast) = -\frac{1}{2} \bE_{\y \sim f_{\vM^\ast}}\bs{w^0(\y)\cdot\norm{\y - \vmu^0}_2^2 + w^1(\y)\cdot\norm{\y - \vmu^1}_2^2},
\]
where $w^z(\y) = e^{-\frac{\norm{\y - \vmu^{\ast,z}}_2^2}{2}}\bs{e^{-\frac{\norm{\y - \vmu^{\ast,0}}_2^2}{2}} + e^{-\frac{\norm{\y - \vmu^{\ast,1}}_2^2}{2}}}^{-1}$ for $z = 0,1$. It can be seen that the function $q_{\vM^\ast}(\cdot)$ satisfies $\nabla^2 q_{\vM^\ast}(\cdot) \succeq w\cdot I$ where $w = \min\bc{\E{w^0(\y)},\E{w^1(\y)}} > 0$ and hence this problem satisfies the $(\infty,w)$-LSC property i.e., it is globally strongly concave.

Establishing the FOS property is more involved. However, it can be shown that the problem does satisfy the $(r,\alpha)$-FOS property with $r = \Omega(\norm{\vM^\ast}_2)$ and $\alpha = \exp(-\Om{\norm{\vM^\ast}_2^2})$ under suitable conditions.\\

\noindent\textbf{Mixed Regression} We again use the point-wise $Q$-function construction to construct the population $Q$-function. For any $\vW = (\bt^0,\bt^1)$,
\[
Q_{(\x,y)}(\vW|\vW^\ast) = -\frac{1}{2}\E{\alpha^0_{(\x,y)}\cdot(y - \x^\top\bt^0)^2 - \alpha^1_{(\x,y)}\cdot(y - \x^\top\bt^1)^2},
\]
where $\alpha^z_{(\x,y)} = e^{-\frac{(y - \x^\top\bt^{\ast,z})^2}{2}}\bs{e^{-\frac{(y - \x^\top\bt^{\ast,0})^2}{2}} + e^{-\frac{(y - \x^\top\bt^{\ast,1})^2}{2}}}^{-1}$. Assuming $\E{\x\x^\top} = I$ for sake of simplicity, we get $\nabla^2 q_{\vW^\ast}(\cdot) \succeq \alpha\cdot I$ where $\alpha = \min\bc{\lambda_{\min}\br{\E{w^0(\x,y)\cdot\x\x^\top}}, \lambda_{\min}\br{\E{w^1(\x,y)\cdot\x\x^\top}}}$ i.e., the problem satisfies the $(\infty,w)$-LSC property i.e., it is globally strongly concave. Establishing the FOS property is more involved but the problem does satisfy the $(r,\alpha)$-FOS property with $r = \Omega(\norm{\vW^\ast}_2)$ and $\alpha = \Om{1}$ under suitable conditions.\\

\section{Exercises}
\begin{exer}
\label{exer:em-em-k-means}
Show that for Gaussian mixture models with a balanced isotropic mixture, the AM-LVM algorithm (Algorithm~\ref{algo:em-hard-altopt}) implements exactly recovers Lloyd's algorithm for k-means clustering. Note that AM-LVM in this case prescribes setting $w^0_t(\y) = 1$ if $\norm{\y - \vmu^{t,0}}_2 \leq \norm{\y - \vmu^{t,1}}_2$ and $0$ otherwise and also setting $w^1_t(\y) = 1 - w^0_t(\y)$.
\end{exer}
\begin{exer}
\label{exer:em-sgd}
Show that on expectation, the stochastic EM update rule is, on expectation, equivalent to the sample E and the gradient M-step i.e., $\E{M^{\text{sto}}(\vth^t,Q_t)\cond\vth^t} = M^{\text{grad}}(\vth^t,Q^\text{sam}_t(\cdot\cond\vth^t))$.
\end{exer}
\begin{exer}
\label{exer:em-gmm-m}
Derive the fully corrective and gradient M-step in the Gaussian mixture modeling problem. Show that they have closed forms.
\end{exer}
\begin{exer}
\label{exer:em-mr-em}
Derive the E and M-step constructions for the mixed regression problem with fair Bernoulli trials.
\end{exer}
\begin{exer}
\label{exer:em-monotone}
Prove Lemma~\ref{lem:em-monotone-point}.
\end{exer}
\begin{exer}
\label{exer:em-self-consis-q}
Let $\hat\vth$ be a population likelihood maximizer i.e.,
\[
\hat\vth \in \underset{\vth \in \Theta}{\arg\max}\ \Ee{\y \sim f(\y|\popt)}{f(\y\cond\vth)}
\]
Then show that $\hat\vth \in \underset{\vth\in\Theta}{\arg\max}\ q_{\hat\vth}(\vth)$. \textit{Hint}: One way to show this result is to use Theorem~\ref{thm:em-monotone} and Lemma~\ref{lem:em-monotone-point}.
\end{exer}
\begin{exer}
\label{exer:em-ss-lip-grad}
Show that a function $f$ (whether convex or not) that has $L$-Lipschitz gradients is necessarily $L$-strongly smooth. Also show that for any given $L > 0$, there exist functions that do not have $L$-Lipschitz gradients which are also not $L$-strongly smooth.\\
\textit{Hint}: Use the fundamental theorem for calculus for line integrals for the first part. For the second part try using a quadratic function.
\end{exer}
\begin{exer}
\label{exer:em-conv-like}
For any statistical estimation problem with population likelihood maximizer $\popt$ that satisfies the LSC and LSS properties with appropriate constants, show that parameters close to $\popt$ are approximate population likelihood maximizers themselves. Show this by finding constants $\epsilon_0, D > 0$ (that may depend on the LSC, LSS constants) such that for any $\vth$, if $\norm{\vth - \popt}_2 \leq \epsilon < \epsilon_0$, then
\[
\Ee{\y \sim f(\y|\popt)}{f(\y\cond\vth)} \geq \Ee{\y \sim f(\y|\popt)}{f(\y\cond\popt)} - D\cdot\epsilon.
\]
\text{Hint}: Use Lemma~\ref{lem:em-monotone-point} and Exercise~\ref{exer:em-ss-lip-grad}.
\end{exer}
\begin{exer}
\label{exer:em-altmax-monotone}
Similar to how the EM algorithm ensures monotone progress with respect to the likelihood objective (see \S~\ref{sec:em-monotone}), show that the the AM-LVM algorithm (see Algorithm~\ref{algo:em-hard-altopt}) ensures monotone progress with respect to the following optimization problem
\[
\max_{\substack{\vth\in\Theta\\\hat\z_i,\ldots,\hat\z_n\in\cZ}}\ \cL\br{\vth;\bc{(\y_i,\hat\z_i)}_{i=1}^n}.
\]
More specifically, show that if the iterates obtained by AM-LVM at time $t$ are $\vth^t,\bc{\hat\z^t_i}_{i=1}^n$, then we have for all $t \geq 1$
\[
\cL\br{\vth^{t+1};\bc{(\y_i,\hat\z^{t+1}_i)}_{i=1}^n} \geq \cL\br{\vth^t;\bc{(\y_i,\hat\z^t_i)}_{i=1}^n}
\]
\end{exer}

\section{Bibliographic Notes}
\label{sec:em-bib}
The EM algorithm was given its name and formally introduced in the seminal work of \citet{DempsterLR1977}. Although several results exist that attempt to characterize the convergence properties of the EM algorithm, prominent among them is the work of \citet{Wu1983} which showed that for problems with a uni-modal likelihood function and some regularity conditions that make the estimation problem well-posed, the EM algorithm converges to the (unique) global optimum.

We have largely followed the work of \cite{BalakrishnanWY2017} in the treatment of the topic in this section. Recent results have attempted to give local convergence results for the EM algorithm, similar to the ones we discussed. Examples include analyses of the Baum-Welch algorithm used to learn hidden Markov models \citep{YangBW2015}, the EM algorithm \citep{BalakrishnanWY2017}, and the more general problem of M-estimation  \citep{AndresenS2016} (recall that maximum likelihood estimators are members of the broader class of M-estimators). The recurring message here is that upon proper initialization, the EM algorithm does converge to the global optimum.

For certain problems, such as mixed regression, it is known how to get such initializations in polynomial time under regularity assumptions \citep{YiCS2014} which makes the entire procedure a polynomial time solution to get globally consistent parameters. Other recent advances in the area include extensions to high dimensional estimation estimation problems by way of regularization \citep{YiC2015} or directly incorporating sparsity structure into the procedure \citep{WangGNL2015}. We refer the reader to the work of \citet{BalakrishnanWY2017} and other more recent works for a more detailed review of literature in this area.

As a closing note, we reiterate our previous comment on similarities between EM and alternating minimization. Both have an alternating pattern in their execution with each alternation typically being cheap and easy to execute. Indeed, for several areas in which the EM approach has been found successful, such as learning mixture models, regression with missing or corrupted data, phase retrieval etc, there also exist efficient alternating minimization algorithms. We will look at some of them in later sections. Having said that, the EM algorithm does distinguish itself within the general alternating minimization approach in its single minded goal of approaching the MLE by promoting parameters that increase the likelihood of the data.
\chapter{Stochastic Optimization Techniques}
\label{chap:saddle}

In previous sections, we have looked at specific instances of optimization problems with non-convex objective functions. In \S~\ref{chap:altmin} we looked at problems in which the objective can be expressed as a function of two variables, whereas in \S~\ref{chap:em} we looked at objective functions with latent variables that arose in probabilistic settings. In this section, we will look at the problem of optimization with non-convex objectives in a more general setting.

Several machine learning and signal processing applications such as deep learning, topic modeling etc, generate optimization problems that have non-convex objective functions. The global optimization of non-convex objectives, i.e., finding the global optimum of the objective function, is an NP hard problem in general. Even the seemingly simple problem of minimizing quadratic functions of the kind $\x^\top A\x$ over convex constraint sets becomes NP-hard the moment the matrix $A$ is allowed to have even one negative eigenvalue.

As a result, a much sought after goal in applications with non-convex objectives is to find a local minimum of the objective function. The main hurdle in achieving local optimality is the presence of saddle points which can mislead optimization methods such as gradient descent by stalling their progress. Saddle points are best avoided as they signal inflection in the objective surface and unlike local optima, they need not optimize the objective function in any meaningful way.

The recent years have seen much interest in this problem, particularly with the advent of deep learning where folklore wisdom tells us that in the presence of sufficient data, even locally optimal solutions to the problem of learning the edge weights of a network perform quite well \citep{ChoromanskaHMALC2015}. In these settings, techniques such as convex relaxations, and non-convex optimization techniques that we have studied such as EM, gAM, gPGD, do not apply directly. In these settings, one has to attempt optimizing a non-convex objective directly.

The problem of avoiding or escaping saddle points is actually quite challenging in itself given the wide variety of configurations saddle points can appear in, especially in high dimensional problems. It should be noted that there exist saddle configurations, bypassing which is intractable in itself. For such cases, even finding locally optimal solutions is an NP-hard problem \citep{AnandkumarG2016}.

In our discussion, we will look recent results which show that if the function being optimized possesses certain nice structural properties, then an application of very intuitive algorithmic techniques can guarantee local optimality of the solutions. This will be yet another instance of a result where the presence of a structural property (such as RSC/RSS, MSC/MSS, or LSC/LSS as studied in previous sections) makes the problem well behaved and allows efficient algorithms to offer provably good solutions. Our discussion will largely aligned to the work of \cite{GeHJY2015}. The bibliographic notes will point to other works.

\section{Motivating Applications}
\label{sec:saddle-app}
A wide range of problems in machine learning and signal processing generate optimization problems with non-convex objectives. Of particular interest to us would be the problem of \emph{Orthogonal Tensor Decomposition}. This problem has been shown to be especially useful in modeling a large variety of learning problems including training deep Recurrent Neural Networks \citep{SedghiA2016}, Topic Modeling, learning Gaussian Mixture Models and Hidden Markov Models, Independent Component Analysis \citep{AnandkumarGHKT2014}, and reinforcement learning \citep{AzizzadenesheliLA2016}.

The details of how these machine learning problems can be reduced to tensor decomposition problems will involve getting into the details which will distract us from our main objective. To keep the discussion focused and brief, we request the reader to refer to these papers for the reductions to tensor decomposition. We ourselves will be most interested in the problem of tensor decomposition itself.

We will restrict our study to $4\th$-order tensors which can be interpreted as $4$-dimensional arrays. Tensors are easily constructed using \emph{outer products}, also known as \emph{tensor products}. An outer product of $2\nd$ order produces a $2\nd$ order tensor which is nothing but a matrix. For any $\vu, \vv \in \R^p$, their outer product is defined as $\vu \otimes \vv := \vu\vv^\top \in \R^{p\times p}$ which is a $p \times p$ matrix, whose $(i,j)$-th entry is $\vu_i\vv_j$.

We can similarly construct a $4\th$-order tensors. For any $\vu,\vv,\vw,\vx \in \bR^p$, let $T = \vu\otimes\vv\otimes\vw\otimes\vx \in \R^{p\times p\times p\times p}$. The $(i,j,k,l)$-th entry of this tensor, for any $i,j,k,l \in [p]$, will be $T_{i,j,k,l} = \vu_i\cdot\vv_j\cdot\vw_k\cdot\vx_l$. The set of $4\th$-order tensors is closed under addition and scalar multiplication. We will study a special class of $4\th$-order tensors known as \emph{orthonormal} tensors which have an \emph{orthonormal decomposition} as follows
\[
T = \sum_{i=1}^r \vu_i \otimes \vu_i \otimes \vu_i\otimes \vu_i,
\]
where the vectors $\vu_i$ are orthonormal \emph{components} of the tensor $T$ i.e., $\vu_i^\top\vu_j = 0$ if $i \neq j$ and $\norm{\vu_i}_2 = 1$. The above tensor is said to have rank $r$ since it has $r$ components in its decomposition. If an orthonormal decomposition of a tensor exists, it can be shown to be unique.

Just as a matrix $A \in \R^{p \times p}$ defines a bi-linear form $A: (\vx,\vy) \mapsto \vx^\top A \vy$, similarly a tensor defines a \emph{multi-linear form}. For orthonormal tensors, the multilinear form has a simple expression. In particular, if $T$ has the orthonormal form described above, we have
\begin{equation*}
\begin{array}{cl}
T(\vv,\vv,\vv,\vv) &= \sum_{i=1}^r (\vu_i^\top\vv)^4 \in \R\\
T(I,\vv,\vv,\vv) &= \sum_{i=1}^r (\vu_i^\top\vv)^3\cdot\vu_i \in \R^p\\
T(I,I,\vv,\vv) &= \sum_{i=1}^r (\vu_i^\top\vv)^2\cdot\vu_i\vu_i^\top \in \R^{p \times p}\\
T(I,I,I,\vv) &= \sum_{i=1}^r (\vu_i^\top\vv)\cdot(\vu_i\otimes\vu_i\otimes\vu_i) \in \R^{p \times p \times p}
\end{array}
\end{equation*}
The problem of orthonormal tensor decomposition involves recovering all $r$ components of a rank-$r$, 4-th order tensor $T$. An intuitive way to do this is to iteratively find all the components.

First we recover, say w.l.o.g., the first component $\vu_1$. Then we perform a \emph{peeling/deflation} step by subtracting that component and creating a new tensor $T^{(1)} = T - \vu_1 \otimes \vu_1 \otimes \vu_1\otimes \vu_1$ and repeating the process. Note that the rank of the tensor $T^{(1)}$ is only $r-1$ and thus, this procedure terminates in just $r$ steps.

To execute the above algorithm, all we are required is to solve the individual steps of recovering a single component. This requires us to solve the following optimization problem.

\begin{equation}
\begin{array}{cl}
	\max & T(\vu,\vu,\vu,\vu) = \sum_{i=1}^r (\vu_i^\top\vu)^4\\
	\text{s.t.} & \norm{\vu}_2 = 1,
\end{array}
\tag*{(LRTD)}\label{eq:lrtd}
\end{equation}

This is the non-convex optimization problem\elink{exer:saddle-non-conv} that we will explore in more detail. We will revisit this problem later after looking at some techniques to optimize non-convex objectives.

\section{Saddles and why they Proliferate}
To better understand the challenges posed by saddle points, let us take good old gradient descent as an optimization algorithm. As we studied in \S~\ref{chap:tools}, the procedure involves repeatedly taking steps away from the direction of the gradient.
\[
\x^{t+1} = \x^t - \eta_t\cdot\nabla f(\x^t)
\]
Now, it can be shown\elink{exer:saddle-smooth-nc}\elink{exer:saddle-smooth-nc-2}, that the procedure is guaranteed to make progress at \emph{every} time step, provided the function $f$ is strongly smooth and the step length is small enough. However, the procedure stalls at \emph{stationary points} where the gradient of the function vanishes i.e., $\nabla f(\x) = \vzero$. This includes local optima, which are of interest, and saddle points, which simply stall descent algorithms.

One way to distinguish saddle points from local optima is by using the \emph{second derivative test}. The Hessian of a doubly differentiable function has only positive eigenvalues at local minima and only negative ones at local maxima. Saddles on the other hand are unpredictable. The \emph{simple saddles} which we shall study here, reveal themselves by having both positive and negative eigenvalues in the Hessian. The bibliographic notes discuss more complex saddle structures.

\begin{figure}
\includegraphics[width=0.5\columnwidth]{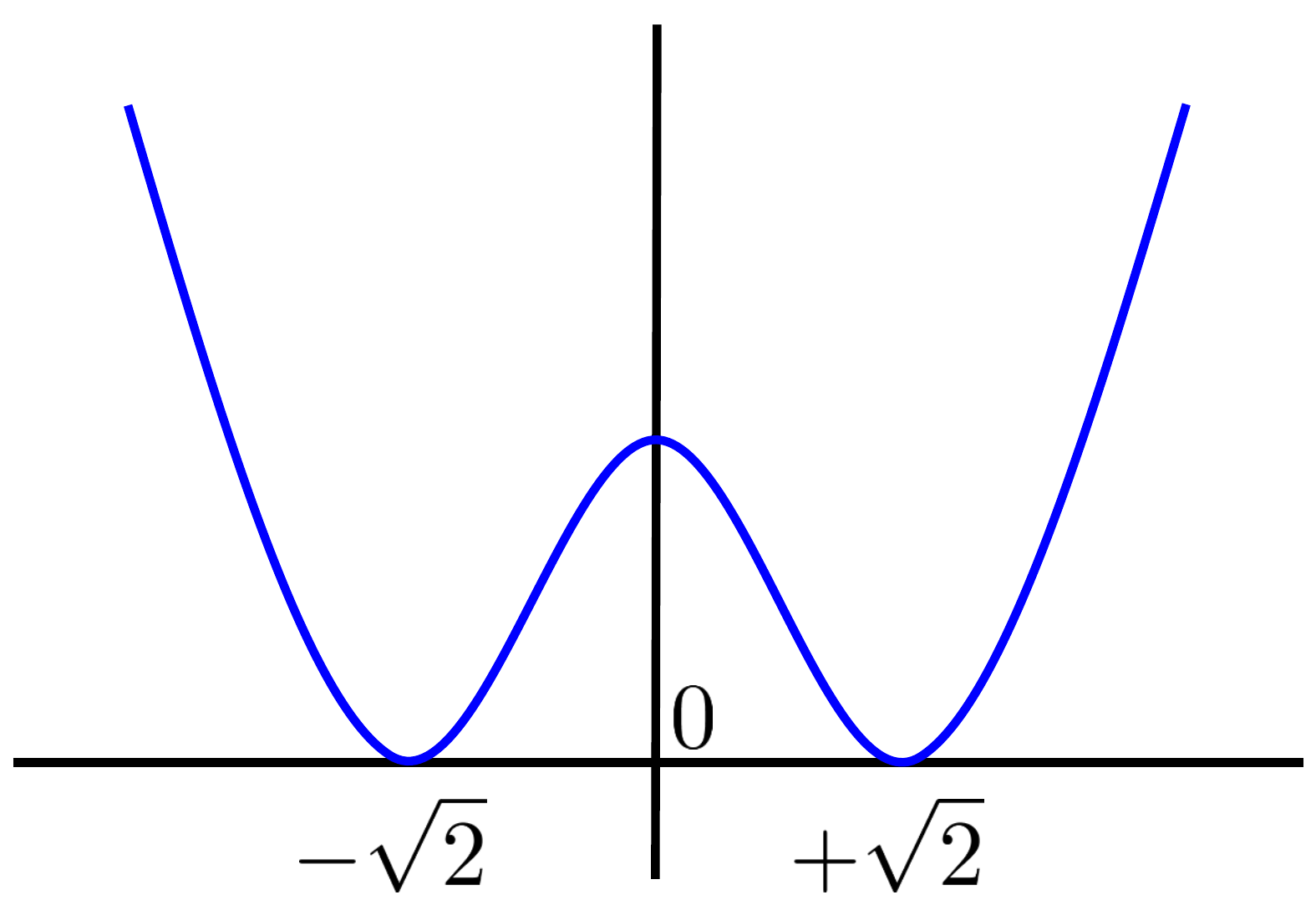}
\includegraphics[width=0.5\columnwidth]{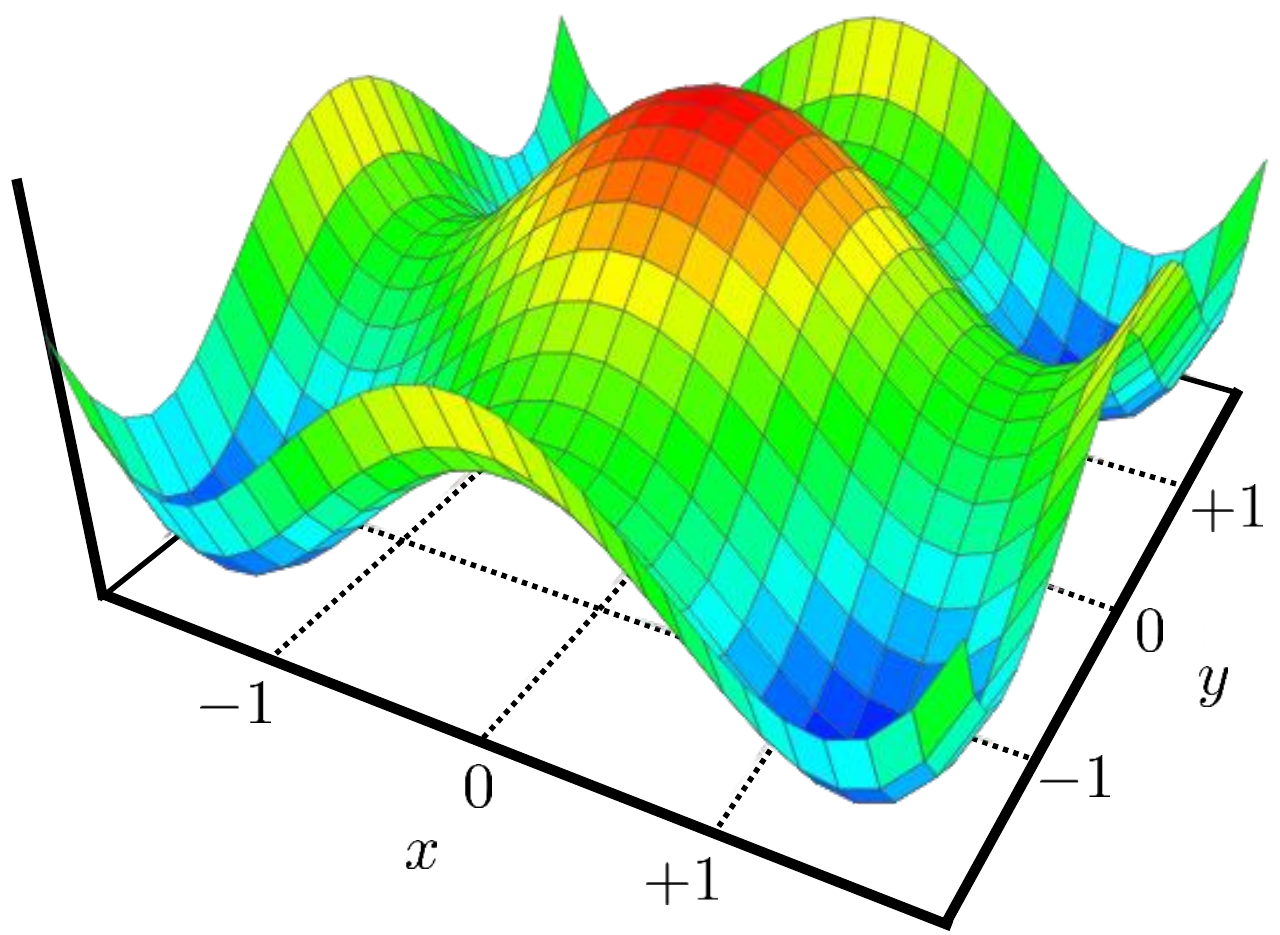}
\caption[Emergence of Saddle Points]{The function on the left $f(x) = x^4 - 4\cdot x^2 + 4$ has two global optima $\bc{-\sqrt 2, \sqrt 2}$ separated by a local maxima at $0$. Using this function, we construct on the right, a higher dimensional function $g(x,y) = f(x) + f(y) +8$ which now has $4$ global minima separated by $4$ saddle points. The number of such minima and saddle points can explode exponentially in learning problems with symmetry (indeed $g(x,y,z) = f(x) + f(y) + f(z) + 12$ has $8$ local minima and saddle points). Plot on the right courtesy \url{academo.org}}
\label{fig:saddle}
\end{figure}

The reasons for the origin of saddle points is quite intriguing too. Figure~\ref{fig:saddle} shows how saddles may emerge and their numbers increase exponentially with increasing dimensionality. Consider the tensor decomposition problem in \eqref{eq:lrtd}. It can be shown\elink{exer:saddle-multi-opt} that all the $r$ components are optimal solutions to this problem. Thus, the problem possesses a beautiful symmetry which allows us to recover the components in any order we like. However, it is also easy to show\elink{exer:saddle-multi-non-opt} that general convex combinations of the components are not optimal solutions to this problem. Thus, we automatically obtain $r$ isolated optima spread out in space, interspersed with saddle points.

The applications we discussed, such as Gaussian mixture models, also have such an internal symmetry -- the optimum is unique only up to permutation. Indeed, it does not matter in which order do we recover the components of a mixture model, so long as we recover all of them. However, this very symmetry gives rise to saddle points \citep{GeHJY2015}, since taking two permutations of the optimal solution and taking a convex combination of them is in general not an optimal solution as well. This gives us, in general, an exponential number of optima, separated by (exponentially many) saddle points.

Before moving forward, we remind the reader that techniques we have studied so far for non-convex optimization, namely EM, gAM, and gPGD are far too specific to be applied to non-convex objectives in general, and to the problems we encounter with tensor decomposition in particular. We need more generic solutions for the task of local optimization of non-convex objectives.

\section{The Strict Saddle Property}
\label{sec:saddle-fsp}
Given that the Hessian of the function is the first point of inquiry when trying to distinguish saddle points from other stationary points, it seems natural to use second order methods to escape points. Indeed, the bibliographic notes discuss several such approaches that use the Hessian, or estimates thereof, to escape saddle points. However, these are expensive and do not scale to high dimensional problems. Consequently, we would be more interested in scalable first order methods.

Given the above considerations, a natural course of action is to identify properties that an objective function should satisfy in order to allow gradient descent-style techniques to escape its saddle points. The recent works of \cite{GeHJY2015,SunQW2015} give an intuitive answer to this question. They observe that if a saddle point $\x$ for a function $f$ contains directions of steep descent, then it is possible, at least in principle, for a gradient descent procedure to discover this direction and ``fall'' along it. The existence of such directions makes a saddle unstable -- it behaves like a local maxima along these directions and a slight perturbation is very likely to cause gradient descent to roll down the function surface. We notice that this will indeed be the case if for some direction $\vu \in \R^p$, we have $\vu^\top \nabla^2f(\x)\vu \ll 0$. Figure~\ref{fig:strict-saddle} depicts a toy case with the function $f(x,y) = x^2-y^2$ which exhibits a saddle at the origin but nevertheless, also presents a direction of steep descent.

\begin{figure}
\centering
\includegraphics[width=0.5\columnwidth]{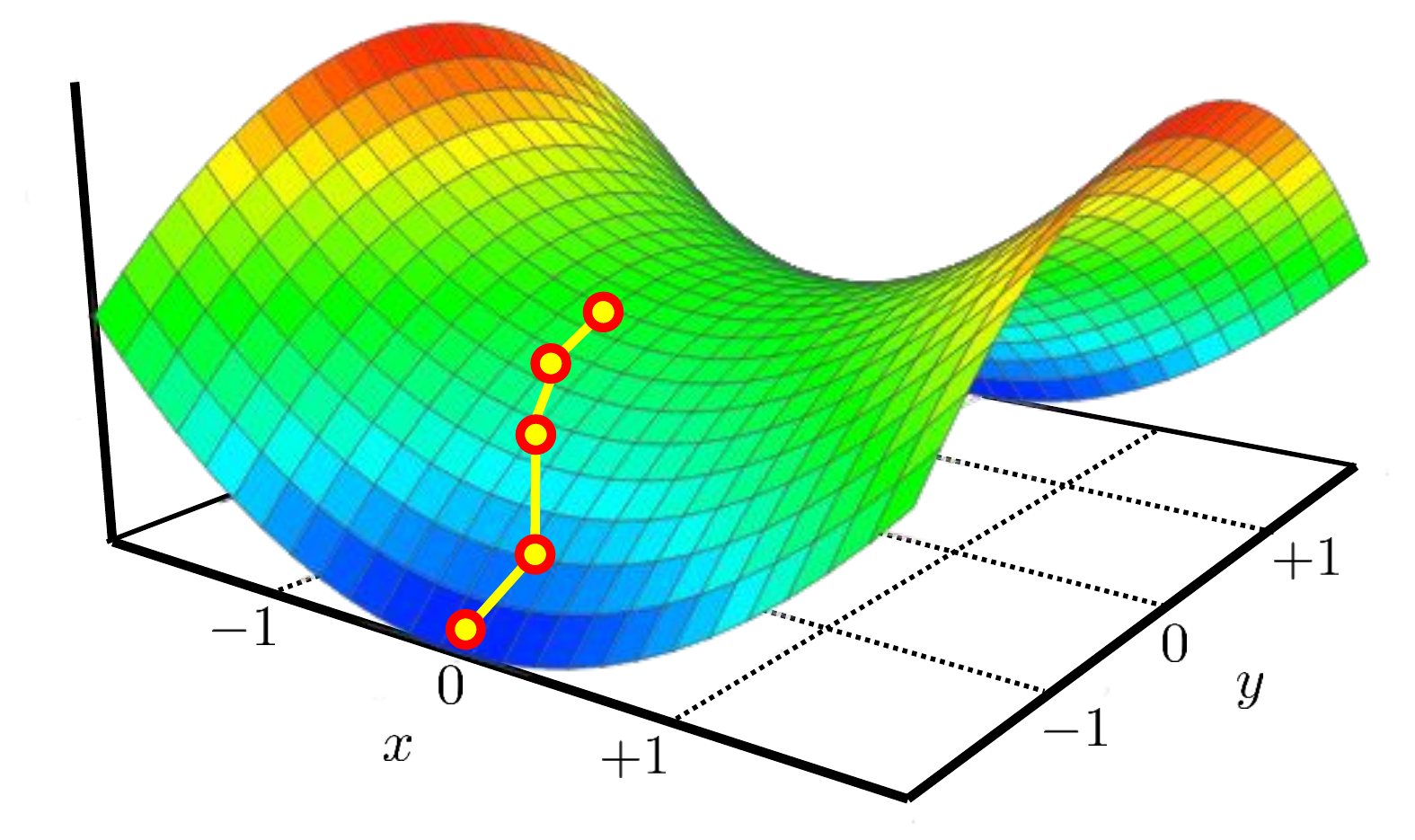}
\caption[The Strict Saddle Property]{The function $f(x,y) = x^2 - y^2$ exhibits a saddle at the origin $(0,0)$. The Hessian of this function at the origin is $-2\cdot I$ and since $\lambda_{\min}(\nabla^2f((0,0))) = -2$, the saddle satisfies the strict-saddle property. Indeed, the saddle does offer a prominent descent path along the $y$ axis which can be used to escape the saddle point. In \S~\ref{sec:saddle_analysis} we will see how the NGD algorithm is able to provably escape this saddle point. Plot courtesy \url{academo.org}}
\label{fig:strict-saddle}
\end{figure}

Consider the following unconstrained optimization problem.
\begin{equation}
	\min_{\x\in\R^p} f(\x)
\tag*{(NOPT)}\label{eq:nopt}
\end{equation}
The following \emph{strict saddle property} formalizes the requirements we have discussed in a more robust manner.
 
\begin{definition}[Strict Saddle Property \citep{GeHJY2015}]
\label{defn:ssp}
A twice differentiable function $f(\x)$ is said to satisfy the $(\alpha, \gamma, \kappa, \xi)$-strict saddle (SSa) property, if for every local minimum $\xopt$ of the function, the function is $\alpha$-strongly convex in the region $\cB_2(\xopt,2\xi)$ and moreover, every point $\x_0 \in \R^p$ satisfies at least one of the following properties: 
\begin{enumerate}
	\item (Non-stationary Point) $\|\nabla f(\x_0)\|_2 \geq \kappa$
	\item	(Strict Saddle Point) $\lambda_{\min}(\nabla^2f(\x_0))\leq -\gamma$
	\item	(Approx. Local Min.) For some local minimum $\xopt$, $\|\x_0-\xopt\|_2 \leq \xi$.
\end{enumerate}
\end{definition}

The above property places quite a few restrictions on the function. The function must be strongly-convex in the neighborhood of every local optima and every point that is not an approximate local minimum, must offer a direction of steep descent. This the point may do by having a steep gradient (case 1) or else (if the point is a saddle point) have its Hessian offer an eigenvector with a large negative eigenvalue which then offers a steep descent direction (case 2). We shall later see that there exist interesting applications that do satisfy this property.

\section{The Noisy Gradient Descent Algorithm}
\label{sec:saddle-algo}
Given that the strict-saddle property assures us of the existence of a steep descent direction until we reach close to an approximate-local minimum, it should come as no surprise that simple techniques should be able to achieve local optimality on functions that are strict saddle. We will now look at a one such simple approach to exploit the strict saddle property and achieve local optimality.

The idea behind the approach is very simple: at every saddle point which may stall gradient descent, the SSa property ensures that there exists a direction of steep descent. If we perturb the gradient, there is a chance it will point in the general direction of the steep descent and escape the saddle. However, if we are at a local minimum or a non-stationary point, then this perturbation would not affect us much.

Algorithm~\ref{algo:saddle-ngd} gives the details of this approach. At each time step, it perturbs the gradient using a unit vector pointing at a random direction in the hope that if we are currently stuck at a saddle point, the perturbation will cause us to discover the steep (enough) descent direction, allowing us to continue making progress. One easy way to obtain a random point on the unit sphere is to sample a random standard Gaussian vector $\vw \sim \cN(\vzero,I_{p \times p})$ and normalize it $\vzeta = \frac{1}{\norm{\vw}_2}\cdot\vw$. Notice that since the perturbation $\vzeta^t$ is chosen uniformly from the unit sphere, we have $\E{\vzeta^t} = \vzero$ and by linearity of expectation, $\E{\vg^t\cond\x^t} = \nabla f(\x^t)$.

Such an unbiased estimate $\vg^t$ is often called a \emph{stochastic gradient} and is widely used in machine learning and optimization. Recall that even the EM algorithm studied in \S~\ref{chap:em} had a variant that used stochastic gradients. In several machine learning applications, the objective can be written as a finite sum $f(\x) = \frac{1}{n}\sum_{i=1}^n f(\x;\vtheta^i)$ where $\vtheta^i$ may denote the $i$-th data point. This allows us to construct a stochastic gradient estimate in an even more inexpensive manner. At each time step, simply sample a data point $I_t \sim \text{\textsf{Unif}}([n])$ and let
\[
\vg^t = \nabla f(\x^t,\vtheta^{I_t}) + \vzeta^t
\]
Note that we still have $\E{\vg^t\cond\x^t} = \nabla f(\x^t)$ but with a much cheaper construction for $\vg^t$. However, in order to simplify the discussion, we will continue to work with the setting $\vg^t = \nabla f(\x^t) + \vzeta^t$. 

We note that we have set the step lengths to be around $1/\sqrt{T}$, where $T$ is the total number of iterations we are going to execute the NGD algorithm. This will seem similar to what we used in the projected gradient descent approach (see Algorithm~\ref{algo:pgd} and Theorem~\ref{thm:pgd-conv-proof}). Although in practice one may set the step length to $\eta_t \approx 1/\sqrt{t}$ here as well, the analysis becomes more involved.

\begin{algorithm}[t]
	\caption{Noisy Gradient Descent (NGD)}
	\label{algo:saddle-ngd}
	\begin{algorithmic}[1]
		{
			\REQUIRE Objective $f$, max step length $\eta_{\max}$, tolerance $\epsilon$
			\ENSURE A locally optimal point $\hat\x \in \R^p$
			\STATE $\x^1 \leftarrow \text{\textsf{INITALIZE}}()$
			\STATE Set $T \leftarrow 1/\eta^2$, where $\eta = \min\bc{\epsilon^2/\log^2(1/\epsilon), \eta_{\max}}$
			\FOR{$t= 1, 2, \ldots, T$}
				\STATE Sample perturbation $\vzeta^t \sim S^{p-1}$\hfill//\texttt{Random pt. on unit sphere}%
				\STATE $\vg^t \leftarrow \nabla f(\x^t) + \vzeta^t$
				\STATE $x^{t+1} \leftarrow \x^t - \eta\cdot\vg^t$
			\ENDFOR
			\STATE \textbf{return} $\x^T$
		}
	\end{algorithmic}
\end{algorithm}

\section{A Local Convergence Guarantee for NGD}
\label{sec:saddle_analysis}
We will now analyze the NGD algorithm for its convergence properties. We note that the complete proof requires results that are quite technical and beyond the scope of this monograph. We will instead, present the essential elements of the proof and point the curious reader to \citep{GeHJY2015} for the complete analysis. To simplify the notation, we will often omit specifying the exact constants as well.

In the following, we will assume that the function $f$ satisfies the strict saddle property and is $\beta$-strongly smooth (see Definition~\ref{defn:strong-cvx-smooth-fn}). We will also assume that the function is bounded i.e., $\abs{f(\vx)} \leq B$ for all $\vx \in \bR^p$ and has $\rho$-Lipschitz Hessians i.e., for every $\x,\y \in \R^p$, we have $\|\nabla^2_\bt f(\x)-\nabla^2_\bt f(\y)\|_2 \leq \rho\cdot\|\x-\y\|_2$ where $\norm{\cdot}_2$ for matrices denotes the spectral/operator norm of the matrix.

Before commencing with the actual proof, we first present an overview of the proof. The definition of the SSa property clearly demarcates three different regimes:
\begin{enumerate}
	\item Non-stationary points, i.e., points where gradient is ``large'' enough: in this case, standard (stochastic) gradient descent is powerful enough to ensure a large enough decrease in the objective function value in a single step\elink{exer:saddle-smooth-nc}
	\item Saddle points, i.e., points where the gradient is close to $\vzero$. Here, the SSa property ensures that at least one highly `negative'' Hessian direction exists: in this case, traditional (stochastic) gradient descent may fail but the additional noise ensures an escape from the saddle point with high probability
	\item Local minima, i.e., points where gradient is close to $\vzero$ but the have a positive semi definite Hessian due to strong convexity: in this case, standard (stochastic) gradient descent by itself would converge to the corresponding local minima
\end{enumerate}

The above three regimes will be formally studied below in three separate lemmata. Note that the analysis for non-stationary points as well as for points near local minima is similar to the standard stochastic gradient descent analysis for convex functions. However, the analysis for saddle points is quite interesting and shows that the added random noise ensures an escape from the saddle point.

To further understand the inner workings of the NGD algorithm, let us perform a warm-up exercise by showing that the NGD algorithm will, with high probability, escape the saddle point in the function $f(x,y) = x^2-y^2$ that we considered in Figure~\ref{fig:strict-saddle}.

\begin{theorem}
\label{thm:saddle-toy}
Consider the function $f(x,y) = x^2 - y^2$ on two variables. If initialized at the saddle point $(0,0)$ of this function with step length $\eta < 1$, with high probability, we have $f(x^t,y^t) \rightarrow -\infty$ as $t \rightarrow \infty$.
\end{theorem}

\begin{proof}
For an illustration, see Figure~\ref{fig:strict-saddle}. Note that the function $f(x,y) = x^2 - y^2$ has trivial minima at the limiting points $(0,\pm\infty)$ where the function value approaches $-\infty$. Thus, the statement of the theorem claims that NGD approaches the ``minimum'' function value.

In any case, we are interested in showing that NGD escapes the saddle point $(0,0)$.  The gradient of $f$ is $\vzero$ at the origin $(0,0)$. Thus, if a gradient descent procedure is initialized at the origin for this function, it will remain stuck there forever making no non-trivial updates.

The NGD algorithm on the other hand, when initialized at the saddle point $(0,0)$, after $t$ iterations, can be shown to reach the point $(x^t,y^t)$ where $x^t = \sum_{\tau=0}^{t-1} (1-\eta)^{t-\tau-1}\zeta_1^\tau$ and $y^t=\sum_{\tau=0}^{t-1} (1+\eta)^{t-\tau-1}\zeta_2^\tau$ and $(\zeta^\tau_1,\zeta^\tau_2) \in \R^2$ is the noise vector added to the gradient at each step. Since $\eta < 1$, as $t \rightarrow \infty$, it is easy to see that with high probability, we have $x^t\rightarrow 0$ while $|y^t|\rightarrow \infty$ which indicates a successful escape from the saddle point, as well as progress towards the global optima. 
\end{proof} 

We now formalize the intuitions developed above. The following lemma shows that even if we are at a saddle point, NGD will still ensure a large drop in function value in not too many steps. The proof of this result is a bit subtle and we will just provide a sketch.

\begin{lemma}
	\label{lem:ngd-saddle}
	If NGD is executed on a function $f: \bR^p \rightarrow \bR$ that is $\beta$-strongly smooth, satisfies the $(\alpha, \gamma, \kappa, \xi)$-SSa property and has $\rho$-Lipschitz Hessians, with step length $\eta \leq 1/(\beta+\rho)^2$, then if an iterate $\x^t$ satisfies $\norm{\nabla f(\x^t)}_2 \leq \sqrt{\eta\beta}$ and $\lambda_{\min}\nabla^2f(\x^t) \leq -\gamma$, then NGD ensures that after at most $s \leq \frac{3\log p}{\eta\gamma}$ steps,
	\[
	\E{f(\x^{t+s})\cond\x^t} \leq f(\x^t) - \eta/4
	\]
\end{lemma}
\begin{proof}
	For the sake of notational simplicity, let $t = 0$. The overall idea of the proof is to rely on that one large negative eigendirection in the Hessian to induce a drop in function value. The hope is that random fluctuations will eventually nudge NGD in the steep descent direction and upon discovery, the larger and larger gradient values will accumulate to let the NGD procedure escape the saddle.
	
	Since the effects of the Hessian are most apparent in the second order Taylor expansion, we will consider the following function
	\[
	\hat f(\x) = f(\x^0) + \ip{\nabla f(\x^0)}{\x - \x^0} + \frac{1}{2}(\x - \x^0)^\top\cH(\x - \x^0),
	\]
	where $\cH = \nabla^2 f(\x^0) \in \bR^{p\times p}$. The proof will proceed by first imagining that NGD was executed on the function $\hat f(\cdot)$ instead, showing that the function value indeed drops, and then finishing off by showing that things do not change too much if NGD is executed on the function $f(\cdot)$ instead. Note that to make this claim, we will need Hessians to vary smoothly which is why we assumed the Lipschitz Hessian condition. This seems to be a requirement for several follow-up results as well \citep{Jin0NKJ17, AgarwalA-ZBHM2017}. An interesting and challenging open problem is to obtain similar results for non-convex optimization without the Lipschitz-Hessian assumption. 
	
We will let $\x^t, t \geq 1$ denote the iterates of NGD when executed on $f(\cdot)$ and $\hat\x^t$ denote the iterates of NGD when executed on $\hat f(\cdot)$. We will fix $\hat\x^0 = \x^0$. Using some careful calculations, we can get
	\[
	\hat\x^t - \hat\x^0 = -\eta\sum_{\tau=0}^{t-1}(I-\eta\cH)^\tau\nabla f(\x^0) - \eta\sum_{\tau=0}^{t-1}(I-\eta\cH)^{t-\tau-1}\vzeta^\tau
	\]
We note that both terms in the right hand expression above correspond to ``small'' vectors. The first term is small as we know that $\hat\x^0$ satisfies $\norm{\nabla f(\hat\x^0)}_2 \leq \sqrt{\eta\beta}$ by virtue of being close to a stationary point as is assumed in the statement of this result. The second term is small as $\vzeta^\tau$ are random unit vectors with expectation $\vzero$. Using these intuitions, we will first show that $\hat f(\hat\x^t)$ is significantly smaller than $\hat f(\x^0) = f(\x^0)$ after sufficiently many steps. Then using the property of Lipschitz Hessians, we will obtain obtain a descent guarantee for $f(x^t)$. 
	
Now, notice that NGD chooses the noise vectors $\vzeta^\tau$ for any $\tau \geq 0$, independently of $\x^0$ and $\cH$. Moreover, for any two $\tau \neq \tau'$, the vectors $\vzeta^\tau,\vzeta^{\tau'}$ are also chosen independently. We also know that the noise vectors are isotropic i.e., $\E{\vzeta^\tau (\vzeta^\tau)^\top}=I_p$. These observations and some straightforward calculations give us the following upper bound on the suboptimality of $\x^T$ with respect to the Taylor approximation $\hat f$. As we have noted, this upper bound can be converted to an upper bound on the suboptimality of $\x^T$ with respect to the actual function $f$ using some more effort.
\begin{align*}
\bE[\hat f(\hat \x^T)-f(\hat\x^0)&]\\
={}& -\eta \nabla f(\x^0)^\top \sum_{\tau=0}^{t-1}(I_p-\eta \cH)^{\tau}\nabla f(\x^0)\\
&{}+\frac{\eta^2}{2}\nabla f(\x^0)^\top \sum_{\tau=0}^{t-1}(I_p-\eta \cH)^{\tau} H \sum_{\tau=0}^{t-1}(I_p-\eta \cH)^{\tau}\nabla f(\x^0)\\
&{}+\frac{\eta^2}{2}\text{tr}\left(\sum_{\tau=0}^{t-1}(I_p-\eta \cH)^{2\tau}\cH\right)\\
={}& -\eta \nabla f(\x^0)^\top B \nabla f(\x^0) + \frac{\eta^2}{2}\text{tr}\left(\sum_{\tau=0}^{t-1}(I_p-\eta \cH)^{2\tau}\cH\right), \label{eq:ngd1}
\end{align*}
where $\text{tr}(\cdot)$ is the trace operator and $B=\sum_{\tau=0}^{t-1}(I_p-\eta \cH)^{\tau}-\frac{\eta}{2}  \sum_{\tau=0}^{t-1}(I_p-\eta \cH)^{\tau} \cH \sum_{\tau=0}^{t-1}(I_p-\eta \cH)^{\tau}$. It is easy to verify that $B \succeq 0$ for all step lengths $\eta\leq \frac{1}{\|\cH\|_2}$ i.e., all $\eta \leq \frac{1}{\beta}$. 
	
For any value of $t \geq \frac{3\log p}{\eta\gamma}$ (which is the setting for the parameter $s$ in the statement of the theorem), the second term in the final expression above can be simplified to give us
\[
\frac{\eta^2}{2}\text{tr}\left(\sum_{\tau=0}^{t-1}(I-\eta \cH)^{2\tau}\cH\right)=\sum_{i=1}^p \lambda_i (\sum_{\tau=0}^{t-1} (1-\eta \lambda_i)^\tau)\leq -\frac{\eta}{2},
\]
where $\lambda_i$ is the $i$-th eigenvalue of $\cH$, by using $\lambda_p \leq -\gamma$. This gives us
\[
\E{\hat f(\hat \x^T)-f(\x^0)}\leq -\frac{\eta}{2}.
\]
Note that the above equation only shows descent for $\hat f(\hat x^T)$. One can now show \cite[Lemma 19]{GeHJY2015} that the iterates obtained by NGD on $\hat f(\cdot)$ do not deviate too far from those obtained on the actual function $f(\cdot)$ using the Lipschitz-Hessian property. The proof is concluded by combining these two results. 
\end{proof}

\begin{lemma}
\label{lem:ngd-non-stationary}
If NGD is executed on a function that is $\beta$-strongly smooth and satisfies the $(\alpha, \gamma, \kappa, \xi)$-SSa property, with step length $\eta \leq \frac{1}{\beta}\cdot\min\bc{1,\kappa^2}$, then for any iterate $\x^t$ that satisfies $\norm{\nabla f(\x^t)}_2 \geq \sqrt{\eta\beta}$, NGD ensures that
\[
\E{f(\x^{t+1})\sep\x^t} \leq f(\x^t) - \frac{\beta}{2}\cdot\eta^2.
\]
\end{lemma}
\begin{proof}
This is the most carefree case as we are neither close to any local optima, nor a saddle point. Unsurprisingly, the proof of this lemma follows from standard arguments as we have assumed the function to be strongly smooth. Since $\x^{t+1} = \x^t - \eta\cdot(\nabla f(\x^t) + \vzeta^t)$, we have, by an application of the strong smoothness property (see Definition~\ref{defn:strong-cvx-smooth-fn})
\[
f(\x^{t+1}) \leq f(\x^t) - \ip{\nabla f(\x^t)}{\eta\cdot(\nabla f(\x^t) + \vzeta^t)} + \frac{\beta\eta^2}{2}\cdot\norm{\nabla f(\x^t) + \vzeta^t}_2^2.
\]
Using the facts $\norm{\vzeta^t}_2 = 1$, $\E{\vzeta^t\sep\x^t} = \vzero$, we get
\[
\E{f(\x^{t+1})\sep\x^t} \leq f(\x^t) - \eta\br{1 - \frac{\beta\eta}{2}}\norm{\nabla f(\x^t)}_2^2 + \frac{\beta\eta^2}{2},
\]
Since we have $\eta \leq \frac{1}{\beta}\cdot\min\bc{1,\kappa^2}$ by assumption and $\norm{f(\x^t)}_2 \geq \sqrt{\eta\beta}$, we get
\[
\E{f(\x^{t+1}\sep\x^t} \leq f(\x^t) - \frac{\beta}{2}\cdot\eta^2,
\]
which proves the result.
\end{proof}

The final intermediate result is an \emph{entrapment} lemma. It shows that once NGD gets sufficiently close to a local optimum, it gets trapped there for a really long time. Although the function $f$ satisfies strong convexity and smoothness properties in the neighborhood $\cB_2(\xopt,2\xi)$, the proof of this result is still non-trivial due to the perturbations $\vzeta^t$. Had the perturbations not been there, we could have utilized the analysis of the PGD algorithm\elink{exer:saddle-smooth-nc-2} to show that we would converge to the local optimum $\xopt$ at a linear rate.

The problem is that the perturbations do not diminish -- we always have $\norm{\vzeta^t}_2 = 1$. This prevents us from ever converging to the local optima. Moreover, a sequence of unfortunate perturbations may have us kicked out of this nice neighborhood. The next result shows that we will not get kicked out of the neighborhood of $\xopt$ for a really long time.

\begin{lemma}
\label{lem:ngd-local-nbhd}
If NGD is executed on a function that is $\beta$-strongly smooth and satisfies the $(\alpha, \gamma, \kappa, \xi)$-SSa property, with step length $\eta \leq \min\bc{\frac{\alpha}{\beta^2}, \xi^2\log^{-1}(\frac{1}{\delta\xi})}$ for some $\delta > 0$, then if some iterate $\x^t$ satisfies $\norm{\x^t - \xopt}_2 \leq \xi$, then NGD ensures that with probability at least $1 - \delta$, for all $s \in \bs{t,t+\frac{1}{\eta^2}\log\frac{1}{\delta}}$, we have
\[
\norm{\x^s - \xopt}_2 \leq \sqrt{\eta\log\frac{1}{\eta\delta}} \leq \xi.
\]
\end{lemma}
\begin{proof}
Using strong convexity in the neighborhood of $\xopt$ and the fact that $\xopt$, being a local minimum, satisfies $\nabla f(\xopt) = \vzero$, gives us
\begin{align*}
f(\vx^t) &\geq f(\xopt) + \frac{\alpha}{2}\norm{\x^t - \xopt}_2^2\\
f(\xopt) &\geq f(\x^t) + \ip{\nabla f(\x^t)}{\xopt - \x^t} + \frac{\alpha}{2}\norm{\x^t - \xopt}_2^2.
\end{align*}
Together, the above two expressions give us
\[
\ip{\nabla f(\x^t)}{\x^t - \xopt} \geq \alpha\cdot\norm{\x^t - \xopt}_2^2.
\]
Since $f$ is $\beta$-smooth, using the co-coercivity of the gradient for smooth convex functions (recall that $f$ is strongly convex, in this neighborhood of $\xopt$), we conclude that $f$ has $\beta$-Lipschitz gradients, which gives us
\[
\norm{\nabla f(\vx^t)}_2 = \norm{\nabla f(\vx^t) - \nabla f(\xopt)}_2 \leq \beta\cdot\norm{\x^t - \xopt}_2.
\]
Using the above results, $\E{\vzeta^t\sep\x^t} = \vzero$ and $\eta \leq \frac{\alpha}{\beta^2}$ gives us
\begin{align*}
\E{\left.\norm{\x^{t+1} - \xopt}_2^2\ \right|\ \x^t} ={}& \E{\left.\norm{\x^t - \eta(\nabla f(\x^t) + \vzeta^t) - \xopt}_2^2\ \right|\ \x^t}\\
														={}& \norm{\x^t - \xopt}_2^2 - 2\eta\ip{\nabla f(\x^t)}{\x^t - \xopt}\\
														&{}+ \eta^2\norm{\nabla f(\vx^t)}_2^2 + \eta^2\\
														\leq{}& (1-2\eta\alpha+\eta^2\beta^2)\cdot\norm{\x^t - \xopt}_2^2 + \eta^2\\
														\leq{}& (1-\eta\alpha)\cdot\norm{\x^t - \xopt}_2^2 + \eta^2,
\end{align*}
which, upon some manipulation, gives us
\[
\E{\left.\norm{\x^{t+1} - \xopt}_2^2 - \frac{\eta}{\alpha}\ \right|\ \x^t} \leq (1-\eta\alpha)\br{\norm{\x^t - \xopt}_2^2 - \frac{\eta}{\alpha}}
\]
This tells us that on expectation, the distance of the iterates $\x^t$ from the local optimum $\xopt$ will hover around $\sqrt\frac{\eta}{\alpha}$ which indicates towards the claimed result. The subsequent steps in the proof of this result require techniques from martingale theory which we wish to avoid. We refer the reader to \cite[Lemma 16]{GeHJY2015} for the details.
\end{proof}

Notice that the above result traps the iterates within a radius $\xi$ ball around the local minimum $\xopt$. Also notice that all points that are approximate local minima satisfy the preconditions of this theorem due to the SSa property and consequently the NGD gets trapped for these points. We now present the final convergence guarantee for NGD.

\begin{theorem}
\label{thm:conv-ngd}
For any $\epsilon,\delta > 0$, suppose NGD is executed on a function that is $\beta$-strongly smooth, has $\rho$-Lipschitz Hessians, and satisfies the $(\alpha, \gamma, \kappa, \xi)$-SSa property, with a step length $\eta < \eta_{\max} = \min\bc{\frac{\epsilon^2}{\log(1/\epsilon\delta)},\frac{\alpha}{\beta^2},\frac{\xi^2}{\log(1/\xi\delta)},\frac{1}{(\beta+\rho)^2},\frac{\kappa^2}{\beta}}$. Then, with probability at least $1 - \delta$, after $T \geq \log p/\eta^2\cdot\log(2/\delta)$ iterations, NGD produces an iterate $\x^T$ that is $\epsilon$-close to some local optimum $\xopt$ i.e., $\norm{\x^T - \xopt}_2 \leq \epsilon$.
\end{theorem}
\clearpage
\begin{proof}
We partition the space $\bR^p$ into 3 regions
\begin{enumerate}
	\item $\kR_1 = \bc{\vx: \norm{f(\vx)}_2 \geq \sqrt{\eta\beta}}$
	\item $\kR_2 = \bc{\vx: \norm{f(\vx)}_2 < \sqrt{\eta\beta}, \lambda_{\min}(\nabla^2 f(\vx)) \leq -\gamma}$
	\item $\kR_3 = \bR^p\backslash(\kR_1\cup\kR_2)$
\end{enumerate}
Since $\sqrt{\eta\beta}\leq \kappa$ due to the setting of $\eta_{\max}$, the region $\kR_1$ contains all points considered non-stationary by the SSa property (Definition~\ref{defn:ssp}) and possibly some other points as well. For this reason, region $\kR_2$ can be shown to contain only saddle points. Since the SSa property assures us that a point that is neither non-stationary nor a saddle point is definitely an approximate local minimum, we deduce that the region $\kR_3$ contains only approximately local minima.

The proof will use the following line of argument: since Lemmata~\ref{lem:ngd-non-stationary} and \ref{lem:ngd-saddle} assure us that whenever the NGD procedure is in regions $\kR_1$ or $\kR_2$ there is a large drop in function value, we should expect the procedure to enter region $\kR_3$ sooner or later, since the function value cannot go on decreasing indefinitely. However, Lemma~\ref{lem:ngd-local-nbhd} shows that once we are in region $\kR_3$, we are trapped there. In the following analysis, we will ignore all non-essential constants and log factors.

Recall that we let the NGD procedure last for $T = 1/\eta^2\log(2/\delta)$ steps. Below we will show that in any sequence of $1/\eta^2$ steps, there is at least a $1/2$ chance of encountering an iterate $\vx^t \in \kR_3$. Since the entire procedure lasts $\log(1/\delta)$ such sequences, we will conclude, by union bound, that with probability at least $1 - \delta/2$, we will encounter at least one iterate in the region $\kR_3$ in the $T$ steps we execute.

However, Lemma~\ref{lem:ngd-local-nbhd} shows that once we enter the $\kR_3$ neighborhood, with probability at least $1-\delta/2$, we are trapped there for at least $T$ steps. Applying the union bound will establish that with probability at least $1 - \delta$, the NGD procedure will output $\vx^T \in \kR_3$. Since we set $\eta \leq \epsilon^2/\log(1/\epsilon\delta)$, this will conclude the proof.

We now left with proving that in every sequence of $1/\eta^2$ steps, there is at least a $1/2$ chance of NGD encountering an iterate $\vx^t \in \kR_3$. To do so, we set up the notion of \emph{epochs}. These will basically correspond to the amount of time taken by NGD to reduce the function value by a significant amount. The first epoch starts at time $\tau_1 = 0$. Subsequently, we define
\[
\tau_{i+1} = \left\{
\begin{array}{cl}
	\tau_i + 1 &\quad \text{ if } \vx^{\tau_i} \in \kR_1 \cup \kR_3\\
	\tau_i + \frac{1}{\eta} &\quad \text{ if } \vx^{\tau_i} \in \kR_2
\end{array}
\right.
\]
Ignoring constants and other non-essential factors, we can rewrite the results of Lemmata~\ref{lem:ngd-non-stationary} and \ref{lem:ngd-saddle} as follows
\begin{align*}
\E{f(\vx^{\tau_{i+1}}) - f(\vx^{\tau_i})\sep\vx^{\tau_i} \in \kR_1} &\leq -\eta^2\\
\E{f(\vx^{\tau_{i+1}}) - f(\vx^{\tau_i})\sep\vx^{\tau_i} \in \kR_2} &\leq -\eta
\end{align*}
Putting these together gives us
\[
\E{f(\vx^{\tau_{i+1}}) - f(\vx^{\tau_i})\sep\vx^{\tau_i} \notin \kR_3} \leq -\E{(\tau_{i+1}-\tau_i)\sep\vx^{\tau_i} \notin \kR_3}\cdot\eta^2
\]
Define the event $\kE_t := \bc{\nexists\ j \leq t: \vx^j \in \kR_3}$ and let $\vone_E$ denote the indicator variable for event $E$ i.e., $\vone_E = 1$ if $E$ occurs and $\vone_E = 0$ otherwise. Then we have
\begin{align*}
\E{f(\vx^{\tau_{i+1}})\cdot\vone_{\kE_{\tau_{i+1}}} - f(\vx^{\tau_i})\cdot\vone_{\kE_{\tau_i}}} = \E{f(\vx^{\tau_{i+1}})\cdot(\vone_{\kE_{\tau_{i+1}}} - \vone_{\kE_{\tau_i}})} + \E{(f(\vx^{\tau_{i+1}}) - f(\vx^{\tau_i}))\cdot\vone_{\kE_{\tau_i}}}\\
\leq B\cdot(\Pr{\kE_{\tau_{i+1}}} - \Pr{\kE_{\tau_i}}) - \eta^2\cdot\E{\tau_{i+1} - \tau_i\sep\vone_{\kE_{\tau_i}}}\cdot\Pr{\kE_{\tau_i}},
\end{align*}
where we have used the fact that $\abs{f(\x)} \leq B$ for all $\x \in \bR^p$. Since $\kE_{t+1} \Rightarrow \kE_t$, we have $\Pr{\kE_{t+1}} \leq \Pr{\kE_t}$. Summing the expressions above from $i = 1$ to $j$ and using $\x^1 \notin \kR_3$ gives us
\[
\E{f(\vx^{\tau_{j+1}})\cdot\vone_{\kE_{\tau_{j+1}}}} - f(\vx^1) \leq - \eta^2\E{\tau_{j+1}}\cdot\Pr{\kE_{\tau_j}}
\]
However, since the function is bounded $\abs{f(\vx)} \leq B$, the left hand side cannot be smaller than $-2B$. Thus, if $\E{\tau_{j+1}} \geq 4B/\eta^2$ then we must have $\Pr{\kE_{\tau_j}} \leq 1/2$. This concludes the proof.
\end{proof}

We have not presented exact constants in the results to avoid clutter. The NGD algorithm actually requires the step length to be set to $\eta < \eta_{\max} \leq \frac{1}{p}$ where $p$ is the ambient dimensionality. Now, since NGD is run for $\Om{1/\eta^2}$ iterations and each iteration takes $\bigO{p}$ time to execute, the total run-time of NGD is $\bigO{p^3}$ which can be prohibitive. The bibliographic notes discuss more recent results that offer run-times that are linear in $p$. Also note that the NGD procedure requires $\softO{1/\epsilon^4}$ iterations to converge within an $\epsilon$ distance of a local optimum.

\section{Constrained Optimization with Non-convex Objectives}

\begin{algorithm}[t]
	\caption{Projected Noisy Gradient Descent (PNGD)}
	\label{algo:saddle-pngd}
	\begin{algorithmic}[1]
		{
			\REQUIRE Objective $f$, max step length $\eta_{\max}$, tolerance $\epsilon$
			\ENSURE A locally optimal point $\hat\x \in \R^p$
			\STATE $\x^1 \leftarrow \text{\textsf{INITALIZE}}()$
			\STATE Set $T \leftarrow 1/\eta^2$, where $\eta = \min\bc{\epsilon^2/\log^2(1/\epsilon), \eta_{\max}}$
			\FOR{$t= 1, 2, \ldots, T$}
				\STATE Sample perturbation $\vzeta^t \sim S^{p-1}$\hfill//\texttt{Random pt. on unit sphere}%
				\STATE $\vg^t \leftarrow \nabla f(\x^t) + \vzeta^t$
				\STATE $x^{t+1} \leftarrow \Pi_\cW(\x^t - \eta\cdot\vg^t)$\hfill//\texttt{Project onto constraint set}
			\ENDFOR
			\STATE \textbf{return} $\x^T$
		}
	\end{algorithmic}
\end{algorithm}

We will now present extensions to the above discussion to cases where the optimization problem is constrained. We will concentrate on constrained optimization problems with equality constraints.
\begin{equation}
\begin{array}{cl}
	\underset{\x\in\R^p}\min & f(\x)\\
	\text{s.t.} & c_i(\vx) = 0, i \in [m]
\end{array}
\tag*{(CNOPT)}\label{eq:nopt-cons}
\end{equation}

Let $\cW := \bc{\x \in \R^p: c_i(\x) = 0, i \in [m]}$ denote the constraint set. In general $\cW$ is a \emph{manifold} and we will assume that this manifold is \emph{nice} in that it is smooth and does not have corners. It is natural to attempt to solve the problem using a gPGD-like approach. Indeed, Algorithm~\ref{algo:saddle-pngd} extends the NGD algorithm to include a projection step. The algorithm is actually similar to the NGD algorithm save the step that projects the iterates onto the manifold $\cW$. This projection step can be tricky but can be efficiently solved, for instance when the constraint functions $c_i$ are linear (see \cite[Section 6.2]{BoydV2004}).

The complete analysis of this algorithm (see \cite[Appendix B]{GeHJY2015}), although similar in parts to that of the NGD algorithm, is beyond the scope of this monograph. Instead, we just develop the design concepts and intuition used in the analysis. The first step is to convert the above constrained optimization into an unconstrained one so that some of the tools used for NGD may be reapplied here.

A very common way to do so is to first construct the \emph{Lagrangian} \cite[Chapter 5]{BoydV2004} of the problem defined as
\[
\cL(\vx,\vlambda) = f(\x) - \sum_{i=1}^m\lambda_ic_i(\x),
\]
where $\lambda_i$ are \emph{Lagrange multipliers}. It is easy to verify that the solution to the problem \eqref{eq:nopt-cons} coincides with that of the following problem
\[
\min_{\vx \in \bR^p}\max_{\vlambda \in \bR^m}\ \cL(\vx,\vlambda)
\]
Note that the above problem is unconstrained. For any $\x \in \R^p$, define $\vlambda^\ast(\x) := \arg\min_{\vlambda}\ \norm{\nabla f(\x) - \sum_{i=1}^m\lambda_i\nabla c_i(\x)}$ and $\cL^\ast(\x) := \cL(\x,\vlambda^\ast(\x))$. We also define the tangent and normal spaces of the manifold $\cW$ as follows.

\begin{definition}[Normal and Tangent Space]
\label{defn:saddle_tangent}
For a manifold $\cW$ defined as the intersection of $m$ constraints of the form $c_i(\x) = 0$, given any $\x \in \cW$, define its tangent space as $\cT(\x) = \bc{\vv\sep\ip{\nabla c_i(\x)}{\vv} = 0, i \in [m]}$ and its normal space as $\cT^c(\x) = \text{span}\bc{\nabla c_1(\x),\ldots,\nabla c_m(\x)}$.
\end{definition}

If we think of $\cW$ as a smooth surface, then at any point, the tangent space defines the \emph{tangent plane} of the surface at that point and the normal space consists of all vectors orthogonal to the tangent plane. The reason behind defining all the above quantities is the following. In unconstrained optimization we have the first and second order optimality conditions: if $\x^\ast$ is a local minimum for $f(\cdot)$ then $\nabla f(\x^\ast) = \vzero$ and $\nabla^2 f(\x^\ast) \succ 0$.

Along the same lines, for constrained optimization problems, similar conditions exist characterizing stationary and optimal points: one can show \citep{WrightN1999} that if $\xopt$ is a local optimum for the constrained problem \eqref{eq:nopt-cons}, then we must have $\nabla \cL^\ast(\xopt) = \vzero$, as well as $\vv^\top\nabla^2 \cL^\ast(\xopt)\vv \geq 0$ for all $\vv \in \cT(\xopt)$. This motivates us to propose the following extension of the strict saddle property for constrained optimization problems.

\begin{definition}[Strict Constrained Saddle Property \citep{GeHJY2015}]
\label{defn:scsp}
A twice differentiable function $f(\x)$ with a constraint set $\cW$ is said to satisfy the $(\alpha, \gamma, \kappa, \xi)$-strict constrained saddle (SCSa) property, if for every local minimum $\xopt \in \cW$, we have $\vv^\top\cL^\ast(\x')\vv \geq \alpha$ for all $\vv \in \cT(\x'), \norm{\vv}_2 = 1$ and for all $\x'$ in the region $\cB_2(\xopt,2\xi)$, and moreover, any point $\x_0 \in \R^p$ satisfies at least one of the following properties: 
\begin{enumerate}
	\item (Non-Stationary) $\|\nabla \cL^\ast(\x_0)\|_2 \geq \kappa$
	\item	(Strict Saddle) $\vv^\top\nabla^2\cL^\ast(\x_0)\vv \leq -\gamma$ for some $\vv \in \cT(\x_0), \norm{\vv}_2 = 1$
	\item	(Approx. Local Min.) For some local minimum $\xopt$, $\|\x_0-\xopt\|_2 \leq \xi$.
\end{enumerate}
\end{definition}
The following local convergence result can then be shown for the PNGD algorithm.

\begin{theorem}
\label{thm:conv-pngd}
Suppose PNGD is executed on a constrained optimization problem that satisfies the $(\alpha, \gamma, \kappa, \xi)$-SCSa property, has $\rho$-Lipschitz Hessians and whose constraint set is a smooth manifold $\cW$. Then there exists a constant $\eta_{\max} = \Theta(1)$ such that for any $\epsilon,\delta > 0$, if we set $\eta = \min\bc{\eta_{\max},\frac{\epsilon^2}{\log(1/\epsilon\delta)}}$, then with probability at least $1 - \delta$, after $T \geq \log p/\eta^2\cdot\log(2/\delta)$ iterations, PNGD produces an iterate $\x^T$ that is $\epsilon$-close to some local optimum $\xopt$ i.e., $\norm{\x^T - \xopt}_2 \leq \epsilon$.
\end{theorem}

The PNGD procedure requires $\softO{1/\epsilon^4}$ iterations to converge within an $\epsilon$ distance of a local optimum, similar to NGD. The proof of the above theorem \cite[Lemma 35]{GeHJY2015} proceeds by showing that the PNGD updates can be rewritten as
\[
\x^{t+1} = \x^t - \eta\cdot(\nabla\cL^\ast(\x^t) + \Pi_{\cT(\x^t)}(\vzeta^t)) + \viota^t,
\]
where $\Pi_{\cT(\x^t)}(\cdot)$ is the projection of a vector onto the tangent space $\cT(\x^t)$ and $\viota^t$ is an error vector with small norm $\norm{\viota^t}_2 \leq \eta^2$.

\section{Application to Orthogonal Tensor Decomposition}
\label{sec:saddle-app-otd}
Recall that we reduced the tensor decomposition problem to solving the following optimization problem \eqref{eq:lrtd}
\begin{equation*}
\begin{array}{cl}
	\max & T(\vu,\vu,\vu,\vu) = \sum_{i=1}^r (\vu_i^\top\vu)^4\\
	\text{s.t.} & \norm{\vu}_2 = 1,
\end{array}
\end{equation*}
The above problem has internal symmetry with respect to permutation of the components, as well as sign flips (both $\vu_i$ and $-\vu_i$ are valid components) which gives rise to saddle points as we discussed earlier. Using some simple but slightly tedious calculations (see \cite[Theorem 44, Lemmata 45, 46]{GeHJY2015}) it is possible to show that the above optimization problem does satisfy the $(3,7/p,1/p^c,1/p^c)$-SCSa property for some constant $c > 0$. All other requirements for Theorem~\ref{thm:conv-pngd} can also be shown to be satisfied.

It is easy to see that any solution to the problem must lie in the span of the components. Since the components $\vu_i$ are orthonormal, this means that it suffices to look for solutions of the form $\vu = \sum_{i=1}^rx_i\vu_i$. This gives us $T(\vu,\vu,\vu,\vu) = \sum_{i=1}^rx_i^4$, and $\norm{\vu}_2 = \norm{\x}_2$ where $\x = [x_1,x_2,\ldots,x_r]$. This allows us to formulate the equivalent problem as
\begin{equation*}
\begin{array}{cl}
	\min & -\norm{\x}_4^4\\
	\text{s.t.} & \norm{\x}_2 = 1,
\end{array}
\end{equation*}
Note that the above problem is non-convex as the objective function is concave. However, it is also possible to show\elink{exer:saddle-multi-opt}\elink{exer:saddle-multi-non-opt} that the only local minima of the optimization problem \eqref{eq:lrtd} are $\pm\vu_i$. This is most fortunate since it shows that all local minima are actually global minima! Thus, the \emph{deflation} strategy alluded to earlier can be successfully applied. We discover one component, say $\vu_1$ by applying the PNGD algorithm to \eqref{eq:lrtd}. Having recovered this component, we create a new tensor $T' = T - \vu_1\otimes\vu_1\otimes\vu_1\otimes\vu_1$, apply the procedure again to discover a second component and so on.

The work of \cite{GeHJY2015} also discusses techniques to recover all $r$ components simultaneously, tensors with different positive weights on the components, as well as tensors of other orders.

\section{Exercises}
\begin{exer}
\label{exer:saddle-non-conv}
Show that the optimization problem in the formulation~\eqref{eq:lrtd} is non-convex in that it has a non-convex objective as well as a non-convex constraint set. Show that constraint set may be convexified without changing the optimum of the problem.
\end{exer}
\begin{exer}
\label{exer:saddle-smooth-nc}
Consider a differentiable function $f$ that is $\beta$-strongly smooth but possibly non-convex. Show that if gradient descent is performed on $f$ with a static step length $\eta \leq \frac{2}{\beta}$ i.e.,
\[
\x^{t+1} = \x^t - \eta\cdot\nabla f(\x^t)
\]
then the function value $f$ will never increase across iterations i.e., $f(\x^{t+1}) \leq f(\x^t)$. This shows that on smooth functions, gradient descent enjoys monotonic progress whenever the step length is small enough.\\
\textit{Hint}: apply the SS property relating consecutive iterates.
\end{exer}
\begin{exer}
\label{exer:saddle-smooth-nc-2}
For the same setting as the previous problem, show that if we have $\eta \in\bs{\frac{1}{2\beta},\frac{1}{\beta}}$ instead, then within $T = \bigO{\frac{1}{\epsilon^2}}$ steps, gradient descent is guaranteed to identify an $\epsilon$-stationary point i.e for some $t \leq T$, we must have $\norm{\nabla f(\x^t)}_2 \leq \epsilon$.\\
\textit{Hint}: first apply SS to show that $f(x^t) - f(x^{t+1}) \geq \frac{1}{4\beta}\cdot\norm{\nabla f(x^t)}_2^2$. Then use the fact that the total improvement in function value over $T$ time steps cannot exceed $f(x^0) - f^\ast$.
\end{exer}
\begin{exer}
\label{exer:saddle-multi-opt}
Show that every component vector $\vu_i$ is a globally optimal point for the optimization problem in~\eqref{eq:lrtd}.\\
\textit{Hint}: Observe the reduction in \S~\ref{sec:saddle-app-otd}. Show that it is equivalent to finding the minimum $L_2$ norm vector(s) among unit $L_1$ norm vectors. Find the Lagrangian dual of this optimization problem and simplify.
\end{exer}
\begin{exer}
\label{exer:saddle-multi-non-opt}
Given a rank-$r$ orthonormal tensor $T$, construct a non-trivial convex combination of its components that has unit $L_2$ norm but achieves a suboptimal objective value in the formulation~\eqref{eq:lrtd}.
\end{exer}

\section{Bibliographic Notes}
Progress in this area was extremely rapid at the time of writing this monograph. The following discussion outlines some recent advances.

The general problem of local optimization has been of interest in the optimization community for several years. Some of the earlier results in the area concentrated on second order or quasi-Newton methods, understandably since the Hessians can be used to distinguish between local optima and saddle points. Some key works include those of \cite{NesterovP2006,DauphinPGCGB2014,SunQW2015,Yuan2015}. However, these techniques can be expensive and hence, are not used in high dimensional applications.

Algorithms for specialized applications, such as the EM algorithm we studied in \S~\ref{chap:em}, have received independent attention with respect to their local convergence properties. We refer the reader to the bibliographic notes \S~\ref{sec:em-bib} for the same. In recent years, the problem of optimization of general non-convex problems has been looked at from several perspectives. Below we outline some of the major thrust areas.\\

\noindent\textbf{First Order Methods} Given their efficiency, the question of whether first order methods can offer local optimality has captured the interest of the field. In particular, much attention has been paid to the standard gradient descent and its variants. The work of \cite{LeeSJR2016} showed that when initialized at a random point, gradient descent avoids saddle points almost surely, although the work provides no definite rate of convergence in general.

The subsequent works of \cite{SunQW2015,GeHJY2015,AnandkumarG2016} introduced structural properties such as the strict saddle property, and demonstrated that crisp convergence rates can be ensured for problems that do satisfy these structural properties. It is notable that the work of \cite{SunQW2015} reconsidered second order methods while \cite{GeHJY2015} were able to show that noisy stochastic gradient descent itself suffices.

The technique of using randomly perturbed (stochastic) gradients to escape saddle points receives attention in a more general framework of \emph{Langevin Dynamics} which study cases when the perturbations are non-isotropic or else are applied at a scale that adapts to the problem. The recent work of \cite{ZhangLC2017} shows a powerful result that offers, for empirical risk minimization problems that are ubiquitous in machine learning, a convergence guarantee that ensures convergence to a local optimum of the population risk functional. This is a useful result since a majority of works in literature focus on identifying local optima of the empirical risk functional, which depend on the training data, but which may correspond to bad solutions with respect to the population risk.\\

\noindent\textbf{Non-smooth Optimization} All techniques we discussed in thus far in this monograph require the objective function to at least be differentiable. In fact, we go one step ahead to assume (restricted) smoothness. The work of \cite{ReddiHSPS2016} shows how variance reduction techniques may be exploited to ensure $\epsilon$-convergence to a first-order stationary point $\x^t$ (effectively a saddle point) where we have $\norm{f(\vx^t)}_2 \leq \epsilon$, in no more than $\bigO{1/\epsilon}$ iterations of stochastic mini-batch gradient descent even when the objective function is non-convex and non-smooth at the same time. We urge the reader to confer this with Exercise~\ref{exer:saddle-smooth-nc-2} where it took $\bigO{1/\epsilon^2}$ iterations to do the same, that too using full gradient descent. Also notable is the fact that in our analysis (see Theorem~\ref{thm:conv-ngd}), it took NGD $\bigO{1/\epsilon^4}$ steps to reach a local optimum. However, two caveats exist: 1) the method proposed in \citep{ReddiHSPS2016} only reaches a saddle point, and 2) it assumes a \emph{finite-sum} objective, i.e., one that has a decomposable form.\\

\noindent\textbf{Accelerated Optimization} The works of \cite{CarmonDHS2017,AgarwalA-ZBHM2017} extend the work of \cite{ReddiHSPS2016} by offering faster than $\bigO{1/\epsilon^2}$ convergence to a stationary point for general (non finite-sum) non-convex objectives. However, it should be noted that these techniques assume a smooth objective and convergence is guaranteed only to a saddle point, not a local minimum. Whereas the work of \cite{CarmonDHS2017} invokes a variant of Nesterov's accelerated gradient technique to offer $\epsilon$-convergence in $\bigO{1/\epsilon^{7/4}}$ iterations, the work of \cite{AgarwalA-ZBHM2017} employs a second-order method to offer $\epsilon$-convergence, also in $\bigO{1/\epsilon^{7/4}}$ iterations.\\

\noindent\textbf{High-dimensional Optimization} When considering problems in high-dimensions, it becomes difficult to execute algorithms whose run-times scale super-linearly in the ambient dimensionality. This is one reason why first-order methods are preferred. However, as we saw, the overall run-time of Theorem~\ref{thm:conv-ngd} scales as $\bigO{p^3}$ which makes the method prohibitive even for moderate dimensional problems. The work of \cite{Jin0NKJ17} proposes another perturbed gradient method that offers a run-time that depends only quasi-linearly on the ambient dimension. It should be noted that the works of \cite{CarmonDHS2017,AgarwalA-ZBHM2017} also offer run-times that are linear in the ambient dimensionality although, as noted earlier, convergence is only guaranteed to a saddle point by these methods.\\

\noindent\textbf{Escaping Higher-order Saddles} In our discussion, we were occupied with the problem of avoiding getting stuck at simple saddle points which readily reveal themselves by having distinctly positive and negative eigenvalues in the Hessian. However, there may exist more complex \emph{degenerate saddle} points where the Hessian has only non-negative eigenvalues and thus, masquerades as a local minima. Such configurations yield complex cases such as \emph{monkey saddles} and \emph{connected saddles}. We did not address these. The work of \cite{AnandkumarG2016} proposes a method based on the Cubic Regularization algorithm of \cite{NesterovP2006} which is able to escape some of these more complex saddle points and achieve convergence to a point that enjoys \emph{third order optimality}.\\

\noindent\textbf{Training Deep Networks} Given the popularity of deep networks in several areas of learning and signal processing, as well as the fact that the task of training deep networks corresponds to non-convex optimization, a lot of recent efforts have focused on the problem of efficiently and provably training deep networks using non-convex optimization techniques. Some notable works include provable methods for training multi-layered perceptrons \citep{GoelKlivans2017,ZhongSJBD2017,LiYuan2017} and special cases of convoluational networks known as non-overlapping convolutional networks \citep{BrutzkusGloberson2017}. Whereas the works \citep{BrutzkusGloberson2017,ZhongSJBD2017,LiYuan2017} utilize gradient-descent based techniques, \cite{GoelKlivans2017} uses an application of isotonic regression and kernel techniques. The work of \cite{LiYuan2017} shows that the inclusion of an \emph{identity map} into the network eases optimization by making the training problem well-posed.

\part{Applications}

\chapter{Sparse Recovery}
\label{chap:spreg}

In this section, we will take a look at the sparse recovery and sparse linear regression as applications of non-convex optimization. These are extremely well studied problems and find applications in several practical settings. This will be the first of four ``application'' sections where we apply non-convex optimization techniques to real-world problems.

\section{Motivating Applications}
We will take the following two running examples to motivate the sparse regression problem in two different settings.\\

\noindent\textbf{Gene Expression Analysis} The availability of DNA micro-array gene expression data makes it possible to identify genetic explanations for a wide range of phenotypical traits such as physiological properties or even disease progressions.	In such data, we are given say, for $n$ human test subjects participating in the study, the expression levels of a large number $p$ of genes (encoded as a real vector $\x_i \in \R^p$), and the corresponding phenotypical trait $y_i \in \R$. Figure~\ref{fig:gene} depicts this for a hypothetical study on Type-I diabetes. For the sake of simplicity, we are considering cases where the phenotypical trait can be modeled as a real number -- this real number may indicate the severity of a condition or the level of some other biological measurement. More expressive models exist in literature where the target phenotypical trait is itself represented as a vector \cite[see for example,][]{JainT2015}.

\begin{figure}[t]
\includegraphics[width=\columnwidth]{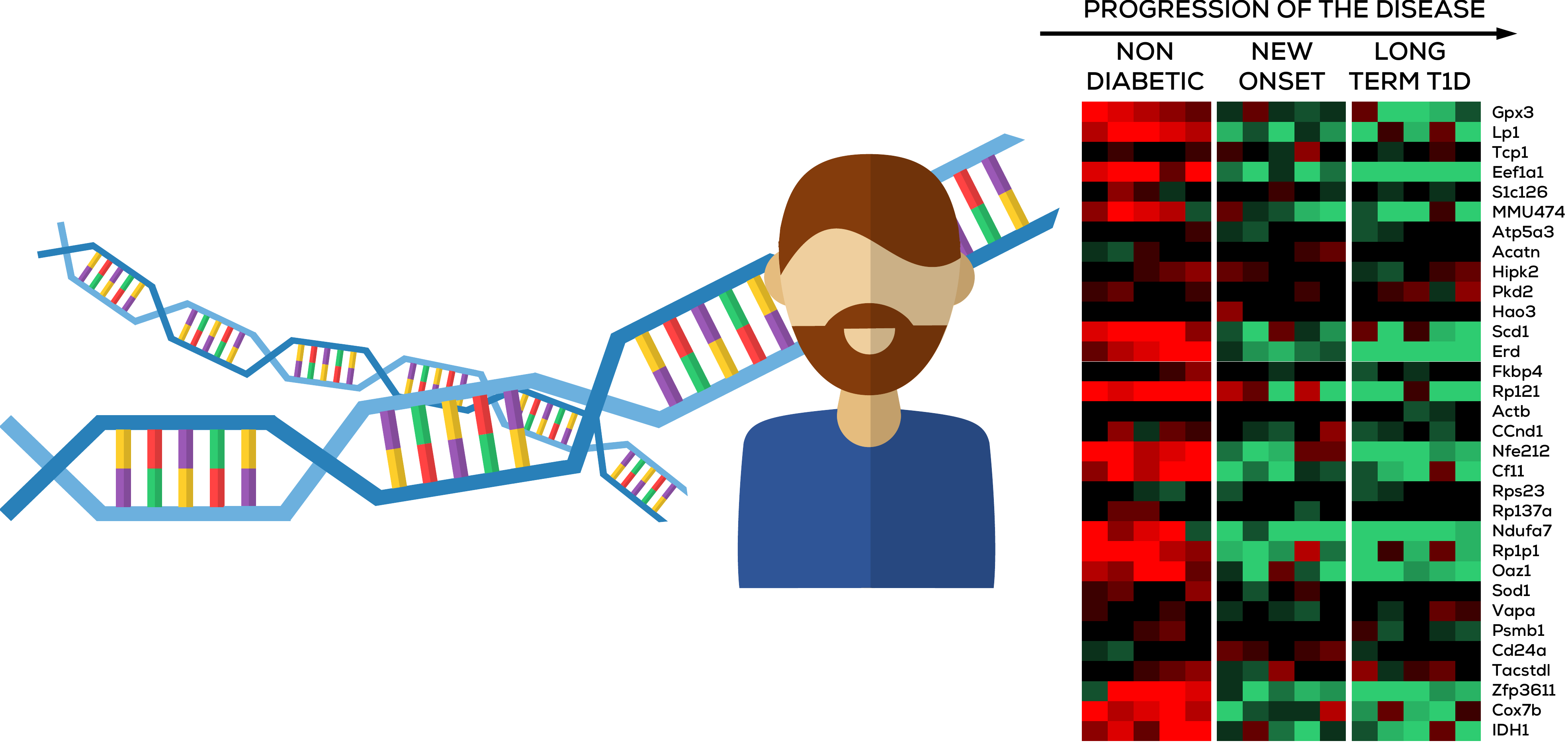}
\caption[Gene Expression Analysis as Sparse Regression]{Gene expression analysis can help identify genetic bases for physiological conditions. The expression matrix on the right has 32 rows and 15 columns: each row represents one gene being tracked and each column represents one test subject. A bright red (green) shade in a cell indicates an elevated (depressed) expression level of the corresponding gene in the test subject with respect to a reference population. A black/dark shade indicates an expression level identical to the reference population. Notice that most genes do not participate in the progression of the disease in a significant manner. Moreover, the number of genes being tracked is much larger than the number of test subjects. This makes the problem of gene expression analysis, an ideal application for sparse recovery techniques. Please note that names of genes and expression levels in the figure are illustrative and do not correspond to actual experimental observations. Figure adapted from \citep{WilsonELRYSD-SMCS2003}.}%
\label{fig:gene}
\end{figure}

For the sake of simplicity, we assume that the phenotypical response is linearly linked to the gene expression levels i.e. for some $\bto \in \R^p$, we have $y_i = \x_i^\top\bto + \eta_i$ where $\eta_i$ is some noise. The goal then is to use gene expression data to deduce an estimate for $\bto$. Having access to the model $\bto$ can be instrumental in discovering possible genetic bases for diseases, traits etc. Consequently, this problem has significant implications for understanding physiology and developing novel medical interventions to treat and prevent diseases/conditions.

However, the problem fails to reduce to a simple linear regression problem for two important reasons. Firstly, although the number of genes whose expression levels are being recorded is usually very large (running into several tens of thousands), the number of samples (test subjects) is usually not nearly as large, i.e. $n \ll p$. Traditional regression algorithms fall silent in such data-starved settings as they usually expect $n > p$. Secondly, and more importantly, we do not expect all genes being tracked to participate in realizing the phenotype. Indeed, the whole objective of this exercise is to identify a small set of genes which most prominently influence the given phenotype. Note that this implies that the vector $\bto$ is very sparse. Traditional linear regression cannot guarantee the recovery of a sparse model.\\

\noindent\textbf{Sparse Signal Transmission and Recovery} The task of transmitting and acquiring signals is a key problem in engineering. In several application areas such as magnetic resonance imagery and radio communication, \emph{linear} measurement techniques, for example, sampling, are commonly used to acquire a signal. The task then is to reconstruct the original signal from these measurements. For sake of simplicity, suppose we wish to sense/transmit signals represented as vectors in $\R^p$. For various reasons (conserving energy, protection against data corruption etc), we may want to not transmit the signal directly and instead, create a \emph{sensing mechanism} wherein a signal $\bt \in \R^p$ is encoded into a signal $\y \in \R^n$ and it is $\y$ that is transmitted. At the receiving end $\y$ must be decoded back into $\bt$. A popular way of creating sensing mechanisms -- also called \emph{designs} -- is to come up with a set of $n$ linear functionals $\x_i: \R^p \rightarrow \R$ and for any signal $\bt \in \R^p$, record the values $y_i = \x_i^\top\bt$. If we denote $X = [\x_1,\ldots,\x_n]^\top$ and $\y = [y_1,\ldots,y_n]^\top$, then $\y = X\bt$ is transmitted. Note as a special case that if $n = p$ and $\x_i = \e_i$, then $X = I_{p\times p}$ and $\y = \bt$, i.e. we transmit the original signal itself.

If $p$ is very large then we naturally look for designs with $n \ll p$. However, elementary results in algebra dictate that the recovery of $\bt$ from $\y$ cannot be guaranteed even if $n = p - 1$. There is irrecoverable loss of information and there could be (infinitely) many signals $\w$ all of which map to the same transmitted signal $\y$ making it impossible to recover the original signal uniquely. A result similar in spirit called the Shannon-Nyquist theorem holds for analog or continuous-time signals. Although this seems to spell doom for any efforts to perform \emph{compressed} sensing and transmission, these negative results can actually be overcome by observing that in several useful settings, the signals we are interested in, are actually very sparse i.e. $\bt \in \B_0(s) \subset \R^p$, $s \ll p$. This realization is critical since it allows the possibility of specialized design matrices to be used to transmit sparse signals in a highly compressed manner i.e. with $n \ll p$ but without any loss of information. However the recovery problem now requires a sparse vector to be recovered from the transmitted signal $\y$ and the design matrix $X$.

\section{Problem Formulation}
\label{sec:spreg-prob-form}
In both the examples considered above, we wish to recover a sparse linear model $\bt \in \B_0(s) \subset \R^p$ that fits the data, i.e., $\y_i \approx \x_i^\top\bt$, hence the name \emph{sparse recovery}. In the gene analysis problem, the support of such a model $\bt$ is valuable in revealing the identity of the genes involved in promoting the given phenotype/disease. Similarly, in the sparse signal recovery problem, $\bt$ is the (sparse) signal itself.

This motivates the sparse linear regression problem. In the following, we shall use $\x_i \in \R^p$ to denote the features (gene expression levels/measurement functionals). Each feature will constitute a data point. There will be a \emph{response} variable $y_i \in \R$ (phenotype response/measurement) associated with each data point. We will assume that the response variables are being generated using some underlying sparse \emph{model} $\bto \in \B_0(s)$ (gene influence pattern/sparse signal) as $y_i = \x_i^\top\bto + \eta_i$ where $\eta_i$ is some benign noise.

In both the gene expression analysis problem, as well as the sparse signal recovery problem, recovering $\bto$ from the data $(\x_1,y_1),\ldots,(\x_n,y_n)$ then requires us to solve the following optimization problem: $\underset{\bt\in\R^p, \norm{\bt}_0 \leq s}{\min}\ \sum_{i=1}^n\br{y_i - \x_i^\top\bt}^2$. Rewriting the above in more succinct notation gives us
\begin{equation}
\underset{\substack{\bt\in\R^p\\\norm{\bt}_0 \leq s}}{\min}\ \norm{\y - X\bt}_2^2,\tag*{(SP-REG)}\label{eq:spreg}
\end{equation}
where $X = [\x_1,\ldots,\x_n]^\top$ and $\y = [y_1,\ldots,y_n]^\top$. It is common to model the additive noise as \emph{white noise} i.e. $\eta_i \sim \cN(0,\sigma^2)$ for some $\sigma > 0$. It should be noted that the sparse regression problem in \eqref{eq:spreg} is an NP-hard problem \citep{Natarajan1995}.

\section{Sparse Regression: Two Perspectives}
Although we cast both the gene analysis and sparse signal recovery problems in the same framework of sparse linear regression, there are subtle but crucial differences between the two problem settings.

Notice that the problem in sparse signal recovery is to come up with both, a design matrix $X$, as well as a recovery algorithm $\A: \R^n \times \R^{n \times p} \rightarrow \R^p$ such that all sparse signals can be accurately recovered from the measurements, i.e. $\forall \bt \in \B_0(s) \subset \R^p$, we have $\A(X\bt, X) \approx \bt$.

On the other hand, in the gene expression analysis task, we do not have as direct a control over the effective design matrix $X$. In this case, the role of the design matrix is played by the gene expression data of the $n$ test subjects. Although we may choose which individuals we wish to include in our study, even this choice does not give us a direct handle on the properties of the design matrix. Our job here is restricted to coming up with an algorithm $\A$ which can, given the gene expression data for $p$ genes in $n$ test subjects, figure out a sparse set of $s$ genes which collectively promote the given phenotype i.e. for any $\bt \in \B_0(s) \subset \R^p$ and given $X \in \bR^{n \times p}$, we desire $\A(X\bt, X) \approx \bt$.

The distinction between the two settings is now apparent: in the first setting, the design matrix is totally in our control. We may design it to have specific properties as required by the recovery algorithm. However, in the second case, the design matrix is mostly given to us. We have no fine control over its properties.

This will make an important difference in the algorithms that operate in these settings since algorithms for sparse signal recovery would be able to make very stringent assumptions regarding the design matrix since we ourselves create this matrix from scratch. However, for the same reason, algorithms working in \emph{statistical learning} settings such as the gene expression analysis problem, would have to work with relaxed assumptions that can be expected to be satisfied by natural data. We will revisit this point later once we have introduced the reader to algorithms for performing sparse regression.

\section{Sparse Recovery via Projected Gradient Descent}
The formulation in \eqref{eq:spreg} looks strikingly similar to the convex optimization problem \eqref{eq:cons-cvx-opt} we analyzed in \S~\ref{chap:tools} if we take $f(\bt) = \norm{\y - X\bt}_2^2$ as the (convex) objective function and $\cC = \B_0(s)$ as the (non-convex) constraint set. Given this, it is indeed tempting to adapt Algorithm~\ref{algo:pgd} to solve this problem. The only difference would be that the projection step would now have to project onto a non-convex set $\B_0(s)$. However, as we have seen in \S~\ref{sec:non-cvx-proj}, this can be done efficiently. The resulting algorithm is a variant of the gPGD algorithm (Algorithm~\ref{algo:gpgd}) that we studied in \S~\ref{chap:pgd} and is referred to as \emph{Iterative Hard Thresholding} (IHT) in literature. The IHT algorithm is outlined in Algorithm~\ref{algo:iht}.

\begin{algorithm}[t]
	\caption{Iterative Hard-thresholding (IHT)}
	\label{algo:iht}
	\begin{algorithmic}[1]
			\REQUIRE Data $X, \y$, step length $\eta$, projection sparsity level $k$
			\ENSURE A sparse model $\bth \in \B_0(k)$
			\STATE $\bt^1 \leftarrow \vzero$
			\FOR{$t = 1, 2, \ldots$}
				\STATE $\z^{t+1} \leftarrow \btt - \eta\cdot \frac{1}{n}X^\top(X\btt - \y)$
				\STATE $\btn \leftarrow \Pi_{\B_0(k)}(\z^{t+1})$ \hfill//\texttt{see \S~\ref{sec:non-cvx-proj}}
			\ENDFOR
			\STATE \textbf{return} {$\btt$}
	\end{algorithmic}
\end{algorithm}

It should come as no surprise that Algorithm~\ref{algo:iht} turns out to be extremely simple to implement, as well as extremely fast in execution, given that only gradient steps are required. Indeed, IHT is a method of choice for practitioners given its ease of use and speed. However, much less is clear about the recovery guarantees of this algorithm, especially since we have already stated that \eqref{eq:spreg} is NP-hard to solve \citep{Natarajan1995}. Note that since the problem involves non-convex constraints, Theorem~\ref{thm:pgd-conv-proof} no longer applies. This seems to destroy all hope of proving a recovery guarantee, until one observes that the NP-hardness result does not preclude the possibility of solving this problem efficiently when there is special structure in the design matrix $X$.

Indeed if $X = I_{p\times p}$, it is trivial to recover \emph{any} underlying sparse model $\bto$ by simply returning $\y$. This toy case actually holds the key to efficient sparse recovery. Notice that when $X = I$, the design matrix is an \emph{isometry} -- it completely preserves the geometry of the space $\R^p$. However, one could argue that this is an uninteresting and expensive design with $n = p$. In a long line of illuminating results, which we shall now discuss, it was revealed that even if the design matrix is not a global isometry such as $I$ but a \emph{restricted isometry} that only preserves the geometry of sparse vectors, recovery is possible with $n \ll p$.

The observant reader would be wondering why are we not applying the gPGD analysis from \S~\ref{chap:pgd} directly here, and if notions similar to the RSC/RSS notions discussed there make sense here too. We request the reader to read on. We will find that not only do those notions extend here, but have beautiful interpretations. Moreover, instead of directly applying the gPGD analysis (Theorem~\ref{thm:gpgd-rsc-rss-proof}), we will see a simpler convergence proof tailored to the sparse recovery problem which also gives a sharper result.

\section{Restricted Isometry and Other Design Properties}
As we observed previously, design matrices such as $I_p$ which preserve the geometry of signals/models seem to allow for recovery of sparse signals. We now formalize this intuition further and develop specific conditions on the design matrix $X$ which guarantee \emph{universal recovery} i.e. for every $\bt \in \B_0(s)$, it is possible to uniquely recover $\bt$ from the measurements $X\bt$.

It is easy to see that a design matrix that \emph{identifies} sparse vectors cannot guarantee universal recovery. Suppose we have a design matrix $X \in \R^{n \times p}$ such that for some $\bt_1, \bt_2 \in \B_0(s)$ and $\bt_1 \neq \bt_2$, we get $\y_1  = \y_2$ where $\y_1 = X\bt_1$ and $\y_2 = X\bt_2$. In this case, it is information theoretically impossible to distinguish between $\bt_1$ and $\bt_2$ on the basis of measurements made using $X$ i.e. using $\y_1$ (or $\y_2$). Consequently, this design matrix cannot be used for universal recovery since it produces measurements that confuse between sparse vectors. It can be seen\elink{exer:spreg-confuse} that such a design matrix will not identify just one pair of sparse vectors but an infinite set of pairs (indeed an entire subspace) of sparse vectors.

Thus, it is clear that the design matrix must preserve the geometry of the set of sparse vectors while projecting them from a high $p$-dimensional space to a low $n$-dimensional space. Recall that in sparse recovery settings, we usually have $n \ll p$. The \emph{Nullspace Property} presented below, and the others thereafter, are formalizations of this intuition. For any subset of coordinates $S \subset [p]$, let us define the set $\cC(S) := \bc{\w \in \R^p, \norm{\w_S}_1 \geq \norm{\w_{\bar S}}_1}$. This is the (convex) set of points that place a majority of their weight on coordinates in the set $S$. Define $\cC(k) := \bigcup_{S: |S| = k}\cC(S)$ to be the (non-convex\elink{exer:spreg-c-k-non-conv}) set of points that place a majority of their weight on some $k$ coordinates. Note that $\cC(k) \supset \B_0(k)$ since $k$-sparse vectors put \emph{all} their weight on some $k$ coordinates.

\begin{definition}[Nullspace Property \citep{CohenDDeV2009}]
\label{defn:ns-prop}
A matrix $X \in \R^{n \times p}$ is said to satisfy the null-space property of order $k$ if $\text{ker}(X) \cap \cC(k) = \bc{\vzero}$, where $\text{ker}(X) = \{\w \in \bR^p: X\w = \vzero\}$ is the kernel of the linear transformation induced by $X$ (also called its null-space).
\end{definition}
If a design matrix satisfies this property, then vectors in its null-space are disallowed from concentrating a majority of their weight on any $k$ coordinates. Clearly no $k$-sparse vector is present in the null-space either. If a design matrix has the null-space property of order $2s$, then it can never identify two $s$-sparse vectors\elink{exer:spreg-nsp-iden} -- something that we have already seen as essential to ensure global recovery. A strengthening of the Nullspace Property gives us the Restricted Eigenvalue Property.

\begin{definition}[Restricted Eigenvalue Property \citep{RaskuttiWY2010}]
\label{def:re}
A matrix $X \in \R^{n \times p}$ is said to satisfy the restricted eigenvalue property of order $k$ with constant $\alpha$ if for all $\bt \in \cC(k)$, we have $\frac{1}{n}\norm{X\bt}_2^2 \geq \alpha\cdot\norm{\bt}_2^2$.
\end{definition}
This means that not only are $k$-sparse vectors absent from the null-space, they actually retain a good fraction of their length after projection as well. This means that if $k = 2s$, then for any $\bt_1,\bt_2 \in \B_0(s)$, we have $\frac{1}{n}\norm{X(\bt_1 - \bt_2)}_2^2 \geq \alpha\cdot\norm{\bt_1 - \bt_2}_2^2$. Thus, the distance between \emph{any two} sparse vectors never greatly diminished after projection. Such behavior is the hallmark of an isometry, which preserves the geometry of vectors. The next property further explicates this and is, not surprisingly, called the Restricted Isometry Property.

\begin{definition}[Restricted Isometry Property \citep{CandesT2005}]
A matrix $X \in \R^{n \times p}$ is said to satisfy the restricted isometry property (RIP) of order $k$ with constant $\delta_k \in [0,1)$ if for all $\bt \in \B_0(k)$, we have
\begin{center}
	$(1-\delta_k)\cdot\norm{\bt}_2^2 \leq \frac{1}{n}\norm{X\bt}_2^2 \leq (1+\delta_k)\cdot\norm{\bt}_2^2$.
\end{center}
\end{definition}
The above property is most widely used in analyzing sparse recovery and compressive sensing algorithms. However, it is a bit restrictive since it requires the distortion parameters to be of the kind $(1 \pm \delta)$ for $\delta \in [0,1)$. A generalization of this property that is especially useful in settings where the properties of the design matrix are not strictly controlled by us, such as the gene expression analysis problem, is the following notion of restricted strong convexity and smoothness.

\begin{definition}[Restricted Strong Convexity/Smoothness Property \citep{JainTK2014, JalaliJR2011}]
\label{defn:rsc-rss}
A matrix $X \in \R^{n \times p}$ is said to satisfy the $\alpha$-restricted strong convexity (RSC) property and the $\beta$-restricted smoothness (RSS) property of order $k$ if for all $\bt \in \B_0(k)$, we have
\begin{center}
	$\alpha\cdot\norm{\bt}_2^2 \leq \frac{1}{n}\norm{X\bt}_2^2 \leq \beta\cdot\norm{\bt}_2^2$.
\end{center}
\end{definition}

The only difference between the RIP and the RSC/RSS properties is that the former forces constants to be of the form $1 \pm \delta_k$ whereas the latter does not impose any such constraints. The reader will notice the similarities in the definition of restricted strong convexity and smoothness as given here and Definition~\ref{defn:res-strong-cvx-smooth-fn} where we defined restricted strongly convexity and smoothness notions for general functions. The reader is invited to verify\elink{exer:spreg-rsc-is-rsc} that the two are indeed related.

Indeed, Definition~\ref{defn:res-strong-cvx-smooth-fn} can be seen as a generalization of Definition~\ref{defn:rsc-rss} to general functions \citep{JainTK2014}. For twice differentiable functions, both definitions can be seen as placing restrictions on the (restricted) eigenvalues of the Hessian of the function.

It is a useful exercise to verify\elink{exer:spreg-hierarchy} that these properties fall in a hierarchy: RSC-RSS $\Rightarrow$ REP $\Rightarrow$ NSP for an appropriate setting of constants. We will next establish the main result of this section: if the design matrix satisfies the RIP condition with appropriate constants, then the IHT algorithm does indeed guarantee universal sparse recovery. Subsequently, we will give pointers to recent results that guarantee universal recovery in gene expression analysis-like settings.

\section{Ensuring RIP and other Properties}
\label{sec:rip-ensure}
Since properties such as RIP, RE and RSC are so crucial for guaranteed sparse recovery, it is important to study problem settings in which these are satisfied by actual data. A lot of research has gone into explicit construction of matrices that provably satisfy the RIP property.\\

\noindent\textbf{Random Designs}: The simplest of these results are the so-called random design constructions which guarantee that if the matrix is sampled from certain well behaved distributions, then it will satisfy the RIP property with high probability. For instance, the work of \citet{BaraniukDdVW2008} shows the following result:

\begin{theorem}\cite[Theorem 5.2]{BaraniukDdVW2008}
Let $\cD$ be a distribution over matrices in $\bR^{n \times p}$ such that for any fixed $\vv \in \bR^p, \epsilon > 0$,
\[
\Prr{X \sim \cD^{n \times p}}{\abs{\norm{X\vv}_2^2 - \norm{\vv}_2^2} > \epsilon\cdot\norm{\vv}_2^2} \leq 2\exp(-\Omega(n))
\]
Then, for any $k < p/2$, matrices $X$ generated from this distribution also satisfy the RIP property at order $k$ with constant $\delta$ with probability at least $1 - 2\exp(-\Omega(n))$ whenever $n \geq \Omega\br{\frac{k}{\delta^2}\log\frac{p}{k}}$.
\end{theorem}

Thus, a distribution over matrices that, for every \emph{fixed} vector, acts as an almost isometry with high probability, is also guaranteed to, with very high probability, generate matrices that act as a restricted isometry \emph{simulataneously} over all sparse vectors. Such matrix distributions are easy to construct -- one simply needs to sample each entry of the matrix independently according to one of the following distributions:
\begin{enumerate}
	\item sample each entry from the Gaussian distribution $\cN(0,1/n)$.
	\item set each entry to $\pm 1/\sqrt n$ with equal probability.
	\item set each entry to $0$ w.p. $2/3$ and $\pm \sqrt{3/n}$ w.p. $1/6$.
\end{enumerate}

The work of \citet{AgarwalNW2012} shows that the RSC/RSS properties are satisfied whenever rows of the matrix $X$ are drawn from a sub-Gaussian distribution over $p$-dimensional vectors with a non-singular covariance matrix. This result is useful since it shows that real-life data, which can often be modeled as vectors being drawn from sub-Gaussian distributions, will satisfy these properties with high probability. This is crucial for sparse recovery and other algorithms to be applicable to real life problems such as the gene-expression analysis problem.

If one can tolerate a slight blowup in the number of rows of the matrix $X$, then there exist better constructions with the added benefit of allowing fast matrix vector products. The initial work of \citet{CandesT2005} itself showed that selecting each row of a Fourier transform matrix independently with probability $\bigO{k\frac{\log^6 p}{p}}$ results in an RIP matrix with high probability. More recently, this was improved to $\bigO{k\log^2 k\frac{\log p}{p}}$ in the work of \cite{HavivR17}. A matrix-vector product of a $k$-sparse vector with such a matrix takes only $\bigO{k\log^2 p}$ time whereas a dense matrix filled with Gaussians would have taken up to $\bigO{k^2\log p}$ time. There exist more involved hashing-based constructions that simultaneously offer reduced sample complexity and fast matrix-vector multiplications \citep{NelsonPW2014}.\\

\noindent\textbf{Deterministic Designs}: There exist far fewer and far weaker results for deterministic constructions of RIP matrices. The initial results in this direction all involved constructing \emph{incoherent} matrices. A matrix $X \in \bR^{n \times p}$ with unit norm columns is said to be $\mu$-incoherent if for all $i \neq j \in [p]$, $\ip{X_i}{X_j} \leq \mu$. A $\mu$-incoherent matrix always satisfies\elink{exer:spreg-inc-rip} RIP at order $k$ with parameter $\delta = (k-1)\mu$.

Deterministic constructions of incoherent matrices with $\mu = \bigO{\frac{\log p}{\sqrt n\log n}}$ are well known since the work of \citet{Kashin1975}. However, such constructions require $n = \tilde\Omega\br{\frac{k^2\log^2 p}{\delta^2}}$ rows which is quadratically more than what random designs require. The first result to improve upon these constructions came from the work of \citet{BourgainDFKK2011} which gave deterministic combinatorial constructions that assured the RIP property with $n = \softO{\frac{k^{(2-\epsilon)}}{\delta^2}}$ for some constant $\epsilon > 0$. However, till date, substantially better constructions are not known.

\section{A Sparse Recovery Guarantee for IHT}
We will now establish a convergence result for the IHT algorithm. Although the analysis for the gPGD algorithm (Theorem~\ref{thm:gpgd-rsc-rss-proof}) can be adapted here, the following proof is much more tuned to the sparse recovery problem and offers a tighter analysis and several problem-specific insights.
\begin{theorem}
\label{thm:iht-conv-proof-rip}
Suppose $X \in \R^{n \times p}$ is a design matrix that satisfies the RIP property of order $3s$ with constant $\delta_{3s} < \frac{1}{2}$. Let $\bto \in B_0(s) \subset \R^p$ be any arbitrary sparse vector and let $\y = X\bto$. Then the IHT algorithm (Algorithm~\ref{algo:iht}), when executed with a step length $\eta = 1$, and a projection sparsity level $k = s$, ensures $\norm{\btt - \bto}_2 \leq \epsilon$ after at most $t = \bigO{\log\frac{\norm{\bto}_2}{\epsilon}}$ iterations of the algorithm.
\end{theorem}
\begin{proof}
We start off with some notation. Let $S^\ast := \supp(\bto)$ and $S^t := \supp(\btt)$. Let $I^t := S^t \cup S^{t+1} \cup S^\ast$ denote the union of the supports of the two consecutive iterates and the optimal model. The reason behind defining this quantity is that we are assured that while analyzing this update step, the two \emph{error vectors} $\btt - \bto$ and $\btn - \bto$, which will be the focal point of the analysis, have support within $I^t$. Note that $|I^t| \leq 3s$. Please refer to the notation section at the beginning of this monograph for the interpretation of the notation $\x_I$ and $A_I$ for a vector $\x$, matrix $A$ and set $I$.

With $\eta = 1$, we have (refer to Algorithm~\ref{algo:iht}), $\z^{t+1} = \btt - \frac{1}{n}X^\top(X\btt - \y)$. However, due to the (non-convex) projection step $\btn = \Pi_{\B_0(k)}(\z^{t+1})$, applying projection property-O gives us
\[
\norm{\btn - \z^{t+1}}_2^2 \leq \norm{\bto - \z^{t+1}}_2^2.
\]
Note that none of the other projection properties are applicable here since the set of sparse vectors is a non-convex set. Now, by Pythagoras' theorem, for any vector $\vv \in \R^p$, we have $\norm{\vv}_2^2 = \norm{\vv_I}_2^2 + \norm{\vv_{\bar I}}_2^2$ which gives us
\[
\norm{\btn_I - \z^{t+1}_I}_2^2 + \norm{\btn_{\bar I} - \z^{t+1}_{\bar I}}_2^2  \leq \norm{\bto_I - \z^{t+1}_I}_2^2 + \norm{\bto_{\bar I} - \z^{t+1}_{\bar I}}_2^2
\]
Now it is easy to see that $\btn_{\bar I} = \bto_{\bar I} = \vzero$. Hence we have
\[
\norm{\btn_I - \z^{t+1}_I}_2 \leq \norm{\bto_I - \z^{t+1}_I}_2
\]
Using the fact that $y = X\bto$, and denoting $\bar X = \frac{1}{\sqrt n} X$ we get
\[
\norm{\btn_I - (\btt_I - \bar X_I^\top \bar X(\btt - \bto))}_2 \leq \norm{\bto_I - (\btt_I - \bar X_I^\top \bar X(\btt - \bto))}_2
\]
Adding and subtracting $\bto$ from the expression inside the norm operator on the LHS, rearranging, and applying the triangle inequality for norms gives us
\[
\norm{\btn_I - \bto_I}_2 \leq 2\norm{(\btt_I - \bto_I) - \bar X_I^\top \bar X(\btt - \bto)}_2
\]
As $\btt_I = \btt, \btn_I = \btn$, observing that $X(\btt - \bto) = X_I(\btt - \bto)$ gives us
\begin{align*}
\norm{\btn - \bto}_2 &\leq 2\norm{(I - \bar X_I^\top\bar X_I)(\btt - \bto)}_2\\
										 &\leq 2\br{\norm{\btt - \bto}_2 - \norm{\bar X_I^\top\bar X_I(\btt - \bto)}_2}\\
										 &\leq 2\delta_{3s}\norm{\btt - \bto}_2,
\end{align*}
which finishes the proof. The second inequality above follows due to the triangle inequality and the third inequality follows from the fact that RIP implies\elink{exer:spreg-rip-eigen} that for any $\abs I \leq 3s$, the smallest eigenvalue of the matrix $\bar X_I^\top\bar X_I$ is lower bounded by $(1-\delta_{3s})$.
\end{proof}

We note that this result holds even if the hard thresholding level is set to $k > s$. It is easy to see that the condition $\delta_{3s} < \frac 12$ is equivalent to the \emph{restricted condition number} (over $3s$-sparse vectors) of the corresponding sparse recovery problem being upper bounded by $\kappa_{3s} < 3$. Similar to Theorem~\ref{thm:gpgd-rsc-rss-proof}, here also we require an upper bound on the restricted condition number of the problem. It is interesting to note that a direct application of Theorem~\ref{thm:gpgd-rsc-rss-proof} would have instead required $\delta_{2s} < \frac 13$ (or equivalently $\kappa_{2s} < 2$) which can be shown to be a harsher requirement than what we have achieved. Moreover, applying Theorem~\ref{thm:gpgd-rsc-rss-proof} would have also required us to set the step length to a specific quantity $\eta = \frac{1}{1+\delta_s}$ while executing the gPGD algorithm whereas while executing the IHT algorithm, we need only set $\eta = 1$.

An alternate proof of this result appears in the work of \cite{GargK2009} which also requires the condition $\delta_{2s} < \frac{1}{3}$. The above result extends to a more general setting where there is additive noise in the model $\y = X\bto + \veta$. In this setting, it is known (see for example, \cite[][Theorem 3]{JainTK2014} or \cite[][Theorem 2.3]{GargK2009}) that if the objective function in question (for the sparse recovery problem the objective function is the least squares objective) satisfies the $(\alpha,\beta)$ RSC/RSS properties at level $2s$, then the following is guaranteed for the output $\hat\bt$ of the IHT algorithm (assuming the algorithm is run for roughly $\bigO{\log n}$ iterations)
\[
\norm{\hat\bt - \bto}_2 \leq \frac{3\sqrt s}{\alpha}\norm{\frac{X^\top\veta}{n}}_\infty
\]
The consistency of the above solution can be verified in several interesting situations. For example, if the design matrix has normalized columns i.e. $\norm{X_i}_2 \leq \sqrt{n}$ and the noise $\eta_i$ is generated i.i.d. and independently of the design $X$ from some Gaussian distribution $\cN(0,\sigma^2)$, then the quantity $\norm{\frac{X^\top\veta}{n}}_\infty$ is of the order of $\sigma\sqrt{\frac{\log p}{n}}$ with high probability. In the above setting IHT guarantees with high probability
\[
\norm{\hat\bt - \bto}_2 \leq \softO{\frac{\sigma}{\alpha}\sqrt\frac{s\log p}{n}},
\]
i.e. $\norm{\hat\bt - \bto}_2 \rightarrow 0$ as $n \rightarrow \infty$, thus establishing consistency.

\section{Other Popular Techniques for Sparse Recovery}
\label{sec:spreg-other}
The IHT method is a part of a larger class of \emph{hard thresholding techniques}, which include algorithms such as Iterative Hard Thresholding (IHT) \citep{Blumensath2011}, Gradient Descent with Sparsification (GraDeS) \citep{GargK2009}, and Hard Thresholding Pursuit (HTP) \citep{Foucart2011}. Apart from these gradient descent-style techniques, several other approaches have been developed for the sparse recovery problem over the years. Here we briefly survey them.

\subsection{Pursuit Techniques}
A popular non-convex optimization technique for sparse recovery and a few related optimization problems is that of discovering support elements iteratively. This technique is embodied in pursuit-style algorithms. We warn the reader that the popular \emph{Basis Pursuit} algorithm is actually a convex relaxation technique and not related to the other pursuit algorithms we discuss here. The terminology is a bit confusing but seems to be a matter of legacy.

The pursuit family of algorithms includes Orthogonal Matching Pursuit (OMP) \citep{TroppG2007}, Orthogonal Matching Pursuit with Replacement (OMPR) \citep{JainTD2011}, Compressive Sampling Matching Pursuit (CoSaMP) \citep{NeedellT2008}, and the Forward-backward (FoBa) algorithm \citep{Zhang2011}.

Pursuit methods work by gradually discovering the elements in the support of the true model vector $\bto$. At every time step, these techniques add a new support element to an active support set (which is empty to begin with) and solve a traditional least-squares problem on the active support set. This least-squares problem has no sparsity constraints, and is hence a convex problem which can be solved easily.

The support set is then updated by adding a new support element. It is common to add the coordinate where the gradient of the objective function has the highest magnitude among coordinates not already in the support. FoBa-style techniques augment this method by having \emph{backward} steps where support elements that were erroneously picked earlier are discarded when the error is detected.

Pursuit-style methods are, in general, applicable whenever the structure in the (non-convex) constraint set in question can be represented as a combination of a small number of \emph{atoms}. Examples include sparse recovery, where the atoms are individual coordinates: every $s$-sparse vector is a linear combination of some $s$ of these atoms.

Other examples include low-rank matrix recovery, which we will study in detail in \S~\ref{chap:matrec}, where the atoms are rank-one matrices. The SVD theorem tells us that every $r$-rank matrix can indeed be expressed as a sum of $r$ rank-one matrices. There exist works \citep{TewariRD2011} that give generic methods to perform sparse recovery in such structurally constrained settings.

\subsection{Convex Relaxation Techniques for Sparse Recovery}
Convex relaxation techniques have been extremely popular for the sparse recovery problem. In fact they formed the first line of attack on non-convex optimization problems starting with the seminal work of \cite{CandesT2005,CandesRT2006,Donoho2006} that, for the first time, established polynomial time, globally convergent solutions for the compressive sensing problem.

A flurry of work then followed on relaxation-based techniques \citep{Candes2008,DonohoMM2009,Foucart2010,NegahbanRWY2012} that vastly expanded the scope of the problems being studied, the techniques being applied, as well as their analyses. It is important to note that all methods, whether relaxation based or not, have to assume some design property such as NSP/REP/RIP/RSC-RSS that we discussed earlier, in order to give provable guarantees.

The relaxation approach converts non-convex problems to convex problems first before solving them. This approach, applied to the sparse regression problem, gives us the so-called LASSO problem which has been studied extensively. Consider the sparse recovery problem \eqref{eq:spreg}.
\begin{equation*}
\underset{\substack{\bt\in\R^p\\\norm{\bt}_0 \leq s}}{\min}\ \norm{\y - X\bt}_2^2.
\end{equation*}
Non-convexity arises in the problem due to the non-convex constraint $\norm{\bt}_0 \leq s$ as the sparsity operator is not a valid norm. The relaxation approach fixes this problem by changing the constraint to use the $L_1$ norm instead i.e.
\begin{equation}
\underset{\substack{\bt\in\R^p\\\norm{\bt}_1 \leq R}}{\min}\ \norm{\y - X\bt}_2^2,\tag*{(LASSO-1)}\label{eq:lasso-1}
\end{equation}
or by using its regularized version instead
\begin{equation}
\underset{\substack{\bt\in\R^p}}{\min}\ \frac{1}{2n}\norm{\y - X\bt}_2^2 + \lambda_n\norm{\bt}_1.\tag*{(LASSO-2)}\label{eq:lasso-2}
\end{equation}
The choice of the $L_1$ norm is motivated mainly by its convexity as well as formal results that assure us that the relaxation gap is small or non-existent. Both the above formulations \eqref{eq:lasso-1} and \eqref{eq:lasso-2} are indeed convex but include parameters such as $R$ and $\lambda_n$ that must be tuned properly to ensure proper convergence. Although the optimization problems \eqref{eq:lasso-1} and \eqref{eq:lasso-2} are vastly different from \eqref{eq:spreg}, a long line of beautiful results, starting from the seminal work of \cite{CandesT2005,CandesRT2006,Donoho2006}, showed that if the design matrix $X$ satisfies RIP with appropriate constants, and if the parameters of the relaxations $R$ and $\lambda_n$ are appropriately tuned, then the solutions to the relaxations are indeed solutions to the original problems as well.

Below we state one such result from the recent text by \citet{HastieTW2016}. We recommend this text to any reader looking for a well-curated compendium of techniques and results on the relaxation approach to several non-convex optimization problems arising in machine learning.

\begin{theorem}\cite[Theorem 11.1]{HastieTW2016}
Consider a sparse recovery problem $\y = X\bto + \veta$ where the model $\bto$ is $s$-sparse and the design matrix $X$ satisfies the restricted-eigenvalue condition (see Definition~\ref{def:re}) of the order $s$ with constant $\alpha$, then the following hold
\begin{enumerate}
	\item Any solution $\hat\bt_1$ to \eqref{eq:lasso-1} with $R = \norm{\bto}_1$ satisfies
	\[
	\norm{\hat\bt_1 - \bto}_2 \leq \frac{4}{\alpha}\sqrt s \norm{\frac{X^\top\veta}{n}}_\infty.
	\]
	\item Any solution $\hat\bt_2$ to \eqref{eq:lasso-2} with $\lambda_n \geq 2\norm{X^\top\veta/n}_\infty$ satisfies
	\[
	\norm{\hat\bt_1 - \bto}_2 \leq \frac{3}{\alpha}\sqrt s \lambda_n.
	\]
\end{enumerate}
\end{theorem}

The reader can verify that the above bounds are competitive with the bounds for the IHT algorithm that we discussed previously. We refer the reader to \cite[Chapter 11]{HastieTW2016} for more consistency results for the LASSO formulations.

\subsection{Non-convex Regularization Techniques}
Instead of performing a complete convex relaxation of the problem, there exist approaches that only partly relax the problem. A popular approach in this direction uses $L_q$ regularization with $0 < q < 1$ \citep{Chartrand2007,Foucart-Lai2009,WangXT2011}. The resulting problem still remains non-convex but becomes a little well behaved in terms of having objective functions that are almost-everywhere differentiable. For instance, the following optimization problem may be used to solve the noiseless sparse recovery problem.
\begin{align*}
\underset{\bt\in\R^p}{\min}&\ \norm{\bt}_q,\\
\text{s.t.}&\ \y = X\bt
\end{align*}
For noisy settings, one may replace the constraint with a soft constraint such as $\norm{\y - X\bt}_2 \leq \epsilon$, or else move to an unconstrained version like LASSO with the $L_1$ norm replaced by the $L_q$ norm. The choice of the regularization norm $q$ is dictated by application and usually any value within a certain range within the interval $(0,1)$ can be chosen.

There has been interest in characterizing both the global and the local optima of these optimization problems for their recovery properties \citep{Chen-Gu2015}. In general, $L_q$ regularized formulations, if solved exactly, can guarantee recovery under much weaker conditions than what LASSO formulations, and IHT require. For instance, the RIP condition that $L_q$-regularized formulations need in order to guarantee universal recovery can be as weak as $\delta_{2k+1} < 1$ \citep{Chartrand2007}. This is very close to the requirement $\delta_{2k} < 1$ that must be made by any algorithm in order to ensure that the solution even be unique. However, solving these non-convex regularized problems at large scale itself remains challenging and an active area of research.

\begin{figure}[t]
\centering \includegraphics[width=0.5\columnwidth]{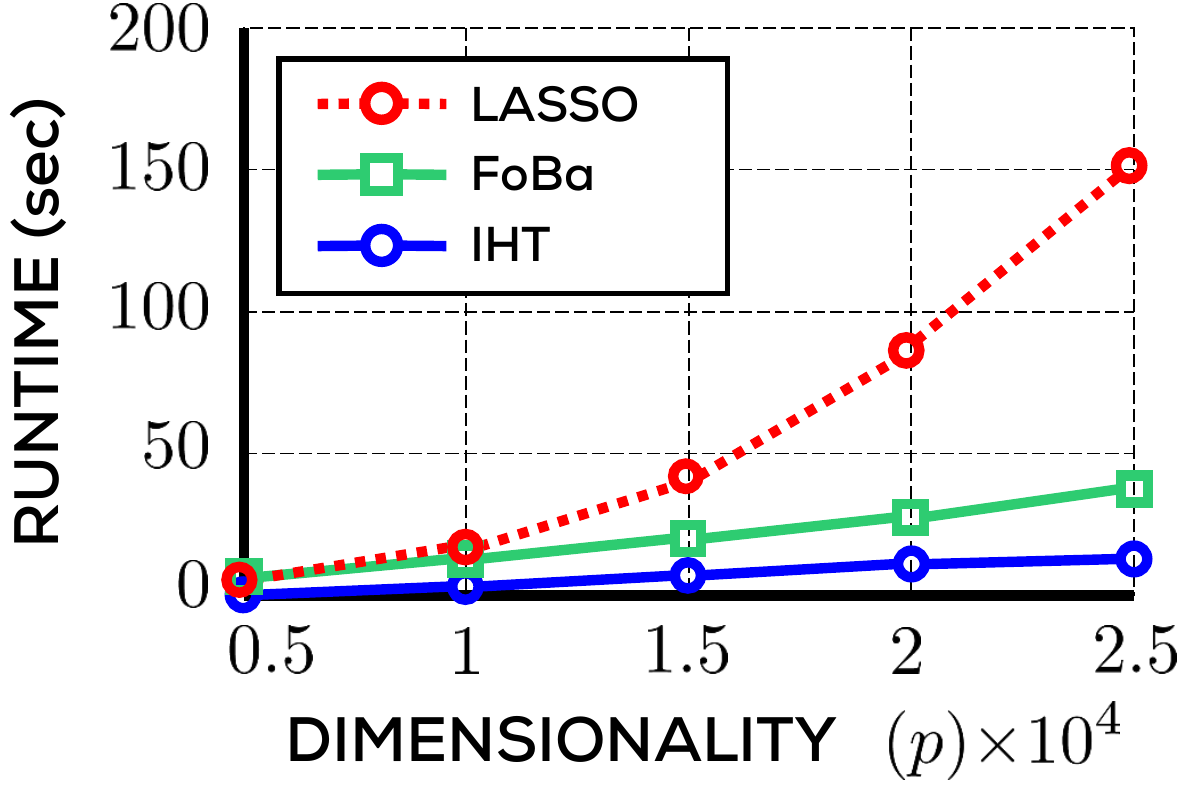}
\caption[Relaxation vs. Non-convex Methods for Sparse Recovery]{An empirical comparison of run-times offered by the LASSO, FoBA and IHT methods on sparse regression problems with varying dimensionality $p$. All problems enjoyed sparsity $s = 100$ and were offered $n = 2s\cdot\log p$ data points. IHT is clearly the most scalable of the methods followed by FoBa. The relaxation technique does not scale very well to high dimensions. Figure adapted from \citep{JainTK2014}.}%
\label{fig:spreg-comparison-spreg}
\end{figure}

\subsection{Empirical Comparison}
To give the reader an appreciation of the empirical performance of the various methods we have discussed for sparse recovery, Figure~\ref{fig:spreg-comparison-spreg} provides a comparison of some of these methods on a synthetic sparse regression problem. The graph plots the running time taken by the various methods to solve the same sparse linear regression problem with sparsity $s = 100$ but with dimensionalities increasing from $p = 5000$ to $p= 25000$. The graph indicates that non-convex optimization methods such as IHT and FoBa are far more scalable than relaxation-based methods. It should be noted that although pursuit-style techniques are scalable, they can become sluggish if the true support set size $s$ is not very small since these techniques discover support elements one by one.

\section{Extensions}
\label{sec:spreg-ext}
In the preceding discussion, we studied the problem of sparse linear regression and the IHT technique to solve the problem. These basic results can be augmented and generalized in several ways. The work of \cite{NegahbanRWY2012} greatly expanded the scope of sparse recovery techniques beyond simple least-squares to the more general M-estimation problem. The work of \cite{BhatiaJK2015} offered solutions to the \emph{robust} sparse regression problem  where the responses may be corrupted by an adversary. We will explore the robust regression problem in more detail in \S~\ref{chap:rreg}. We discuss a few more such extensions below.

\subsection{Sparse Recovery in Ill-Conditioned Settings}
As we discussed before, the bound on the RIP constant $\delta_{3s} < \frac{1}{2}$ as required by Theorem~\ref{thm:iht-conv-proof-rip}, effectively places a bound on the \emph{restricted condition number} $\kappa_{3s}$ of the design matrix. In our case the bound translates to $\kappa_{3s} = \frac{1 + \delta_{3s}}{1 - \delta_{3s}} < 3$. However, in cases such as the gene expression analysis problem where the design matrix is not totally under our control, the restricted condition number might be much larger than $3$.

For instance, it can be shown that if the expression levels of two genes are highly correlated then this results in ill-conditioned design matrices. In such settings, it is much more appropriate to assume that the design matrix satisfies restricted strong convexity and smoothness (RSC/RSS) which allows us to work with design matrices with arbitrarily large condition numbers. It turns out that the IHT algorithm can be modified \cite[see for example,][]{JainTK2014} to work in these ill-conditioned recovery settings.

\begin{theorem}
\label{thm:iht-conv-proof-rsc-rss}
Suppose $X \in \R^{n \times p}$ is a design matrix that satisfies the restricted strong convexity and smoothness property of order $2k+s$ with constants $\alpha_{2k+s}$ and $\beta_{2k+s}$ respectively. Let $\bto \in B_0(s) \subset \R^p$ be any arbitrary sparse vector and let $\y = X\bto$. Then the IHT algorithm, when executed with a step length $\eta < \frac{2}{\beta_{2k+s}}$, and a projection sparsity level $k \geq 32\br{\frac{\beta_{2k+s}}{\alpha_{2k+s}}}^2s$, ensures $\norm{\btt - \bto}_2 \leq \epsilon$ after $t = \bigO{\frac{\beta_{2k+s}}{\alpha_{2k+s}}\log\frac{\norm{\bto}_2}{\epsilon}}$ iterations of the algorithm.
\end{theorem}

Note that the above result does not place any restrictions on the condition number or the RSC/RSS constants of the problem. The result also mimics Theorem~\ref{thm:gpgd-rsc-rss-proof} in its dependence on the (restricted) condition number of the optimization problem i.e. $\kappa_{2k+s} = \frac{\beta_{2k+s}}{\alpha_{2k+s}}$. The proof of this result is a bit tedious, hence omitted.

\subsection{Recovery from a Union of Subspaces}
If we look closely, the set $\B_0(s)$ is simply a union of $p \choose s$ linear subspaces, each subspace encoding a specific sparsity pattern. It is natural to wonder whether the methods and analyses described above also hold when the vector to be recovered belongs to a general union of subspaces. More specifically, consider a family of linear subspaces $\cH_1,\ldots,\cH_L \subset \R^p$ and denote the union of these subspaces by $\cH = \bigcup_{i=1}^L\cH_i$. The restricted strong convexity and restricted strong smoothness conditions can be appropriately modified to suit this setting by requiring a design matrix $X: \R^p \rightarrow \R^n$ to satisfy, for every $\bt_1, \bt_2 \in \cH$,
\[
\alpha\cdot\norm{\bt_1 - \bt_2}_2^2 \leq \norm{X(\bt_1 - \bt_2)}_2^2 \leq L\cdot\norm{\bt_1 - \bt_2}_2^2
\]
It turns out that IHT, with an appropriately modified projection operator $\Pi_{\cH}(\cdot)$, can ensure recovery of vectors that are guaranteed to reside in a small union of low-dimensional subspaces. Moreover, a linear rate of convergence, as we have seen for the IHT algorithm in the sparse regression case, can still be achieved. We refer the reader to the work of \cite{Blumensath2011} for more details of this extension.

\subsection{Dictionary Learning}
A very useful extension of sparse recovery, or sparse \emph{coding} emerges when we attempt to learn the design matrix as well. Thus, all we are given are observations $\vy_1,\ldots,\vy_m \in \bR^n$ and we wish to learn a design matrix $X \in \bR^{n \times p}$ such that the observations $\vy_i$ can be represented as sparse combinations $\vw_i \in \bR^p$ of the columns of the design matrix i.e. $\vy_i \approx X\vw_i$ such that $\norm{\vw_i}_0 \leq s \ll p$. The problem has several applications in the fields of computer vision and signal processing and has seen a lot of interest in the recent past.

The alternating minimization technique where one alternates between estimating the design matrix and the sparse representations, is especially popular for this problem. Methods mostly differ in the exact implementation of these alternations. Some notable works in this area include \citep{AgarwalAJN2016,AroraGM2014,GribonvalJB2015,SpielmanWW2012}.

\section{Exercises}
\begin{exer}
\label{exer:spreg-confuse}
Suppose a design matrix $X \in \bR^{n \times p}$ satisfies $X\w_1 = X\w_2$ for some $\w_1 \neq \w_2 \in \bR^p$. Then show that there exists an entire subspace $\cH \subset \bR^p$ such that for all $\w,\w' \in \cH$, we have $X\w = X\w'$.
\end{exer}
\begin{exer}
\label{exer:spreg-c-k-non-conv}
Show that the set $\cC(k) := \bigcup_{S: |S| = k}\cC(S)$ is non-convex.
\end{exer}
\begin{exer}
\label{exer:spreg-nsp-iden}
Show that if a design matrix $X$ satisfies the null-space property of order $2s$, then for any two distinct $s$-sparse vectors $\vv^1,\vv^2 \in \cB_0(s)$, $\vv^1 \neq \vv^2$, it must be the case that $X\vv^1 \neq X\vv^2$. 
\end{exer}
\begin{exer}
\label{exer:spreg-rsc-is-rsc}
Show that the RSC/RSS notion introduced in Definition~\ref{defn:rsc-rss} is equivalent to the RSC/RSS notion in Definition~\ref{defn:res-strong-cvx-smooth-fn} defined in \S~\ref{chap:pgd} for an appropriate choice of function and constraint sets.
\end{exer}
\begin{exer}
\label{exer:spreg-hierarchy}
Show that RSC-RSS $\Rightarrow$ REP $\Rightarrow$ NSP i.e. a matrix that satisfies the RSC/RSS condition for some constants, must satisfy the REP condition for some constants which in turn must force it to satisfy the null-space property.
\end{exer}
\begin{exer}
\label{exer:spreg-inc-rip}
Show that every $\mu$-incoherent matrix satisfies the RIP property at order $k$ with parameter $\delta = (k-1)\mu$.
\end{exer}
\begin{exer}
\label{exer:spreg-rip-eigen} 
Suppose the matrix $X \in \R^{n \times p}$ satisfies RIP at order $s$ with constant $\delta_s$. Then show that for any set $I \subset [p], \abs I \leq s$, the smallest eigenvalue of the matrix $X_I^\top X_I$ is lower bounded by $(1-\delta_{s})$.
\end{exer}
\begin{exer}
\label{exer:spreg-monotone}
Show that the RIP constant is monotonic in its order i.e. if a matrix $X$ satisfies RIP of order $k$ with constant $\delta_k$, then it also satisfies RIP for all orders $k' \leq k$ with $\delta_{k'} \leq \delta_{k}$.
\end{exer}

\section{Bibliographic Notes}
\label{sec:spreg-bib}
The literature on sparse recovery techniques is too vast for this note to cover. We have already covered several directions in \S\S~\ref{sec:spreg-other} and \ref{sec:spreg-ext} and point the reader to references therein.
\chapter{Low-rank Matrix Recovery}
\label{chap:matrec}

In this section, we will look at the problem of low-rank matrix recovery in detail. Although simple to motivate as an extension of the sparse recovery problem that we studied in \S~\ref{chap:spreg}, the problem rapidly distinguishes itself in requiring specific tools, both algorithmic and analytic. We will start our discussion with a milder version of the problem as a warm up and move on to the problem of low-rank matrix completion which is an active area of research.

\section{Motivating Applications}
We will take the following two running examples to motivate the problem of low-rank matrix recovery.\\

\noindent\textbf{Collaborative Filtering} Recommendation systems are popularly used to model the preference patterns of users, say at an e-commerce website, for items being sold on that website, although the principle of recommendation extends to several other domains that demand \emph{personalization} such as education and healthcare. Collaborative filtering is a popular technique for building recommendation systems.

The collaborative filtering approach seeks to exploit co-occurring patterns in the observed behavior across users in order to predict future user behavior. This approach has proven successful in addressing users that interact very sparingly with the system. Consider a set of $m$ users $u_1,\ldots,u_m$, and $n$ items $a_1,\ldots,a_n$. Our goal is to predict the preference score $s_{(i,j)}$ that is indicative of the interest user $u_i$ has in item $a_j$.

However, we get direct access to (noisy estimates of) actual preference scores for only a few items per user by looking at clicks, purchases etc. That is to say, if we consider the $m \times n$ \emph{preference matrix}  $A = [A_{ij}]$ where $A_{ij} = s_{(i,j)}$ encodes the (true) preference of the $i\th$ user for the $j\th$ item, we get to see only $k \ll m\cdot n$ entries of $A$, as depicted in Figure~\ref{fig:matrec-cf}. Our goal is to recover the remaining entries.

\begin{figure}
\includegraphics[width=\columnwidth]{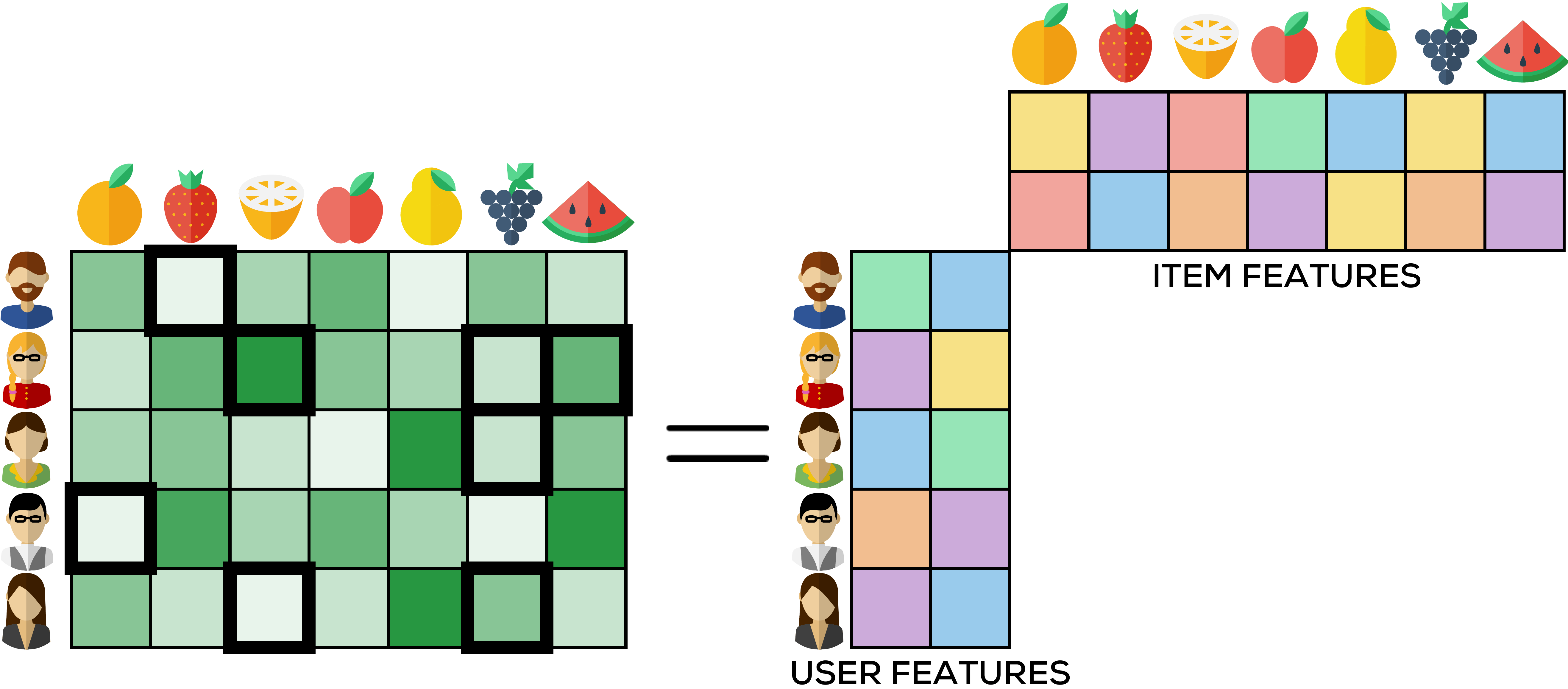}
\caption[Low-rank Matrix Completion for Recommendation]{In a typical recommendation system, users rate items very infrequently and certain items may not get rated even once. The figure depicts a ratings matrix. Only the matrix entries with a bold border are observed. Low-rank matrix completion can help recover the unobserved entries, as well as reveal hidden features that are descriptive of user and item properties, as shown on the right hand side.}%
\label{fig:matrec-cf}
\end{figure}

The problem of paucity of available data is readily apparent in this setting. In its nascent form, the problem is not even well posed and does not admit a unique solution. A popular way of overcoming these problems is to assume a low-rank structure in the preference matrix. 

As we saw in Exercise~\ref{exer:pgd-low-rank}, this is equivalent to assuming that there is an $r$-dimensional vector $\vu_i$ denoting the $i\th$ user and an $r$-dimensional vector $\va_j$ denoting the $j\th$ such that $s_{(i,j)} \approx \ip{\vu_i}{\va_j}$. Thus, if $\Omega \subset [m] \times [n]$ is the set of entries that have been observed by us, then the problem of recovering the unobserved entries can be cast as the following optimization problem:
\[
\underset{\substack{X \in \R^{m \times n}\\\rank(X) \leq r}}{\min}\ \sum_{(i,j) \in \Omega}\br{X_{ij} - A_{ij}}^2.
\]
This problem can be shown to be NP-hard \citep{HardtMRW2014} but has generated an enormous amount of interest across communities. We shall give special emphasis to this \emph{matrix completion} problem in the second part of this section.\\

\noindent\textbf{Linear Time-invariant Systems} Linear Time-invariant (LTI) systems are widely used in modeling dynamical systems in fields such as engineering and finance. The \emph{response behavior} of these systems is characterized by a model vector $\vh = [h(0),h(1),\ldots,h(2N-1)]$. The \emph{order} of such a system is given by the rank of the following Hankel matrix
\[
\text{hank}(\vh) =\left[
\begin{array}{cccc}
	h(0) & h(1) & \ldots & h(N)\\
	h(1) & h(2) & \ldots & h(N+1)\\
	\vdots & \vdots & \ddots & \vdots\\
	h(N-1) & h(N) & \ldots & h(2N-1)
\end{array}
\right]
\]
Given a sequence of inputs $\va = [a(1),a(2),\ldots,a(N)]$ to the system, the output of the system is given by
\[
y(N) = \sum_{t=0}^{N-1}a(N-t)h(t)
\]
In order to recover the model parameters of a system, we repeatedly apply i.i.d. Gaussian impulses $a(i)$ to the system for $N$ time steps and then observe the output of the system. This process is repeated, say $k$ times, to yield observation pairs $\bc{(\va^i,y^i)}_{i=1}^k$. Our goal now, is to take these observations and identify an LTI vector $\vh$ that best fits the data. However, for the sake of accuracy and ease of analysis \citep{FazelPST2013}, it is advisable to fit a low-order model to the data. Let the matrix $A \in \R^{k\times N}$ contain the i.i.d. Gaussian impulses applied to the system. Then the problem of fitting a low-order model can be shown to reduce to the following constrained optimization problem with a rank objective and an affine constraint.
\begin{equation*}
\begin{array}{cc}
	\min & \rank(\text{hank}(\vh))\\
	\text{s.t.} & A\vh = \vy,
\end{array}
\end{equation*}
The above problem is a non-convex optimization problem due to the objective being the minimization of the rank of a matrix. Several other problems in metric embedding and multi-dimensional scaling, image compression, low rank kernel learning and spatio-temporal imaging can also be reduced to low rank matrix recovery problems \citep{JainMD2010, RechtFP2010}.

\section{Problem Formulation}
The two problems considered above can actually be cast in a single problem formulation, that of \emph{Affine Rank Minimization} (ARM). Consider a low rank matrix $X^\ast \in \R^{m \times n}$ that we wish to recover and an affine transformation $\cA: \R^{m \times n} \rightarrow \R^k$. The transformation can be seen as a concatenation of $k$ real valued affine transformations $\cA_i: \R^{m \times n} \rightarrow \R$. We are given the transformation $\cA$, and its (possibly noisy) action $\y = \cA(X^\ast) \in \R^k$ on the matrix $X^\ast$ and our goal is to recover this matrix by solving the following optimization problem.
\begin{equation}
\begin{array}{cl}
	\min & \rank(X)\\
	\text{s.t.} & \cA(X) = \y,
\end{array}
\tag*{(ARM)}\label{eq:arm}
\end{equation}
This problem can be shown to be NP-hard due to a reduction to the sparse recovery problem\elink{exer:matrec-np-hard}. The LTI modeling problem can be easily seen to be an instance of ARM with the Gaussian impulses being delivered to the system resulting in a $k$-dimensional affine transformation of the Hankel matrix corresponding to the system. However, the Collaborative Filtering problem is also an instance of ARM. To see this, for any $(i,j) \in [m]\times[n]$, let $O^{(i,j)} \in \bR^{m\times n}$ be the matrix such that its $(i,j)$-th entry $O^{(i,j)}_{ij}=1$ and all other entries are zero. Then, simply define the affine transformation $\cA_{(i,j)} : X \mapsto \text{tr}(X^\top O^{(i,j)}) = X_{ij}$. Thus, if we observe $k$ user-item ratings, the ARM problem effectively operates with a $k$-dimensional affine transformation of the underlying rating matrix.

Due to its similarity to the sparse recovery problem, we will first discuss the general ARM problem. However, we will find it beneficial to cast the collaborative filtering problem as a \emph{Low-rank Matrix Completion} problem instead. In this problem, we have an underlying low rank matrix $X^\ast$ of which, we observe entries in a set $\Omega \subset {[m]\times[n]}$. Then the low rank matrix completion problem can be stated as
\begin{equation}
\underset{\substack{X \in \R^{m \times n}\\ \rank(X) \leq r}}{\min}\ \norm{\Pi_{\Omega}(X - X^\ast)}_F^2,\tag*{(LRMC)}\label{eq:lrmc}
\end{equation}
where $\Pi_\Omega(X)$ is defined, for any matrix $X$ as
\[
\Pi_\Omega(X)_{i,j} = \left\{
\begin{array}{cc}
	X_{i,j} & \text{ if } (i,j) \in \Omega\\
	0 & \text{ otherwise}.
\end{array}
\right.
\]
The above formulation succinctly captures our objective to find a \emph{completion} of the ratings matrix that is both, low rank, as well as agrees on the user ratings that are actually observed. As pointed out earlier, this problem is NP-hard \citep{HardtMRW2014}.

Before moving on to present algorithms for the ARM and LRMC problems, we discuss some matrix design properties that would be required in the convergence analyses of the algorithms.

\section{Matrix Design Properties}
Similar to sparse recovery, there exist design properties that ensure that the general NP-hardness of the ARM and LRMC problems can be overcome in well-behaved problem settings. In fact given the similarity between ARM and sparse recovery problems, it is tempting to try and import concepts such as RIP into the matrix-recovery setting.

In fact this is exactly the first line of attack that was adopted in literature. What followed was a beautiful body of work that generalized, both structural notions such as RIP, as well as algorithmic techniques such as IHT, to address the ARM problem. Given the generality of these constructs, as well as the smooth transition it offers having studied sparse recovery, we feel compelled to present them to the reader.

\subsection{The Matrix Restricted Isometry Property}
The generalization of RIP to matrix settings, referred to as matrix RIP, follows in a relatively straightforward manner and was first elucidated by \cite{RechtFP2010}. Quite in line with the sparse recovery setting, the intuition dictates that recovery should be possible only if the affine transformation does not identify two distinct low-rank matrices. A more robust version dictates that no low-rank matrix should be distorted significantly by this transformation which gives us the following.

\begin{definition}[Matrix Restricted Isometry Property \citep{RechtFP2010}]
\label{defn:matrix-rip}
A linear map $\cA: \R^{m \times n} \rightarrow \R^k$ is said to satisfy the matrix restricted isometry property of order $r$ with constant $\delta_r \in [0,1)$ if for all matrices $X$ of rank at most $r$, we have
\begin{center}
	$(1-\delta_r)\cdot\norm{X}_F^2 \leq \norm{\cA(X)}_2^2 \leq (1+\delta_r)\cdot\norm{X}_F^2$.
\end{center}
\end{definition}
Furthermore, the work of \cite{RechtFP2010} also showed that linear maps or affine transformations arising in random measurement models, such as those in image compression and LTI systems, do satisfy RIP with requisite constants whenever the number of affine measurements satisfies $k = \bigO{nr}$ \citep{OymakRS15a}. Note however, that these are settings in which the design of the affine map is within our control. For settings, where the restricted condition number of the affine map is not within our control, more involved analysis is required. The bibliographic notes point to some of these results.

Given the relatively simple extension of the RIP definitions to the matrix setting, it is all the more tempting to attempt to apply gPGD-style techniques to solve the ARM problem, particularly since we saw how IHT succeeded in offering scalable solutions to the sparse recovery problem. The works of \citep{GoldfarbM2011,JainMD2010} showed that this is indeed possible. We will explore this shortly.

\subsection{The Matrix Incoherence Property}
We begin this discussion by warning the reader that there are two distinct notions prevalent in literature, both of which are given the same name, that of the matrix incoherence property. The first of these notions was introduced in \S~\ref{sec:rip-ensure} as a property that can be used to ensure the RIP property in matrices. However a different property, but bearing the same name, finds application in matrix completion problems which we now introduce. We note that the two properties are not known to be strongly related in a formal sense and the coincidental clash of the names seems to be a matter of legacy.

Nevertheless, the intuition behind the second notion of matrix incoherence is similar to that for RIP in that it seeks to make the problem well posed. Consider the matrix $A = \sum_{t=1}^r s_t\cdot\ve_{i_t}\bar\ve_{j_t}^\top \in \bR^{m \times n}$ where $\ve_i$ are the canonical orthonormal vectors in $\bR^m$ and $\bar\ve_j$ are the canonical orthonormal vectors in $\bR^n$. Clearly $A$ has rank at most $r$.

However, this matrix $A$ is non-zero only at $r$ locations. Thus, it is impossible to recover the entire matrix uniquely unless these very $r$ locations $\bc{(i_t,j_t)}_{t=1,\ldots,r}$ are actually observed. Since in recommendation settings, we only observe a few random entries of the matrix, there is a good possibility that none of these entries will ever be observed. This presents a serious challenge for the matrix completion problem -- the low rank structure is not sufficient to ensure unique recovery!

To overcome this and make the LRMC problem well posed with a unique solution, an additional property is imposed. This so-called matrix incoherence property prohibits low rank matrices that are also sparse. A side effect of this imposition is that for incoherent matrices, observing a small random set of entries is enough to uniquely determine the unobserved entries of the matrix.

\begin{definition}[Matrix Incoherence Property \citep{CandesR2009}]
\label{defn:matrix-incoherence}
A matrix $A \in \R^{m \times n}$ of rank $r$ is said to be incoherent with parameter $\mu$ if its left and right singular matrices have bounded row norms. More specifically, let $A = U \Sigma V^\top$ be the SVD of $A$. Then $\mu$-incoherence dictates that $\norm{U^i}_2 \leq \frac{\mu\sqrt r}{\sqrt m}$ for all $i \in [m]$ and $\norm{V^j}_2 \leq \frac{\mu\sqrt r}{\sqrt n}$ for all $j \in [n]$. A stricter version of this property requires all entries of $U$ to satisfy $\abs{U_{ij}} \leq \frac{\mu}{\sqrt m}$ and all entries of $V$ to satisfy $\abs{V_{ij}} \leq \frac{\mu}{\sqrt n}$.
\end{definition}

A low rank incoherent matrix is guaranteed to be \emph{far}, i.e., well distinguished, from any sparse matrix, something that is exploited by algorithms to give guarantees for the LRMC problem.

\begin{algorithm}[t]
	\caption{Singular Value Projection (SVP)}
	\label{algo:svp}
	\begin{algorithmic}[1]
			\REQUIRE Linear map $\cA$, measurements $\y$, target rank $q$, step length $\eta$
			\ENSURE A matrix $\hat X$ with rank at most $q$
			\STATE $X^1 \leftarrow \vzero^{m \times n}$
			\FOR{$t = 1, 2, \ldots$}
				\STATE $Y^{t+1} \leftarrow X^t - \eta\cdot \cA^\top(\cA(X^t) - \y)$
				\STATE Compute top $q$ singular vectors/values of $Y^{t+1}$: $U^t_q,\Sigma^t_q,V^t_q$
				\STATE $X^{t+1} \leftarrow U^t_q\Sigma^t_q(V^t_q)^\top$
			\ENDFOR
			\STATE \textbf{return} {$X^t$}
	\end{algorithmic}
\end{algorithm}

\section{Low-rank Matrix Recovery via Proj. Gradient Descent}
We will now apply the gPGD algorithm to the ARM problem. To do so, first consider the following reformulation of the ARM problem to make it more compatible to the projected gradient descent iterations.
\begin{equation}
\begin{array}{cl}
	\min & \frac{1}{2}\norm{\cA(X) - \y}_2^2\\
	\text{s.t.} & \rank(X) \leq r
\end{array}
\tag*{(ARM-2)}\label{eq:arm-2}
\end{equation}
Applying the gPGD algorithm to the above formulation gives us the \emph{Singular Value Projection} (SVP) algorithm (Algorithm~\ref{algo:svp}). Note that in this case, the projection needs to be carried out onto the set of low rank matrices. However, as we saw in \S~\ref{sec:non-cvx-proj}, this can be efficiently done by computing the singular value decomposition of the iterates.

SVP offers ease of implementation and speed similar to IHT. Moreover, it applies to ARM problems in general. If the Matrix RIP property is appropriately satisfied, then SVP guarantees a linear rate of convergence to the optimum, much like IHT. All these make SVP a very attractive choice for solving low rank matrix recovery problems.

Below, we give a convergence proof for SVP in the noiseless case, i.e., when $\y = \cA(X^\ast)$. The proof is similar in spirit to the convergence proof we saw for the IHT algorithm in \S~\ref{chap:spreg} but differs in crucial aspects since sparsity in this case is apparent not in the signal domain (the matrix is not itself sparse) but the spectral domain (the set of singular values of the matrix is sparse). The analysis can be extended to noisy measurements as well and can be found in \citep{JainMD2010}.

\section{A Low-rank Matrix Recovery Guarantee for SVP}
We will now present a convergence result for the SVP algorithm. As before, although the general convergence result for gPGD can indeed be applied here, we will see, just as we did in \S~\ref{chap:spreg}, that the problem specific analysis we present here is finer and reveals more insights about the ARM problem structure.

\begin{theorem}
\label{thm:svp-conv-proof-rip}
Suppose $\cA: \R^{m \times n} \rightarrow \R^k$ is an affine transformation that satisfies the matrix RIP property of order $2r$ with constant $\delta_{2r} \leq \frac{1}{3}$. Let $X^\ast \in \R^{m \times n}$ be a matrix of rank at most $r$ and let $\y = \cA(X^\ast)$. Then the SVP Algorithm (Algorithm~\ref{algo:svp}), when executed with a step length $\eta = 1/(1+\delta_{2r})$, and a target rank $q = r$, ensures $\norm{X^t - X^\ast}_F^2 \leq \epsilon$ after $t = \bigO{\log\frac{\norm{\y}_2^2}{2\epsilon}}$ iterations of the algorithm.
\end{theorem}
\begin{proof}
Notice that the notions of \emph{sparsity} and \emph{support} are very different in ARM than what they were for sparse regression. Consequently, the exact convergence proof for IHT (Theorem~\ref{thm:iht-conv-proof-rip}) is not applicable here. We will first establish an intermediate result that will show, that after $t = \bigO{\log\frac{\norm{\y}_2^2}{2\epsilon}}$ iterations, SVP ensures $\norm{\cA(X^t) - \y}_2^2 \leq \epsilon$. We will then use the matrix RIP property (Definition~\ref{defn:matrix-rip}) to deduce
\[
\norm{\cA(X^t) - \y}_2^2 = \norm{\cA(X^t - X^\ast)}_2^2 \geq (1 - \delta_{2r})\cdot\norm{X^t - X^\ast}_F^2,
\]
which will conclude the proof. To prove this intermediate result, let us denote the objective function as
\[
f(X) = \frac{1}{2}\norm{\cA(X) - \y}_2^2 = \frac{1}{2}\norm{\cA(X - X^\ast)}_2^2.
\]
An application of the matrix RIP property then gives us
\begin{align*}
&f(X^{t+1})\\
&= f(X^t) + \ip{\cA(X^t - X^\ast)}{\cA(X^{t+1} - X^t)} + \frac{1}{2}\norm{\cA(X^{t+1} - X^{t})}_2^2\\
										&\leq f(X^t) + \ip{\cA(X^t - X^\ast)}{\cA(X^{t+1} - X^t)} + \frac{(1 + \delta_{2r})}{2}\norm{X^{t+1} - X^{t}}_F^2.
\end{align*}
The following steps now introduce the intermediate variable $Y^{t+1}$ into the analysis in order to link the successive iterates by using the fact that $X^{t+1}$ was the result of a non-convex projection operation.
\begin{align*}
&\ip{\cA(X^t - X^\ast)}{\cA(X^{t+1} - X^t)} + \frac{(1 + \delta_{2r})}{2}\cdot\norm{X^{t+1} - X^{t}}_F^2\\
&\quad = \frac{1 + \delta_{2r}}{2}\cdot\norm{X^{t+1} - Y^{t+1}}_F^2 - \frac{1}{2(1+\delta_{2r})}\cdot\norm{\cA^\top(\cA(X^t - X^\ast))}_F^2\\
&\quad \leq \frac{1 + \delta_{2r}}{2}\cdot\norm{X^\ast - Y^{t+1}}_F^2 - \frac{1}{2(1+\delta_{2r})}\cdot\norm{\cA^\top(\cA(X^t - X^\ast))}_F^2\\
&\quad = \ip{\cA(X^t - X^\ast)}{\cA(X^\ast - X^t)} + \frac{(1 + \delta_{2r})}{2}\cdot\norm{X^\ast - X^{t}}_F^2\\
&\quad \leq \ip{\cA(X^t - X^\ast)}{\cA(X^\ast - X^t)} + \frac{(1 + \delta_{2r})}{2(1 - \delta_{2r})}\cdot\norm{\cA(X^\ast - X^{t})}_2^2\\
&\quad = -f(X^t) - \frac{1}{2}\norm{\cA(X^\ast - X^{t})}_2^2 + \frac{(1 + \delta_{2r})}{2(1 - \delta_{2r})}\cdot\norm{\cA(X^\ast - X^{t})}_2^2.
\end{align*}
The first step uses the identity $Y^{t+1} = X^t - \eta\cdot \cA^\top(\cA(X^t) - \y)$ from Algorithm~\ref{algo:svp}, the fact that we set $\eta = \frac{1}{1 + \delta_{2r}}$, and elementary rearrangements. The second step follows from the fact that $\norm{X^{t+1} - Y^{t+1}}_F^2 \leq \norm{X^\ast - Y^{t+1}}_F^2$ by virtue of the SVD step which makes $X^{t+1}$ the best rank-$(2r)$ approximation to $Y^{t+1}$ in terms of the Frobenius norm. The third step simply rearranges things in the reverse order of the way they were arranged in the first step, the fourth step uses the matrix RIP property and the fifth step makes elementary manipulations. This, upon rearrangement, and using $\norm{\cA(X^{t} - X^\ast)}_2^2 = 2f(X^t)$, gives us
\[
f(X^{t+1}) \leq \frac{2\delta_{2r}}{1 - \delta_{2r}}\cdot f(X^t).
\]
Since $f(X^1) = \norm{\y}_2^2$, as we set $X^1 = \vzero^{m \times n}$, if $\delta_{2r} < 1/3$ (i.e., $\frac{2\delta_{2r}}{1 - \delta_{2r}} < 1$), we have the claimed convergence result.
\end{proof}

One can, in principle, apply the SVP technique to the matrix completion problem as well. However, on the LMRC problem, SVP is outperformed by gAM-style approaches which we study next. Although the superior performance of gAM on the LMRC problem was well documented empirically, it took some time before a theoretical understanding could be obtained. This was first done in the works of \cite{Keshavan2012, JainNS2013}. These results set off a long line of works that progressively improved both the algorithm, as well as its analysis.

\begin{algorithm}[t]
	\caption{AltMin for Matrix Completion (AM-MC)}
	\label{algo:altmin-lrmc}
	\begin{algorithmic}[1]
			\REQUIRE Matrix $A \in \R^{m \times n}$ of rank $r$ observed at entries in the set $\Omega$, sampling probability $p$, stopping time $T$
			\ENSURE A matrix $\hat X$ with rank at most $r$
			\STATE Partition $\Omega$ into $2T+1$ sets $\Omega_0,\Omega_1,\ldots,\Omega_{2T}$ uniformly and randomly%
			\STATE $U^1 \leftarrow \text{SVD}(\frac{1}{p}\Pi_{\Omega_0}(A),r)$, the top $r$ left singular vectors of $\frac{1}{p}\Pi_{\Omega_0}(A)$%
			\FOR{$t = 1, 2, \ldots, T$}
				\STATE $V^{t+1} \leftarrow \arg\min_{V \in \R^{n \times r}}\ \norm{\Pi_{\Omega_{t}}(U^tV^\top - A)}_F^2$
				\STATE $U^{t+1} \leftarrow \arg\min_{U \in \R^{m \times r}}\ \norm{\Pi_{\Omega_{T+t}}(U(V^{t+1})^\top - A)}_F^2$
			\ENDFOR
			\STATE \textbf{return} {$U^\top(V^\top)^\top$}
	\end{algorithmic}
\end{algorithm}

\section{Matrix Completion via Alternating Minimization}
We will now look at the alternating minimization technique for solving the low-rank matrix completion problem. As we have observed before, the LRMC problem admits an equivalent reformulation where the low rank structure constraint is eliminated and instead, the solution is described in terms of two low-rank components
\begin{equation}
\underset{\substack{U \in \R^{m \times k}\\ V \in \R^{n \times k}}}{\min}\ \norm{\Pi_{\Omega}(UV^\top - X^\ast)}_F^2.\tag*{(LRMC*)}\label{eq:lrmc2}
\end{equation}
In this case, fixing either $U$ or $V$ reduces the above problem to a simple least squares problem for which we have very efficient and scalable solvers. As we saw in \S~\ref{chap:altmin}, such problems are excellent candidates for the gAM algorithm to be applied. The AM-MC algorithm (see Algorithm~\ref{algo:altmin-lrmc}) applies the gAM approach to the reformulated LMRC problem. The AM-MC approach is the choice of practitioners in the context of collaborative filtering \citep{ChenH2012,KorenBV2009,ZhouWSP2008}. However, AM-MC, like other gAM-style algorithms, does require proper initialization and tuning.

\section{A Low-rank Matrix Completion Guarantee for AM-MC}
We will now analyze the convergence properties of the AM-MC method for matrix completion. To simplify the presentation, we will restrict ourselves to the case when the matrix $A \in \R^{m \times n}$ is rank one. This will allow us to present the essential arguments without getting involved with technical details. Let $A = \vu^\ast(\vv^\ast)^\top$ be a $\mu$-incoherent matrix that needs to be recovered, where $\norm{\vu^\ast}_2 = \norm{\vv^\ast}_2 = 1$. It is easy to see that $\mu$-incoherence implies $\norm{\vu^\ast}_\infty \leq \frac{\mu}{\sqrt m}$ and $\norm{\vv^\ast}_\infty \leq \frac{\mu}{\sqrt n}$.

We will also assume the Bernoulli sampling model i.e., that the set of observed indices $\Omega$ is generated by selecting each entry $(i,j)$ for inclusion in $\Omega$ in an i.i.d. fashion with probability $p$. More specifically, $(i,j)\in \Omega$ iff $\delta_{ij} = 1$ where $\delta_{ij}=1$ with probability $p$ and $0$ otherwise. 

For simplicity, we will assume that each iteration of the AM-MC procedure receives a fresh set of samples $\Omega$ from $A$. This can be achieved in practice by randomly partitioning the available set of samples into as many groups as the iterations of the procedure. The completion guarantee will proceed in two steps. In the first step, we will show that the initialization step in Algorithm~\ref{algo:altmin-lrmc} itself brings AM-MC within a constant distance of the optimal solution. Next, we will show that this close initialization is sufficient for AM-MC to ensure a linear rate of convergence to the optimal solution.\\

\noindent{\bf Initialization}: We will now show that the initialization step (Step 2 in Algorithm~\ref{algo:altmin-lrmc}) provides a point $(\vu^1, \vv^1)$ which is at most a constant $c>0$ distance away from $(\vu^\ast, \vv^\ast)$. To this we need a Bernstein-style argument which we provide here for the rank-$1$ case.

\begin{theorem}\cite[Theorem 1.6]{Tropp2012}
\label{thm:matrix-bern}
Consider a finite sequence $\bc{Z_k}$ of independent random matrices of dimension $m \times n$. Assume each matrix satisfies $\E{Z_k} = \vzero$ and $\norm{Z_k}_2 \leq R$ almost surely and denote
\[
\textstyle\sigma^2 := \max\bc{\norm{\sum_k Z_kZ_k^\top}_2, \norm{\sum_k Z_k^\top Z_k}_2}.
\]
Then, for all $t \geq 0$, we have
\[
\textstyle\Pr{\norm{\sum_k Z_k}_2 \geq t} \leq (m+n)\cdot\exp\br{\frac{-t^2}{\sigma^2 + Rt/3}}.
\]
\end{theorem}

Below we apply this inequality to analyze the initialization step for AM-MC in the rank-$1$ case. We point the reader to \citep{Recht2011} and \citep{KeshavanMO2010} for a more precise argument analyzing the initialization step in the general rank-$r$ case.

\begin{theorem}
\label{thm:mc_init}
For a rank one matrix $A$ satisfying the $\mu$-incoherence property, let the observed samples $\Omega$ be generated with sampling probability $p$ as described above. Let $\vu^1, \vv^1$ be the singular vectors of $\frac{1}{p}P_{\Omega}(A)$ corresponding to its largest singular value. Then for any $\epsilon > 0$, if $p \geq \frac{45\mu^2}{\epsilon^2}\cdot\frac{\log(m+n)}{\min\bc{m,n}}$, then with probability at least $1-1/(m+n)^{10}$:
\[
\norm{\frac{1}{p}\Pi_\Omega(A)-A}_2\leq \epsilon, \quad \norm{\vu^1 - \vu^\ast}_2 \leq \epsilon, \quad \norm{\vv^1 - \vv^\ast}_2 \leq \epsilon.
\]
Moreover, the vectors $\vu^1$ and $\vv^1$ are also $2\mu$-incoherent.
\end{theorem}
\begin{proof}
Notice that the statement of the theorem essentially states that once enough entries in the matrix have been observed (as dictated by the requirement $p \geq \frac{45\mu^2}{\epsilon^2}\cdot\frac{\log(m+n)}{\min\bc{m,n}}$) an SVD step on the incomplete matrix will yield components $\vu^1,\vv^1$ that are very close to the components of the complete matrix $\vu^\ast,\vv^\ast$. Moreover, since $\vu^\ast,\vv^\ast$ are incoherent by assumption, the estimated components $\vu^1,\vv^1$ will be so too.

To apply Theorem~\ref{thm:matrix-bern} to prove this result, we will first express $\frac{1}{p}\Pi_{\Omega}(A)$ as a sum of random matrices. We first rewrite $\frac{1}{p}\Pi_{\Omega}(A)=\frac{1}{p}\sum_{ij}\delta_{ij} A_{ij}\ve_i\ve_j^\top=\sum_{ij} W_{ij}$ where $\delta_{ij} = 1$ if $(i,j) \in \Omega$ and $0$ otherwise. Note that the Bernoulli sampling model assures us that the random variables $W_{ij}=\frac{1}{p}\delta_{ij} A_{ij}\ve_i\ve_j^\top$ are independent and that $\E{\delta_{ij}} = p$. This gives us $\E{W_{ij}}=A_{ij}\ve_i\ve_j^\top$. Note that $\sum_{ij}A_{ij}\ve_i\ve_j^\top = A$.

The matrices $Z_{ij} = W_{ij} - A_{ij}\ve_i \ve_j^\top$ shall serve as our \emph{random matrices} in the application of Theorem~\ref{thm:matrix-bern}.  Clearly $\E{Z_{ij}} = \vzero$. We also have $\max_{ij}\|W_{ij}\|_2 \leq \frac{1}{p} \max_{ij} |A_{ij}|\leq \frac{\mu^2}{p\sqrt{mn}}$ due to the incoherence assumption. Applying the triangle inequality gives us $\max_{ij} \norm{Z_{ij}}_2 \leq \max_{ij}\|W_{ij}\|_2 + \max_{ij}\|A_{ij}\ve_i \ve_j^\top\|_2 \leq \br{1 + \frac{1}{p}}\frac{\mu^2}{\sqrt{mn}} \leq \frac{2\mu^2}{p\sqrt{mn}}$.

Moreover, as $A_{ij} = \vu^\ast_i\vv^\ast_j$ and $\norm{\vv^\ast}_2 = 1$, we have $\E{\sum_{ij} W_{ij} W_{ij}^\top}=\frac{1}{p}\sum_i\sum_j A_{ij}^2 \ve_i \ve_i^\top = \frac{1}{p} \sum_{i} (\vu^\ast_i)^2\ve_i \ve_i^\top$. Due to incoherence $\norm{\vu^\ast}_\infty \leq \frac{\mu}{\sqrt m}$, we get $\norm{\E{\sum_{ij} W_{ij} W_{ij}^\top}}_2\leq \frac{\mu^2}{p\cdot m}$, which can be shown to give us
\[
\norm{\E{\sum_{ij} Z_{ij} Z_{ij}^\top}}_2 \leq \br{\frac{1}{p}-1}\cdot\frac{\mu^2}{m} \leq \frac{\mu^2}{p\cdot m}
\]
Similarly, we can also get $\norm{\E{\sum_{ij} Z_{ij}^\top Z_{ij}}}_2\leq \frac{\mu^2}{p\cdot n}$. Now using Theorem~\ref{thm:matrix-bern} gives us, with probability at least $1 - \delta$,
\[
\norm{\frac{1}{p}\Pi_\Omega(A) - A}_2 \leq \frac{2\mu^2}{3p\sqrt{mn}}\log\bs{\frac{m+n}{\delta}} + \sqrt{\frac{\mu^2}{p\cdot\min\bc{m,n}}\log\bs{\frac{m+n}{\delta}}}
\]
If $p \geq \frac{45\mu^2\log(m+n)}{\epsilon^2\cdot\min\bc{m,n}}$, we have with probability at least $1-1/(m+n)^{10}$,
\[
\norm{\frac{1}{p}\Pi_\Omega(A) - A}_2 \leq \epsilon.
\]
The proof now follows by applying the Davis-Kahan inequality \citep{GolubVL1996} with the above bound. It can be shown \citep{JainN2015} that the vectors that are recovered as a result of this initialization are incoherent as well.
\end{proof}

\noindent{\bf Linear Convergence}: We will now show that, given the initialization above, the AM-MC procedure converges to the true solution with a linear rate of convergence. This will involve showing a few intermediate results, such as showing that the alternation steps preserve incoherence. Since the Theorem~\ref{thm:mc_init} shows that $\vu^1$ is $2\mu$-incoherent, this will establish the incoherence of all future iterates. Preserving incoherence will be crucial in showing the next result which shows that successive iterates get increasingly close to the optimum. Put together, these will establish the convergence result. First, recall that in the $t\th$ iteration of the AM-MC algorithm, $\vv^{t+1}$ is updated as 
\[
\vv^{t+1} = \underset{\vv}{\arg\min}\ \sum_{ij} \delta_{ij}(\vu_i^t \vv_j-\vu^\ast_i\vv^\ast_j)^2, 
\]
which gives us
\begin{equation}
\vv^{t+1}_j=\frac{\sum_{i} \delta_{ij}\vu^\ast_i \vu^t_i}{\sum_i \delta_{ij} (\vu^t_i)^2}\cdot \vv^\ast_j.
\label{eq:v-update}
\end{equation}
Note that this means that if $\vu^\ast=\vu^t$, then $\vv^{t+1}=\vv^\ast$. Also, note that if $\delta_{ij}=1$ for all $(i,j)$ which happens when the sampling probability satisfies $p=1$, we have $\vv^{t+1} = \frac{\ip{\vu^t}{\vu^\ast}}{\norm{\vu^t}_2^2}\cdot\vv^\ast$. This is reminiscent of the \emph{power method} used to recover the leading singular vectors of a matrix. Indeed if we let $\tilde\vu = \vu^t/\norm{\vu^t}_2$, we get $\norm{\vu^t}_2\cdot\vv^{t+1} = \ip{\tilde\vu}{\vu^\ast}\cdot\vv^\ast$ if $p=1$.

This allows us to rewrite the update \eqref{eq:v-update} as a noisy power update.
\begin{equation}
\norm{\vu^t}_2\cdot\vv^{t+1} = \ip{\tilde\vu}{\vu^\ast}\cdot\vv^\ast - B^{-1}(\ip{\tilde\vu}{\vu^\ast} B - C)\vv^\ast
\label{eq:v-update-mod}
\end{equation}
where $B,C\in \R^{n\times n}$ are diagonal matrices with $B_{jj}=\frac{1}{p}\sum_i \delta_{ij} (\tilde{\vu}_i)^2$ and $C_{jj}=\frac{1}{p}\sum_i \delta_{ij} \tilde{\vu}_i \vu^\ast_i$. The following two lemmata show that if $\vu^t$ is $2\mu$ incoherent and if $p$ is large enough, then: a) $\vv^{t+1}$ is also $2\mu$ incoherent, and b) the angular distance between $\vv^{t+1}$ and $\vv^\ast$ decreases as compared to that between $\vu^t$ and $\vu^\ast$. The following lemma will aid the analysis.
\begin{lemma}
\label{lem:ip-preserve}
Suppose $\va,\vb\in\R^n$ are two fixed $\mu$-incoherent unit vectors. Also suppose $\delta_i, i \in [n]$ are i.i.d. Bernoulli random variables such that $\delta_i = 1$ with probability $p$ and 0 otherwise. Then, for any $\epsilon > 0$, if $p > \frac{27\mu^2\log n}{n\epsilon^2}$, then with probability at least $1-1/n^{10}$, $\abs{\frac{1}{p}\sum_i\delta_i\va_i\vb_i - \ip{\va}{\vb}} \leq \epsilon$.
\end{lemma}
\begin{proof}
Define $Z_i = \br{\frac{\delta_i}{p}-1}\va_i\vb_i$. Using the incoherence of the vectors, we get $\E{Z_i} = 0$, $\sum_{i=1}^n\E{Z_i^2} = \br{\frac{1}{p}-1}\sum_{i=1}^n(\va_i\vb_i)^2 \leq \frac{\mu^2}{pn}$ since $\norm{\vb}_2 = 1$, and $\abs{Z_i} \leq \frac{\mu^2}{pn}$ almost surely. Applying the Bernstein inequality gives us
\[
\Pr{\abs{\frac{1}{p}\sum_i\delta_i\va_i\vb_i - \ip{\va}{\vb}} > t} \leq \exp\br{\frac{-3pnt^2}{6\mu^4+2\mu^2t}},
\]
which upon simple manipulations, gives us the result.
\end{proof}

\begin{lemma}
\label{lem:incoherence-preserve}
With probability at least $\min\bc{1 - 1/n^{10}, 1 - 1/m^{10}}$, if a pair of iterates $(\vu^t,\vv^t)$ in the execution of the AM-MC procedure are $2\mu$-incoherent, then so are the next pair of iterates $(\vu^{t+1},\vv^{t+1})$.
\end{lemma}
\begin{proof}
Since $\norm{\tilde\vu}_2 = 1$, using Lemma~\ref{lem:ip-preserve} tells us that with high probability, for all $j$, we have $\abs{B_{jj} - 1} \leq \epsilon$ as well as $\abs{C_{jj} - \ip{\tilde\vu}{\vu^\ast}} \leq \epsilon$. Also, using triangle inequality, we get $\norm{\vu^t}_2 \geq 1 - \epsilon$. Using these and the incoherence of $\vv^\ast$ in the update equation for $\vv^{t+1}$ \eqref{eq:v-update-mod}, we have
\begin{align*}
\abs{\vv^{t+1}_j} &= \frac{1}{1-\epsilon}\abs{\ip{\tilde\vu}{\vu^\ast}\vv^\ast_j - \frac{1}{B_{jj}}(\ip{\tilde\vu}{\vu^\ast}B_{jj} - C_{jj})\vv^\ast_j}\\
									&\leq \frac{1}{1-\epsilon}\abs{\ip{\tilde\vu}{\vu^\ast}\vv^\ast_j} + \frac{1}{1-\epsilon}\abs{\frac{1}{B_{jj}}(\ip{\tilde\vu}{\vu^\ast}B_{jj} - C_{jj})\vv^\ast_j}\\
									&\leq \frac{1}{(1-\epsilon)^2}\br{\abs{\ip{\tilde\vu}{\vu^\ast}} + \abs{\ip{\tilde\vu}{\vu^\ast}(1+\epsilon) - (\ip{\tilde\vu}{\vu^\ast} - \epsilon)}}\frac{\mu}{\sqrt n} \leq \frac{1+2\epsilon}{(1-\epsilon)^2}\frac{\mu}{\sqrt n}
\end{align*}
For $\epsilon < 1/6$, the result now holds.
\end{proof}

We note that whereas Lemma~\ref{lem:ip-preserve} is proved for fixed vectors, we seem to have inappropriately applied it to $\tilde\vu$ in the proof of Lemma~\ref{lem:incoherence-preserve} which is not a fixed vector as it depends on the randomness used in sampling the entries of the matrix revealed to the algorithm. However notice that the AM-MC procedure in Algorithm~\ref{algo:altmin-lrmc} uses fresh samples $\Omega_t$ and $\Omega_{T+t}$ for each iteration. This ensures that $\tilde\vu$ does behave like a fixed vector with respect to Lemma~\ref{lem:ip-preserve}.

\begin{lemma}
\label{lem:lin-conv-altmin}
For any $\epsilon > 0$, if $p > \frac{80\mu^2\log(m+n)}{\epsilon^2\min\bc{m,n}}$ and $\vu^t$ is $2\mu$-incoherent, the next iterate $\vv^{t+1}$ satisfies
\[
1- \ip{\frac{\vv^{t+1}}{\norm{\vv^{t+1}}_2}}{\vv^\ast}^2 \leq \frac{\epsilon}{(1-\epsilon)^3}\br{1- \ip{\frac{\vu^t}{\norm{\vu^t}_2}}{\vu^\ast}^2}
\]
Similarly, for any $2\mu$-incoherent iterate $\vv^{t+1}$, the next iterate satisfies 
\[
1- \ip{\frac{\vu^{t+1}}{\norm{\vu^{t+1}}_2}}{\vu^\ast}^2 \leq \frac{\epsilon}{(1-\epsilon)^3}\br{1- \ip{\frac{\vv^{t+1}}{\norm{\vv^{t+1}}_2}}{\vv^\ast}^2}.
\]
\end{lemma}
\begin{proof}
Using the modified form of the update for $\vu^{t+1}$ \eqref{eq:v-update-mod}, we get, for any unit vector $\vv_\bot$ such that $\ip{\vv_\bot}{\vv^\ast} = 0$,
\begin{align*}
\norm{\vu^t}_2\cdot\ip{\vv^{t+1}}{\vv_\bot} &= \ip{\vv_\bot}{B^{-1}(\ip{\tilde\vu}{\vu^\ast} B - C)\vv^\ast}\\
&\leq \norm{B^{-1}}_2\norm{(\ip{\tilde\vu}{\vu^\ast} B - C)\vv^\ast}_2\\
&\leq \frac{1}{1-\epsilon}\norm{(\ip{\tilde\vu}{\vu^\ast} B - C)\vv^\ast}_2,
\end{align*}
where the last step follows from an application of Lemma~\ref{lem:ip-preserve}. To bound the other term let $Z_{ij} = \frac{1}{p}\delta_{ij}(\ip{\tilde\vu}{\vu^\ast}(\tilde\vu_i)^2 - \tilde\vu_i\vu^\ast_i)\vv^\ast_j\ve_j \in \R^n$. Clearly $\sum_{i=1}^m\sum_{j=1}^nZ_{ij} = (\ip{\tilde\vu}{\vu^\ast} B - C)\vv^\ast$. Note that due to fresh samples being used by Algorithm~\ref{algo:altmin-lrmc} at every step, the vector $\tilde\vu$ appears as a constant vector to the random variables $\delta_{ij}$. Given this, note that
\begin{align*}
\E{\sum_{i=1}^mZ_{ij}} &= \sum_{i=1}^m(\ip{\tilde\vu}{\vu^\ast}(\tilde\vu_i)^2 - \tilde\vu_i\vu^\ast_i)\vv^\ast_j\ve_j\\
											 &= \ip{\tilde\vu}{\vu^\ast}\sum_{i=1}^m(\tilde\vu_i)^2\vv^\ast_j\ve_j - \sum_{i=1}^m\tilde\vu_i\vu^\ast_i\vv^\ast_j\ve_j\\
											 &= (\ip{\tilde\vu}{\vu^\ast}\norm{\tilde\vu}_2^2 - \ip{\tilde\vu}{\vu^\ast})\vv^\ast_j\ve_j = \vzero,
\end{align*}
since $\norm{\tilde\vu}_2 = 1$. Thus $\E{\sum_{ij}Z_{ij}} = \vzero$ as well. Now, we have $\max_i(\tilde\vu_i)^2 = \frac{1}{\norm{\vu^t}_2^2}\cdot\max_i(\vu^t_i)^2 \leq \frac{4\mu^2}{m\norm{\vu^t}_2^2} \leq \frac{\mu^2}{m(1-\epsilon)}$ since $\norm{\vu^t - \vu^\ast}_2 \leq \epsilon$. This allows us to bound
\begin{align*}
\abs{\E{\sum_{ij} Z_{ij}^\top Z_{ij}}} &= \frac{1}{p}\sum_{i=1}^m\sum_{j=1}^n (\ip{\tilde\vu}{\vu^\ast}(\tilde\vu_i)^2 - \tilde\vu_i\vu^\ast_i)^2 (\vv^\ast_j)^2\\
											&= \frac{1}{p}\sum_{i=1}^m(\tilde\vu_i)^2(\ip{\tilde\vu}{\vu^\ast}\tilde\vu_i - \vu^\ast_i)^2\\
											&\leq \frac{\mu^2}{pm(1-\epsilon)}\sum_{i=1}^m \ip{\tilde\vu}{\vu^\ast}^2(\tilde\vu_i)^2 + (\vu^\ast_i)^2 - 2\ip{\tilde\vu}{\vu^\ast}\tilde\vu_i\vu^\ast_i\\
											&\leq \frac{8\mu^2}{pm} (1- \ip{\tilde\vu}{\vu^\ast}^2),
\end{align*}
where we set $\epsilon = 0.5$. In the same way we can show $\norm{\E{\sum_{ij} Z_{ij}^\top Z_{ij}}}_2 \leq \frac{8\mu^2}{pm} (1- \ip{\tilde\vu}{\vu^\ast}^2)$ as well. Using a similar argument we can show $\norm{Z_{ij}}_2 \leq \frac{4\mu^2}{p\sqrt{mn}}\sqrt{1- \ip{\tilde\vu}{\vu^\ast}^2}$. Applying the Bernstein inequality now tells us, that for any $\epsilon > 0$, if $p > \frac{80\mu^2\log(m+n)}{\epsilon^2\min\bc{m,n}}$, then with probability at least $1 - 1/n^{10}$, we have
\[
\norm{(\ip{\tilde\vu}{\vu^\ast} B - C)\vv^\ast}_2 \leq \epsilon\cdot\sqrt{1- \ip{\tilde\vu}{\vu^\ast}^2}.
\]
Since $\norm{\vu^t}_2 \geq 1 -\epsilon$ is guaranteed by the initialization step, we now get,
\[
\ip{\vv^{t+1}}{\vv_\bot} \leq \frac{\epsilon}{(1-\epsilon)^2}\sqrt{1- \ip{\tilde\vu}{\vu^\ast}^2}.
\]
If $\vv^{t+1}_\bot$ and $\vv^{t+1}_\parallel$ be the components of $\vv^{t+1}$ perpendicular and parallel to $\vv^\ast$. Then the above guarantees that $\norm{\vv^\bot}_2 \leq c\cdot\sqrt{1- \ip{\tilde\vu}{\vu^\ast}^2}$. This gives us, upon applying the Pythagoras theorem,
\[
\norm{\vv^{t+1}}^2_2 = \norm{\vv^{t+1}_\bot}^2_2 + \norm{\vv^{t+1}_\parallel}^2_2 \leq \frac{\epsilon}{(1-\epsilon)^2}\br{1- \ip{\tilde\vu}{\vu^\ast}^2} + \norm{\vv^{t+1}_\parallel}^2_2
\]
Since $\norm{\vv^{t+1}_\parallel}_2 = \ip{\vv^{t+1}}{\vv^\ast}$ and $\norm{\vv^{t+1}}_2 \geq 1-\epsilon$ as $\norm{\vv^{t+1} - \vv^\ast} \leq \epsilon$ due to the initialization, rearranging the terms gives us the result.
\end{proof}

Using these results, it is easy to establish the main theorem.

\begin{theorem}
\label{thm:altmin-conv}
Let $A=\vu^\ast(\vv^\ast)^\top$ be a unit rank matrix where $\vu^\ast \in \R^m$ and $\vv^\ast \in \R^n$ are two $\mu$-incoherent unit vectors. Let the matrix be observed at a set of indices $\Omega\subseteq [m]\times [n]$ where each index is observed with probability $p$. Then if $p \geq C\cdot\frac{\mu^2\log(m+n)}{\epsilon^2\min\bc{m,n}}$ for some large enough constant $C$, then with probability at least $1 - 1/\min\bc{m,n}^{10}$, AM-MC generates iterates which are $2\mu$-incoherent. Moreover, within $\bigO{\log\frac{1}{\epsilon}}$ iterations, AM-MC also ensures that $\norm{\frac{\vu^t}{\norm{\vu^t}_2}-\vu^\ast}_2 \leq \epsilon$ and $\norm{\frac{\vv^t}{\norm{\vv^t}_2} - \vv^\ast}_2 \leq \epsilon$.
\end{theorem}

\section{Other Popular Techniques for Matrix Recovery}
As has been the case with the other problems we have studied so far, the first approaches to solving the ARM and LRMC problems were relaxation based approaches \citep{CandesR2009,CandesT2009,RechtFP2010,Recht2011}. These approaches relax the non-convex rank objective in the \eqref{eq:arm} formulation using the (convex) \emph{nuclear norm}
\begin{equation*}
\begin{array}{cl}
	\min & \norm{X}_\ast\\
	\text{s.t.} & \cA(X) = \y,
\end{array}
\label{eq:arm-conv}
\end{equation*}
where the nuclear norm of a matrix $\norm{X}_\ast$ is the sum of all singular values of the matrix $X$. The nuclear norm is known to provide the tightest convex envelope of the rank function, just as the $\ell_1$ norm provides a relaxation to the sparsity norm $\norm{\cdot}_0$ \citep{RechtFP2010}. Similar to sparse recovery, under matrix-RIP settings, these relaxations can be shown to offer exact recovery \citep{RechtFP2010,HastieTW2016}.

\begin{figure}[t]
\begin{subfigure}[t]{.5\columnwidth}
\centering \includegraphics[width=\columnwidth]{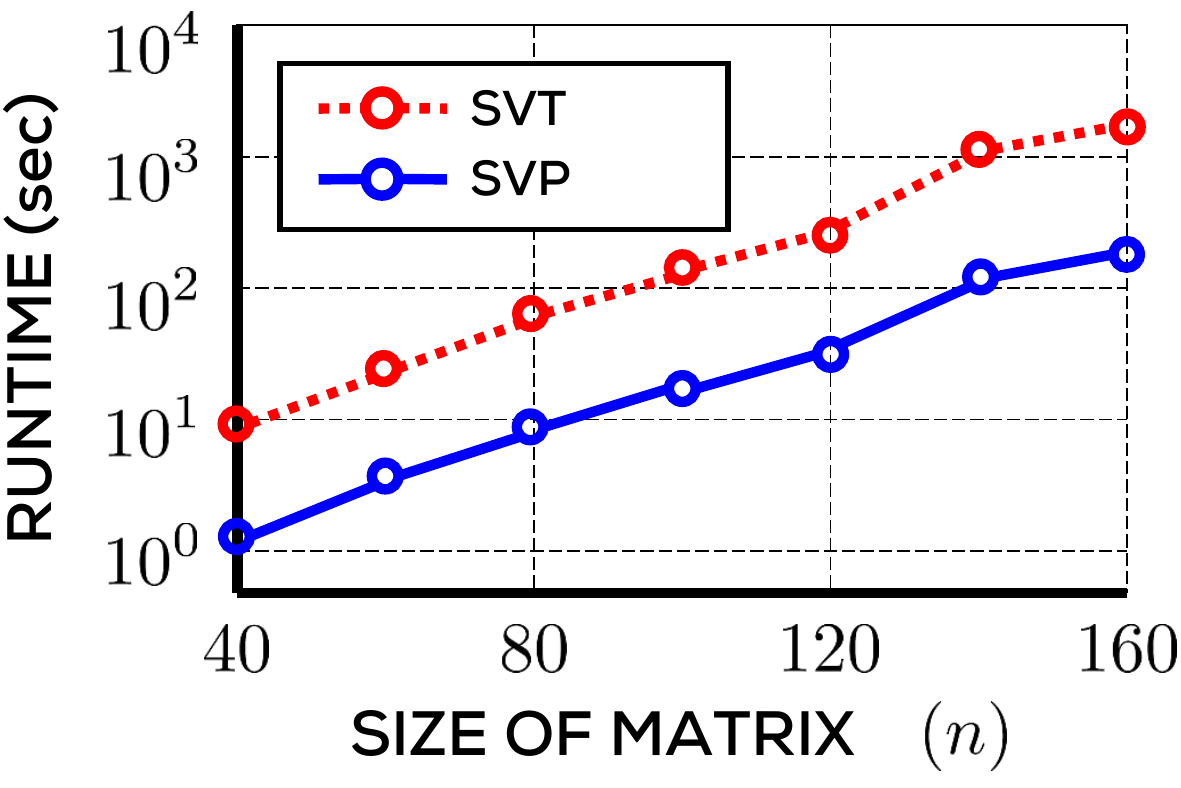}
\caption{Matrix Recovery}
\label{fig:matrec-comparison-matrec}
\end{subfigure}
\hfill
\begin{subfigure}[t]{.5\columnwidth}
\centering \includegraphics[width=\columnwidth]{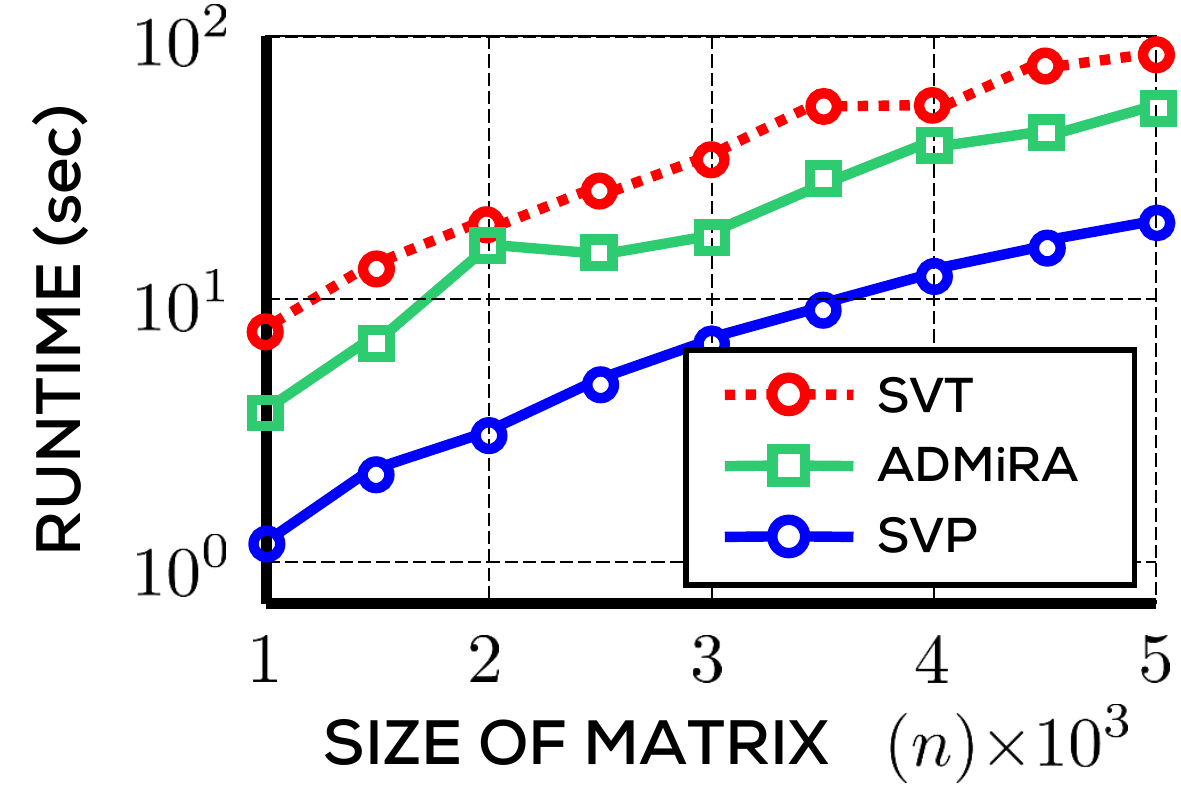}
\caption{Matrix Completion}
\label{fig:matrec-comparison-matcomp}
\end{subfigure}%
\caption[Relaxation vs. Non-convex Methods for Matrix Recovery]{An empirical comparison of run-times offered by the SVT, ADMiRA and SVP methods on synthetic matrix recovery and matrix completion problems with varying matrix sizes. The SVT method due to \cite{CaiCS2010} is an efficient implementation of the nuclear norm relaxation technique. ADMiRA is a pursuit-style method due to \cite{LeeB2010}. For the ARM task in Figure~\ref{fig:matrec-comparison-matrec}, the rank of the true matrix was set to $r=5$ whereas it was set to $r=2$ for the LRMC task in Figure~\ref{fig:matrec-comparison-matcomp}. SVT is clearly the most scalable of the methods in both cases whereas the relaxation-based SVT technique does not scale very well to large matrices. Note however, that for the LRMC problem, AM-MC (not shown in the figure) outperforms even SVP. Figures adapted from \citep{MekaJCD2008}.}%
\label{fig:matrec-comparison}
\end{figure}

Also similar to sparse recovery, there exist pursuit-style techniques for matrix recovery, most notable among them being the ADMiRA method \citep{LeeB2010} that extends the orthogonal matching pursuit approach to the matrix recovery setting. However, this method can be a bit sluggish when recovering matrices with slightly large rank since it discovers a matrix with larger and larger rank incrementally.

Before concluding, we present the reader with an empirical performance of these various methods. Figure~\ref{fig:matrec-comparison} provides a comparison of these methods on synthetic matrix recovery and matrix completion problems with increasing dimensionality of the (low-rank) matrix being recovered. The graphs indicate that non-convex optimization methods such as IHT are far more scalable, often by an order of magnitude, than relaxation-based methods.

\section{Exercises}
\begin{exer}
\label{exer:matrec-np-hard}
Show that low-rank matrix recovery is NP-hard.\\
\textit{Hint}: Take the sparse recovery problem in~\eqref{eq:spreg} and reduce it to the reformulation~\eqref{eq:arm-2} of the matrix recovery problem.\end{exer}
\begin{exer}
\label{exer:matrec-monotone}
Show that the matrix RIP constant is monotonic in its order i.e., if a linear map $\cA$ satisfies matrix RIP of order $r$ with constant $\delta_r$, then it also satisfies matrix RIP for all orders $r' \leq r$ with $\delta_{r'} \leq \delta_{r}$.
\end{exer}

\section{Bibliographic Notes}
There are a lot of aspects of low-rank matrix recovery that we have not covered in our discussion. Here we briefly point out some of these.

Similar to the sparse regression setting, the problem of ill-conditioned problems requires special care in the matrix recovery setting as well. For the general ARM problem, the work of \cite{JainTK2014} does this by first proposing appropriate versions of RSC-RSS properties (see Definition~\ref{defn:rsc-rss}) for the matrix case, and the suitably modifying SVP-style techniques to function even in high condition number settings. The final result is similar to the one for sparse recovery (see Theorem~\ref{thm:iht-conv-proof-rsc-rss}) wherein a more ``relaxed'' projection step is required by using a rank $q > r$ while executing the SVP algorithm.

It turns out to be challenging to prove convergence results for SVP for the LRMC problem. This is primarily due to the difficulty in establishing the matrix RIP property for the affine transformation used in the problem. The affine map simply selects a few elements of the matrix and reproduces them which makes establishing RIP properties harder in this setting. Specifically, even though the initialization step can be shown to yield a matrix that satisfies matrix RIP \citep{JainMD2010}, if the underlying matrix is low-rank and incoherent, it becomes challenging to show that RIP-ness is maintained across iterates. \cite{JainN2015} overcome this by executing the SVP algorithm in a \emph{stage-wise} fashion which resembles ADMiRA-like pursuit approaches.

Several works have furthered the alternating minimization approach itself by reducing its sample complexity \citep{Hardt2014}, giving recovery guarantees independent of the condition number of the problem \citep{HardtW2014,JainN2015}, designing universal sampling schemes for recovery \citep{BhojanapalliJ2014}, as well as tackling settings where some of the revealed entries of the matrix may be corrupted \citep{ChenXCS2016, CherapanamjeriGJ2017}.

Another interesting line of work for matrix completion is that of \citep{GeLM2016, SunL2015} which shows that under certain regularization assumptions, the matrix completion problem does not have any non-optimal stationary points once one gets close-enough to the global minimum. Thus, one can use any method for convex optimization such as alternating minimization, gradient descent, stochastic gradient descent, and its variants we studied in \S~\ref{chap:saddle}, once one is close enough to the global minimum. 
\chapter{Robust Linear Regression}
\label{chap:rreg}

\newcommand{\oc}{\vb^\ast}
\newcommand{\ovy}{\vy^\ast}
\newcommand{\mc}{\vb^\ast}
\newcommand{\seto}{S_\ast}
\newcommand{\sett}{S_t}
\newcommand{\setn}{S_{t+1}}

In this section, we will look at the problem of robust linear regression. Simply put, it is the task of performing linear regression in the presence of adversarial \emph{outliers} or \emph{corruptions}. Let us take a look at some motivating applications. 

\section{Motivating Applications}
\label{sec:rreg-intro}
The problem of regression has widespread application in signal processing, financial and economic analysis, as well as machine learning and data analytics. In most real life applications, the data presented to the algorithm has some amount of noise in it. However, at times, data may be riddled with missing values and corruptions, and that too of a malicious nature. It is useful to design algorithms that can assure stable operation even in the presence of such corruptions.\\

\noindent\textbf{Face Recognition} The task of face recognition is widely useful in areas such as biometrics and automated image annotation. In biometrics, a fundamental problem is to identify if a new face image belongs to that of a registered individual or not. This problem can be cast as a regression problem by trying to fit various features of the new image to corresponding features of existing images of the individual in the registered database. More specifically, assume that images are represented as $n$-dimensional feature vectors say, using simple pixel-based features. Also assume that there already exist $p$ images of the person in the database.

Our task is to represent the new image $\x^t \in \R^n$ in terms of the database images $X = [\x_1,\ldots,\x_p] \in \bR^{n \times p}$ of that person. A nice way to do this is to perform a linear interpolation as follows
\[
\min_{\bt \in \bR^p} \norm{\x^t - X\bt}_2^2 = \sum_{i=1}^n(\x^t_i - X^i\bt)^2.
\]
If the person is genuine, then there will exist a combination $\bto$ such that for all $i$, we have $\x^t_i \approx X^i\bto$ i.e., all features can be faithfully reconstructed. Thus, the fit will be nice and we will admit the person. However, the same becomes problematic if the new image has occlusions. For example, if the person is genuine but wearing a pair of sunglasses or sporting a beard. In such cases, some of the pixels $\x^t_i$ will appear corrupted, cause us to get a poor fit, and result in a false alarm. More specifically
\[
\x^t_i = X^i\bto + \oc_i
\]
where $\oc_i = 0$ on uncorrupted pixels but can take abnormally large and unpredictable values for corrupted pixels, such as those corresponding to the sunglasses. Being able to still correctly identify the person involves computing the least squares fit in the presence of such corruptions. The challenge is to do this without requiring any manual effort to identify the locations of the corrupted pixels. Figure~\ref{fig:rreg-facerec} depicts this problem setting visually.\\

\begin{figure}[t]
\includegraphics[width=\columnwidth]{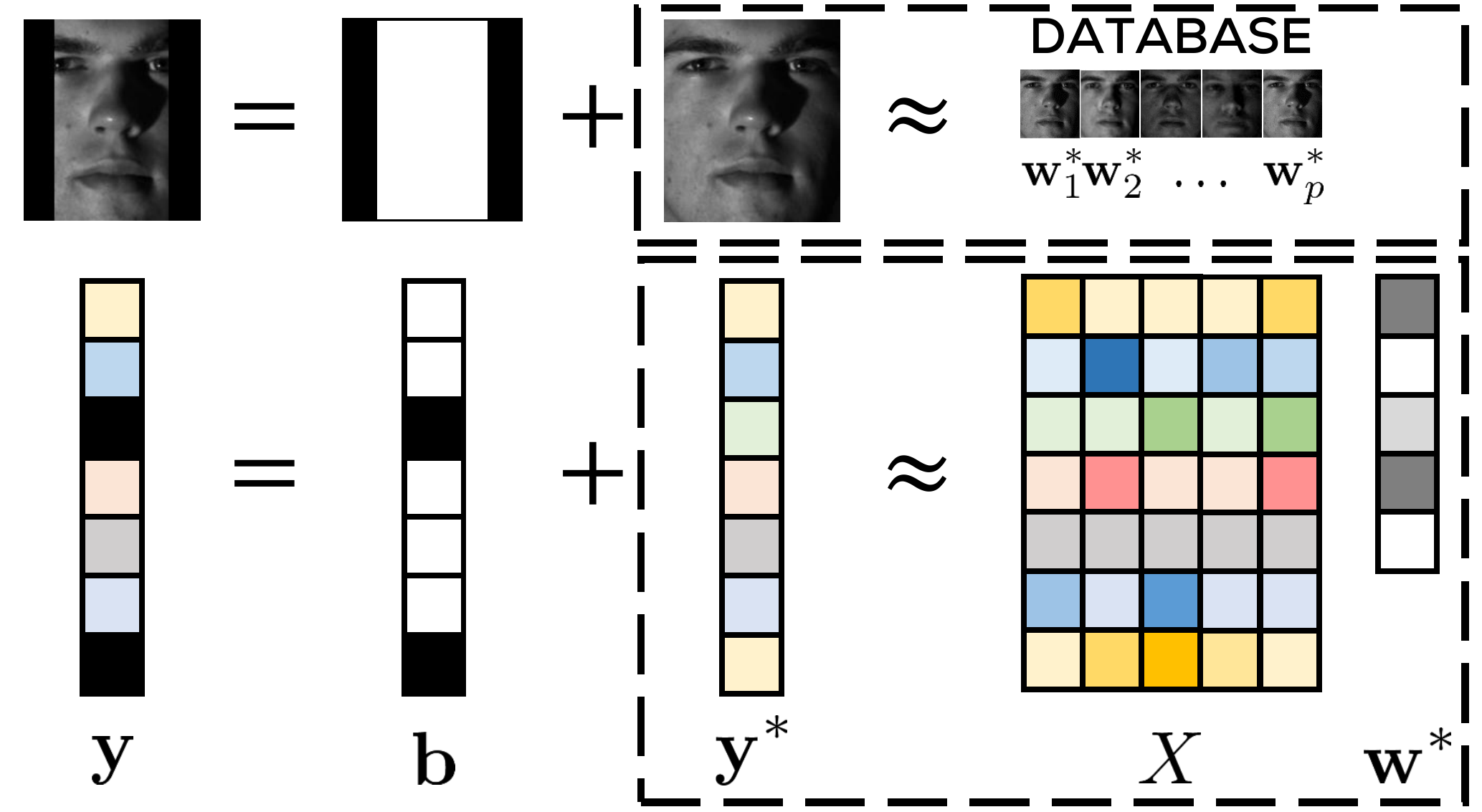}
\caption[Face Recognition under Occlusions]{A corrupted image $\vy$ can be interpreted as a combination of a clean image $\ovy$ and a corruption mask $\oc$ i.e., $\vy = \ovy + \oc$. The mask encodes the locations of the corrupted pixels as well as the values of the corruptions. The clean image can be (approximately) recovered as an affine combination of existing images in a database as $\ovy \approx X\bto$. Face reconstruction and recognition in such a scenario constitutes a robust regression problem. Note that the corruption mask $\oc$ is sparse since only a few pixels are corrupted. Images courtesy the Yale Face Database B.}%
\label{fig:rreg-facerec}
\end{figure}

\noindent\textbf{Time Series Analysis} This is a problem that has received much independent attention in statistics and signal processing due to its applications in modeling sequence data such as weather data, financial data, and DNA sequences. However, the underlying problem is similar to that of regression. A sequence of timestamped observations $\bc{y_t}$ for $t = 0, 1, \ldots$ are made which constitute the time series. Note that the ordering of the samples is critical here.

The popular auto-regressive (AR) time series model uses a generative mechanism wherein the observation $y_t$ at time $t$ is obtained as a fixed linear combination of $p$ previous observations plus some noise.
\[
y_t = \sum_{i=1}^p\bto_iy_{t-i} + \eta_t
\]
We can cast this into a regression framework by constructing covariates $\x_t := \bs{y_{t-1},y_{t-2},\ldots,y_{t-p}}^\top \in \bR^p$ and rewriting the above as
\[
y_t = \x_t^\top\bto + \eta_t,
\]
where $\bto \in \bR^p$ is an unknown model and $\eta_i$ is Gaussian noise generated independent of previous observations $\eta_t | \bc{y_{t-1},y_{t-2},\ldots} \sim \cN(0,\sigma^2)$. The number $p$ is known as the \emph{order} of the time series and captures how many historical observations affect the current one.

Is not uncommon to encounter situations where the time series experiences gross corruptions. Examples may include observation or sensing errors or unmodeled factors such as dips and upheavals in stock prices due to political or socio-economic events. Thus, we have
\[
y_t = \x_t^\top\bto + \eta_t + \oc_t
\]
where $\oc_t$ can take unpredictable values for corrupted time instances and $0$ otherwise. In time series literature, two corruption models are popular. In the \emph{additive} model, observations at time steps subsequent to a corruption are constructed using the uncorrupted values i.e., $\oc_t$ does not influence the values $y_\tau$ for $\tau > t$. Of course, observers may detect the corruption at time $t$ but the underlying time series goes on as though nothing happened. 

However, in the \emph{innovative} model, observations at time instances subsequent to a corruption use the corrupted value i.e., $\oc_t$ is involved in constructing the values $y_\tau$ for $\tau = t+1,\ldots,t+p$. Innovative corruptions are simpler to handle as although the observation at the moment of the corruption i.e., $y_t$, appears to deviate from that predicted by the base model i.e., $\x_t^\top\bto$, subsequent observations fall in line with the predictions once more (unless there are more corruptions down the line). In the additive model however, observations can seem to deviate from the predictions of the base model for several iterations. In particular, $y_\tau$ can disagree with $\x_\tau^\top\bto$ for times $\tau = t,t+1,\ldots,t+p$, even if there is only a single corruption at time $t$.

A time series analysis technique which seeks to study the ``usual'' behavior of the model involved, such as stock prices, might wish to exclude such aberrations, whether additive or innovative. However it is unreasonable, as well as error prone, to expect manual exclusion of such corruptions which motivates the problem of robust time series analysis. Note that time series analysis is a more challenging problem than regression since the ``covariates'' in this case $\x_t, \x_{t+1},\ldots$ are heavily correlated with each other as they share a large number of coordinates whereas in regression they are usually assumed to be independent.

\section{Problem Formulation}
Let us recall the regression model studied in \S~\ref{sec:em-pml} and \S~\ref{sec:spreg-prob-form}. In a linear regression setting, responses are modeled as a (possibly sparse) linear function of the data features to which some additive noise is added, i.e., we have
\[
y_i = \x_i^\top\bto + \eta_i.
\]
However, notice that in previous discussions, we termed the noise as benign, and even found it appropriate to assume it could be modeled as a Gaussian random variable. This is usually acceptable when the noise is expected to be non-deliberate or the result of small and unstructured perturbations. However, the outliers or corruptions in the examples that we just saw, seem to defy these assumptions.

In the face recognition case, salt and pepper noise due to sensor disturbances, or lighting changes can be considered benign noise. However, it is improper to consider structured occlusions such as sunglasses or beards as benign as some of them may even be a part of a malicious attempt to fool the system. Even in the time series analysis setting, there might be malicious forces at play in the stock market which cause stock prices to significantly deviate from their usual variations and trends.

Such corruptions can severely disrupt the modeling process and hamper our understanding of the system. To handle them, we will need a more refined model that distinguishes between benign and unstructured errors, and errors that are deliberate, malicious and structured. The \emph{robust regression} model best describes this problem setting
\[
y_i = \x_i^\top\bto + \eta_i + \oc_i
\]
where the variable $\oc_i$ encodes the \emph{additional} corruption introduced into the response. Our goal is to take a set of $n$ (possibly) corrupted data points $(\x_i,y_i)_{i=1}^n$ and recover the underlying parameter vector $\bto$, i.e.,
\begin{equation}
\underset{\substack{\bt\in\R^p, \vb\in\R^n\\\norm{\vb}_0 \leq k}}{\min}\ \norm{\y - X\bt - \vb}_2^2,
\tag*{(ROB-REG)}\label{eq:rreg}
\end{equation}

The variables $\oc_i$ can be unbounded in magnitude and of arbitrary sign. However, we assume that only a few data points are corrupted i.e., the vector $\mc = [\oc_1,\oc_2,\ldots,\oc_n]$ is sparse $\norm{\mc}_0 \leq k$. Indeed it is impossible\elink{exer:rreg-impossible} to recover the model $\bto$ if more than half the points are corrupted i.e., $k \geq n/2$. A worthy goal is to develop algorithms that can tolerate as large a value of $k$ as possible. We will study how two non-convex optimization techniques, namely gAM and gPGD, can be used to solve this problem. We point to other approaches, as well as extensions such as robust sparse recovery, in the bibliographic notes.

\section{Robust Regression via Alternating Minimization}
The key to applying alternating minimization to the robust regression problem is to identify the two critical parameters in this problem. Let us assume that there is no Gaussian noise in the model i.e., $\y = X\bto + \mc$ where $\mc$ is $k$-sparse but can contain unbounded entries in its support.

\begin{algorithm}[t]
	\caption{AltMin for Robust Regression (AM-RR)}
	\label{algo:torrent}
	\begin{algorithmic}[1]
			\REQUIRE Data $X, \y$, number of corruptions $k$
			\ENSURE An accurate model $\bth \in \bR^p$
			\STATE $\bt^1 \leftarrow \vzero$, $S_1 = [1:n-k]$
			\FOR{$t = 1, 2, \ldots$}
				\STATE $\btn \leftarrow \arg\min_{\bt \in \bR^p}\ \sum_{i \in \sett}(y_i - \x_i^\top\bt)^2$
				\STATE $\setn \leftarrow \arg\min_{|S| = n-k}\ \sum_{i \in S}(y_i - \x_i^\top\btn)^2$
			\ENDFOR
			\STATE \textbf{return} {$\btt$}
	\end{algorithmic}
\end{algorithm}

It can be seen that $\bto$ and $\bar{\supp(\mc)} =: \seto$, i.e., the true model and the locations of the uncorrupted points, are the two most crucial elements since given one, finding the other is very simple. Indeed, if someone were to magically hand us $\bto$, it is trivial to identify $\seto$ by simply identifying data points where $y_i = \x_i^\top\bto$. On the other hand, given $\seto$, it is simple to obtain $\bto$ by simply solving a least squares regression problem on the set of data points in the set $\seto$. Thus we can rewrite the robust regression problem as
\begin{equation}
\underset{\substack{\bt \in \bR^p\\|S| = n-k}}{\min}\ \norm{\y_S - X^S\bt}_2^2
\tag*{(ROB-REG-2)}\label{eq:rreg-2}
\end{equation}
This gives us a direct way of applying the gAM approach to this problem as outlined in the AM-RR algorithm (Algorithm~\ref{algo:torrent}). The work of \cite{BhatiaJK2015} showed that this technique and its variants offer scalable solutions to the robust regression problem.

In order to execute the gAM protocol, AM-RR maintains a model estimate $\btt$ and an \emph{active set} $S_t \subset [n]$ of points that are deemed clean at the moment. At every time step, true to the gAM philosophy, AM-RR first fixes the active set and updates the model, and then fixes the model and updates the active set. The first step turns out to be nothing but least squares over the active set. For the second step, it is easy to see that the optimal solution is achieved simply by taking the $n-k$ data points with the smallest residuals (by magnitude) with respect to the updated model and designating them to be the active set.

\section{A Robust Recovery Guarantee for AM-RR}
To simplify the analysis, we will assume that there is no Gaussian noise in the model i.e., $\y = X\bto + \mc$ where $\mc$ is $k$-sparse. To present the analysis of the AM-RR algorithm, we will need the notions of subset strong convexity and smoothness.
\begin{definition}[Subset Strong Convexity/Smoothness Property \citep{BhatiaJK2015}]
\label{defn:ssc-sss}
A matrix $X \in \R^{n \times p}$ is said to satisfy the $\alpha_k$-subset strong convexity (SSC) property and the $\beta_k$-subset smoothness property (SSS) of order $k$ if for all sets $S \subset [n]$ of size $|S| \leq k$, we have, for all $\vv \in \bR^p$,
\begin{center}
	$\alpha_k\cdot\norm{\vv}_2^2 \leq \norm{X^S\vv}_2^2 \leq \beta_k\cdot\norm{\vv}_2^2$.
\end{center}
\end{definition}
The SSC/SSS properties require that the design matrix formed by taking any subset of $k$ points from the data set of $n$ points act as an approximate isometry on all $p$ dimensional points. These properties are related to the traditional RSC/RSS properties and it can be shown (see for example, \citep{BhatiaJK2015}) that RIP-inducing distributions over matrices (see \S~\ref{sec:rip-ensure}) also produce matrices that satisfy the SSC/SSS properties, with high probability. However, it is interesting to note that whereas the RIP definition is concerned with column subsets of the design matrix, SSC/SSS concerns itself with row subsets.

The nature of the SSC/SSS properties is readily seen to be very appropriate for AM-RR to succeed. Since the algorithm uses only a subset of data points to estimate the model vector, it is essential that smaller subsets of data points of size $n-k$ (in particular the true subset of clean points $S_\ast$) also allow the model to be recovered. This is equivalent\elink{exer:rreg-ssc} to requiring that the design matrices formed by smaller subsets of data points not identify distinct model vectors. This is exactly what the SSC property demands.

Given this, we can prove the following convergence guarantee for the AM-RR algorithm. The reader would notice that the algorithm, despite being a gAM-style algorithm, does not require precise and careful initialization of the model and active set. This is in stark contrast to other gAM-style approaches we have seen so far, namely EM and AM-MC, both of which demanded careful initialization.

\begin{theorem}
Let $X \in \bR^{n \times p}$ satisfy the SSC property at order $n-k$ with parameter $\alpha_{n-k}$ and the SSS property at order $k$ with parameter $\beta_k$ such that $\beta_k/\alpha_{n-k} < \frac{1}{\sqrt 2 + 1}$. Let $\bto \in \bR^p$ be an arbitrary model vector and $\y = X\bto + \mc$ where $\norm{\mc}_0 \leq k$ is a sparse vector of possibly unbounded corruptions. Then AM-RR yields an $\epsilon$-accurate solution $\norm{\btt - \bto}_2 \leq \epsilon$ in no more than $\bigO{\log\frac{\norm{\mc}_2}{\epsilon}}$ steps.
\end{theorem}
\begin{proof}
Let $\vr^t = \y - X\btt$ denote the vector of residuals at time $t$ and let $C_t = (X^{\sett})^\top X^{\sett}$ and $\seto = \bar{\supp(\mc)}$. Then the model update step of AM-RR solves a least squares problem ensuring
\begin{align*}
\btn &= C_t^{-1}(X^{\sett})^\top\y_{\sett}\\
&= C_t^{-1}(X^{\sett})^\top(X^{\sett}\bto + \mc_{\sett})\\
&= \bto + C_t^{-1}(X^{\sett})^\top\mc_{\sett}.
\end{align*}
The residuals with respect to this new model can be computed as
\[
\vr^{t+1} = \y - X\btn = \mc + XC_t^{-1}(X^{\sett})^\top\mc_{\sett}.
\]
However, the active-set update step selects the set with smallest residuals, in particular, ensuring that
\[
\norm{\vr^{t+1}_{\setn}}^2_2 \leq \norm{\vr^{t+1}_{\seto}}^2_2.
\]
Plugging in the expression for $\vr^{t+1}$ into both sides of the equation, using $\mc_{\seto} = \vzero$ and the fact that that for any matrix $X$ and vector $\vv$ we have
\[
\norm{X^S\vv}_2^2 - \norm{X^T\vv}_2^2 = \norm{X^{S\backslash T}\vv}_2^2 - \norm{X^{T\backslash S}\vv}_2^2 \leq \norm{X^{S\backslash T}\vv}_2^2,
\]
gives us, upon some simplification,
\begin{align*}
\norm{\mc_{\setn}}_2^2 \leq{}& \norm{X^{\seto\backslash\setn}C_t^{-1}(X^{\sett})^\top\mc_{\sett}}_2^2\\
&{}- 2(\mc_{\setn})^\top X^{\setn}C_t^{-1}(X^{\sett})^\top\mc_{\sett}\\
\leq{}& \frac{\beta_k^2}{\alpha_{n-k}^2}\norm{\mc_{\sett}}_2^2 + 2\cdot\frac{\beta_k}{\alpha_{n-k}}\cdot\norm{\mc_{\setn}}_2\cdot\norm{\mc_{\sett}}_2,
\end{align*}
where the last step follows from an application of the SSC/SSS properties by noticing that $|\seto\backslash\setn| \leq k$ and that $\mc_{\sett}$ and $\mc_{\setn}$ are all $k$-sparse vectors since $\mc$ itself is a $k$-sparse vector. Solving the above equation gives us
\[
\norm{\mc_{\setn}}_2 \leq (\sqrt 2 + 1)\cdot\frac{\beta_k}{\alpha_{n-k}}\cdot\norm{\mc_{\sett}}_2
\]
The above result proves that in $t = \bigO{\log\frac{\norm{\mc}_2}{\epsilon}}$ iterations, the alternating minimization procedure will identify an active set $\sett$ such that $\norm{\mc_{\sett}}_2 \leq \epsilon$. It is easy\elink{exer:rreg-ls} to see that a least squares step on this active set will yield a model $\bth$ satisfying
\[
\norm{\btt - \bto}_2 = \norm{C_t^{-1}(X^{\sett})^\top\mc_{\sett}}_2 \leq \frac{\beta_k}{\alpha_{n-k}}\cdot\epsilon \leq \epsilon,
\]
since $\beta_k/\alpha_{n-k} < 1$. This concludes the convergence guarantee.
\end{proof}

The crucial assumption in the previous result is the requirement $\beta_k/\alpha_{n-k} < \frac{1}{\sqrt 2 + 1}$. Clearly, as $k \rightarrow 0$, we have $\beta_k \rightarrow 0$ but if the matrix $X$ is well conditioned we still have $\alpha_{n-k} > 0$. Thus, for small enough $k$, it is assured that we will have $\beta_k/\alpha_{n-k} < \frac{1}{\sqrt 2 + 1}$. The point at which this occurs is the so-called \emph{breakdown point} of the algorithm -- it is the largest number $k$ such that the algorithm can tolerate $k$ possibly adversarial corruptions and yet guarantee recovery.

Note that the quantity $\kappa_k = \beta_k/\alpha_{n-k}$ acts as the effective condition number of the problem. It plays the same role as the condition number did in the analysis of the gPGD and IHT algorithms. It can be shown that for RIP-inducing distributions (see \citep{BhatiaJK2015}), AM-RR can tolerate $k = \Omega(n)$ corruptions.

For the specific case of the design matrix being generated from a Gaussian distribution, it can be shown that we have $\alpha_{n-k} = \Om{n-k}$ and $\beta_k = \bigO{k}$. This in turn can be used to show that we have $\beta_k/\alpha_{n-k} < \frac{1}{\sqrt 2 + 1}$ whenever $k \leq n/70$. This means that AM-RR can tolerate up to $n/70$ corruptions when the design matrix is Gaussian. This indicates a high degree of robustness in the algorithm since these corruptions can be completely adversarial in terms of their location, as well as their magnitude. Note that AM-RR is able to ensure this without requiring any specific initialization.

\section{Alternating Minimization via Gradient Updates}
\label{sec:rreg-hyb}
Similar to the gradient-based EM heuristic we looked at in \S~\ref{sec:em-implement}, we can make the alternating minimization process in AM-RR much cheaper by executing a gradient step for the alternations. More specifically, we can execute step 3 of AM-RR as
\[
\btn \leftarrow \btt - \eta\cdot\sum_{i \in \sett}(\x_i^\top\bt - y_i)\x_i,
\]
for some step size parameter $\eta$. It can be shown (see \citep{BhatiaJK2015}) that this process enjoys the same linear rate of convergence as AM-RR. However, notice that both alternations (model update as well as active set update) in this gradient-descent version can be carried out in (near-)linear time. This is in contrast with AM-RR which takes super-quadratic time in each alternation to discover the least squares solution. In practice, this makes the gradient version much faster (often by an order of magnitude) as compared to the fully corrective version.

However, once we have obtained a reasonably good estimate of $\seto$, it is better to execute the least squares solution to obtain the final solution in a single stroke. Thus, the gradient and fully corrective steps can be mixed together to great effect. In practice, such \emph{hybrid} techniques offer the fastest convergence. We refer the reader to \S~\ref{sec:rreg-empcomp} for a brief discussion on this and to \citep{BhatiaJK2015} for details.

\section{Robust Regression via Projected Gradient Descent}
\label{sec:rreg-gpgd}
It is possible to devise an alternate formulation for the robust regression problem that allows us to apply the gPGD technique instead. The work of \citep{BhatiaJKK2017} uses this alternate formulation to arrive at a solution that enjoys consistency properties. Note that if someone gave us a good estimate $\hat\vb$ of the corruption vector, we could use it to clean up the responses as $\y - \hat\vb$, and re-estimate the model as
\[
\hat\bt(\hat\vb) = \underset{\bt \in \bR^p}{\arg\min}\ \norm{(\y - \hat\vb) - X\bt}_2^2 = (X^\top X)^{-1}X^\top(\y-\hat\vb).
\]
The residuals corresponding to this new model estimate would be
\[
\norm{\y - \hat\vb - X\cdot\hat\bt(\hat\vb)}_2^2 = \norm{(I - P_X)(\y - \hat\vb)}_2^2,
\]
where $P_X = X(X^\top X)^{-1}X^\top$. The above calculation shows that an equivalent formulation for the robust regression problem \eqref{eq:rreg} is the following
\[
\underset{\norm{\vb}_0 = k}\min\ \norm{(I - P_X)(\y - \vb)}_2^2,
\]
to which we can apply gPGD (Algorithm~\ref{algo:gpgd}) since it now resembles a sparse recovery problem! This problem enjoys\elink{exer:rreg-rip} the restricted isometry property whenever the design matrix $X$ is sampled from an RIP-inducing distribution (see \S~\ref{sec:rip-ensure}). This shows that an application of the gPGD technique will guarantee recovery of the optimal corruption vector $\mc$ at a linear rate. Once we have an $\epsilon$-optimal estimate $\hat\vb$ of $\mc$, a good model $\hat\bt$ can be found by solving the least squares problem.
\[
\hat\bt(\hat\vb) = (X^\top X)^{-1}X^\top(\y-\hat\vb),
\]
as before. It is a simple exercise to show that if $\norm{\hat\vb - \mc}_2 \leq \epsilon$, then
\[
\norm{\hat\bt(\hat\vb) - \bto} \leq \bigO{\frac{\epsilon}{\alpha}},
\]
where $\alpha$ is the RSC parameter of the problem. Note that this shows how the gPGD technique can be applied to perform recovery not only when the parameter is sparse in the model domain (for instance in the gene expression analysis problem), but also when the parameter is sparse in the data domain, as in the robust regression example.

\begin{figure}[t]
\begin{subfigure}[t]{.5\columnwidth}
\centering \includegraphics[width=\columnwidth]{rreg-comp.pdf}
\caption{Run-time Comparison}
\label{fig:rreg-comparison-rreg}
\end{subfigure}
\hfill
\begin{subfigure}[t]{.5\columnwidth}
\centering \includegraphics[width=\columnwidth]{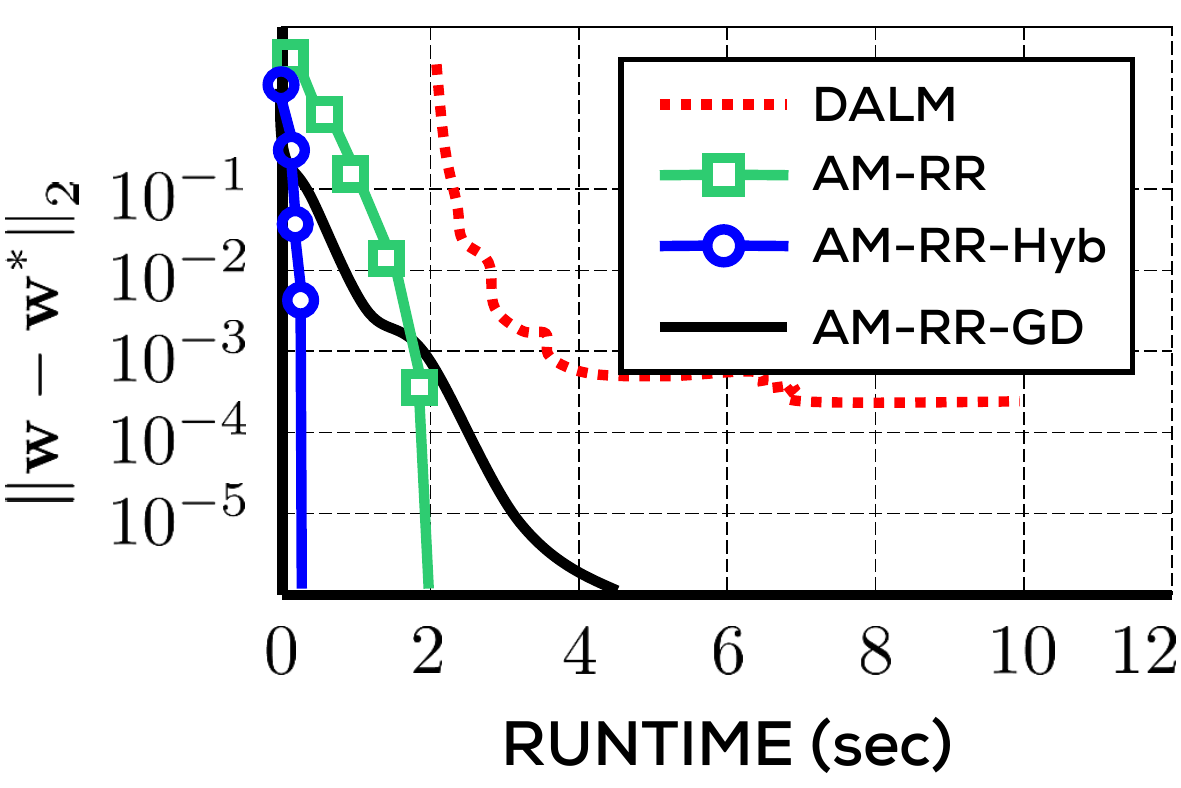}
\caption{Convergence Comparison}
\label{fig:rreg-comparison-hyb}
\end{subfigure}%
\caption[Empirical Performance on Robust Regression Problems]{An empirical comparison of the performance offered by various approaches for robust regression. Figure~\ref{fig:rreg-comparison-rreg} (adapted from \citep{BhatiaJKK2017}) compares Extended LASSO, a state-of-the-art relaxation method by \cite{NguyenT2013b}, AM-RR and the gPGD method from \S~\ref{sec:rreg-gpgd} on a robust regression problem in $p = 1000$ dimensions with $30\%$ data points corrupted. Non-convex techniques such as AM-RR and gPGD are more than an order of magnitude faster, and scale much better, than Extended LASSO. Figure~\ref{fig:rreg-comparison-hyb} (adapted from \citep{BhatiaJK2015}) compares various solvers on a robust regression problem in $p = 300$ dimensions with $1800$ data points of which $40\%$ are corrupted. The solvers include the gAM-style solver AM-RR, a variant using gradient-based updates, a hybrid method (see \S~\ref{sec:rreg-hyb}), and the DALM method \citep{YangZBSM2013}, a state-of-the-art solver for relaxed LASSO-style formulations. The hybrid method is the fastest of all the techniques. In general, all AM-RR variants are much faster than the relaxation-based method.}%
\label{fig:rreg-comparison}
\end{figure}

\section{Empirical Comparison}
\label{sec:rreg-empcomp}
Before concluding, we present some discussion on the empirical performance of various algorithms on the robust regression problem. Figure~\ref{fig:rreg-comparison-rreg} compares the running times of various robust regression approaches on synthetic problems as the number of data points available in the dataset increases. Similar trends are seen if the data dimensionality increases. The graph shows that non-convex techniques such as AM-RR and gPGD offer much more scalable solutions than relaxation-based methods. Figure~\ref{fig:rreg-comparison-hyb} similarly demonstrates how variants of the basic AM-RR procedure can offer significantly faster convergence. Figure~\ref{fig:faceamrr} looks at a realistic example of face reconstruction under occlusions and demonstrates how AM-RR is able to successfully recover the face image even when significant portions of the face are occluded.

\begin{figure}[t]
\includegraphics[width=\columnwidth]{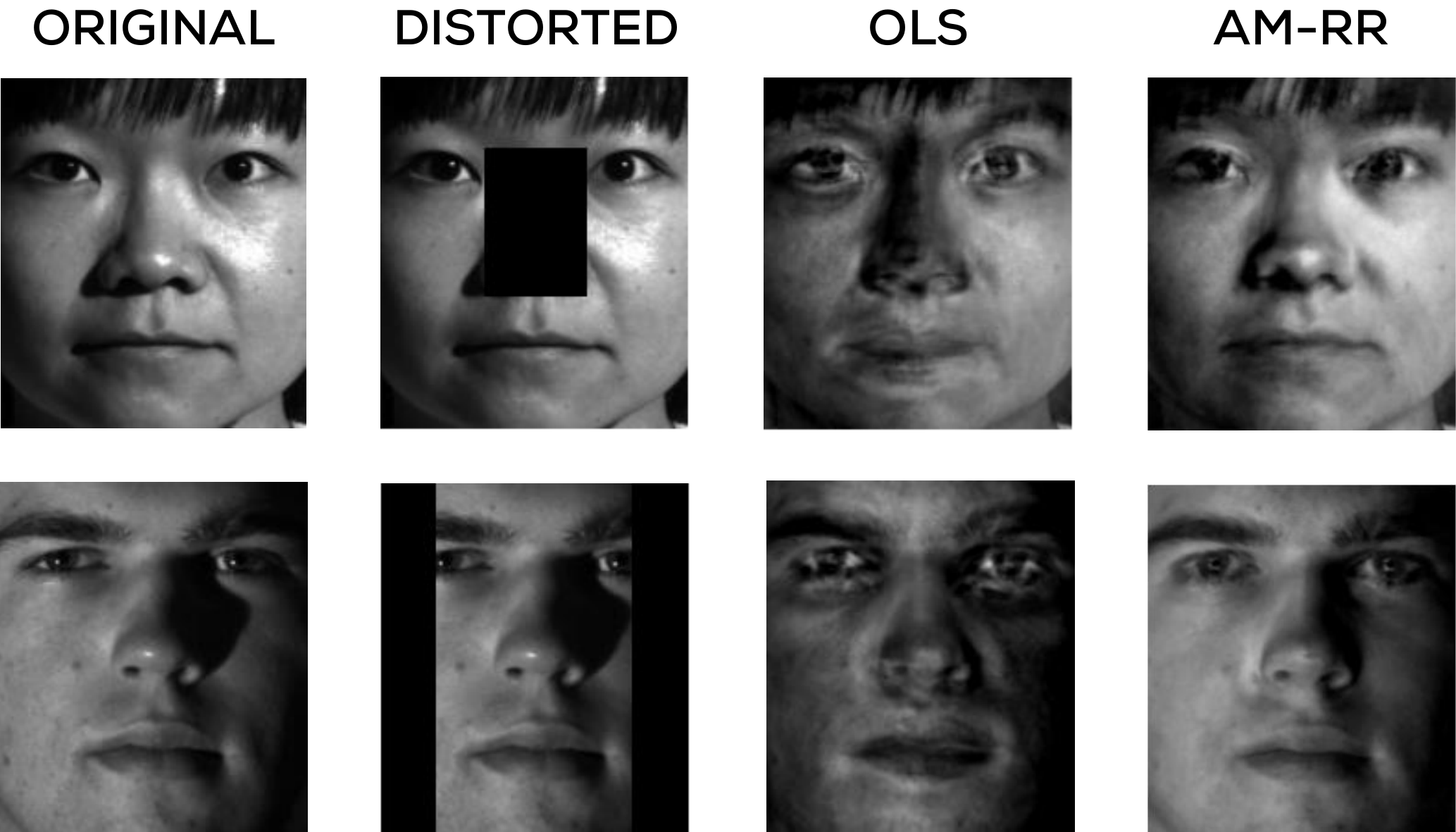}
\caption[Robust Face Reconstruction]{An experiment on face reconstruction using robust regression techniques. Two face images were taken and different occlusions were applied to them. Using the model described in \S~\ref{sec:rreg-intro}, reconstruction was attempted using both, ordinary least squares (OLS) and robust regression (AM-RR). It is clear that AM-RR achieves far superior reconstruction of the images and is able to correctly figure out the locations of the occlusions. Images courtesy the Yale Face Database B.}%
\label{fig:faceamrr}
\end{figure}

\section{Exercises}
\begin{exer}
\label{exer:rreg-impossible}
Show that it is impossible to recover the model vector if a fully adaptive adversary is able to corrupt more than half the responses i.e., if $k \geq n/2$. A fully adaptive adversary is one that is allowed to perform corruptions after observing the clean covariates as well as the uncorrupted responses.\\
\textit{Hint}: The adversary can make it impossible to distinguish between two models, the real model, and another one of its choosing.
\end{exer}
\begin{exer}
\label{exer:rreg-ssc}
Show that the SSC property ensures that there exists no subset $S$ of the data, $|S| \leq k$ and no two distinct model vectors $\vv^1,\vv^2\in\R^p$ such that $X^S\vv^1 = X^S\vv^2$.
\end{exer}
\begin{exer}
\label{exer:rreg-ls}
Show that executing the AM-RR algorithm for a single step starting from an active set $S_t$ such that $\norm{\mc_{\sett}}_2 \leq \epsilon$ ensures in the very next step that $\norm{\btt - \bto}_2 \leq \epsilon$.
\end{exer}
\begin{exer}
\label{exer:rreg-rip}
Show that if the design matrix $X$ satisfies RIP, then the objective function $f(\vb) = \norm{(I - P_X)(\y - \vb)}_2^2,$ enjoys RIP (of the same order as $X$ but with possibly different constants) as well.
\end{exer}
\begin{exer}
\label{exer:rreg-1-2}
Prove that ~\eqref{eq:rreg}~and~\eqref{eq:rreg-2} are equivalent formulations i.e., they yield the same model.
\end{exer}

\section{Bibliographic Notes}
The problem of robust estimation has been well studied in the statistics community. Indeed, there exist entire texts devoted to this area \citep{RousseeuwL1987,MaronnaMY2006} which look at robust estimators for regression and other problems. However, these methods often involve estimators, such as the least median of squares estimator, that have an exponential time complexity.

It is notable that these infeasible estimators often have attractive theoretical properties such as a high breakdown point. For instance, the work of \citet{Rousseeuw1984} shows that the least median of squares method enjoys a breakdown point of as high as $n/2 - p$. In contrast, AM-RR is only able to handle $n/70$ errors. However, whereas the gradient descent version of AM-RR can be executed in near-linear time, the least median of squares method requires time exponential in $p$.

There have been relaxation based approaches to solving robust regression and time series problems as well. Chief of them include the works of \cite{ChenD2012, ChenCM2013} which look at the Dantzig selector methods and the \emph{trimmed product} techniques to perform estimation, and the works of \cite{WrightYGSM2009, NguyenT2013}. The work of \cite{ChenCM2013} considers corruptions not only in the responses, but in the covariates as well. However, these methods tend to scale poorly to really large scale problems owing to the non-smooth nature of the optimization problems that they end up solving. The non-convex optimization techniques we have studied, on the other hand, require linear time or else have closed-form updates. 

Recent years have seen the application of non-convex techniques to robust estimation. However these works can be traced back to the classical work of \citet{FischlerB1981} that developed the RANSAC algorithm that is very widely used in fields such as computer vision. The RANSAC algorithm samples multiple candidate active sets and returns the least squares estimate on the set with least residual error.

Although the RANSAC method does not enjoy strong theoretical guarantees in the face of an adaptive adversary and a large number of corruptions, the method is seen to work well when there are very few outliers. Later works, such as that of \citet{SheO2011} applied soft-thresholding techniques to the problem followed by the work of \cite{BhatiaJK2015} which applied the alternating minimization algorithm we studied here. \cite{BhatiaJK2015} also looked at the problem of robust sparse recovery.

The time series literature has also seen the application of various techniques for robust estimation including robust M-estimators in both the additive and innovative outlier models \citep{MartinZ1978, StockingD1987} and least trimmed squares \citep{CrouxJ2008}.
\chapter{Phase Retrieval}
\label{chap:phret}

\newcommand{\dist}[1]{\text{dist}\br{#1}}

In this section, we will take a look at the phase retrieval problem, a non-convex optimization problem with applications in several domains. We briefly mentioned this problem in \S~\ref{chap:em} when we were discussing mixed regression. At a high level, phase retrieval is equivalent to discovering a complex signal using observations that reveal only the magnitudes of (complex) linear measurements over that signal. The phases of the measurements are not revealed to us.

Over reals, this reduces to solving a system of quadratic equations which is known to be computationally intractable to solve exactly. Fortunately however, typical phase retrieval systems usually require solving quadratic systems that have nice randomized structures in the coefficients of the quadratic equations. These structures can be exploited to efficiently solve these systems. In this section we will look at various algorithms that achieve this. 

\section{Motivating Applications}
The problem of phase retrieval arises in areas of signal processing where the phase of a signal being measured is irrevocably lost, or at best, unreliably obtained. The recovery of the signal in such situations becomes closely linked to our ability to perform phase retrieval.\\

\noindent\textbf{X-ray Crystallography} The goal in X-ray crystallography is to find the structure of a small molecule by bombarding it with X-rays from various angles and measuring the trajectories and intensities of the diffracted photons on a film. These quantities can be used to glean the internal three-dimensional structure of the electron cloud within the crystal, revealing the atomic arrangements therein. This technique has been found to be immensely useful in imaging specimens with a crystalline structure and has been historically significant in revealing the interior composition of several compounds, both inorganic and organic.

A notable example is that of nucleic acids such as DNA whose structure was revealed in the seminal work of \cite{FranklinG1953a} which immediately led Francis and Crick to propose its double helix structure. The reader may be intrigued by the now famous \emph{Photo 51} which provided critical support to the helical structure-theory of DNA \citep{FranklinG1953b}. We refer the reader to the expository work of \cite{Lucas2008} for a technical and historical account of how these discoveries came to be.\\

\noindent\textbf{Transmission Electron Microscopy (TEM)} This technique uses a focused beam of electrons instead of high energy photons to image the object of study. This technique works best with ultra thin specimens which do not absorb electrons but let the beam of electrons pass through. The electrons interact with the atomic structure of the specimen and are captured at the other end using a sensing mechanism. TEM is capable of resolutions far higher than those possible using photonic imaging techniques due to the extremely small de-Broglie wavelength of electrons.\\

\noindent\textbf{Coherent Diffraction Imaging (CDI)} This is a widely used technique for studying nanostructures such as nanotubes, nanocrystals and the like. A highly coherent beam of X-rays is made incident on the object of study and the diffracted rays allowed to interfere to produce a diffraction pattern which is used to recover the structure of the object. A key differentiating factor in CDI is the absence of any lenses to focus light onto the specimen, as opposed to other methods such as TEM/X-Ray crystallography which use optical or electromagnetic lenses to focus the incident beam and then refocus the diffracted beam. The absence of any lenses in CDI is very advantageous since it results in aberration-free patterns. Moreover, this way the resolution of the technique is only dependent on the wavelength and other properties of the incident rays rather than the material of the lens etc.

\section{Problem Formulation}
Bombarding a structure with X-rays or other beams can be shown to be equivalent to taking random Fourier measurements of its internal structure. More specifically, let $\bto \in \bC^p$ denote the vector that represents the density of electrons throughout the crystal, signifying its internal structure. Then, under simplifying regularity assumptions such as perfect periodicity, X-ray crystallography can be modeled as transforming $\bto$ into $\y = X\bto \in \bC^n$ where $X \in \bC^{n\times p}$ has each of its rows sampled from a Fourier measurement matrix, i.e, $X_{ab}=\exp\br{-\frac{2\pi(a-1)(b-1)}{p}}, b = 1,\ldots,p$ for some random $a \in [p]$.

Unfortunately, the measurements made by the sensors in the above applications are not able to observe the Fourier measurements exactly. Instead of observing complex valued transformations $y_k \in \bC$ of the signal $\y$, all we observe are the real value magnitudes $\abs{y_k} \in \bR$, the phase being lost in the signal acquisition process. A natural question arises whether it is still possible to recover $\bto$ or not.

Note that if the signal and the measurement matrices are real then $|y_k|^2=(\x_k^\top\bto)^2$. Thus, the goal is simply to recover a model $\bt$ by solving a system of $n$ quadratic equations described above. Although problem of solving a system of quadratic equations is intractable in general, in the special case of phase retrieval, it is possible to exploit the fact that the vectors $\x_k$ are sampled randomly, to develop recovery algorithms that are efficient, as well as require only a small number $n$ of measurements.

We note that the state of the art in phase retrieval literature is yet unable to guarantee recovery from random Fourier measurements, which closely model situations arising in crystallography and other imaging applications. Current analyses mostly consider the case when the sensing matrix $X$ has random Gaussian entries. A notable exception to the above is the work of \cite{CandesLS2015} which is able to address	\emph{coded diffraction} measurements. These are more relevant to what is used in practice but still complicated to design. For sake of simplicity, we will only focus on the Gaussian measurement case in our discussions. This would allow us to study the algorithmic techniques used to solve the problem without getting involved in the intricacies of Fourier or coded diffraction measurements.

To formalize our problem statement, we let $\bto \in \bC^p$ be an unknown $p$-dimensional signal, $X = [\x_1,\x_2,\ldots,\x_n]^\top \in \bC^{n \times p}$ be a measurement matrix. Our goal is to recover $\bto$ given $X$ and $\abs{\y}$, where $\y = [y_1,y_2,\ldots,y_k]^\top \in \bC^n = X\bto$ and $\abs{y_k} = \abs{\x_i^\top\bto}$. Our focus would be on the special case where $X_{kj} \stackrel{\text{i.i.d.}}{\sim} \cN(0,1) +i\cdot \cN(0,1)$ where $i^2=-1$. We would like to recover $\bto$ efficiently using only $n = \softO{p}$ measurements. We will study how gAM and gPGD-style techniques can be used to solve this problem. We will point to approaches using the relaxation technique in the bibliographic notes.\\

\noindent{\bf Notation}: We will abuse the notation $\x^\top$ to denote the complex row-conjugate of a complex vector $\x \in \bC^p$, something that is usually denoted by $\x^\ast$. A random vector $\x = \va + i\vb \in \bC^n$ will be said to be distributed according to the standard Gaussian distribution over $\bC^n$, denoted as $\cN_\bC(\vzero,I)$ if $\va,\vb \in \bR^n$ are independently distributed as standard (real valued) Gaussian vectors i.e., $\va,\vb \sim \cN(\vzero,I_n)$.

\section{Phase Retrieval via Alternating Minimization}
One of the most popular approaches for solving the phase retrieval problem is a simple application of the gAM approach, first proposed by \cite{GerchbergO1972} more than 4 decades ago. The intuition behind the approach is routine -- there exist two parameters of interest in the problem -- the phase of the data points i.e., $y_k/\abs{y_k}$, and the true signal $\bto$. Just as before, we observe a two-way connection between these parameters. Given the true phase values of the points $\phi_k = y_k/|y_k|$, the signal $\bto$ can be recovered simply by solving a system of linear equations: $\phi_k \cdot|y_k|=\x_k^\top\bt, k = 1,\ldots,n$. On the other hand, given $\bto$, estimating the phase of the points is straightforward as $\phi_k=(\x_k^\top\bto)/|\x_k^\top\bto|$.

Given the above, a gAM-style approach naturally arises: we alternately estimate $\phi_k$ and $\bto$. More precisely, at the $t$-th step, we first estimate $\phi_k^t$ using $\btp$ as $\phi_k^t=(\x_k^\top\btp)/|\x_k^\top\btp|$ and then update our estimate of $\bt$ by solving a least squares problem $\btt=\arg\min_{\bt} \sum_k |\phi_k^t |y_k| - \x_k^\top\bt|^2$ over complex variables. Algorithm~\ref{algo:phret_am} presents the details of this \emph{Gerchberg-Saxton Alternating Minimization} (GSAM) method. To avoid correlations, the algorithm performs these alternations on distinct sets of points at each time step. These disjoint sets can be created by sub-sampling the overall available data.

\begin{algorithm}[t]
	\caption{Gerchberg-Saxton Alternating Minimization (GSAM)}
	\label{algo:phret_am}
	\begin{algorithmic}[1]
			\REQUIRE Measurement matrix $X\in \bC^{n \times p}$, observed response magnitudes $|\y|\in \R_+^n$, desired accuracy $\epsilon$
			\ENSURE A signal $\bth \in \bC^p$
			\STATE Set $T \leftarrow \log 1/\epsilon$
			\STATE Partition $n$ data points into $T+1$ sets $S_0,S_1,\ldots,S_T$
			\STATE $\bt^0 \leftarrow \text{eig}\br{\frac{1}{|S_0|}\sum_{k \in S_0} |y_k|^2\cdot\x_k\x_k^\top, 1}$ \hfill //\emph{Leading eigenvector}
			\FOR{$t = 1, 2, \ldots, T$}
				\STATE Phase Estimation: $\phi_k = \x_k^\top \btp/|\x_k^\top \btp|$, for all $k \in S_t$
				\STATE Signal Estimation: $\btt = \arg\min_{\bt \in \bC^p}\sum_{k \in S_t} ||y_k|\cdot\phi_k- \x_k^\top\bt|^2$
			\ENDFOR
			\STATE \textbf{return} {$\bt^T$}
	\end{algorithmic}
\end{algorithm}

In their original work, \cite{GerchbergO1972} proposed to use a random vector $\bt^0$ for initialization. However, the recent work of \cite{NetrapalliJS13} demonstrated that a more careful initialization, in particular the largest eigenvector of $M=\sum_k |y_k|^2\cdot\x_k \x_k^\top$, is beneficial. Such a spectral initialization leads to an initial solution that is already at most a (small) constant distance away from the optimal solution $\bto$.

As we have seen to be the case with most gAM-style approaches, including the EM algorithm, this approximately optimal initialization is crucial to allow the alternating minimization procedure to take over and push the iterates toward the globally optimal solution.\\

\noindent\textbf{Notions of Convergence}: Before we proceed to give convergence guarantees for the GSAM algorithm, note that an exact recovery of $\bto$ is impossible, since phase information is totally lost. More specifically, two signals $\bto$ and $e^{i\theta}\cdot\bto$ for some $\theta \in \R$ will generate exactly the same responses when phase information is eliminated. Thus, the best we can hope for is to recover $\bto$ up to a phase shift. There are several ways of formalizing notions of convergence modulo a phase shift. We also note that complete proofs of the convergence results will be tedious and hence we will only give proof sketches for them.

\section{A Phase Retrieval Guarantee for GSAM}
While the GSAM approach has been known for decades, its convergence properties and rates were not understood until the work of \cite{NetrapalliJS13} who analyzed the heuristic (with spectral initialization as described in Algorithm~\ref{algo:phret_am}) for Gaussian measurement matrices. In particular, they proved that GSAM recovers an $\epsilon$-optimal estimate of $\bto$ in roughly $T = \log(1/\epsilon)$ iterations so long as the number of measurements satisfies $n = \softOm{p \log^3 p/\epsilon}$.

This result is established by first proving a linear convergence guarantee for the GSAM procedure assuming a sufficiently nice initialization. Next, it is shown that the spectral initialization described in Algorithm~\ref{algo:phret_am} does indeed satisfy this condition.\\

\noindent{\bf Linear Convergence}: For this part, it is assumed that the GSAM procedure has been initialized at $\bt^0$ such that $\norm{\bt^0 - \bto}_2 \leq \frac{1}{100}$. The following result (which we state without proof) shows that the alternating procedure hereafter ensures a linear rate of convergence.

\begin{theorem}
\label{thm:phret_am}
Let $y_k= \x_k^\top\bto$ for $k = 1,\ldots,n$ where $\x_k \sim \cN_\bC(\vzero,I)$ and $n \geq C\cdot p \log^3(p/\epsilon)$ for a suitably large constant $C$. Then, if the initialization satisfies $\norm{\bt^0 - \bto}_2 \leq \frac{1}{100}$, then with probability at least $1-1/n^2$, GSAM outputs an $\epsilon$-accurate solution $\norm{\bt^T - \bto}_2 \leq \epsilon$ in no more than $T = \bigO{\log\frac{\norm{\y}_2}{\epsilon}}$ steps.
\end{theorem}

In practice, one can use fast approximate solvers such as the conjugate gradient method to solve the least squares problem at each iteration. These take $\bigO{np\log(1/\epsilon)}$ time to solve a least squares instance. Since $n = \softO{p\log^3p}$ samples are enough, the GSAM algorithm operates with computation time at most $\softO{p^2\log^3(p/\epsilon)}$.\\

\noindent{\bf Initialization}: We now establish the utility of the initialization step. The proof hinges on a simple observation. Consider the random variable $Z = |y|^2\cdot\x\x^\top$ corresponding to a randomly chosen vector $\x \sim \cN_\bC(\vzero, I)$ and $y = \x^\top\bto$. For sake of simplicity, let us assume that $\norm{\bto}_2 = 1$. Since $\x = [x_1,x_2,\ldots,x_p]^\top$ is a spherically symmetric vector in the complex space $\bC^p$, the random variable $\x^\top\bto$ has an identical distribution as the vector $e^{i\theta}\cdot\x^\top\ve_1$ where $\ve_1 = [1,0,0,\ldots,0]^\top$, for any $\theta \in \R$. Using the above, it is easy to see that
\[
\E{|y|^2\cdot\x\x^\top} = \E{|x_1|^2\cdot\x\x^\top} = 4\cdot\ve_1\ve_1^\top + 4\cdot I
\]
Using a slightly more tedious calculation involving unitary transformations, we can extend the above to show that, in general,
\[
\E{|y|^2\cdot\x\x^\top} = 4\cdot\bto(\bto)^\top + 4\cdot I =: D
\]
The above clearly indicates that the largest eigenvector of the matrix $D$ is along $\bto$. Now notice that the matrix whose leading eigenvector we are interested in during initialization,
\[
S := \frac{1}{|S_0|}\sum_{k \in S_0} |y_k|^2\cdot\x_k\x_k^\top,
\]
is simply an empirical estimate to the expectation $\E{|y|^2\cdot\x\x^\top} = D$. Indeed, we have $\E{S} = D$. Thus, it is reasonable to expect that the leading eigenvector of $S$ would also be aligned to $\bto$. We can make this statement precise using results from the concentration of finite sums of self-adjoint independent random matrices from \citep{Tropp2012}.

\begin{theorem}
\label{thm:phret_init}
The spectral initialization method (Step 3 of Algorithm~\ref{algo:phret_am}), with probability at least $1 - 1/|S_0|^2 \leq 1 - 1/p^2$, ensures an initialization $\bt^0$ such that $\norm{\bt^0 - \bto}_2 \leq c$ for any constant $c > 0$, so long as it is executed with a randomly chosen set $S_0$ of data points of size $|S_0| \geq C\cdot p\log p$ for a suitably large constant $C$ depending on $c$.
\end{theorem}
\begin{proof}
To make the analysis simple, we will continue to assume that $\bto = e^{i\theta}\cdot\ve_1$ for some $\theta \in \R$. We can use Bernstein-style results for matrix concentration (for instance, see \cite[Theorem 1.5]{Tropp2012}) to show that for any chosen constant $c > 0$, if $n \geq C\cdot p\log p$ for a large enough constant $C$ that depends on the constant $c$, then with probability at least $1 - 1/|S_0|^2$, we have
\[
\norm{S - D}_2 \leq c
\]
Note that the norm being used above is the spectral/operator norm on matrices. Given this, it is possible to get a handle on the leading eigenvalue of $S$. Observe that since $\bt^0$ is the leading eigenvector of $S$, and since we have assumed $\bto = e^{i\theta}\cdot\ve_1$, we have
\begin{align*}
\textstyle\abs{\ip{\bt^0}{S\bt^0}} &\geq \abs{\ip{\bto}{S\bto}}\\
												 &\geq \abs{\ip{\bto}{D\bto}} - \abs{\ip{\bto}{(S-D)\bto}}\\
												 &\geq 8 - \norm{S-D}_2\norm{\bto}_2^2\\
												 &\geq 8 - c
\end{align*}
On the other hand, using the triangle inequality again, we have
\begin{align*}
\textstyle\abs{\ip{\bt^0}{S\bt^0}} &= \textstyle\abs{\ip{\bt^0}{(S-D)\bt^0} + \ip{\bt^0}{D\bt^0}}\\
												 &\leq \textstyle\norm{S-D}_2\norm{\bt^0}_2^2 + 4\abs{\bt^0_1}^2 + 4\norm{\bt^0}_2^2\\
												 &= \textstyle c + 4\abs{\bt^0_1}^2 + 4,
\end{align*}
where in the second step we have used $D = 4\cdot\ve_1\ve_1^\top + 4\cdot I$. The two opposing inequalities on the quantity $\abs{\ip{\bt^0}{S\bt^0}}$ give us $\abs{\bt^0_1}^2 \geq 1 - c/2$. Noticing that $\abs{\bt^0_1}^2 = \abs{\ip{\bt^0}{\ve_1}} = \abs{\ip{\bt^0}{\bto}}$ then establishes
\[
\textstyle\norm{\bt^0-\bto}_2^2 = 2(1 - \abs{\ip{\bt^0}{\bto}}) \leq 2(1 - \sqrt{1-c/2}) \leq c/2
\]
which finishes the proof.
\end{proof}

\section{Phase Retrieval via Gradient Descent}
There also exists a relatively straightforward reformulation of the phase retrieval problem that allows a simple gPGD-style approach to be applied. The recent work of \cite{CandesLS2015} did this by reformulating the phase retrieval problem in terms of the following objective
\[
f(\bt)=\sum_{k=1}^n\br{|y_k|^2 - |\x_k^\top \bt|^2}^2,
\]
and then performing gradient descent (over complex variables) on this unconstrained optimization problem. In the same work, this technique was named the\emph{Wirtinger's flow} algorithm, presumably as a reference to the notions of Wirtinger derivatives, and shown to offer provable convergence to the global optimum, just like the Gerchberg-Saxton method, when initialization is performed using a spectral method. Algorithm~\ref{algo:phret_wf} outlines the Wirtinger's Flow (WF) algorithm. Note that WF offers accelerated update times. Note that unlike the GSAM approach, the WF algorithm does not require sub-sampling but needs to choose a step size parameter instead. 

\begin{algorithm}[t]
	\caption{Wirtinger's Flow for Phase Retrieval (WF)}
	\label{algo:phret_wf}
	\begin{algorithmic}[1]
			\REQUIRE  Measurement matrix $X \in \bC^{n \times p}$, observed response magnitudes $|\y| \in \R_+^n$, step size $\eta$
			\ENSURE A signal $\bth \in \bC^p$
			\STATE $\bt^0 \leftarrow \text{eig}(\frac{1}{n}\sum_{k=1}^n |y_k|^2\cdot\x_k\x_k^\top, 1)$\hfill//{\em Leading eigenvector}
			\FOR{$t = 1, 2,\dots$}
			\STATE $\btt \leftarrow \btp - 2\eta\cdot\sum_k(|\x_k^\top\btp|^2-|y_k|^2)\x_k\x_k^\top\btp$
			\ENDFOR
			\STATE \textbf{return} {$\btt$}
	\end{algorithmic}
\end{algorithm}

\section{A Phase Retrieval Guarantee for WF}
Since the Wirtinger's Flow algorithm uses an initialization technique that is identical to that used by the Gerchberg-Saxton method, we can straightaway apply Theorem~\ref{thm:phret_init} to assure ourselves that with probability at least $1-1/n^2$, we have $\norm{\bt^0 - \bto}_2 \leq c$ for any constant $c > 0$. Starting from such an initial point, \cite{CandesLS2015} argue that each step of the gradient descent procedure decreases the distance to optima by at least a constant (multiplicative) factor. This allows a linear convergence result to be established for the WF procedure, similar to the GSAM approach, however, with each iteration being much less expensive, being a gradient descent step, rather than the solution to a complex-valued least squares problem. 
\begin{theorem}
\label{thm:phret_wf}
Let $y_k= \x_k^\top\bto$ for $k = 1,\ldots,n$ where $\x_k \sim \cN_\bC(\vzero,I)$. Also, let $n \geq C\cdot p \log p$ for a suitably large constant $C$. Then, if the initialization satisfies $\norm{\bt^0 - \bto}_2 \leq \frac{1}{100}$, then with probability at least $1-1/p^2$, WF outputs an $\epsilon$-accurate solution $\norm{\bt^T - \bto}_2 \leq \epsilon$ in no more than $T = \bigO{\log\frac{\norm{\y}_2}{\epsilon}}$ steps.
\end{theorem}
 \cite{CandesLS2015} also studied a coded diffraction pattern (CDP) model which uses measurements $X$ that are more ``practical'' for X-ray crystallography style applications and based on a combination of random multiplicative perturbations of a standard Fourier measurement. For such measurements, \cite{CandesLS2015} provided a result similar to Theorem~\ref{thm:phret_wf} but with a slightly inferior rate of convergence: the new procedure is able to guarantee an $\epsilon$-optimal solution only after $T = \bigO{p\cdot\log\frac{\norm{\y}_2}{\epsilon}}$ steps, i.e., a multiplicative factor of $p$ larger than that required by the WF algorithm for Gaussian measurements.

\section{Bibliographic Notes}
The phase retrieval problem has been studied extensively in the X-ray crystallography literature where the focus is mostly on studying Fourier measurement matrices \citep{CandesL2014}. For Gaussian measurements, initial results were obtained using a technique called \emph{Phase-lift}, which essentially viewed the quadratic equation $|y_k|^2=(\x_k^\top\bto)^2$ as the linear measurement of a rank-one matrix, i.e., $|y_k|^2= \ip{\x_k\x_k^\top}{W^\ast}$ where $W^\ast = \bto(\bto)^\top$.

This rank-one constraint was then replaced by a nuclear norm constraint and the resulting problem was solved as a semi-definite program (SDP). This technique was shown to achieve the information theoretically optimal sample complexity of $n = \Om{p}$ (recall that the GSAM and WF techniques require $n = \Om{p \log p}$. However, the running time of the Phase-lift algorithm is prohibitive at $O(np^2+p^3)$. 

In addition to the standard phase retrieval problem, several works \citep{NetrapalliJS13, JaganathanOH2013} have also studied the sparse phase retrieval problem where the goal is to recover a sparse signal $\bto \in \bC^p$, with $\norm{\bto}_0 \leq s \ll p$, using only magnitude measurements $|y_k|=|\x_k^\top\bto|$. The best known results for such problems require $n \geq s^3 \log p$ measurements for $s$-sparse signals. This is significantly worse than information theoretically optimal $\bigO{s\log p}$ number of measurements. However, \cite{JaganathanOH2013} showed that for a phase-lift style technique, one cannot hope to solve the problem using less than $\bigO{s^2\log p}$ measurements. 

\backmatter  

\bibliographystyle{plainnat}
\bibliography{refs}

\begin{thebibliography}{132}
\providecommand{\natexlab}[1]{#1}
\providecommand{\url}[1]{\texttt{#1}}
\expandafter\ifx\csname urlstyle\endcsname\relax
  \providecommand{\doi}[1]{doi: #1}\else
  \providecommand{\doi}{doi: \begingroup \urlstyle{rm}\Url}\fi

\bibitem[Agarwal et~al.(2012)Agarwal, Negahban, and Wainwright]{AgarwalNW2012}
Alekh Agarwal, Sahand~N. Negahban, and Martin~J. Wainwright.
\newblock {Fast global convergence of gradient methods for high-dimensional
  statistical recovery}.
\newblock \emph{The Annals of Statistics}, 40\penalty0 (5):\penalty0
  2452--2482, 2012.

\bibitem[Agarwal et~al.(2016)Agarwal, Anandkumar, Jain, and
  Netrapalli]{AgarwalAJN2016}
Alekh Agarwal, Animashree Anandkumar, Prateek Jain, and Praneeth Netrapalli.
\newblock {Learning Sparsely Used Overcomplete Dictionaries via Alternating
  Minimization}.
\newblock \emph{SIAM Journal of Optimization}, 26\penalty0 (4):\penalty0
  2775--2799, 2016.

\bibitem[Agarwal et~al.(2017)Agarwal, Allen-Zhu, Bullins, Hazan, and
  Ma]{AgarwalA-ZBHM2017}
Naman Agarwal, Zeyuan Allen-Zhu, Brian Bullins, Elad Hazan, and Tengyu Ma.
\newblock {Finding Approximate Local Minima Faster than Gradient Descent}.
\newblock In \emph{Proceedings of the 49th Annual ACM SIGACT Symposium on
  Theory of Computing (STOC)}, 2017.

\bibitem[Anandkumar and Ge(2016)]{AnandkumarG2016}
Animashree Anandkumar and Rong Ge.
\newblock {Efficient approaches for escaping higher order saddle points in
  non-convex optimization}.
\newblock In \emph{Proceedings of the 29th Conference on Learning Theory
  (COLT)}, pages 81--102, 2016.

\bibitem[Anandkumar et~al.(2014)Anandkumar, Ge, Hsu, Kakade, and
  Telgarsky]{AnandkumarGHKT2014}
Animashree Anandkumar, Rong Ge, Daniel Hsu, Sham~M. Kakade, and Matus
  Telgarsky.
\newblock {Tensor Decompositions for Learning Latent Variable Models}.
\newblock \emph{Journal of Machine Learning Research}, 15:\penalty0 2773--2832,
  2014.

\bibitem[Andresen and Spokoiny(2016)]{AndresenS2016}
Andreas Andresen and Vladimir Spokoiny.
\newblock {Convergence of an Alternating Maximization Procedure}.
\newblock \emph{Journal of Machine Learning Research}, 17:\penalty0 1--53,
  2016.

\bibitem[Arora et~al.(2014)Arora, Ge, and Moitra]{AroraGM2014}
Sanjeev Arora, Rong Ge, and Ankur Moitra.
\newblock {New Algorithms for Learning Incoherent and Overcomplete
  Dictionaries}.
\newblock In \emph{Proceedings of The 27th Conference on Learning Theory
  (COLT)}, 2014.

\bibitem[Azizzadenesheli et~al.(2016)Azizzadenesheli, Lazaric, and
  Anandkumar]{AzizzadenesheliLA2016}
Kamyar Azizzadenesheli, Alessandro Lazaric, and Anima Anandkumar.
\newblock {Reinforcement Learning of POMDPs using Spectral Methods}.
\newblock In \emph{Proceedings of the 29th Conference on Learning Theory
  (COLT)}, 2016.

\bibitem[Balakrishnan et~al.(2017)Balakrishnan, Wainwright, and
  Yu]{BalakrishnanWY2017}
Sivaraman Balakrishnan, Martin~J. Wainwright, and Bin Yu.
\newblock {Statistical Guarantees for the EM Algorithm: From Population to
  Sample-based Analysis}.
\newblock \emph{Annals of Statistics}, 45\penalty0 (1):\penalty0 77--120, 2017.

\bibitem[Baraniuk et~al.(2008)Baraniuk, Davenport, DeVore, and
  Wakin]{BaraniukDdVW2008}
Richard Baraniuk, Mark Davenport, Ronald DeVore, and Michael Wakin.
\newblock {A Simple Proof of the Restricted Isometry Property for Random
  Matrices}.
\newblock \emph{Constructive Approximation}, 28\penalty0 (3):\penalty0
  253--263, 2008.

\bibitem[Bertsekas(2016)]{Bertsekas2016}
Dimitri~P. Bertsekas.
\newblock \emph{{Nonlinear Programming}}.
\newblock Athena Scientific, 3rd edition, 2016.

\bibitem[Bhatia et~al.(2015)Bhatia, Jain, and Kar]{BhatiaJK2015}
Kush Bhatia, Prateek Jain, and Purushottam Kar.
\newblock {Robust Regression via Hard Thresholding}.
\newblock In \emph{Proceedings of the 29th Annual Conference on Neural
  Information Processing Systems (NIPS)}, 2015.

\bibitem[Bhatia et~al.(2017)Bhatia, Jain, Kamalaruban, and Kar]{BhatiaJKK2017}
Kush Bhatia, Prateek Jain, Parameswaran Kamalaruban, and Purushottam Kar.
\newblock {Consistent Robust Regression}.
\newblock In \emph{Proceedings of the 31st Annual Conference on Neural
  Information Processing Systems (NIPS)}, 2017.

\bibitem[Bhojanapalli and Jain(2014)]{BhojanapalliJ2014}
Srinadh Bhojanapalli and Prateek Jain.
\newblock {Universal Matrix Completion}.
\newblock In \emph{Proceedings of the 31st International Conference on Machine
  Learning (ICML)}, 2014.

\bibitem[Blumensath(2011)]{Blumensath2011}
Thomas Blumensath.
\newblock {Sampling and Reconstructing Signals From a Union of Linear
  Subspaces}.
\newblock \emph{IEEE Transactions on Information Theory}, 57\penalty0
  (7):\penalty0 4660--4671, 2011.

\bibitem[Bourgain et~al.(2011)Bourgain, Dilworth, Ford, Konyagin, and
  Kutzarova]{BourgainDFKK2011}
Jean Bourgain, Stephen Dilworth, Kevin Ford, Sergei Konyagin, and Denka
  Kutzarova.
\newblock {Explicit constructions of RIP matrices and related problems}.
\newblock \emph{Duke Mathematical Journal}, 159\penalty0 (1):\penalty0
  145--185, 2011.

\bibitem[Boyd and Vandenberghe(2004)]{BoydV2004}
Stephen Boyd and Lieven Vandenberghe.
\newblock \emph{{Convex Optimization}}.
\newblock Cambridge University Press, 2004.

\bibitem[Brutzkus and Globerson(2017)]{BrutzkusGloberson2017}
Alon Brutzkus and Amir Globerson.
\newblock {Globally Optimal Gradient Descent for a ConvNet with Gaussian
  Inputs}.
\newblock In \emph{Proceedings of the 34th International Conference on Machine
  Learning (ICML)}, 2017.

\bibitem[Bubeck(2015)]{Bubeck2015}
Sebastien Bubeck.
\newblock {Convex Optimization: Algorithms and Complexity}.
\newblock \emph{Foundations and Trends\textsuperscript{\textregistered}~in
  Machine Learning}, 8\penalty0 (34):\penalty0 231--357, 2015.

\bibitem[Cai et~al.(2010)Cai, Candès, and Shen]{CaiCS2010}
Jian-Feng Cai, Emmanuel~J. Candès, and Zuowei Shen.
\newblock {A Singular Value Thresholding Algorithm for Matrix Completion}.
\newblock \emph{SIAM Journal of Optimization}, 20\penalty0 (4):\penalty0
  1956--1982, 2010.

\bibitem[Cand\`es and Tao(2005)]{CandesT2005}
Emmanuel Cand\`es and Terence Tao.
\newblock {Decoding by Linear Programming}.
\newblock \emph{IEEE Transactions on Information Theory}, 51\penalty0
  (12):\penalty0 4203--4215, 2005.

\bibitem[Cand\`es(2008)]{Candes2008}
Emmanuel~J. Cand\`es.
\newblock {The Restricted Isometry Property and Its Implications for Compressed
  Sensing}.
\newblock \emph{Comptes Rendus Mathematique}, 346\penalty0 (9-10):\penalty0
  589--592, 2008.

\bibitem[Cand\`es and Li(2014)]{CandesL2014}
Emmanuel~J. Cand\`es and Xiaodong Li.
\newblock {Solving Quadratic Equations via PhaseLift When There Are About as
  Many Equations as Unknowns}.
\newblock \emph{Foundations of Computational Mathematics}, 14\penalty0
  (5):\penalty0 1017--1026, 2014.

\bibitem[Cand\`es and Recht(2009)]{CandesR2009}
Emmanuel~J. Cand\`es and Benjamin Recht.
\newblock {Exact Matrix Completion via Convex Optimization}.
\newblock \emph{Foundations of Computational Mathematics}, 9\penalty0
  (6):\penalty0 717--772, 2009.

\bibitem[Cand\`es and Tao(2009)]{CandesT2009}
Emmanuel~J. Cand\`es and Terence Tao.
\newblock {The power of convex relaxation: Near-optimal matrix completion}.
\newblock \emph{IEEE Transactions on Information Theory}, 56\penalty0
  (5):\penalty0 2053--2080, 2009.

\bibitem[Cand\`es et~al.(2006)Cand\`es, Romberg, and Tao]{CandesRT2006}
Emmanuel~J. Cand\`es, Justin~K. Romberg, and Terence Tao.
\newblock {Stable Signal Recovery from Incomplete and Inaccurate Measurements}.
\newblock \emph{Communications on Pure and Applied Mathematics}, 59\penalty0
  (8):\penalty0 1207--1223, 2006.

\bibitem[Cand\`es et~al.(2015)Cand\`es, Li, and Soltanolkotabi]{CandesLS2015}
Emmanuel~J. Cand\`es, Xiaodong Li, and Mahdi Soltanolkotabi.
\newblock {Phase Retrieval via Wirtinger Flow: Theory and Algorithms}.
\newblock \emph{IEEE Transactions on Information Theory}, 61\penalty0
  (4):\penalty0 1985--2007, 2015.

\bibitem[Carmon et~al.(2017)Carmon, Duchi, Hinder, and Sidford]{CarmonDHS2017}
Yair Carmon, John~C. Duchi, Oliver Hinder, and Aaron Sidford.
\newblock {``Convex Until Proven Guilty'': Dimension-Free Acceleration of
  Gradient Descent on Non-Convex Functions}.
\newblock In \emph{Proceedings of the 34th International Conference on Machine
  Learning (ICML)}, 2017.

\bibitem[Chartrand(2007)]{Chartrand2007}
Rick Chartrand.
\newblock {Exact Reconstruction of Sparse Signals via Nonconvex Minimization}.
\newblock \emph{IEEE Information Processing Letters}, 14\penalty0
  (10):\penalty0 707--710, 2007.

\bibitem[Chen and He(2012)]{ChenH2012}
Caihua Chen and Bingsheng He.
\newblock {Matrix Completion via an Alternating Direction Method}.
\newblock \emph{IMA Journal of Numerical Analysis}, 32\penalty0 (1):\penalty0
  227--245, 2012.

\bibitem[Chen and Gu(2015)]{Chen-Gu2015}
Laming Chen and Yuantao Gu.
\newblock {Local and global optimality of LP minimization for sparse recovery}.
\newblock In \emph{Proceedings of the IEEE International Conference on
  Acoustics, Speech and Signal Processing (ICASSP)}, 2015.

\bibitem[Chen and Dalalyan(2012)]{ChenD2012}
Yin Chen and Arnak~S. Dalalyan.
\newblock {Fused sparsity and robust estimation for linear models with unknown
  variance}.
\newblock In \emph{Proceedings of the 26th Annual Conference on Neural
  Information Processing Systems (NIPS)}, 2012.

\bibitem[Chen et~al.(2013)Chen, Caramanis, and Mannor]{ChenCM2013}
Yudong Chen, Constantine Caramanis, and Shie Mannor.
\newblock {Robust Sparse Regression under Adversarial Corruption}.
\newblock In \emph{Proceedings of the 30th International Conference on Machine
  Learning (ICML)}, 2013.

\bibitem[Chen et~al.(2016)Chen, Xu, Caramanis, and Sanghavi]{ChenXCS2016}
Yudong Chen, Huan Xu, Constantine Caramanis, and Sujay Sanghavi.
\newblock {Matrix Completion with Column Manipulation: Near-Optimal
  Sample-Robustness-Rank Tradeoffs}.
\newblock \emph{IEEE Transactions on Information Theory}, 62\penalty0
  (1):\penalty0 503--526, 2016.

\bibitem[Cherapanamjeri et~al.(2017)Cherapanamjeri, Gupta, and
  Jain]{CherapanamjeriGJ2017}
Yeshwanth Cherapanamjeri, Kartik Gupta, and Prateek Jain.
\newblock {Nearly-optimal Robust Matrix Completion}.
\newblock In \emph{Proceedings of the 34th International Conference on Machine
  Learning (ICML)}, 2017.

\bibitem[Choromanska et~al.(2015)Choromanska, Hena, Mathieu, Arous, and
  LeCun]{ChoromanskaHMALC2015}
Anna Choromanska, Mikael Hena, Michael Mathieu, G\'erard~Ben Arous, and Yann
  LeCun.
\newblock {The Loss Surfaces of Multilayer Networks}.
\newblock In \emph{Proceedings of the 18th International Conference on Arti
  cial Intelligence and Statistics (AISTATS)}, 2015.

\bibitem[Cohen et~al.(2009)Cohen, Dahmen, and DeVore]{CohenDDeV2009}
Albert Cohen, Wolfgang Dahmen, and Ronald DeVore.
\newblock {Compressed Sensing and Best $k$-term Approximation}.
\newblock \emph{Journal of the American Mathematical Society}, 22\penalty0
  (1):\penalty0 211--231, 2009.

\bibitem[Croux and Joossens(2008)]{CrouxJ2008}
Christophe Croux and Kristel Joossens.
\newblock {Robust Estimation of the Vector Autoregressive Model by a Least
  Trimmed Squares Procedure}.
\newblock In \emph{Proceedings in Computational Statistics (COMPSTAT)}, 2008.

\bibitem[Dauphin et~al.(2014)Dauphin, Pascanu, G{\"{u}}l{\c{c}}ehre, Cho,
  Ganguli, and Bengio]{DauphinPGCGB2014}
Yann~N. Dauphin, Razvan Pascanu, {\c{C}}aglar G{\"{u}}l{\c{c}}ehre, Kyunghyun
  Cho, Surya Ganguli, and Yoshua Bengio.
\newblock {Identifying and attacking the saddle point problem in
  high-dimensional non-convex optimization}.
\newblock In \emph{Proceedings of the 28th Annual Conference on Neural
  Information Processing Systems (NIPS)}, pages 2933--2941, 2014.

\bibitem[Dempster et~al.(1977)Dempster, Laird, and Rubin]{DempsterLR1977}
Arthur~P. Dempster, Nan~M. Laird, and Donald~B. Rubin.
\newblock {Maximum Likelihood from Incomplete Data via the EM Algorithm}.
\newblock \emph{Journal of the Royal Statistical Society, Series B},
  39\penalty0 (1):\penalty0 1--38, 1977.

\bibitem[Donoho(2006)]{Donoho2006}
David~L. Donoho.
\newblock {Compressed Sensing}.
\newblock \emph{IEEE Transactions on Information Theory}, 52\penalty0
  (4):\penalty0 1289--1306, 2006.

\bibitem[Donoho et~al.(2009)Donoho, Maleki, and Montanari]{DonohoMM2009}
David~L. Donoho, Arian Maleki, and Andrea Montanari.
\newblock {Message Passing Algorithms for Compressed Sensing: I. Motivation and
  Construction}.
\newblock \emph{Proceedings of the National Academy of Sciences USA},
  106\penalty0 (45):\penalty0 18914--18919, 2009.

\bibitem[Duchi et~al.(2008)Duchi, Shalev-Shwartz, Singer, and
  Chandra]{DuchiS-SSC2008}
John Duchi, Shai Shalev-Shwartz, Yoram Singer, and Tushar Chandra.
\newblock {Efficient Projections onto the $\ell_1$-Ball for Learning in High
  Dimensions}.
\newblock In \emph{Proceedings of the 25th International Conference on Machine
  Learning (ICML)}, 2008.

\bibitem[Fan et~al.(2008)Fan, Chang, Hsieh, Wang, and Lin]{FanCHWL2008}
Rong-En Fan, Kai-Wei Chang, Cho-Jui Hsieh, Xiang-Rui Wang, and Chih-Jen Lin.
\newblock {LIBLINEAR: A Library for Large Linear Classification}.
\newblock \emph{Journal of Machine Learning Research}, 9:\penalty0 1871--1874,
  2008.

\bibitem[Fazel et~al.(2013)Fazel, Pong, Sun, and Tseng]{FazelPST2013}
Maryam Fazel, Ting~Kei Pong, Defeng Sun, and Paul Tseng.
\newblock {Hankel matrix rank minimization with applications in system
  identification and realization}.
\newblock \emph{SIAM Journal on Matrix Analysis and Applications}, 34\penalty0
  (3):\penalty0 946--977, 2013.

\bibitem[Fischler and Bolles(1981)]{FischlerB1981}
Martin~A. Fischler and Robert~C. Bolles.
\newblock {Random Sample Consensus: A Paradigm for Model Fitting with
  Applications to Image Analysis and Automated Cartography}.
\newblock \emph{Communications of the ACM}, 24\penalty0 (6):\penalty0 381--395,
  1981.

\bibitem[Foucart(2010)]{Foucart2010}
Simon Foucart.
\newblock {A Note on Guaranteed Sparse Recovery via $\ell_1$-minimization}.
\newblock \emph{Applied and Computational Harmonic Analysis}, 29\penalty0
  (1):\penalty0 97--103, 2010.

\bibitem[Foucart(2011)]{Foucart2011}
Simon Foucart.
\newblock {Hard Thresholding Pursuit: an Algorithm for Compressive Sensing}.
\newblock \emph{SIAM Journal on Numerical Analysis}, 49\penalty0 (6):\penalty0
  2543--2563, 2011.

\bibitem[Foucart and Lai(2009)]{Foucart-Lai2009}
Simon Foucart and Ming-Jun Lai.
\newblock {Sparsest solutions of underdetermined linear systems via
  $\ell_q$-minimization for $0 < q \leq 1$}.
\newblock \emph{Applied and Computational Harmonic Analysis}, 26\penalty0
  (3):\penalty0 395--407, 2009.

\bibitem[Franklin and Gosling(1953{\natexlab{a}})]{FranklinG1953a}
Rosalind Franklin and Raymond~G. Gosling.
\newblock {Evidence for 2-Chain Helix in Crystalline Structure of Sodium
  Deoxyribonucleate}.
\newblock \emph{Nature}, 172:\penalty0 156--157, 1953{\natexlab{a}}.

\bibitem[Franklin and Gosling(1953{\natexlab{b}})]{FranklinG1953b}
Rosalind Franklin and Raymond~G. Gosling.
\newblock {Molecular Configuration in Sodium Thymonucleate}.
\newblock \emph{Nature}, 171:\penalty0 740--741, 1953{\natexlab{b}}.

\bibitem[Garg and Khandekar(2009)]{GargK2009}
Rahul Garg and Rohit Khandekar.
\newblock {Gradient Descent with Sparsification: An iterative algorithm for
  sparse recovery with restricted isometry property}.
\newblock In \emph{Proceedings of the 26th International Conference on Machine
  Learning (ICML)}, 2009.

\bibitem[Ge et~al.(2015)Ge, Huang, Jin, and Yuan]{GeHJY2015}
Rong Ge, Furong Huang, Chi Jin, and Yang Yuan.
\newblock {Escaping From Saddle Points - Online Stochastic Gradient for Tensor
  Decomposition}.
\newblock In \emph{Proceedings of The 28th Conference on Learning Theory
  (COLT)}, pages 797--842, 2015.

\bibitem[Ge et~al.(2016)Ge, Lee, and Ma]{GeLM2016}
Rong Ge, Jason~D. Lee, and Tengyu Ma.
\newblock {Matrix Completion has No Spurious Local Minimum}.
\newblock In \emph{Proceedings of the 30th Annual Conference on Neural
  Information Processing Systems (NIPS)}, 2016.

\bibitem[Gerchberg and Saxton(1972)]{GerchbergO1972}
R.~W. Gerchberg and W.~Owen Saxton.
\newblock {A Practical Algorithm for the Determination of Phase from Image and
  Diffraction Plane Pictures}.
\newblock \emph{Optik}, 35\penalty0 (2):\penalty0 237--246, 1972.

\bibitem[Goel and Klivans(2017)]{GoelKlivans2017}
Surbhi Goel and Adam Klivans.
\newblock {Learning Depth-Three Neural Networks in Polynomial Time}.
\newblock arXiv:1709.06010v1 [cs.DS], 2017.

\bibitem[Goldfarb and Ma(2011)]{GoldfarbM2011}
Donald Goldfarb and Shiqian Ma.
\newblock {Convergence of Fixed-Point Continuation Algorithms for Matrix Rank
  Minimization}.
\newblock \emph{Foundations of Computational Mathematics}, 11\penalty0
  (2):\penalty0 183--210, 2011.

\bibitem[Golub and Loan(1996)]{GolubVL1996}
Gene~H. Golub and Charles F.~Van Loan.
\newblock \emph{{Matrix Computations}}.
\newblock Johns Hopkins Studies in Mathematical Sciences. The John Hopkins
  University Press, 3rd edition, 1996.

\bibitem[Gribonval et~al.(2015)Gribonval, Jenatton, and Bach]{GribonvalJB2015}
R\'emi Gribonval, Rodolphe Jenatton, and Francis Bach.
\newblock {Sparse and Spurious: Dictionary Learning With Noise and Outliers}.
\newblock \emph{IEEE Transaction on Information Theory}, 61\penalty0
  (11):\penalty0 6298--6319, 2015.

\bibitem[Hardt(2014)]{Hardt2014}
Moritz Hardt.
\newblock {Understanding Alternating Minimization for Matrix Completion}.
\newblock In \emph{Proceedings of the 55th IEEE Annual Symposium on Foundations
  of Computer Science (FOCS)}, 2014.

\bibitem[Hardt and Wootters(2014)]{HardtW2014}
Moritz Hardt and Mary Wootters.
\newblock {Fast Matrix Completion Without the Condition Number}.
\newblock In \emph{Proceedings of The 27th Conference on Learning Theory
  (COLT)}, 2014.

\bibitem[Hardt et~al.(2014)Hardt, Meka, Raghavendra, and Weitz]{HardtMRW2014}
Moritz Hardt, Raghu Meka, Prasad Raghavendra, and Benjamin Weitz.
\newblock {Computational limits for matrix completion}.
\newblock In \emph{Proceedings of The 27th Conference on Learning Theory
  (COLT)}, 2014.

\bibitem[Hastie et~al.(2016)Hastie, Tibshirani, and Wainwright]{HastieTW2016}
Trevor Hastie, Robert Tibshirani, and Martin Wainwright.
\newblock \emph{{Statistical Learning with Sparsity: The Lasso and
  Generalizations}}.
\newblock Number 143 in Monographs on Statistics and Applied Probability. The
  CRC Press, 2016.

\bibitem[Haviv and Regev(2017)]{HavivR17}
Ishay Haviv and Oded Regev.
\newblock {The Restricted Isometry Property of Subsampled Fourier Matrices}.
\newblock In Bo'az Klartag and Emanuel Milman, editors, \emph{Geometric Aspects
  of Functional Analysis}, volume 2169 of \emph{Lecture Notes in Mathematics},
  pages 163--179. Springer, Cham, 2017.

\bibitem[Huber and Ronchetti(2009)]{HuberR2009}
Peter~J. Huber and Elvezio~M. Ronchetti.
\newblock \emph{{Robust Statistics}}.
\newblock {Wiley Series in Probability and Statistics}. John Wiley \& Sons, 2nd
  edition, 2009.

\bibitem[Jaganathan et~al.(2013)Jaganathan, Oymak, and
  Hassibi]{JaganathanOH2013}
Kishore Jaganathan, Samet Oymak, and Babak Hassibi.
\newblock {Sparse Phase Retrieval: Convex Algorithms and Limitations}.
\newblock In \emph{Proceedings of the IEEE International Symposium on
  Information Theory (ISIT)}, 2013.

\bibitem[Jain and Netrapalli(2015)]{JainN2015}
Prateek Jain and Praneeth Netrapalli.
\newblock {Fast Exact Matrix Completion with Finite Samples}.
\newblock In \emph{Proceedings of The 28th Conference on Learning Theory
  (COLT)}, 2015.

\bibitem[Jain and Tewari(2015)]{JainT2015}
Prateek Jain and Ambuj Tewari.
\newblock {Alternating Minimization for Regression Problems with Vector-valued
  Outputs}.
\newblock In \emph{Proceedings of the 29th Annual Conference on Neural
  Information Processing Systems (NIPS)}, 2015.

\bibitem[Jain et~al.(2010)Jain, Meka, and Dhillon]{JainMD2010}
Prateek Jain, Raghu Meka, and Inderjit Dhillon.
\newblock {Guaranteed Rank Minimization via Singular Value Projections}.
\newblock In \emph{Proceedings of the 24th Annual Conference on Neural
  Information Processing Systems (NIPS)}, 2010.

\bibitem[Jain et~al.(2011)Jain, Tewari, and Dhillon]{JainTD2011}
Prateek Jain, Ambuj Tewari, and Inderjit~S. Dhillon.
\newblock {Orthogonal Matching Pursuit with Replacement}.
\newblock In \emph{Proceedings of the 25th Annual Conference on Neural
  Information Processing Systems (NIPS)}, 2011.

\bibitem[Jain et~al.(2013)Jain, Netrapalli, and Sanghavi]{JainNS2013}
Prateek Jain, Praneeth Netrapalli, and Sujay Sanghavi.
\newblock {Low-rank Matrix Completion using Alternating Minimization}.
\newblock In \emph{Proceedings of the 45th annual ACM Symposium on Theory of
  Computing (STOC)}, pages 665--674, 2013.

\bibitem[Jain et~al.(2014)Jain, Tewari, and Kar]{JainTK2014}
Prateek Jain, Ambuj Tewari, and Purushottam Kar.
\newblock {On Iterative Hard Thresholding Methods for High-dimensional
  M-Estimation}.
\newblock In \emph{Proceedings of the 28th Annual Conference on Neural
  Information Processing Systems (NIPS)}, 2014.

\bibitem[Jalali et~al.(2011)Jalali, Johnson, and Ravikumar]{JalaliJR2011}
Ali Jalali, Christopher~C Johnson, and Pradeep~D Ravikumar.
\newblock {On Learning Discrete Graphical Models using Greedy Methods}.
\newblock In \emph{Proceedings of the 25th Annual Conference on Neural
  Information Processing Systems (NIPS)}, pages 1935--1943, 2011.

\bibitem[Jin et~al.(2017)Jin, Ge, Netrapalli, Kakade, and Jordan]{Jin0NKJ17}
Chi Jin, Rong Ge, Praneeth Netrapalli, Sham~M. Kakade, and Michael~I. Jordan.
\newblock {How to Escape Saddle Points Efficiently}.
\newblock In \emph{Proceedings of the 34th International Conference on Machine
  Learning (ICML)}, pages 1724--1732, 2017.

\bibitem[Kashin(1975)]{Kashin1975}
Boris~Sergeevich Kashin.
\newblock {The diameters of octahedra}.
\newblock \emph{Uspekhi Matematicheskikh Nauk}, 30\penalty0 (4(184)):\penalty0
  251--252, 1975.

\bibitem[Keshavan(2012)]{Keshavan2012}
Raghunandan~H. Keshavan.
\newblock \emph{{Efficient Algorithms for Collaborative Filtering}}.
\newblock {Ph.D. Thesis}, Stanford University, 2012.

\bibitem[Keshavan et~al.(2010)Keshavan, Montanari, and Oh]{KeshavanMO2010}
Raghunandan~H. Keshavan, Andrea Montanari, and Sewoong Oh.
\newblock {Matrix Completion from a Few Entries}.
\newblock \emph{IEEE Transactions on Information Theory}, 56\penalty0
  (6):\penalty0 2980--2998, 2010.

\bibitem[Koren et~al.(2009)Koren, Bell, and Volinsky]{KorenBV2009}
Yehuda Koren, Robert Bell, and Chris Volinsky.
\newblock {Matrix Factorization Techniques for. Recommender Systems}.
\newblock \emph{IEEE Computer}, 42\penalty0 (8):\penalty0 30--37, 2009.

\bibitem[Lee et~al.(2016)Lee, Simchowitz, Jordan, and Recht]{LeeSJR2016}
Jason~D. Lee, Max Simchowitz, Michael~I. Jordan, and Benjamin Recht.
\newblock {Gradient Descent Only Converges to Minimizers}.
\newblock In \emph{Proceedings of the 29th Conference on Learning Theory
  (COLT)}, pages 1246--1257, 2016.

\bibitem[Lee and Bresler(2010)]{LeeB2010}
Kiryung Lee and Yoram Bresler.
\newblock {ADMiRA: Atomic Decomposition for Minimum Rank Approximation}.
\newblock \emph{IEEE Transactions on Information Theory}, 56\penalty0
  (9):\penalty0 4402--4416, 2010.

\bibitem[Li and Yuan(2017)]{LiYuan2017}
Yuanzhi Li and Yang Yuan.
\newblock {Convergence Analysis of Two-layer Neural Networks with ReLU
  Activation}.
\newblock In \emph{Proceedings of the 31st Annual Conference on Neural
  Information Processing Systems (NIPS)}, 2017.

\bibitem[Lloyd(1982)]{Lloyd1982}
Stuart~P. Lloyd.
\newblock {Least squares quantization in PCM}.
\newblock \emph{IEEE Transactions on Information Theory}, 28\penalty0
  (2):\penalty0 129--137, 1982.

\bibitem[Lucas(2008)]{Lucas2008}
Amand~A. Lucas.
\newblock {A-DNA and B-DNA: Comparing Their Historical X-ray Fiber Diffraction
  Images}.
\newblock \emph{Journal of Chemical Education}, 85\penalty0 (5):\penalty0
  737--743, 2008.

\bibitem[Luo and Tseng(1992)]{LuoT1992}
Zhi-Quan Luo and Paul Tseng.
\newblock {On the Convergence of the Coordinate Descent Method for Convex
  Differentiable Minimization}.
\newblock \emph{Journal of Optimization Theory and Applications}, 72\penalty0
  (1):\penalty0 7--35, 1992.

\bibitem[Luo and Tseng(1993)]{LuoT1993}
Zhi-Quan Luo and Paul Tseng.
\newblock {Error bounds and convergence analysis of feasible descent methods: A
  general approach}.
\newblock \emph{Annals of Operations Research}, 46\penalty0 (1):\penalty0
  157--178, 1993.

\bibitem[Maronna et~al.(2006)Maronna, Martin, and Yoha]{MaronnaMY2006}
Ricardo~A. Maronna, R.~Douglas Martin, and Victor~J. Yoha.
\newblock \emph{{Robust Statistics: Theory and Methods}}.
\newblock John Wiley, 2006.

\bibitem[Martin and Zeh(1978)]{MartinZ1978}
R.~Douglas Martin and Judy Zeh.
\newblock {Robust Generalized M-estimates for Autoregressive Parameters:
  Small-sample Behavior and Applications}.
\newblock Technical Report 214, University of Washington, 1978.

\bibitem[Meka et~al.(2008)Meka, Jain, Caramanis, and Dhillon]{MekaJCD2008}
Raghu Meka, Prateek Jain, Constantine Caramanis, and Inderjit Dhillon.
\newblock {Rank Minimization via Online Learning}.
\newblock In \emph{Proceedings of the 25th International Conference on Machine
  Learning (ICML)}, 2008.

\bibitem[Natarajan(1995)]{Natarajan1995}
Balas~Kausik Natarajan.
\newblock {Sparse approximate solutions to linear systems}.
\newblock \emph{SIAM Journal on Computing}, 24\penalty0 (2):\penalty0 227--234,
  1995.

\bibitem[Needell and Tropp(2008)]{NeedellT2008}
Deanna Needell and Joel~A. Tropp.
\newblock {CoSaMP: Iterative Signal Recovery from Incomplete and Inaccurate
  Samples}.
\newblock \emph{Applied and Computational Harmonic Analysis}, 26:\penalty0
  301--321, 2008.

\bibitem[Negahban et~al.(2012)Negahban, Ravikumar, Wainwright, and
  Yu]{NegahbanRWY2012}
Sahand~N. Negahban, Pradeep Ravikumar, Martin~J. Wainwright, and Bin Yu.
\newblock {A Unified Framework for High-Dimensional Analysis of M-Estimators
  with Decomposable Regularizers}.
\newblock \emph{Statistical Science}, 27\penalty0 (4):\penalty0 538--557, 2012.

\bibitem[Nelson et~al.(2014)Nelson, Price, and Wootters]{NelsonPW2014}
Jelani Nelson, Eric Price, and Mary Wootters.
\newblock {New constructions of RIP matrices with fast multiplication and fewer
  rows}.
\newblock In \emph{Proceedings of the 25th Annual ACM-SIAM Symposium on
  Discrete Algorithms (SODA)}, 2014.

\bibitem[Nesterov(2003)]{Nesterov2013}
Yurii Nesterov.
\newblock \emph{{Introductory Lectures on Convex Optimization: A Basic
  Course}}.
\newblock Kluwer-Academic, 2003.

\bibitem[Nesterov(2012)]{Nesterov2012}
Yurii Nesterov.
\newblock {Efficiency of Coordinate Descent Methods on Huge-Scale Optimization
  Problems}.
\newblock \emph{SIAM Journal of Optimization}, 22\penalty0 (2):\penalty0
  341--362, 2012.

\bibitem[Nesterov and Polyak(2006)]{NesterovP2006}
Yurii Nesterov and B.T. Polyak.
\newblock {Cubic regularization of Newton method and its global performance}.
\newblock \emph{Mathematical Programming}, 108\penalty0 (1):\penalty0 177--205,
  2006.

\bibitem[Netrapalli et~al.(2013)Netrapalli, Jain, and Sanghavi]{NetrapalliJS13}
Praneeth Netrapalli, Prateek Jain, and Sujay Sanghavi.
\newblock {Phase Retrieval using Alternating Minimization}.
\newblock In \emph{Proceedings of the 27th Annual Conference on Neural
  Information Processing Systems (NIPS)}, 2013.

\bibitem[Nguyen and Tran(2013{\natexlab{a}})]{NguyenT2013}
Nam~H. Nguyen and Trac~D. Tran.
\newblock {Exact recoverability from dense corrupted observations via L1
  minimization}.
\newblock \emph{IEEE Transactions on Information Theory}, 59\penalty0
  (4):\penalty0 2036--2058, 2013{\natexlab{a}}.

\bibitem[Nguyen and Tran(2013{\natexlab{b}})]{NguyenT2013b}
Nam~H Nguyen and Trac~D Tran.
\newblock {Robust Lasso With Missing and Grossly Corrupted Observations}.
\newblock \emph{IEEE Transaction on Information Theory}, 59\penalty0
  (4):\penalty0 2036--2058, 2013{\natexlab{b}}.

\bibitem[Oymak et~al.(2015)Oymak, Recht, and Soltanolkotabi]{OymakRS15a}
Samet Oymak, Benjamin Recht, and Mahdi Soltanolkotabi.
\newblock {Sharp Time-Data Tradeoffs for Linear Inverse Problems}.
\newblock arXiv:1507.04793 [cs.IT], 2015.

\bibitem[Raskutti et~al.(2010)Raskutti, Wainwright, and Yu]{RaskuttiWY2010}
Garvesh Raskutti, Martin~J. Wainwright, and Bin Yu.
\newblock {Restricted Eigenvalue Properties for Correlated Gaussian Designs}.
\newblock \emph{Journal of Machine Learning Research}, 11:\penalty0 2241--2259,
  2010.

\bibitem[Recht(2011)]{Recht2011}
Benjamin Recht.
\newblock {A Simpler Approach to Matrix Completion}.
\newblock \emph{Journal of Machine Learning Research}, 12:\penalty0 3413--3430,
  2011.

\bibitem[Recht et~al.(2010)Recht, Fazel, and Parrilo]{RechtFP2010}
Benjamin Recht, Maryam Fazel, and Pablo~A. Parrilo.
\newblock {Guaranteed Minimum Rank Solutions to Linear Matrix Equations via
  Nuclear Norm Minimization}.
\newblock \emph{SIAM Review}, 52\penalty0 (3):\penalty0 471--501, 2010.

\bibitem[Reddi et~al.(2016)Reddi, Hefny, Sra, Poczos, and
  Smola.]{ReddiHSPS2016}
Sashank Reddi, Ahmed Hefny, Suvrit Sra, Barnabas Poczos, and Alexander~J.
  Smola.
\newblock {Fast Stochastic Methods for Nonsmooth Nonconvex Optimization}.
\newblock In \emph{Proceedings of the 33rd International Conference on Machine
  Learning (ICML)}, 2016.

\bibitem[Rousseeuw(1984)]{Rousseeuw1984}
Peter~J. Rousseeuw.
\newblock {Least Median of Squares Regression}.
\newblock \emph{Journal of the American Statistical Association}, 79\penalty0
  (388):\penalty0 871--880, 1984.

\bibitem[Rousseeuw and Leroy(1987)]{RousseeuwL1987}
Peter~J. Rousseeuw and Annick~M. Leroy.
\newblock \emph{{Robust Regression and Outlier Detection}}.
\newblock John Wiley and Sons, 1987.

\bibitem[Saha and Tewari(2013)]{SahaT2013}
Ankan Saha and Ambuj Tewari.
\newblock {On the Non-asymptotic Convergence of Cyclic Coordinate Descent
  Methods}.
\newblock \emph{SIAM Journal on Optimization}, 23\penalty0 (1):\penalty0
  576--601, 2013.

\bibitem[Sedghi and Anandkumar(2016)]{SedghiA2016}
Hanie Sedghi and Anima Anandkumar.
\newblock {Training Input-Output Recurrent Neural Networks through Spectral
  Methods}.
\newblock arXiv:1603.00954 [CS.LG], 2016.

\bibitem[Shalev-Shwartz and Zhang(2013)]{Shalev-ShwartzZ2013}
Shai Shalev-Shwartz and Tong Zhang.
\newblock {Stochastic Dual Coordinate Ascent Methods for Regularized Loss
  Minimization}.
\newblock \emph{Journal of Machine Learning Research}, 14:\penalty0 567--599,
  2013.

\bibitem[She and Owen(2011)]{SheO2011}
Yiyuan She and Art~B. Owen.
\newblock {Outlier Detection Using Nonconvex Penalized Regression}.
\newblock \emph{Journal of the American Statistical Association}, 106\penalty0
  (494):\penalty0 626--639, 2011.

\bibitem[Spielman et~al.(2012)Spielman, Wang, and Wright]{SpielmanWW2012}
Daniel~A. Spielman, Huan Wang, and John Wright.
\newblock {Exact Recovery of Sparsely-Used Dictionaries}.
\newblock In \emph{Proceedings of the 25th Annual Conference on Learning Theory
  (COLT)}, 2012.

\bibitem[Sra et~al.(2011)Sra, Nowozin, and Wright]{SraNW2011}
Suvrit Sra, Sebastian Nowozin, and Stephen~J. Wright, editors.
\newblock \emph{{Optimization for Machine Learning}}.
\newblock The MIT Press, 2011.

\bibitem[Stockinger and Dutter(1987)]{StockingD1987}
Norbert Stockinger and Rudolf Dutter.
\newblock {Robust time series analysis: A survey}.
\newblock \emph{Kybernetika}, 23\penalty0 (7):\penalty0 1--3, 1987.

\bibitem[Sun et~al.(2015)Sun, Qu, and Wright]{SunQW2015}
Ju~Sun, Qing Qu, and John Wright.
\newblock {When Are Nonconvex Problems Not Scary?}
\newblock arXiv:1510.06096 [math.OC], 2015.

\bibitem[Sun and Lu(2015)]{SunL2015}
Ruoyu Sun and Zhi-Quan Lu.
\newblock {Guaranteed Matrix Completion via Non-convex Factorization}.
\newblock In \emph{Proceedings of the 56th IEEE Annual Symposium on Foundations
  of Computer Science (FOCS)}, 2015.

\bibitem[Tewari et~al.(2011)Tewari, Ravikumar, and Dhillon]{TewariRD2011}
Ambuj Tewari, Pradeep Ravikumar, and Inderjit~S. Dhillon.
\newblock {Greedy Algorithms for Structurally Constrained High Dimensional
  Problems}.
\newblock In \emph{Proceedings of the 25th Annual Conference on Neural
  Information Processing Systems (NIPS)}, 2011.

\bibitem[Tropp(2012)]{Tropp2012}
Joel~A. Tropp.
\newblock {User-Friendly Tail Bounds for Sums of Random Matrices}.
\newblock \emph{Foundations of Computational Mathematics}, 12\penalty0
  (4):\penalty0 389--434, 2012.

\bibitem[Tropp and Gilbert(2007)]{TroppG2007}
Joel~A. Tropp and Anna~C. Gilbert.
\newblock {Signal Recovery From Random Measurements Via Orthogonal Matching
  Pursuit}.
\newblock \emph{IEEE Transactions on Information Theory}, 53\penalty0
  (12):\penalty0 4655--4666, Dec. 2007.
\newblock ISSN 0018-9448.

\bibitem[Wang et~al.(2011)Wang, Xu, and Tang]{WangXT2011}
Meng Wang, Weiyu Xu, and Ao~Tang.
\newblock {On the Performance of Sparse Recovery Via $\ell_p$-Minimization
  $(0\leq p \leq 1)$}.
\newblock \emph{IEEE Transactions on Information Theory}, 57\penalty0
  (11):\penalty0 7255--7278, 2011.

\bibitem[Wang et~al.(2015)Wang, Gu, Ning, and Liu]{WangGNL2015}
Zhaoran Wang, Quanquan Gu, Yang Ning, and Han Liu.
\newblock {High Dimensional EM Algorithm: Statistical Optimization and
  Asymptotic Normality}.
\newblock In \emph{Proceedings of the 29th Annual Conference on Neural
  Information Processing Systems (NIPS)}, 2015.

\bibitem[Wilson et~al.(2003)Wilson, Eckenrode, Li, Ruan, Yang, Shi,
  Davoodi-Semiromi, McIndoe, Croker, and She]{WilsonELRYSD-SMCS2003}
Karen~H.S. Wilson, Sarah~E. Eckenrode, Quan-Zhen Li, Qing-Guo Ruan, Ping Yang,
  Jing-Da Shi, Abdoreza Davoodi-Semiromi, Richard~A. McIndoe, Byron~P. Croker,
  and Jin-Xiong She.
\newblock {Microarray Analysis of Gene Expression in the Kidneys of New- and
  Post-Onset Diabetic NOD Mice}.
\newblock \emph{Diabetes}, 52\penalty0 (8):\penalty0 2151--2159, 2003.

\bibitem[Wright et~al.(2009)Wright, Yang, Ganesh, Sastry, and
  Ma]{WrightYGSM2009}
John Wright, Alan~Y. Yang, Arvind Ganesh, S.~Shankar Sastry, and Yi~Ma.
\newblock {Robust Face Recognition via Sparse Representation}.
\newblock \emph{IEEE Transactions on Pattern Analysis and Machine
  Intelligence}, 31\penalty0 (2):\penalty0 210--227, 2009.

\bibitem[Wright and Nocedal(1999)]{WrightN1999}
Stephen~J Wright and Jorge Nocedal.
\newblock \emph{{Numerical Optimization}}, volume~2.
\newblock Springer New York, 1999.

\bibitem[Wu(1983)]{Wu1983}
C.-F.~Jeff Wu.
\newblock {On the Convergence Properties of the EM Algorithm}.
\newblock \emph{The Annals of Statistics}, 11\penalty0 (1):\penalty0 95--103,
  1983.

\bibitem[Yang et~al.(2013)Yang, Zhou, Balasubramanian, Sastry, and
  Ma]{YangZBSM2013}
Allen~Y. Yang, Zihan Zhou, Arvind~Ganesh Balasubramanian, S~Shankar Sastry, and
  Yi~Ma.
\newblock {Fast $\ell_1$-Minimization Algorithms for Robust Face Recognition}.
\newblock \emph{IEEE Transactions on Image Processing}, 22\penalty0
  (8):\penalty0 3234--3246, 2013.

\bibitem[Yang et~al.(2015)Yang, Balakrishnan, and Wainwright]{YangBW2015}
Fanny Yang, Sivaraman Balakrishnan, and Martin~J. Wainwright.
\newblock {Statistical and computational guarantees for the Baum-Welch
  algorithm}.
\newblock In \emph{Proceedings of the 53rd Annual Allerton Conference on
  Communication, Control, and Computing (Allerton)}, 2015.

\bibitem[Yi and Caramanis(2015)]{YiC2015}
Xinyang Yi and Constantine Caramanis.
\newblock {Regularized EM Algorithms: A Unified Framework and Statistical
  Guarantees}.
\newblock In \emph{Proceedings of the 29th Annual Conference on Neural
  Information Processing Systems (NIPS)}, 2015.

\bibitem[Yi et~al.(2014)Yi, Caramanis, and Sanghavi]{YiCS2014}
Xinyang Yi, Constantine Caramanis, and Sujay Sanghavi.
\newblock {Alternating Minimization for Mixed Linear Regression}.
\newblock In \emph{Proceedings of the 31st International Conference on Machine
  Learning (ICML)}, 2014.

\bibitem[Yuan(2015)]{Yuan2015}
Ya-Xiang Yuan.
\newblock {Recent advances in trust region algorithms}.
\newblock \emph{Mathematical Programming}, 151\penalty0 (1):\penalty0 249--281,
  2015.

\bibitem[Zhang(2011)]{Zhang2011}
Tong Zhang.
\newblock {Adaptive Forward-Backward Greedy Algorithm for Learning Sparse
  Representations}.
\newblock \emph{IEEE Transactions on Information Theory}, 57:\penalty0
  4689--4708, 2011.

\bibitem[Zhang et~al.(2017)Zhang, Liang, and Charikar]{ZhangLC2017}
Yuchen Zhang, Percy Liang, and Moses Charikar.
\newblock {A Hitting Time Analysis of Stochastic Gradient Langevin Dynamics}.
\newblock In \emph{Proceedings of the 30th Conference on Learning Theory},
  2017.

\bibitem[Zhong et~al.(2017)Zhong, Song, Jain, Bartlett, and
  Dhillon]{ZhongSJBD2017}
Kai Zhong, Zhao Song, Prateek Jain, Peter~L. Bartlett, and Inderjit~S. Dhillon.
\newblock {Recovery Guarantees for One-hidden-layer Neural Networks}.
\newblock In \emph{Proceedings of the 34th International Conference on Machine
  Learning (ICML)}, 2017.

\bibitem[Zhou et~al.(2008)Zhou, Wilkinson, Schreiber, and Pan]{ZhouWSP2008}
Yunhong Zhou, Dennis Wilkinson, Robert Schreiber, and Rong Pan.
\newblock {Large-scale Parallel Collaborative Filtering for the Netflix Prize}.
\newblock In \emph{Proceedings of the 4th International Conference on
  Algorithmic Aspects in Information and Management (AAIM)}, 2008.

\end{thebibliography}

\end{document}